\documentclass{article}
\usepackage[preprint]{neurips_2024}

\usepackage{amsmath,amsthm}
\newtheorem{theorem}{Theorem}
\newtheorem{definition}[theorem]{Definition}
\newtheorem{example}[theorem]{Example}
\newtheorem{observation}[theorem]{Observation}
\newtheorem{corollary}[theorem]{Corollary}
\newtheorem{proposition}[theorem]{Proposition}
\newtheorem{attempt}[theorem]{Attempt}

\usepackage{amsfonts}
\usepackage{amssymb}
\usepackage{bbm}
\usepackage{algorithm}
\usepackage{algpseudocode}

\usepackage[font=small]{caption}
\usepackage{float}  
\usepackage{placeins}  
\usepackage[dvipsnames,table]{xcolor}

\usepackage{tikz}
\usetikzlibrary{bayesnet}

\usetikzlibrary{shapes,decorations,arrows,calc,arrows.meta,fit,positioning}
\tikzset{
    -Latex,auto,node distance =0.5 cm and 0.5 cm,semithick,
    state/.style ={ellipse, draw, minimum width = 0.7 cm},
    point/.style = {circle, draw, inner sep=0.04cm,fill,node contents={}},
    bidirected/.style={Latex-Latex,dashed},
    el/.style = {inner sep=2pt, align=left, sloped}
}

\usepackage[colorinlistoftodos]{todonotes}
\presetkeys{todonotes}{inline}{}
\usepackage{url}
\usepackage{booktabs}
\usepackage{rotating}
\usepackage{enumitem}
\usepackage{siunitx}
\usepackage{xspace}
\tikzset{%
  zeroarrow/.style = {-stealth,dashed},
  onearrow/.style = {-stealth,solid},
  c/.style = {circle,draw,solid,minimum width=2em,
        minimum height=2em},
  r/.style = {rectangle,draw,solid,minimum width=2em,
        minimum height=2em}
}

\makeatletter
\DeclareRobustCommand{\varname}[1]{\ensuremath{\begingroup\newmcodes@\mathit{#1}\endgroup}}
\makeatother

\newcommand{\pn}{\ensuremath{\mathit{PN}}}
\newcommand{\ps}{\ensuremath{\mathit{PS}}}
\newcommand{\pns}{\ensuremath{\mathit{PNS}}}

\definecolor{petroil2} {RGB} {36, 165, 175}
\definecolor{gold2} {RGB} {255, 130, 0}

\usepackage{pifont}
\usepackage{makecell}
\definecolor{pcihl}{RGB}{230,241,251}
\definecolor{okgreen}{RGB}{29,126,88}
\definecolor{nogray}{RGB}{150,148,140}
\definecolor{partialamber}{RGB}{186,117,23}
\definecolor{pciblue}{RGB}{24,95,165}
\newcommand{\yes}{\textcolor{okgreen}{\ding{51}}}
\newcommand{\no}{\textcolor{nogray}{\ding{55}}}
\newcommand{\partialm}{\textcolor{partialamber}{$\sim$}}
\newcommand{\na}{\textcolor{nogray}{--}}
\newcommand{\pciyes}{\textcolor{pciblue}{\ding{51}}}

\usepackage{listings}

\definecolor{strings}{rgb}{.624,.251,.259}
\definecolor{keywords}{rgb}{.224,.451,.686}
\definecolor{comment}{rgb}{.322,.451,.322}

\lstdefinelanguage{python}{
  morekeywords={from, import, as, for, in, while, def, return, 
  =, +, -, /, *, lambda},
  keywords=[3]{sample, param, module, Marginal, Posterior, Trace, Poutine, Distribution},
  morecomment=[l]{\#},
  morecomment=[s]{"""}{"""},
  morestring=[b]',
  morestring=[b]",
  alsoletter={<>=-+/*},
  sensitive=true
}

\usepackage[most]{tcolorbox}
\newtcolorbox{docode}[1][]{
    colback=gray!20,        
    colframe=black,         
    width=0.9\textwidth,    
    arc=4mm,                
    boxrule=0.5pt,          
    left=2em,               
    enhanced,
    fontupper=\ttfamily,    
    #1                      
}

\renewcommand{\texttt}[1]{\lstinline[basicstyle=\fontsize{8pt}{8.25pt}\selectfont\ttfamily]{#1}}

\newcommand{\ourapproach}{PCI\xspace}
\newcommand{\Ourapproach}{PCI\xspace}
\newcommand{\dom}{\operatorname{dom}\,}
\usepackage{hyperref}
\usepackage[capitalise]{cleveref}

\hypersetup{
    colorlinks=true,      
    linkcolor=blue,       
    urlcolor=blue,        
    citecolor=blue,       
    filecolor=blue        
}

\newcommand{\nbsite}{https://rfl-urbaniak.github.io/explainable\_paper}
\newcommand{\nb}[1]{\href{\nbsite/#1.html}{\texttt{#1}}}

\title{A Computationally Feasible Framework for Causal Probabilistic Explanation\thanks{\raggedright Code at \url{https://github.com/rfl-urbaniak/explainable_paper} and companion notebooks at \url{https://rfl-urbaniak.github.io/explainable_paper/}. Correspondence to \texttt{rfl.urbaniak@gmail.com} (\url{https://rfl-urbaniak.github.io/}).}}

\author{%
  Rafal Urbaniak \\
  Basis Research Institute
  \And
  Sam Witty \\
  Sorbus AI
  \And
  Daniel Waxman \\
  Basis Research Institute \\
  Massachusetts Institute of Technology
  \And
  Andy Zane \\
  Basis Research Institute \\
  University of Massachusetts Amherst
  \And
  Poorva Garg \\
  University of California, Los Angeles
  \And
  Emily Bunnapradist \\
  Basis Research Institute
  \And
  Sankaran Vaidyanathan \\
  University of Massachusetts Amherst
  \And
  Jack Feser \\
  Basis Research Institute
  \And
  Drew Lehe \\
  HouseIQ
  \And
  Eli Bingham \\
  Basis Research Institute
}

\begin{document}

\maketitle

\begin{abstract}%
  Explaining why a specific outcome occurred, and which inputs deserve the
  blame or credit, is central to philosophical, scientific, and policy analysis.
  Existing tools split into two camps. The theory of
  \emph{actual causality} (AC) gives principled verdicts, but only for toy-sized
  models, because computing them requires enumerating counterfactual scenarios.
  Scalable attribution methods like SHAP (or even causal SHAP) at least partially 
  ignore the causal structure that generated the data, and can give answers that conflict with a careful
  causal analysis. We close this gap with Probabilistic Causal Impact
  (\Ourapproach). 
  
  \ourapproach builds on actual causality and on Pearl's notions of probability
  of necessity and sufficiency, but recasts the question of explainability as an estimation problem on a probabilistic causal
  model that is easily approximated via Monte Carlo. By specifying a distribution over ``candidate explanations,'' a distribution over counterfactual values, and a scoring function, \ourapproach provides tractable, causally grounded, graded explanations, generalizing AC and Pearl's probability of causation as degenerate cases. 

  We evaluate \ourapproach in synthetic and real-world examples, spanning consistency checks with AC, scaling experiments, complex continuous-valued dynamical systems, and a real-world deployed causal machine learning model trained on millions of datapoints.

  \noindent\textbf{Keywords:} Probabilistic Causal Models, Explanation, Responsibility, Actual Causality, Probability of Necessity and Sufficiency.
\end{abstract}

\newpage
\begingroup
\renewcommand{\baselinestretch}{0.95}\small
\setlength{\parskip}{1pt}
\let\oldaddvspace\addvspace
\renewcommand{\addvspace}[1]{\oldaddvspace{0.4em}}
\setcounter{tocdepth}{2}
\tableofcontents
\let\addvspace\oldaddvspace
\endgroup
\newpage

\section{Introduction}
\label{sec:intro}

Causal attribution, figuring out which factors brought about an outcome and to what degree, matters wherever we act on a model's predictions: in lending,
healthcare, or automated decision systems, knowing \emph{why} something happened is
often as important as the prediction itself
\citep{Shepherd2024,CheshireMagrini2009,Papageorgiou2022,HeckmanPinto2022CausalPolicyAnalysis}.
The tools most machine learning practitioners reach for by default often mistake
predictive importance for causal responsibility, which makes their attributions
causally uninterpretable.

Consider the Old Boys' Club Bank (OBCB), whose policy is to check an applicant's
credit score only if they are male; everyone else is rejected outright. Alice, a
woman with bad credit, applies and is denied. Was she denied for her gender,
or her credit? Intuitively the answer is clear: the bank never even looked 
at her credit score, so it played no role in her denial. 
SHAP \citep{lundberg2017unified} and a range of other feature-attribution tools,
however, tend to split responsibility between gender and credit roughly evenly, because they
reason from correlation and ignore which causal pathway was actually active for
Alice specifically. Even standard ``but for'' tests fail here: hypothetically, if Alice were male her low score would have been checked, but still rejected; likewise, had Alice's credit score been higher, she still would have been rejected due to that score not being considered. A naive assessment concludes that, since neither hypothetical changed the outcome, \textit{neither} gender nor credit score is a cause. 

\emph{Actual causality} (AC), a line of work spanning philosophy and AI
\citep{actualCausalityHalpern}, gets this case right: it varies candidate causes while holding fixed certain aspects of the causal story, and assesses changes in outcomes. For Alice's case, it hypothetially varies gender while holding fixed that Alice did not, in fact, fail a credit check (she never received one). Had she been a man who did not fail a credit check, Alice would have received a loan (Examples \ref{example:bobatobcb}, \ref{example:aliceatobcb}).
This correctly names her gender, not her credit, as the \textit{actual} cause of her denial. 

Existing formalisations of actual causality, however, involve some combination of the following three limitations: they are
all-or-nothing, treating a variable as either a cause or not with no notion of
degree; they are computationally intractable beyond toy examples, because
verifying that \emph{some} witness set (a selection of variables held fixed at their actual values, our formal stand-in for a context) satisfies the definition requires an exhaustive search over exponentially many candidates;
and they were developed primarily for deterministic models.

This paper introduces \emph{probabilistic causal impact} (\Ourapproach), an
attribution method that returns a
graded score, scales to large stochastic models
without exhaustive search, and applies directly to
machine-learned causal models (neural, Gaussian-process, or otherwise probabilistic) that practitioners actually build. It still gets Alice's case right. At a high level,
\ourapproach asks two counterfactual questions per candidate cause, echoing Pearl's probabilities of necessity
and sufficiency: a \emph{sufficiency} question (would the outcome have persisted had we
changed nothing about a given feature?) and a \emph{necessity} question (would it have changed had we
intervened on it?). \ourapproach restricts both
questions to the causal pathways that were genuinely active, averaging over
the reasonable ways of specifying which pathways those are. Because every quantity involved is a standard counterfactual
query, the analyst can estimate the score by sampling from the same probabilistic-programming
machinery already used to build the underlying model.

\begin{figure*}[t]
\centering
\includegraphics[width=\textwidth]{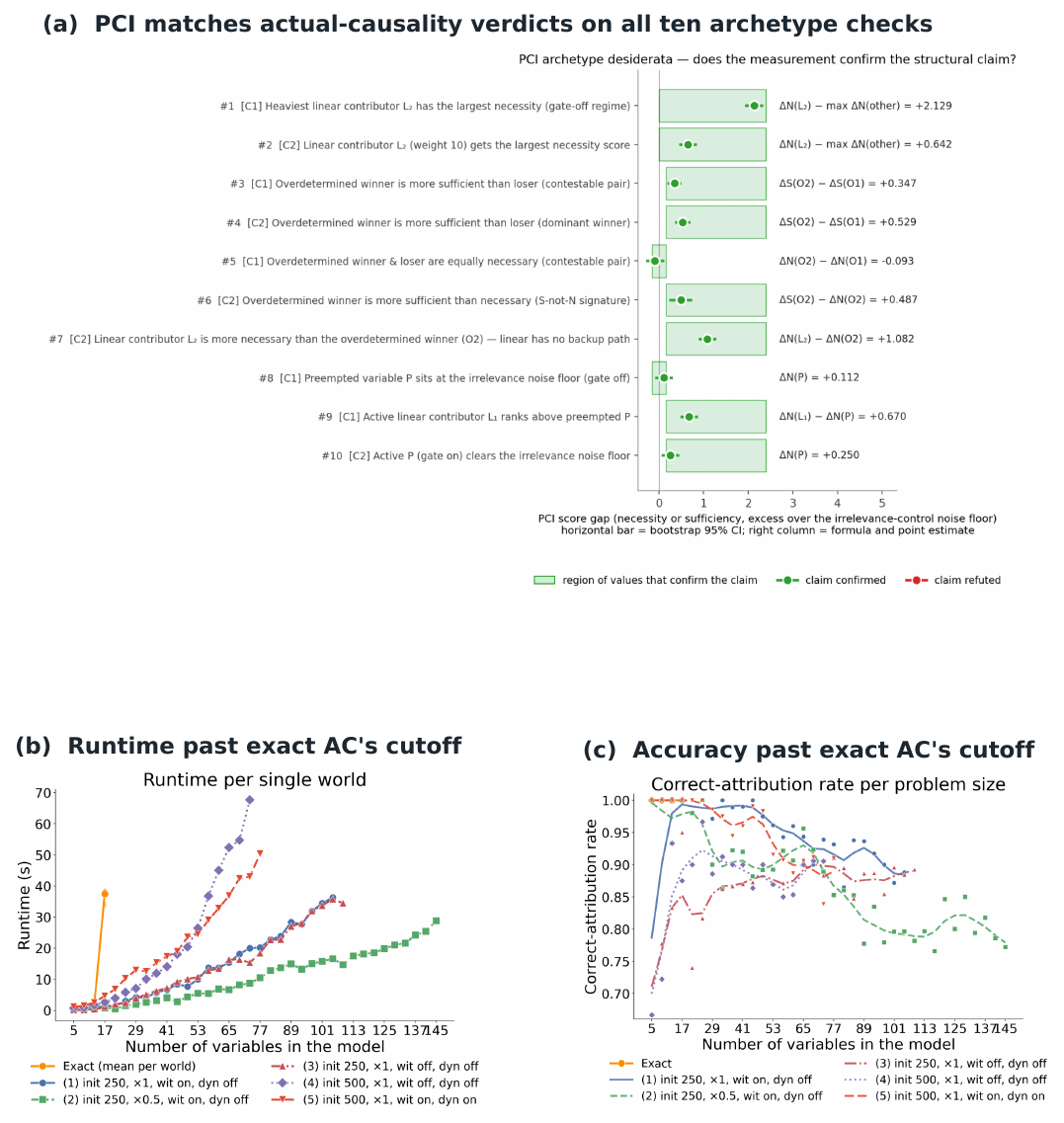}
\caption{\textbf{The headline result.} \textbf{(a)} On a synthetic model
with four causal archetypes (linear necessity/sufficiency,
overdetermination, preemption, irrelevance), \ourapproach's score gaps agree with
the actual-causality verdict across all ten checks, shown in green throughout.
\textbf{(b), (c)} \ourapproach stays tractable, in runtime and in
correct-attribution rate, well past the problem size at which exact
actual-cause enumeration's cost grows exponentially and becomes infeasible.
Both methods run under a modest 60-minute compute budget on inexpensive
hardware, chosen to gauge computational feasibility; the sizes reached
scale with the compute and sample budget allotted.
\cref{sec:evaluation} reads this and the
rest of the empirical story at a glance; full protocols are in
Appendices~\ref{sub:synthetic} and~\ref{sec:ac_benchmark}.}
\label{fig:hero}
\end{figure*}

The rest of this section makes this precise, positions \ourapproach{} against
existing methods, and introduces its central technical move.

At minimum, a satisfactory causal attribution method
should be \phantomsection\label{des:grounded}(i) \textit{causally grounded},
attributing responsibility based on whether ``intervening on'' (manipulating) a candidate cause
would change the outcome; \phantomsection\label{des:context}(ii) \textit{context-sensitive},
distinguishing which causal pathways were active in a specific
instance,\footnote{We use \emph{context-sensitive} in Halpern's sense: sensitivity to which
\emph{endogenous} mediators are held fixed at their factual values during a
counterfactual evaluation. Section~\ref{sec:causal_impact} formalises this as the
witness mechanism. This differs from sensitivity to the exogenous noise realisation
$\mathbf{u}$ itself. Halpern's literature calls $\mathbf{u}$ a \emph{context} only
because his models are deterministic: fixing $\mathbf{u}$ fixes every endogenous
value outright, so ``the context'' and ``the factual world'' coincide there. Once
mechanisms are stochastic, as throughout this paper, that coincidence breaks.
Noise-realisation sensitivity is then a separate axis, on which every counterfactual
or actual-cause query, ours included, trivially qualifies, unlike a population-level
interventional query. \ourapproach{} and actual causality differ from interventional
measures precisely in the mediator-holding sense of context sensitivity: a
``vanilla'' structural counterfactual, which recomputes every mediator downstream of
an intervention instead of holding any fixed, fails this test.}
\phantomsection\label{des:broad}(iii) \textit{broadly applicable}, handling stochastic models,
discrete and continuous variables, and counterfactual changes of any size; and
\phantomsection\label{des:tractable}(iv) \textit{tractable},
estimable from samples
without exhaustive enumeration over witness sets and alternatives.%
\footnote{This trades enumeration for Monte Carlo sampling: achieving a fixed
estimator variance still requires a sample count that scales with the
cardinalities of the suspect (candidate-cause) and witness sets, so the
asymptotic cost does not improve, and we do not
have a general way to know in advance when fewer samples suffice.
Section~\ref{sec:ac_benchmark} reports empirical scaling; Section~\ref{sec:discussion}
returns to this and other limitations.}
A further, stronger demand is a graded notion of causal explanation; we flag it
here but do not treat it as a minimum bar, and develop it in full in
Section~\ref{sec:causal_impact}.

Existing methods each forfeit at least one of these. Actual causality
\citep{actualCausalityHalpern} gets Alice's case right but, as noted above, remains
intractable beyond toy examples, with no clear approximation
\citep{eiterComplexityResultsStructurebased2002}. Methods that scale give up
causal grounding or context-sensitivity instead. SHAP \citep{lundberg2017unified}
attributes by correlation, so a feature can score highly merely for correlating with
a true cause \citep{janzing2020feature}. Causal SHAP \citep{heskes2020causal}
improves on this, yet still conflicts with the causal structure in some cases,
which we return to in Section~\ref{sec:shap_examples}. The Differential Causal Effect
\citep{butler2022differential} is causally grounded but local: relying on gradients,
it misses non-infinitesimal effects, which are often of interest to a practitioner.
The average and conditional treatment effect
(ATE/CATE) are grounded and often tractable but not context-sensitive: they let an
intervention's effects propagate downstream through the structural equations
regardless of which causal paths were active in the specific instance
(Section~\ref{sec:motivations}). Table~\ref{tab:method_comparison} summarises the
tradeoffs.

\begin{table}[thbp]
\centering
\renewcommand{\arraystretch}{1.45}
\begin{tabular}{@{}l c c c c c c >{\columncolor{pcihl}}c @{}}
\toprule
 & \makecell{Actual\\causality} & SHAP & \makecell{Causal\\SHAP} & \makecell{PN/PS\\/PNS} & DCE & \makecell{ATE/\\CATE} & \textcolor{pciblue}{\textbf{\ourapproach}} \\
\midrule
(i) causally grounded    & \yes & \no  & \partialm & \yes & \yes & \yes & \pciyes \\
(ii) context-sensitive   & \yes & \no  & \no       & \no  & \na  & \no  & \pciyes \\
(iii) broadly applicable & \no  & \yes & \yes      & \no  & \no  & \yes & \pciyes \\
(iv) tractable           & \no  & \yes & \yes      & \yes & \yes & \yes & \pciyes \\
\bottomrule
\end{tabular}
\caption{Which of the four desiderata each attribution method satisfies. \partialm\ = partial; \na\ = not addressed by design.}
\label{tab:method_comparison}
\end{table}

Table~\ref{tab:method_comparison}'s row (ii) is Alice's case: actual
causality names her gender as the cause because its witness mechanism holds the
untriggered credit check fixed, while Pearl's probability-of-necessity/sufficiency
family (PN/PS/PNS, three related quantities we define in
Section~\ref{sec:motivations}), lacking any such mechanism, assigns her credit
comparable responsibility despite it never being checked.
Section~\ref{sec:motivations} works through this example, and its stochastic
counterpart, in full.

To overcome these tradeoffs, \Ourapproach combines the context-sensitivity of actual
causality with the scalability of probability-of-causation methods. \Ourapproach
generalises Pearl's probability of necessary and sufficient causation
\citep{pearl2022probabilities} from binary to arbitrary variable
domains.\footnote{\citet{kawakammi_2024_poc_for_cont_and_vec} independently extend PNS to
continuous and vector-valued variables by a different route. Their
Definition~3.1 fixes two treatment levels $x_0,x_1$ and an outcome threshold $y$,
and reads necessity and sufficiency off the inequality $Y_{x_0} < y \le Y_{x_1}$;
exogeneity and monotonicity then identify these quantities nonparametrically, from
observational distributions alone. \Ourapproach substitutes an alternative-value
distribution $\Delta$ for the fixed contrast and a user-chosen impact kernel $ci$
for the fixed threshold. Monotonicity drops out because \ourapproach evaluates
both counterfactual worlds inside a fitted causal model, which supplies the
cross-world quantities by simulation. The identification burden moves onto that
model (\cref{sec:discussion}). The witness mechanism then adds
context-sensitivity; \cref{sec:other-prob-ac} returns to the contrast.}
It also incorporates the \emph{witness} mechanism from actual causality
\citep{actualCausalityHalpern}, but changes its logical form: Halpern's account asks
whether \emph{some} witness set, out of exponentially many, reveals the dependency;
\ourapproach instead asks how much of a $\Gamma$-weighted distribution over
candidate-cause and witness sets reveals it, replacing an existential quantifier
with an expectation.%
\footnote{That context should be weighted is intuitive. Suppose someone drops
a lit match near a gas station and a fire results. In one scenario they are
standing right next to the fuel pump when they drop it. In another they are
two hundred feet away, far from anything flammable, but as they drop it a
squirrel happens to be running past, and the match grazes its tail, setting
it alight. Panicked, the squirrel runs toward the gas station, and its
burning tail starts the fire there. A purely existential account of actual
causality treats both scenarios the same: some witness set makes the
match-drop a cause of the fire in each case, so both get an identical
verdict. But the second scenario's causal chain runs through an improbable
fluke, and intuitively the match-dropper bears less responsibility for it
than in the first. Weighting witness sets gives \ourapproach a natural place
to encode this: a $\Gamma$ is  a distribution over how likely some features are to be a cause or 
a witness - and a $\Gamma$ distribution that would  assign little mass to freak actual contingencies
like the squirrel's detour lets the score track how ordinary the enabling
context was (and similarly, weighting alternative values based on how likely they, as encoded in $\Delta$, are can play a role in such considerations too).}
A witness set records which variables stay at their factual values while
\ourapproach varies a candidate cause, so causal paths inactive in the factual world stay
blocked even under intervention, delivering the
context-sensitivity that ATE and CATE lack.

The witness mechanism also decides how multiple, simultaneously active
candidates rank against each other. A variable that merely correlates with the
one that did the work, or is downstream of it, can outscore it. Holding the
right variables fixed as witnesses reverses that. Without witnesses,
\ourapproach makes the same mistake SHAP does, ranking Alice's credit above
her gender. With them, gender outranks credit, the correct order
(Section~\ref{sec:obcb_shap}). Witnesses turn a $0.034$ margin between a
mediator and its upstream cause into a $0.133$ one (Section~\ref{sec:signal}).
And they let \ourapproach separate a reliable cause from an unreliable one
across forensic scenarios in the desert-traveller example
(Section~\ref{sub:weak-poison}).

This move makes \ourapproach tractable, but it is also the source of its main risk. Existential search over witness sets forces AC's brute-force enumeration; an
expectation admits Monte Carlo sampling from $\Gamma$ instead, at a cost: $\Gamma$-mass on witness sets that do not reveal a dependency can swamp a single witness set that does, or sampling can simply miss it if it is rare enough. In the worst case, rare dependencies can be as expensive to find with sampling as with enumeration (Section~\ref{sec:discussion}).

Four considerations limit the damage. First, converting the search problem to a sampling problem 
lets us study statistical estimation efficiency and provides a "computational effort" knobs (the number of samples, sampling method choice, etc.)
that control success. That is, it lets us fix a computational budget and then study whether it will achieve the desired results. 
The practical aspect of this is that models with
structural regularities let sampling run far cheaper than the worst case above
suggests. In the actual-cause benchmark (Section~\ref{sec:ac_benchmark}),
for instance, any superset of a witness set witnesses the same attribution, so
restricting the search to small witness sets costs little accuracy even though
it never visits the larger sets a wider search would check. \ourapproach's
definition assumes no such regularity, but when a model has it, the sampling
budget can be spent narrowly without losing coverage.

The second is formal: switching from existential search to an expectation does
not make \ourapproach blind to real actual causes. \cref{sec:relation_with_ac}
proves, under conditions given there, that a candidate's \ourapproach score is
positive exactly when it is an actual cause in Halpern's sense.

The third is more general. Missing a rare event under finite sampling is a
familiar tradeoff, not a weakness specific to \ourapproach; every Monte Carlo
procedure faces it. Here, though, the tradeoff is explicit: the sample budget
and the witness-cardinality bound are choices the analyst makes, and
Section~\ref{sec:ac_benchmark} reports how correct-attribution rate changes as
each is loosened or tightened, so increasing either reduces the chance that
sampling misses a rare witness set.

Prior work anticipates pieces of \ourapproach's construction. Halpern's own account
already averages over exogenous noise when computing probabilities of causation
\citep{actualCausalityHalpern}. \citet{Beckers2015CombiningPC} weight alternative
replacement values by how probable they are. Both aggregate over witness sets in
restricted settings \citep{Beckers2015CombiningPC, actualCausalityHalpern}. And
\citet{chocklerResponsibilityBlameStructuralModel2004} let the size of that set
modulate the verdict: their degree of responsibility is $1/(k{+}1)$, where $k$ counts
the variables that must be shifted off their actual values before the outcome depends
counterfactually on the candidate. No prior construction combines all four into a
single estimator that scales to arbitrary variable domains. \ourapproach keeps
Chockler and Halpern's intuition that the size of the set doing the work should count,
replacing their fixed reciprocal with the cardinality band of the variable selection
distribution $\Gamma$, which the analyst sets (\cref{sec:causal_impact});
\cref{sec:recovering-ac} treats this and the other actual-causality quantities in the
same terms.
Every quantity entering the \Ourapproach definition is a standard interventional or
counterfactual quantity, so computing a score asks nothing of a modelling tool beyond
what counterfactual reasoning already asks: a way to write down a structural model,
intervene on it, and sample under the resulting counterfactual semantics. Any system
providing those operations can evaluate \Ourapproach's expectations by Monte Carlo for
any model it expresses, with no per-model derivation. We use ChiRho, a causal extension
of the probabilistic programming language Pyro,\footnote{\url{https://basisresearch.github.io/chirho/explainable_sir.html}}
though nothing in the construction depends on that choice.

That requirement is substantive. \ourapproach presumes access to a trained
probabilistic causal model that supports sampling counterfactual outcomes, whereas
Causal SHAP and the probability-of-necessity/sufficiency (PN/PS/PNS) bounds of
\citet{muellerpearl2022causes} work from observational data and a causal graph alone.
The outputs differ accordingly. Those bounds give \emph{intervals} on the probability
of causation, and \ourapproach, given a fully specified model, gives a full posterior
over its score, summarised as a point estimate via the mean.

Within that requirement, we place no strong restrictions on the form of the model, its
distributions, or its variable types. Such a model can include standard machine
learning components, neural-network-parametrized distributions or Gaussian processes
among them. A feed-forward probabilistic program of this kind is still a structural
causal model, one whose structural equations happen to be parametrized by a neural
network rather than specified by hand.

\citet{xia2022neural} show that such neural causal models, being universal
approximators, are expressive enough to represent the same nonparametric
counterfactual quantities available to an explicitly specified structural
model. This expressiveness does not by itself resolve non-identification,
though: many mechanisms can fit the same observed data equally well, and no
amount of representational flexibility changes that.

Instead, \ourapproach operates in a Bayesian setting: its score averages the
counterfactual outcome over the posterior the trained model induces, spanning both
mechanisms and their parameters, much as Bayesian model averaging marginalizes over
which of several candidate models generated the data.\footnote{\ourapproach, like
actual causality, presupposes a causal model. Whether that model is identified from the
available data, domain knowledge, or experiments is a prior question, and one this
paper inherits rather than settles.} The earlier claim of no strong restrictions is
about the model's functional form only. Obtaining a posterior that reflects the data
well is a separate, nontrivial problem for complex neural or Gaussian-process
mechanisms, and we do not attempt to solve it here.\footnote{Treating non-identification this way is a substantive position,
and its merits have been debated at length in general \citep{poirier1998revising,
gustafson2015partial, moon2012bayesian} and for causal inference specifically
\citep{rubin1978bayesian, richardson2011transparent}. We adopt it here without
defending it. A reader who declines it can still read the posterior spread of the score
as a sensitivity analysis, reporting its range in place of its mean and recovering an
interval of the kind partial-identification methods target \citep{manski2003partial,
zhang2022partial}.}

\paragraph{Paper structure.}
The paper proceeds in five phases, laid out in Figure~\ref{fig:roadmap_inspirations},
with the supporting appendix material collected in a separate band.
The first phase establishes motivation and formalism.
Section~\ref{sec:motivations} works through a running example
to draw out what actual causality, the probability of necessity/sufficiency,
and average/conditional treatment effects each get right and miss, motivating
the design choices behind \ourapproach.
Section~\ref{sec:causal_impact} then defines the method: a kernel built from a
sufficiency check at the factual configuration and a necessity check against
alternatives, integrated over a distribution of suspect/witness sets, with
three components given as user choices: the variable selection distribution
$\Gamma$, the alternative-value distribution $\Delta$, and the causal impact
function.

With \ourapproach defined, the second phase previews the payoff before the
detailed argument. \cref{sec:evaluation} reads each empirical and comparative
result at a glance and defers the full protocols to the appendices:
\ourapproach recovers the actual-causality literature's verdicts on
overdetermination and undercutting (\cref{sub:synthetic}); it diverges from
gradient-based Differential Causal Effect attribution where the local gradient
misreads the counterfactual contrast (\cref{sec:DCE}); it stays tractable where
exact actual-cause enumeration times out at $17$ variables, continuing to
roughly $73$--$145$ variables (where the tests time out after 60 minutes) and holding a $0.85$ correct-attribution rate as
far as $89$
(\cref{sec:ac_benchmark}); it delivers a graded comparison on a
continuous-outcome dynamical SIR model, assigning lockdown roughly twice the
responsibility of masking where but-for reasoning cannot separate them
(\cref{sec:sir_benchmark}); and it scales to a deployed automated valuation
model with machine-learned components trained on millions of points, where SHAP
concentrates attribution on a few downstream variables while \ourapproach
spreads it across the upstream causal structure (\cref{sec:avm}). A reader who
wants the shape of the argument can stop here; the next two phases develop each
comparison in full.

The third phase positions \ourapproach against existing attribution methods.
\cref{sec:shap_examples} compares \ourapproach with SHAP, Causal SHAP, ATE, and
CATE on two worked examples (OBCB, with two binary features, and a
signal-with-mediation model), tabulating which desiderata each satisfies or
fails, including cases where SHAP and even Causal SHAP track correlation rather than
counterfactual dependence.

The fourth phase establishes the formal relationship between \ourapproach and
actual causality. \cref{sec:relation_with_ac} shows that under specific choices of
$\Gamma$ and $\Delta$ the \ourapproach kernel is positive on a configuration
exactly when that configuration is an actual cause in Halpern's sense, so
approximate AC verdicts can be read off the kernel without enumerating witnesses.
\cref{sub:pearl_actual_cause_prob} positions \ourapproach against Pearl's
probability of actual causation: on the canonical desert-traveller example
\ourapproach recovers Pearl's within-scenario ranking with graded magnitudes, and
a weak-poison variant reveals two further responsibility intuitions
(path-reliability and cross-scenario reliability) that \ourapproach's expectation
captures and Pearl's binary indicator cannot express.

\cref{sec:discussion} closes the fifth phase with discussion. We engage related
work where it bears on the argument. The broader literature on
actual-causality-based attribution and responsibility scores (probabilities of
causation, degree of blame and harm, and score-based explanations in databases
and machine learning, surveyed by \citet{bertossiAttributionScores2023}) comes
up mainly in \cref{sec:relation_with_ac,sub:pearl_actual_cause_prob} and the
discussion.

\begin{figure}[H]
  \centering
  \includegraphics[width=\textwidth]{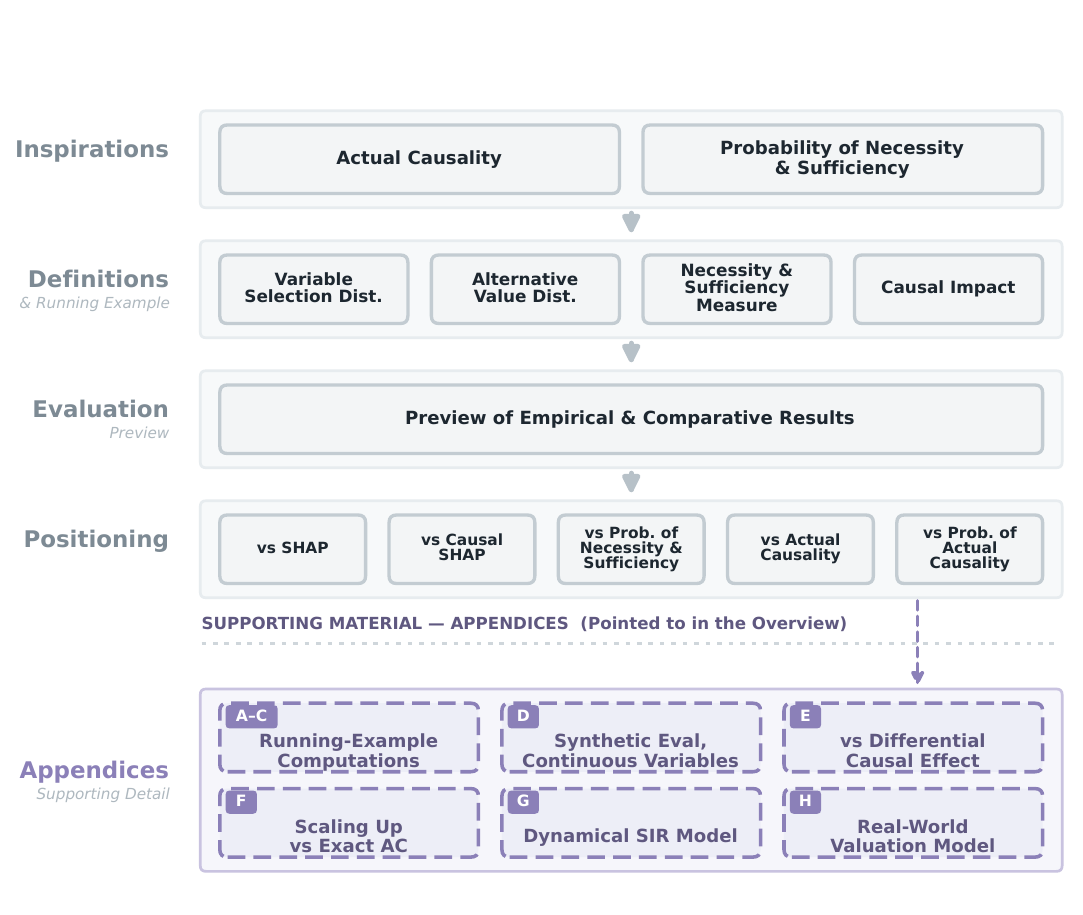}
  \caption{\textbf{Roadmap of the paper.} \ourapproach builds on two ideas,
  \emph{actual causality} and the \emph{probability of necessity and sufficiency}
  (top tier). From there the paper develops the motivations
  (\cref{sec:motivations}), defines \ourapproach (\cref{sec:causal_impact}), previews
  the empirical and comparative payoff (\cref{sec:evaluation}), positions it against
  SHAP, Causal SHAP, PNS, actual causality, and the probability of actual causation
  (\cref{sec:shap_examples,sec:relation_with_ac,sub:pearl_actual_cause_prob}), and
  closes with discussion; we defer the running-example computations and benchmark
  protocols to the appendices, to which the preview points.}
  \label{fig:roadmap_inspirations}
\end{figure}

\paragraph{Reproducibility.}
A set of companion notebooks reproduces every quantitative claim in this paper
(the worked-example numbers, the benchmark tables, and the figures), rendered
with their full outputs at
\url{https://rfl-urbaniak.github.io/explainable_paper/}. Throughout, we point to
the relevant notebook by name (for example \nb{obcb\_computations}); each name
links to its page on that site.

\section{Causal Impact: Motivations}\label{sec:motivations}

The central difficulty in causal attribution is not computing counterfactuals: it is knowing which counterfactuals to compute. Consider two questions: the \emph{type} question (does $X$ tend to cause $Y$?) and the \emph{token}, or actual-causation, question (was $X$'s contribution live in the specific situation we are explaining?). Methods built to answer the former routinely miss the latter, as Section~\ref{sec:intro}'s Alice already showed: a cause can be genuine in general yet causally inactive in one particular case, making no difference to what actually happened. The running example below makes this precise.
There are principled probabilistic theories of type-level causal strength (\citet{sprenger2018foundations}, for instance, derives a difference-making measure $p^*(E \mid C) - p^*(E \mid \neg C)$ axiomatically), but they deliberately bracket the token-level question of which cause was live in the case at hand, and amalgamate mediators just where the context-sensitivity we are after requires tracking them.
Among the closest predecessors of this approach, none individually achieves all of this: actual causality handles context sensitivity but is computationally intractable and not general enough;
the probability of necessity and sufficiency scales well, but lacks context sensitivity.
This section builds intuition for both limitations through a running example, setting up the design choices that motivate \ourapproach in Section~\ref{sec:causal_impact}.
We first revisit Halpern's actual-cause
construction on a deterministic Old Boys' Club Bank (OBCB) example to see what witnesses do and why a plain
but-for clause misses Alice's case. Turning to a stochastic OBCB, we then introduce the
probabilistic machinery actual causality lacks, Pearl's PN/PS/PNS, and show that, without
witnesses, these measures cannot register Alice's missing credit check either. A closing
``Path forward'' subsection distils what \ourapproach must inherit
from each predecessor before Section~\ref{sec:causal_impact} formalises it.

\subsection{Actual Causality}

Consider the following (deterministic, discrete) model:

\begin{example}[Bob at Old Boys' Club Bank] \label{example:bobatobcb}
Bob applied for a home mortgage loan with the local branch of OBCB, and his application was denied.
The OBCB policy requires applicants to pass a credit score check, but it only performs this check if the applicant is male; all non-male applicants are rejected outright. \label{ex:obcb} Thus, the procedure is:
\begin{align*}
\varname{check} &= (\varname{gender} = \text{male}) \\[4pt]
\varname{check\text{-}failed} &= \varname{check} \land (\varname{credit} = \text{bad}) \\[4pt]
\varname{loan} &= \lnot \varname{check\text{-}failed} \land (\varname{gender} = \text{male})
\end{align*}
\end{example}

\noindent An intuitive explanation for Bob's denial is that his credit was bad: had it been good, the outcome would have been different.
In contrast, his gender does not appear to be responsible for the rejection: had it been different,
his application still would have been denied (albeit this time due to $\varname{check}=F$).

This suggests a counterfactual approach to feature responsibility attribution: identify the features responsible for a particular outcome by asking whether it would have been different had a specific feature taken a different value.
This is called the \textit{but-for clause}: Bob wouldn’t have been denied the loan but for his bad credit.

Concretely, consider a model $M$ with \emph{exogenous variables} $\mathbf{U}$, which are the source of stochasticity, and \emph{endogenous variables} $\mathbf{V}$,  which are influenced by other variables in the model. 
Let $\mathbf{S} \subseteq \mathbf{V}$ be the variables we suspect to be causally responsible (call them \emph{causal suspects}), and 
$Y \in  \mathbf{V}$ be an outcome of interest. 
Suppose we consider only one suspect $S \in \mathbf{S}$ and mark factual values with $^\star$. 
Under a setting $\mathbf{U} = \mathbf{u}$ of the exogenous variables, $M$ produces $S = s^\star$ and $Y = y^\star$, denoted $(M, \mathbf{u}) \vDash S=s^\star, Y=y^\star$. We read $\vDash$ as ``yields'' (the model under noise $\mathbf{u}$ makes the stated equalities true), and its crossed form $\not\vDash$ as ``does not yield'' (the stated event fails to hold).
We call this the \emph{factual} scenario.
An intervened model where a variable $S$ is set to an alternative value $s'$ is denoted $(M, \mathbf{u})[S \leftarrow s']$.

A first attempt at defining a counterfactual notion of causal responsibility,
following the counterfactual-dependence tradition \citep{lewis1973causation,
actualCausalityHalpern}, is:
\begin{definition}[But-for, single]
$S=s^\star$ is causally responsible for  $Y=y^\star$  in $(M, \mathbf{u})$ if and only if $(M, \mathbf{u}) \vDash S=s^\star, Y=y^\star$ and $(M, \mathbf{u})[S \leftarrow s'] \not \vDash Y=y^\star$ for some $s' \neq s^\star$.
\end{definition}

\noindent 
In a given context, $S$ is causally responsible for an outcome only when intervening on $S$ with some alternative value would change that outcome. 
Bob’s credit score satisfies this condition: had his credit score been higher, the loan might have been approved.

However, this cause test is too coarse. Consider another case:
\begin{example}[Alice at OBCB] \label{example:aliceatobcb}
Alice is a woman with bad credit.
She applies for a loan with OBCB and is rejected.
\end{example}
\noindent Yet under the but-for clause for single cause candidates, neither her gender nor her credit score is responsible for Alice's denial.
No matter how high her credit score, her application would still have been denied because the bank does not, effectively, loan to women.
Had her gender been different, she would still have been denied for having bad credit.

To address this example, we could consider sets of factors as causes:
\begin{attempt}[But-for, sets]
Variables $\mathbf{C}\subseteq \mathbf{S}$ with factual values $\mathbf{C} = \mathbf{c}^\star$ 
are causally responsible for $Y = y^\star$ in $(M, \mathbf{u})$ if and only if 
$(M, \mathbf{u}) \vDash \mathbf{C} = \mathbf{c}^\star, Y = y^\star$ and $(M, \mathbf{u})[\mathbf{C} \leftarrow \mathbf{c'}] \not \vDash Y= y ^\star$ for some $\mathbf{c'}$ everywhere different from $\mathbf{c}^\star$, where no proper subset of $\mathbf{C}$ has this property. 
\end{attempt}

\noindent The minimality condition reflects our preference for explanations that include only relevant properties.

Under this generalization, the set  $\mathbf{c}^\star = \{\varname{gender}=\text{female}, \varname{credit}=\text{bad}\}$
 is jointly responsible for Alice's application being denied because 
 there exists $\mathbf{c}' = \{\varname{gender}=\text{male}, \varname{credit}=\text{good}\}$ that would change the decision
and $\mathbf{c}^\star$ is minimal; neither gender nor credit alone would have changed the decision. On individual features, we could say that the whole set bears responsibility but individual features do not, or that any feature in a cause set is partially and equally accountable.
Neither is satisfying: the first denies that any single feature is responsible, and the second spreads responsibility evenly between them regardless of context.
This conflicts with our intuition: OBCB did not evaluate Alice’s credit score, so her denial was caused by her gender.
This verdict does not rest on counterfactual reasoning alone. Tracing the model on Alice's inputs, her denial is produced through the $\varname{gender}$ term of the $\varname{loan}$ equation, while $\varname{credit}$ enters only through $\varname{check\text{-}failed}$, which is false and never propagates to the outcome. So on a production (causal-process) view of causation \citep{hall2004two}, credit transmits no influence to Alice's denial at all.
Actual causality's witness-based counterfactual test, developed below, recovers exactly this production-tradition verdict, so the diagnosis that gender alone is responsible is not an artifact of one chosen formalism.

The actual-causality literature calls this structural pattern \emph{preemption}, also called \emph{undercutting}~\citep{lewis1973causation,
actualCausalityHalpern}, the term we use from
Section~\ref{sec:causal_impact} onward, paired with \emph{overdetermination} as
the two archetypes the synthetic evaluation (\cref{sub:synthetic}) targets: two
distinct causal routes
might each have been sufficient to produce the outcome, but only one is actually
engaged, with the other prevented from operating.
Alice's gender preempted her credit by preventing the check from happening, so
credit had no chance to act.
The closely related notion of \emph{over-determination}~\citep{paulhall2013,
actualCausalityHalpern} covers cases in which
multiple sufficient causes coexist without one preempting the other (e.g., two
simultaneous fatal shots).
Both phenomena defeat a simple single-cause but-for test, since intervening on any one of the
candidate causes leaves the outcome unchanged.\footnote{
The canonical preemption example: Jim is both stabbed and poisoned, and dies.
Stabbing works quickly, so it kills Jim before the poison takes effect, and the
stabbing caused his death.
But-for clauses fail to break the symmetry here as well: removing the stabbing
leaves Jim dead by poison, and removing the poison leaves him dead by stab.
}

Halpern's construction offers a way out of this difficulty: hold some of the context in which counterfactuals are evaluated fixed,
calling such fixed variables \emph{witnesses}.
Let us focus on the mediator $\varname{check\text{-}failed}$ between $\varname{gender}$ and $\varname{loan}$.
If we fix this variable at the factual value by intervening, then counterfactually, Alice would have been granted the loan if she were male.
However, there is no context fixed at the actual value such that, had her credit score been higher, she would have been granted the loan.
Enumerating the candidate causes, their counterfactual treatments, the witness held fixed, and the resulting outcome makes the test concrete (Alice's factual run is $\varname{gender}{=}\text{female}$, $\varname{credit}{=}\text{bad}$, giving $\varname{check\text{-}failed}{=}0$ and $\varname{loan}{=}\text{denied}$):

\begin{center}
\begin{tabular}{llll}
\toprule
Candidate cause $\mathbf{C}$ & Treatment $\mathbf{c}'$ & Witness $\mathbf{T}=\mathbf{t}^\star$
  & Outcome ($\varname{loan}$) \\
\midrule
$\varname{gender}{=}\text{female}$ & ${\to}\,\text{male}$ & $\varname{check\text{-}failed}{=}0$
  & denied $\to$ \textbf{granted} \\
$\varname{gender}{=}\text{female}$ & ${\to}\,\text{male}$ & none (recomputed)
  & denied $\to$ denied \\
$\varname{credit}{=}\text{bad}$    & ${\to}\,\text{good}$ & any
  & denied $\to$ denied \\
\bottomrule
\end{tabular}
\end{center}

\noindent Fixing the mediator $\varname{check\text{-}failed}$ at its factual value reveals
$\varname{gender}$ as an actual cause (row~1); without the witness the intervention recomputes the
mediator and the outcome does not flip (row~2); no witness makes $\varname{credit}$ a cause (row~3).
This motivates the following formalization:
\begin{definition}[Actual cause (modified Halpern--Pearl)]\label{def:ac_intuitive}
Let $\mathbf{S}\subseteq\mathbf{V}$ be a suspect pool and
$\mathbf{W}\subseteq\mathbf{V}$ a witness pool, observed in $(M,\mathbf{u})$ at
factual values $\mathbf{S}=\mathbf{s}^\star$ and $\mathbf{W}=\mathbf{w}^\star$.
A cause set $\mathbf{C}\subseteq\mathbf{S}$ at its factual setting
$\mathbf{C}=\mathbf{c}^\star$ is an \emph{actual cause} of $Y=y^\star$ in
$(M,\mathbf{u})$ if there is a witness set $\mathbf{T}\subseteq\mathbf{W}$
disjoint from $\mathbf{C}$ such that:
\begin{itemize}[itemsep=2pt,topsep=2pt,parsep=0pt]
    \item \textbf{Factivity.} The cause, the witness pool, and the outcome
    are all observed at their factual values:
    \[ (M,\mathbf{u})\vDash \mathbf{C}=\mathbf{c}^\star,\;
                              \mathbf{W}=\mathbf{w}^\star,\;
                              Y=y^\star. \]
    \item \textbf{Witnessed necessity.} Holding $\mathbf{T}$ at its factual
    setting $\mathbf{t}^\star$ while intervening on the cause overturns the
    outcome, for some alternative $\mathbf{c}'$ distinct everywhere from
    $\mathbf{c}^\star$\footnote{Halpern's original clause (restated without
    this normalization as Definition~\ref{def:ac} in
    Section~\ref{sec:relation_with_ac}) only requires $\mathbf{c}'\neq\mathbf{c}^\star$,
    not a coordinatewise difference on every variable in $\mathbf{C}$; the two
    are equivalent given minimality, since if $\mathbf{c}'$ agrees with
    $\mathbf{c}^\star$ on some coordinate, dropping that coordinate from
    $\mathbf{C}$ yields a smaller set that already witnesses necessity, so a
    minimal $\mathbf{C}$ can always be witnessed by a fully-distinct
    $\mathbf{c}'$. We state the fully-distinct version here for a cleaner
    search procedure; Section~\ref{sec:relation_with_ac} uses Halpern's
    original statement to keep the correspondence proofs literal.}:
    \[ (M,\mathbf{u})\bigl[\mathbf{C}\leftarrow\mathbf{c}',\;
                            \mathbf{T}\leftarrow\mathbf{t}^\star\bigr]
        \not\vDash Y=y^\star. \]
    \item \textbf{Minimality.} No proper subset of $\mathbf{C}$ satisfies the
    two conditions above.
\end{itemize}
\end{definition}

\noindent Definition~\ref{def:ac_intuitive} is Halpern's \emph{modified} definition of actual
causality \citep{halpern2015modification, actualCausalityHalpern}, specialized to an explicit suspect
pool $\mathbf{S}$ and witness pool $\mathbf{W}$. \ourapproach uses Halpern's modified definition, with witnesses held at their actual values; prior probabilistic extensions of actual causation instead build on the original Halpern--Pearl definition (Section~\ref{sec:other-prob-ac}).
This definition extends but-for causality with a witness set $\mathbf{T}$ that is fixed to its factual value by
an intervention.
Under this definition, to determine the cause of an outcome, we need to find a cause set $\mathbf{C}$, alternative values 
$\mathbf{c}'$ (trivial if the variables are binary and less so if they are categorical or continuous),
and a witness set $\mathbf{T}$.
$\mathbf{C}$ and $\mathbf{T}$ are drawn from sets of size at most $2^{|\mathbf{V}|}$, so a brute-force computation
 is expensive.

\noindent The witnessed-necessity template above is the Halpern--Pearl definition, which has
been refined considerably since: \citet{beckersPrincipled2018} rebuild actual causation from a
production/dependence pair, \citet{beckersTransitivity2017} chart when it is (in)transitive and
asymmetric, and \citet{beckersSufficiency2021,beckersNESS2021} ground it in NESS-style sufficient
sets, where a cause is a necessary element of a set that suffices for the outcome; and
\citet{beckers2025nondeterministic} most recently extends it to nondeterministic
structural models. We build on the
same witnessed-necessity core. It has two halves: necessity under intervention, and
sufficiency of the held-fixed configuration. We treat both as graded quantities,
anticipating the probabilistic turn below.

Putting computational difficulties aside, the deeper limitation of the
construction so far is that it lives entirely in deterministic models.
Real-world attribution problems are typically probabilistic: policies apply
with some probability, mechanisms occasionally fail or are overridden, and the
analyst usually wants to reason about the unit-level outcome in the presence
of this noise.
One can, of course, place a distribution over the exogenous noise $\mathbf{u}$ and read
off the probability of the actual-cause verdict; this is the route taken by Pearl's
probability of causation \citep{causalityPearl} and by the degree-of-blame and harm
constructions of Chockler, Halpern, and Beckers
\citep{chocklerResponsibilityBlameStructuralModel2004, beckersHarm2022,
beckersQuantifyingHarm2023}.
The result is a probabilistic envelope around the same binary kernel, with the same computational demands.
This construction (Pearl's \emph{probability of actual causation}
\citep[Def.~10.3.5]{causalityPearl}) is also the closest neighbour of \ourapproach{}
and agrees with it in simple cases, so we defer a precise side-by-side comparison to
Section~\ref{sub:pearl_actual_cause_prob}. The computational demands it inherits from the
underlying binary kernel are one reason we
turn next to a different family of attribution tools.

\subsection{Probability of Necessity and Sufficiency}

To bring probability into the picture, we now consider a stochastic variant of the
OBCB example in which the bank's gendered policies apply only with some probability:
women are checked only sometimes, men are skipped only sometimes, and even when
the check happens its outcome is occasionally overridden by loan-officer
discretion. This is the natural setting for Pearl's probabilities of necessity,
sufficiency, and their conjunction (PN, PS, PNS)
\citep{pearl1999probabilities,tian2000probabilities},\footnote{Probabilities of
causation have also been put to work outside token-level attribution:
\citet{wangJordanDesiderata2021} formalize representation-learning desiderata
(non-spuriousness, efficiency, and disentanglement) as computable criteria built
from probabilities of causation, paralleling our use of the same primitives for
feature attribution.} introduced in this
subsection. These scale gracefully into the stochastic regime
but, without an analogue of the witness mechanism, mis-attribute Alice's denial.

\begin{example}[OBCB, stochastic]\label{ex:obcb_stochastic}
The bank performs credit checks for a randomly selected 20\% of women and skips checks for a randomly selected 10\% of men.
Unchecked applicants are rejected.
Furthermore, if a man's credit check fails, he is nevertheless accepted 5\% of the time and if a woman's credit check succeeds, she is denied 10\% of the time.
We assume that the population is evenly divided into male and female and good and bad credit.
Formally, the model, using 1 for male and 1 for good credit, is:

\begin{align*}
\varname{loan-prob}(\varname{gender}, \varname{check\text{-}failed}) &=
\begin{array}{r|cc}
 & \varname{check\text{-}failed}{=}0 & \varname{check\text{-}failed}{=}1 \\ \hline
\varname{gender}{=}0 & 0.9 & 0 \\[2pt]
\varname{gender}{=}1 & 1   & 0.05
\end{array}
\\[6pt]
\varname{credit} &\sim \mathrm{Bern}(0.5) \\
\varname{gender} &\sim \mathrm{Bern}(0.5) \\
\varname{check} \mid \varname{gender}
   &\sim \mathrm{Bern}\bigl((1-\varname{gender}) \cdot 0.2
                            + \varname{gender} \cdot 0.9\bigr) \\
\varname{check\text{-}failed} &= \varname{check} \cdot (1-\varname{credit}) \\
\varname{loan\text{-}if\text{-}checked}
   \mid \varname{gender}, \varname{check\text{-}failed}
   &\sim \mathrm{Bern}\bigl(\varname{loan-prob}(\varname{gender},
                                                \varname{check\text{-}failed})\bigr) \\
\varname{loan} &= \varname{loan\text{-}if\text{-}checked} \cdot \varname{check}
\end{align*}
\end{example}

\begin{figure}[htbp]
\centering
\includegraphics[width=0.62\linewidth]{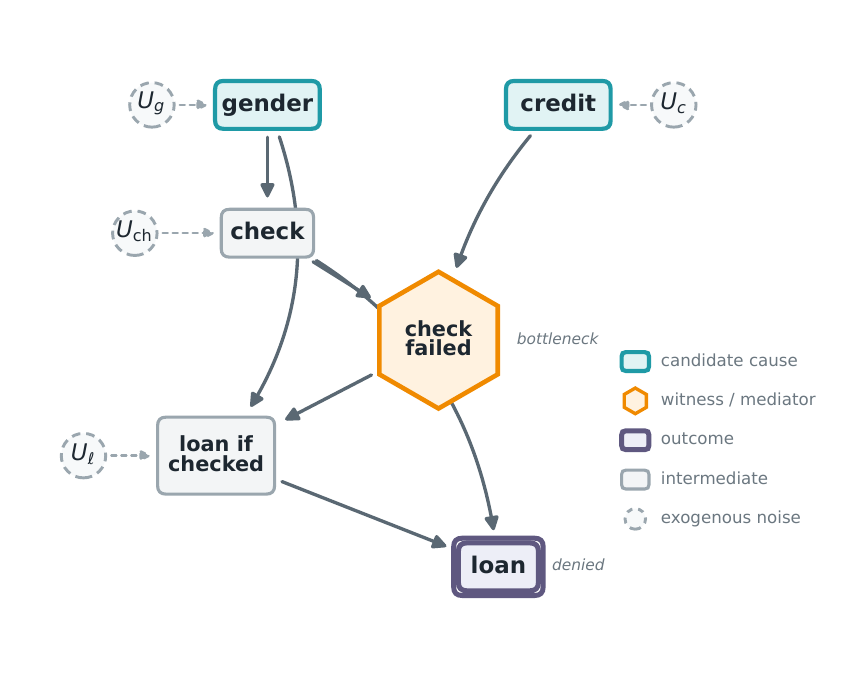}
\caption[OBCB causal structure]{Causal structure of the OBCB loan model
(Example~\ref{ex:obcb_stochastic}). Node shape and colour encode role: the
\emph{candidate causes} \varname{gender} and \varname{credit} (teal boxes), the
\emph{witness} \varname{check-failed} (gold hexagon), and the \emph{outcome}
\varname{loan} (purple double box); intermediate variables are grey. Dashed circles
are the exogenous noise $U_V$ feeding each \emph{stochastic} mechanism; the
deterministic nodes \varname{check-failed} and \varname{loan} carry none. \varname{check} reaches \varname{loan} along two routes: through the
\varname{check-failed} bottleneck and directly as the screening gate
($\varname{loan}=\varname{loan-if-checked}\cdot\varname{check}$). \ourapproach holds
the bottleneck witness fixed to keep these two routes distinct.}
\label{fig:dag_obcb}
\end{figure}

With this stochastic model in hand, we can state seven \emph{desiderata} that any
responsibility attribution for Alice and Bob should satisfy, giving every construction
in the rest of the paper an explicit target.
Let $R(V \rightsquigarrow \varname{loan} \mid \text{person})$ denote the causal responsibility of variable $V$
for the loan outcome at the individual level. $R$ is a placeholder for any attribution score, not a
specific formal quantity; \ourapproach interprets $R$ as the
PCI score $\mathrm{PCI}_{X_k}$ of Definition~\ref{def:causalimpact}.
The desiderata themselves are
method-agnostic: they encode the graded active/preempted/irrelevant ordering and the
cross-individual asymmetry that the actual-causality and preemption literature
already distinguishes \citep{lewis1973causation,hall2004two,actualCausalityHalpern}.
PN, PS, PNS, SHAP,
and Causal SHAP are held to this same independently motivated standard below and in
Section~\ref{sec:shap_examples}. Here $V$ ranges over the candidate variables; in the general development
it is the variable of interest $X_k$.
Attribution methods may produce signed values; the desiderata are stated in terms of absolute
magnitudes $|R(\cdot)|$, allowing for directionality:
\label{desid}
\begin{align*}
\textbf{D-A1} \quad & |R(\varname{gender} \rightsquigarrow \varname{loan} \mid \text{Alice})| > 0 \\[2pt]
\textbf{D-A2} \quad & |R(\varname{credit} \rightsquigarrow \varname{loan} \mid \text{Alice})| > 0 \\[2pt]
\textbf{D-A-rank} \quad & |R(\varname{gender} \rightsquigarrow \varname{loan} \mid \text{Alice})| > |R(\varname{credit} \rightsquigarrow \varname{loan} \mid \text{Alice})| \\[6pt]
\textbf{D-B1} \quad & |R(\varname{gender} \rightsquigarrow \varname{loan} \mid \text{Bob})| > 0 \\[2pt]
\textbf{D-B2} \quad & |R(\varname{credit} \rightsquigarrow \varname{loan} \mid \text{Bob})| > 0 \\[2pt]
\textbf{D-B-rank} \quad & |R(\varname{credit} \rightsquigarrow \varname{loan} \mid \text{Bob})| > |R(\varname{gender} \rightsquigarrow \varname{loan} \mid \text{Bob})| \\[6pt]
\textbf{D-comp} \quad & |R(\varname{gender} \rightsquigarrow \varname{loan} \mid \text{Alice})| > |R(\varname{gender} \rightsquigarrow \varname{loan} \mid \text{Bob})|
\end{align*}

\noindent \textbf{D-A1} and \textbf{D-A2} both require non-zero attribution for Alice because two
distinct causal routes were available. Gender filtered her out at the checking stage. Credit is not
causally inert either: the check is only \emph{probabilistic} (women are still checked with positive
probability, $\gamma_F=0.2$ in the model above), so conditioning on Alice's denial leaves nonzero
probability that she was in fact checked, in which case her bad credit was the operative cause. The
broader principle is a graded ordering of three roles. An \emph{active} cause, on the path that actually
produced the outcome, carries the most responsibility. A \emph{preempted} cause retains a
\emph{residual} share: strictly positive, yet strictly below the active cause (exactly what
\textbf{D-A-rank} demands), governed by the probability that the preemption fails to fire,
leaving the credit route live. An \emph{irrelevant} variable, with no active path to the outcome under
any context, receives the least of all, approaching zero. This is a deliberate departure from the binary
Halpern--Pearl/Beckers treatment of preemption, which collapses the first two roles by zeroing a fully
preempted cause. \textbf{D-A-rank} is the substantive constraint: because Alice never underwent a
credit evaluation, gender's causal role is immediate, whereas credit's is purely
hypothetical. A method that ranks credit above gender for Alice is misattributing the cause of her
rejection.

For Bob, \textbf{D-B2} reflects that the credit check occurred and directly produced the rejection:
the causal chain $\varname{credit} \to \varname{check\text{-}failed} \to \varname{loan}$ was
activated. \textbf{D-B1} requires strictly positive attribution for \varname{gender}: being male
gave Bob a $90\%$ probability of being checked, so gender opened the evaluation that rejected him.
This indirect causal role is non-zero; a method returning exactly zero fails to capture it.
\textbf{D-B-rank} corresponds to the intuition that credit is the proximate cause.
\textbf{D-comp} captures the cross-individual asymmetry that is hardest for naive methods to
express: gender is more causally responsible for Alice's rejection than for Bob's. For
Alice, \varname{gender}=female was overwhelmingly likely to produce the rejection on
its own (evaluation happens only $20\%$ of the time for women, so credit is rarely the
operative factor); for Bob, \varname{gender}=male reliably opened the evaluation (a
$90\%$ chance) that \varname{credit}=bad then caused.
We verify in Section~\ref{sec:causal_impact} that \ourapproach satisfies all seven
desiderata, while PN, PS, PNS, developed next in this section, fail some of them. Later,
we will also revisit this example and its desiderata in the context of SHAP and Causal
SHAP (Section~\ref{sec:shap_examples}).

\noindent Returning to the comparison between the two predecessors: the witness mechanism
works well for deterministic models, but actual causality
in its standard form offers no straightforward way to handle probabilistic ones. Before
extending the witness idea to the stochastic setting, it is instructive to examine a
probabilistic approach that scales more naturally, Pearl's probabilities of necessity
and sufficiency \citep{pearl2022probabilities}, and to ask what it would miss without
witnesses. This will let us identify what probabilistic machinery is needed, and what
context-sensitivity is lost when witnesses are absent, motivating the synthesis that
\ourapproach provides.
The following definitions allow for probabilistic models, but assume that all variables are Boolean (PCI will relax this assumption,
 allowing for continuous variables).
Pearl considers two ways that a cause may relate to an effect: it may be necessary or sufficient for the effect.
A cause is necessary if, counterfactually absent the cause, the effect is unlikely to occur; it is sufficient
if the effect is likely when the cause counterfactually occurs.
We now formalize these ideas and discuss them in the context of the running example.

\begin{definition}[Probability of Necessity]
Let the probability of necessity \citep{pearl1999probabilities} be the probability that the outcome $Y=y$ would not have occurred in the absence of  $C = c^\star$, given that both occurred: \[\pn(C = c^\star, Y = y^\star) = P\!\left(Y_{c'} \neq y^\star \;\middle|\; C = c^\star, Y = y^\star\right),\] where $Y_{c'}$ is the value of $Y$ after intervening on $C$ with $c'\neq c^\star$.\footnote{\label{fn:twin-network}Read in twin-network fashion \citep{balke1994probabilistic, causalityPearl}: one draws an exogenous noise $\mathbf{u}\sim P_{\mathbf{U}}$, computes the factual $Y$ under the unintervened structural equations and the counterfactual $Y_{c'}$ under the intervened equations, and conditions on the joint event $\{C=c^\star,Y=y^\star\}$. Both $Y$ and $Y_{c'}$ live on the same probability space and share the same $\mathbf{u}$.}
\end{definition}

Strictly speaking, the probability of necessity is less specific to a particular situation than the actual-cause test above: as written, it conditions only on the cause--outcome pair, so it asks whether being male causes \emph{someone} to be denied, not whether it caused \emph{Bob} (whose bad credit we already know) to be denied.
We therefore report two readings.
The \emph{population-level} value conditions only on the cause and outcome, exactly as the definition prescribes.
The \emph{individual-level} value additionally conditions on the subject's other observed features (for Bob, on $\varname{credit}=\text{bad}$), specializing the query to that person.
The individual-level reading is the one closer to actual causation; we report both throughout, and the same population/individual split applies to $\ps$ and $\pns$ below.

We illustrate the computation for Alice's gender. Alice is factually
$\varname{gender}=\text{female}$, $\varname{credit}=\text{bad}$,
$\varname{loan}=F$; the counterfactual intervenes to set
$\varname{gender}=\text{male}$ while leaving $\varname{credit}=\text{bad}$.
Tracing the male-intervention branches through the model
(Appendix~\ref{app:obcb}, \S\ref{app:obcb_pn}) gives
$\pn(\varname{gender}=\text{female}, \varname{loan}=F \mid \text{Alice})
= 0.9 \cdot 0.05 = 0.045$.
Computing the analogous quantities for the remaining (variable, person) pairs gives
the individual-level values shown in Table~\ref{tab:pn_ps_pns}; the table also lists
population-level values for comparison.

The PN scores partially align with intuition: Bob's bad credit is high
($\pn=0.90$), and his gender score is zero, matching the reading that his gender played
no role.
But Alice's PN values do not separate gender (0.045) from credit (0.18)
particularly well, and Alice's gender PN is only marginally above Bob's, not the
sharp distinction we would want, given that the bank never even evaluated Alice's
credit.
PN cannot register this asymmetry: it asks counterfactually what would happen
under intervention without any context about which causal paths were active.
For the full picture we need a second quantity, the probability of sufficiency.

\begin{definition}[Probability of Sufficiency]
Let the probability of sufficiency be the probability that the outcome $Y=y^\star$ would have occurred in the presence of $C = c^\star$, given that neither $C=c^\star$ nor $Y=y^\star$ occurred:
\[\ps(C = c^{\star}, Y = y^\star) =  P\!(Y_{c^\star} = y^\star \;\vert\;  C = c',\ Y \neq y^\star)\]
where $c' \neq c^\star$ and $Y_{c^{\star}}$ is the value of $Y$ after intervening on $C$ with $c^\star$.
\end{definition}

To query whether the factual cause was sufficient, we condition on the counterfactual: we assume the cause did \emph{not} take its factual value and the outcome did \emph{not} occur, then intervene to restore the factual cause value and ask whether the outcome now occurs.
Intuitively, this asks: if the cause had been absent and things had gone differently, would reinstating the cause have brought the outcome about?

We illustrate $\ps$ for Alice's gender. Following the same convention as for
$\pn$, the individual-level computation additionally conditions on Alice's
credit ($\varname{credit}=\text{bad}$): we fix her credit but flip the cause and
outcome to their non-factual values ($\varname{gender}=\text{male}$,
$\varname{loan}=T$), then intervene to restore $\varname{gender}=\text{female}$
and ask whether the loan is again denied. Both check branches yield
$\varname{loan}=F$ (Appendix~\ref{app:obcb}, \S\ref{app:obcb_ps}), so
$\ps(\varname{gender}=\text{female}, \varname{loan}=F \mid \text{Alice}) = 1$.
The remaining individual values, along with population-level values for comparison,
appear in Table~\ref{tab:pn_ps_pns}.
Two cases warrant comment: Bob's gender PS is undefined because the required
conditioning event ($\varname{gender}=\text{female}$, $\varname{credit}=\text{bad}$,
$\varname{loan}=T$) has probability zero in the model; Bob's credit PS is $0.955$
rather than $1$ because the model includes a $5\%$ exception rate for males who fail
the credit check.

To combine the insights of the two scores, we would like to find causes that are both necessary and sufficient.
This notion is  formalized as follows:

\begin{definition}[Probability of necessity and sufficiency]
Let the probability of necessity and sufficiency be the probability that the outcome $Y=y^\star$ would have occurred in the presence of $C = c^\star$ and that it would not have occurred under the intervention $C=c'$:
\[\pns(C = c^{\star}, Y = y^\star) =  P\!(Y_{c^{\star}} = y^\star, Y_{c'} \neq y^\star ).\]
This quantity factorises into a weighted sum of $\pn$ and $\ps$ by the
consistency rule of counterfactuals
\citep[\S 9.2]{causalityPearl}:
\[
  \begin{aligned}
    \pns(C = c^\star,\, Y = y^\star)
    &= P(C = c^\star,\, Y = y^\star)\,\pn(C = c^\star,\, Y = y^\star) \\
    &\quad + P(C = c',\, Y \neq y^\star)\,\ps(C = c^\star,\, Y = y^\star).
  \end{aligned}
\]
\end{definition}

\begin{table}[ht]
\centering
\small
\caption{$\pn$, $\ps$, and $\pns$ for Alice and Bob in
  Example~\ref{ex:obcb_stochastic}, with $\varname{loan}=F$ as the outcome.
  ``Population'' values condition only on the cause--outcome pair prescribed by the
  definition; ``Individual'' values additionally condition on the subject's other
  factual variable (e.g.\ for Alice's gender, on $\varname{credit}=\text{bad}$).
  A dash marks an entry that is undefined because the conditioning premises are
  inconsistent with the model.}
\label{tab:pn_ps_pns}
\vspace{4pt}
\begin{tabular}{ll @{\hspace{1.4em}} ccc @{\hspace{1.4em}} ccc}
\toprule
& & \multicolumn{3}{c}{Alice} & \multicolumn{3}{c}{Bob} \\
\cmidrule(lr){3-5}\cmidrule(lr){6-8}
Variable & Level & $\pn$ & $\ps$ & $\pns$ & $\pn$ & $\ps$ & $\pns$ \\
\midrule
$\varname{gender}$ & Population & 0.43  & 0.83 & 0.39  & 0.02  & 0.10  & 0.01 \\
$\varname{gender}$ & Individual & 0.045 & 1.00 & 0.045 & 0     & ---   & 0    \\[3pt]
$\varname{credit}$ & Population & 0.53  & 0.96 & 0.52  & 0.53  & 0.96  & 0.52 \\
$\varname{credit}$ & Individual & 0.18  & 1.00 & 0.18  & 0.90  & 0.955 & 0.86 \\
\bottomrule
\end{tabular}
\end{table}

The population- and individual-level $\pns$ values appear in the rightmost column
group of Table~\ref{tab:pn_ps_pns}.
These values capture the relative causal impact of $\varname{gender}$ and $\varname{credit}$ but are not entirely satisfying.
\Cref{tab:pn_ps_pns} reveals the problem directly: for Alice, $\pns(\varname{gender}) = 0.045 < \pns(\varname{credit}) = 0.18$, so PNS ranks credit as more causally responsible than gender.
Yet the bank never evaluated Alice's credit: the $\varname{check}$ variable was $F$ throughout.
PNS has no mechanism to register this: it asks counterfactually what would happen if credit were different, without fixing $\varname{check\text{-}failed}$, the mediating variable that blocked credit's causal path.

As an attempt to help us better understand Alice's particular situation, we might consider a modification of PNS with additional conditioning:
\[\pns_{c}(C = c^\star, Y = y^\star ~|~ Z=z^\star) = P(Y_{c^\star} = y^\star, Y_{c'} \neq y^\star ~|~ Z = z^\star)\]
Using this definition, we can compute $\pns_c(\varname{gender} = \text{female}, \varname{loan} = F ~|~ \varname{credit} = \text{bad}) = 0.045$ and $\pns_c(\varname{credit} = \text{bad}, \varname{loan} = F ~|~ \varname{gender} = \text{female}) = 0.18$.
Conditioning does not solve the problem described above (Alice's credit was in reality never evaluated) because the interventions override the mediating variable that tracks whether her credit has been checked.
When we compute $\pns_c(\varname{credit}=\text{bad}, \varname{loan}=F \mid \varname{gender}=\text{female})$, conditioning on gender being female is consistent with the factual, but when we then intervene to set credit to good, this intervention propagates downstream and changes $\varname{check\text{-}failed}$, overriding what we learned from conditioning on Alice's factual outcome.
The context we fixed is undone by the very intervention meant to probe it.
We will not consider this use of conditioning further. This observation motivates the need to intervene on (as
opposed to condition) mediating variables that provide important context, which the witness mechanism supplies.

\paragraph{ATE and CATE share the same limitation.}
The two most familiar causal quantities in the ML and statistics literature are the average treatment effect
$\mathrm{ATE}(T) = \mathbb{E}[Y \mid do(T{=}1)] - \mathbb{E}[Y \mid do(T{=}0)]$, where $do(\cdot)$ denotes
Pearl's intervention operator \citep{causalityPearl},
and the conditional version
$\mathrm{CATE}(T \mid X{=}x) = \mathbb{E}[Y \mid do(T{=}1), X{=}x] - \mathbb{E}[Y \mid do(T{=}0), X{=}x]$.
Computed on our stochastic OBCB, population ATE ranks credit ($0.518$) above gender ($0.383$),
mirroring the population PN ranking.
At the individual level for Alice, $\mathrm{CATE}(\varname{credit}\mid\varname{gender}{=}\text{female}) = 0.18$
and $\mathrm{CATE}(\varname{gender}\mid\varname{credit}{=}\text{bad}) = 0.045$.
These are exactly Alice's individual $\pns$ values: with binary outcomes and $\ps = 1$ in the
relevant branches, CATE for a single treatment collapses to $\pns_c$, so the same
$\varname{check\text{-}failed}$ propagation that undermined $\pns_c$ undermines CATE.
A second, sharper failure follows from CATE's inability to condition on the factual
treatment value: $\mathrm{CATE}(\varname{gender}\mid\varname{credit}{=}\text{bad}) = 0.045$ is identical
for Alice and Bob even though one's gender carried her denial outright while the other's
merely opened the credit check that then rejected him.
The structural limitation cuts across PNS, ATE, and CATE alike, so the witness mechanism in Section~\ref{sec:causal_impact} is not
just a fix for PNS but a general repair for the family of methods that average
counterfactual effects without intervention-based context-sensitivity.

\subsection{Path Forward}

The two predecessors examined in this section have complementary strengths and blind spots.
Halpern's actual causality is context-sensitive: by holding witnesses fixed at their factual values, it correctly diagnoses Alice's case, attributing her denial to gender rather than credit.
But in its standard form it is restricted to deterministic models, it requires a brute-force search over witness sets, and it has no native treatment of continuous variables.
Pearl's PN/PS/PNS make the converse trade-off.
They scale naturally to probabilistic models, because every counterfactual reduces to an expectation under the noise distribution.
But they have no analogue of the witness mechanism: every intervention propagates through the structural equations and overrides whatever mediator value the factual context fixed.
Table~\ref{tab:pn_ps_pns} shows the symptom: for Alice, $\pns(\varname{credit}) > \pns(\varname{gender})$, even though the credit check was never performed.

Any replacement therefore needs intervention-based context-sensitivity, holding the relevant mediators fixed by intervention.
\ourapproach{} delivers exactly this, combining three ingredients into a single estimator.
From actual causality it takes the \emph{witness mechanism}: it holds relevant context fixed by intervention, so the structural information about which causal paths were active survives the counterfactual evaluation.
From PN/PS/PNS it takes a \emph{probabilistic vocabulary}: it expresses counterfactuals as expectations under the noise distribution, so the attribution lives on the same probability space as the model and extends naturally to continuous variables.
And in place of brute-force enumeration over witness sets, it \emph{integrates} over a distribution of them, an integral estimable from samples even when the witness pool is large or the variables continuous.

Returning to Alice, \ourapproach{} reverses the ranking $\pns$ produced (Table~\ref{tab:pn_ps_pns}), recovering the intuitive verdict.
We revisit the example in full once the method is defined (Section~\ref{sec:causal_impact}, with the complete computation in Appendix~\ref{app:obcb}).

Two features distinguish \ourapproach{} from both predecessors.
Unlike actual causality, it does not ask whether \emph{some} witness set validates a counterfactual dependency; it integrates over a distribution of witness sets, weighting each by its probability.
This replaces combinatorial enumeration with statistical estimation.
The gain is also conceptual. Existential quantification certifies a cause as soon as \emph{some} witness set validates the dependence, a sensitivity that graded and normality-based accounts were introduced in part to address, by weighting alternatives according to how typical they are instead of treating any single one as decisive. Integration weights each witness set by its probability, so a variable is responsible to the degree that \emph{typical witness sets} render it efficacious.
The witness still does the work of holding context fixed. \ourapproach{} simply declines to privilege one hand-picked set, letting the structural facts decide which witnesses carry the mass (for Alice, those pinning $\varname{check\text{-}failed}$ favour gender; almost none favour credit).
As a by-product, actual causality's binary verdict becomes a graded magnitude.
Unlike PN/PS/PNS, it inherits the context-sensitivity of actual causality through this integration: it downweights causally inert variables automatically, because almost every witness set leaves them disconnected from the outcome, while it upweights variables that drive the outcome, because almost every witness set leaves them connected.
The formal definitions in Section~\ref{sec:causal_impact} make this expectation precise.

\medskip\noindent With the limitations of both predecessors in mind, we turn
in Section~\ref{sec:causal_impact} to the formal development. The next section
introduces \ourapproach piece by piece (PSCM scaffolding, the variable
selection distribution $\Gamma$, the alternative-value distribution $\Delta$, the
joint necessity-sufficiency measure, and the user-chosen causal impact function
$ci$) and verifies on the same Alice/Bob example that the resulting attribution
satisfies all seven desiderata stated above (p.~\pageref{desid}).

\section{Causal Impact: Definitions}\label{sec:causal_impact}

Explanation and causal attribution seek to connect a potentially causal \textit{event} $X=x$ to an outcome event $Y=y$.\footnote{This is quite different from type-level causal queries, where one is interested in some group-level attribution, as when one estimates average treatment effect or conditional average treatment effect.}
We return to Example~\ref{ex:obcb_stochastic} to ground the formal development that follows,
 using Alice and Bob as running illustrations throughout this section. Both are rejected for a loan; the question is
 which of their attributes (\varname{gender} and \varname{credit}) was causally responsible for each rejection,
 and to what degree. While the example is discrete, the machinery applies to continuous and mixed-type variables;
  a continuous illustration is given in Sections~\ref{sec:signal} and~\ref{sub:synthetic}.

\paragraph{Roadmap of this section.}
Section~\ref{sec:motivations} already stated seven \emph{desiderata}
(p.~\pageref{desid}) that any responsibility attribution for Alice and Bob
should satisfy; they give the formal construction below an explicit target. We
introduce the probabilistic
structural causal model (PSCM) and interventions, the substrate on which everything
else rests, and then the two distributions \ourapproach offers to the user: the
variable selection distribution $\Gamma$ over suspect/witness pairs and the
alternative-value distribution $\Delta$ over counterfactual cause settings. We combine
these into the joint necessity-sufficiency measure $P_k^{s,n}$ and the user-chosen
causal impact function $ci$ whose expectation is the \ourapproach score, showing how
the choice of $ci$ generalises the framework from binary to continuous outcomes.
Finally we compute the score on the OBCB example, verify that all seven desiderata
hold, and contrast \ourapproach's necessity/sufficiency decomposition with Pearl's
PN/PS/PNS.

Recall $R(V \rightsquigarrow \varname{loan} \mid \text{person})$, the placeholder for
the causal responsibility of variable $V$ for the loan outcome at the individual level
used to state those desiderata (Section~\ref{sec:motivations}, p.~\pageref{desid}):
\ourapproach interprets $R$ as the PCI score $\mathrm{PCI}_{X_k}$ of
Definition~\ref{def:causalimpact}, but the desiderata themselves are method-agnostic.
Here $V$ ranges over the candidate variables; in the general development
it is the variable of interest $X_k$.

We work within the context of probabilistic structural causal models \citep{causalityPearl}.
\begin{definition}[Probabilistic Structural Causal Model (PSCM)]
A probabilistic structural causal model is a tuple
\[
M = \langle \mathbf{U}, \mathbf{V}, \mathbf{F}, P_{\mathbf{U}} \rangle,
\]
where \(\mathbf{V} = \{V_1,\dots,V_n\}\) is a finite set of endogenous variables,
\(\mathbf{U} = \{U_1,\dots,U_m\}\) is a set of exogenous variables (not
necessarily one per endogenous variable: a deterministic mechanism, $f_i$
depending only on $\mathbf{Pa}_i$, needs no $U_i$ of its own),
and \(\mathbf{F} = \{f_1,\dots,f_n\}\) is a set of structural assignments
\[
V_i := f_i(\mathbf{Pa}_i, U_i), \qquad i = 1,\dots,n,
\]
with \(\mathbf{Pa}_i \subseteq \mathbf{V}\setminus\{V_i\}\) denoting the parents of \(V_i\), and
$U_i\in\mathbf{U}\cup\{\varnothing\}$ its own exogenous input, if any (a purely
deterministic $f_i$ takes $U_i=\varnothing$ and depends on $\mathbf{Pa}_i$ alone).
The distribution \(P_{\mathbf{U}}\) is a probability measure on $\dom(\mathbf{U})$,
and the assignments in \(\mathbf{F}\) are assumed to be acyclic and to have a unique solution for \(\mathbf{V}\) given \(\mathbf{U}\).\footnote{Working within the PSCM framework commits the practitioner to fully specified structural assignments $\mathbf{F}$ and an exogenous distribution $P_{\mathbf{U}}$, analogous to how Bayesian inference requires a likelihood and a prior. This enables \ourapproach to yield individual-level counterfactual attributions rather than population-level bounds, and once $\mathbf{F}$ and $P_{\mathbf{U}}$ are specified, sampling from $P_{\mathbf{U}}$ is a routine probabilistic inference operation.}
\end{definition}

\medskip\noindent A PSCM encodes which variables influence which others, 
through what mechanisms, and subject to what randomness. The structural equations specify
 each variable as a deterministic function of its causal parents and local noise; the exogenous 
 distribution captures all irreducible uncertainty in the system. Together they make every counterfactual
  question well-posed and computable.

The OBCB model (Example~\ref{ex:obcb_stochastic}) is a PSCM with endogenous
variables
\[
\mathbf{V}=\{\varname{gender},\varname{credit},\varname{check},\varname{check\text{-}failed},\varname{loan\text{-}if\text{-}checked},\varname{loan}\},
\]
and exogenous variables
\[
\mathbf{U}=\{U_{\varname{gender}},U_{\varname{credit}},U_{\varname{check}},U_{\varname{loan\text{-}if\text{-}checked}}\}
\]
which induce independent Bernoulli noise.

The structural assignments in the case at hand take the form $V_i = f_i(\mathbf{Pa}_i, U_i)$
promised by the PSCM definition above. Each stochastic node is a Bernoulli draw whose
success probability is a primitive parameter of the model (possibly depending on the
node's parents); writing $V \sim \mathrm{Bernoulli}(p)$ for a node equal to $1$ with
probability $p$, the mechanisms are
\begin{align*}
\varname{gender}        &\sim \mathrm{Bernoulli}(p_g), \\
\varname{credit}        &\sim \mathrm{Bernoulli}(p_c), \\
\varname{check} \mid \varname{gender}        &\sim \mathrm{Bernoulli}\bigl(\gamma_F(1-\varname{gender}) + \gamma_M\,\varname{gender}\bigr), \\
\varname{check\text{-}failed}  &= \varname{check}\,(1-\varname{credit}), \\
\varname{loan\text{-}if\text{-}checked} \mid \varname{gender}, \varname{check\text{-}failed} &\sim \mathrm{Bernoulli}\bigl(\mathrm{loan\text{-}prob}(\varname{gender}, \varname{check\text{-}failed})\bigr), \\
\varname{loan}          &= \varname{loan\text{-}if\text{-}checked} \cdot \varname{check}.
\end{align*}
So \varname{gender} and \varname{credit} are fair coin flips; \varname{check} fires with
probability $\gamma_F$ for women and $\gamma_M$ for men; \varname{check\text{-}failed} and
\varname{loan} are deterministic functions of the nodes above them
(Figure~\ref{fig:dag_obcb}). The Bernoulli draws are coupled across factual and counterfactual worlds, making counterfactuals well-defined. For the running example we take $p_g = p_c = 0.5$,
$\gamma_F = 0.2$, $\gamma_M = 0.9$, and the values of $\mathrm{loan\text{-}prob}$
given in Example~\ref{ex:obcb_stochastic}.

\begin{definition}[Interventions]
Given a model \(M = \langle \mathbf{U}, \mathbf{V}, \mathbf{F}, P_{\mathbf{U}} \rangle\),
an intervention on variables \(\mathbf{X} = \{X_1, \dots, X_k\} \subseteq \mathbf{V}\)
with values \(\mathbf{x} = (x_1, \dots, x_k)\) replaces the structural assignments for \(X_i\) by $X_i := x_i$, where $i = 1, \dots, k$.
We denote this intervention by either \([\mathbf{X} \leftarrow \mathbf{x}]\) or $do(\mathbf{X} = \mathbf{x})$.
\end{definition}

\medskip\noindent An intervention surgically replaces a variable's structural equation with a fixed constant, 
disconnecting it from its usual causes. The rest of the model propagates downstream as normal under the new assignment;
 the exogenous noise is unchanged.

In the OBCB model, $\mathrm{do}(\varname{gender}=1)$ replaces Alice's gender assignment with male, leaving all other structural equations intact. The downstream effect propagates: \varname{check} now draws from $\mathrm{Bern}(0.9)$ instead of $\mathrm{Bern}(0.2)$, so Alice would be checked in 90\% of noise realizations. This is the counterfactual world in which we ask whether her loan outcome would have differed.

Since \ourapproach supports arbitrary variable types, it will be helpful to define a measure-theoretic object corresponding to interventional outcomes; because outcomes are a deterministic function of exogenous noise, this object turns out to be a Dirac measure.

\begin{definition}[Pointwise Interventional Law] \label{def:interlaw}
Let $M=\langle \mathbf{U},\mathbf{V},\mathbf{F},P_{\mathbf{U}}\rangle$ be a probabilistic
structural causal model and let $Y\in\mathbf{V}$.
For an intervention $\mathrm{do}(\mathbf{X}=\mathbf{x})$, let
$Y':\dom(\mathbf{U})\to\dom(Y)$ denote the measurable map induced by the
modified structural assignments.
The \emph{pointwise interventional law} of $Y$ at fixed $\mathbf{u}\in\dom(\mathbf{U})$ is the probability measure on
$\dom(Y)$ defined by
\[
P\!\left(Y\in A \mid \mathbf{u},\mathrm{do}(\mathbf{X}=\mathbf{x})\right)
\;:=\;
\mathbb{I}\!\left\{Y'(\mathbf{u})\in A\right\}
\;=\;\delta_{Y'(\mathbf{u})}(A),
\qquad
A\subseteq\dom(Y).
\]
\end{definition}

\medskip\noindent The pointwise interventional law records the deterministic
outcome of $Y$ under $\mathrm{do}(\mathbf{X}=\mathbf{x})$ for a fixed noise realization
$\mathbf{u}$: since the modified structural equations make $Y$ a function of $\mathbf{u}$
alone, its distribution collapses to a point mass. Integrating over $P_{\mathbf{U}}$
then recovers the full interventional distribution.

The suspect and witness apparatus introduced below is built from a single primitive, defined next.

\begin{definition}[Causal set]
Given a model \(\langle \mathbf{U}, \mathbf{V}, \mathbf{F}, P_{\mathbf{U}}\rangle\), a causal set is a set of pairs $X = x$ where $X \in \mathbf{V}$ and $x \in \mathit{dom}(X)$.
\end{definition}

\medskip\noindent Informally, a causal set is a list of specific variable-value
assignments, a precise statement of which variables are set to which values. It is the
basic unit of intervention: from a causal set one can read off exactly what is being
manipulated or held fixed.

In \ourapproach, causal sets play four distinct roles, two on the cause side and two
on the context side:

\begin{description}
\item[Suspects $\mathbf{S}$.] The full set of variables suspected to bear on the
outcome. Subsets of $\mathbf{S}$ will be sampled and intervened on.
\item[Active suspects $\mathbf{C} \subseteq \mathbf{S}$.] The specific subset
intervened on in a given search step, either to alternative values (necessity
regime) or to factual values (sufficiency regime).
\item[Witnesses $\mathbf{W}$.] The full set of context variables that may be held fixed
at their factual values. Subsets of $\mathbf{W}$ will be sampled and held fixed.
\item[Active witnesses $\mathbf{T} \subseteq \mathbf{W}$.] The specific subset held
fixed in a given search step. To avoid contradicting the active intervention,
$\mathbf{T}$ is pruned so that $\mathbf{T} \cap \mathbf{C} = \emptyset$.
\end{description}

\medskip\noindent
\textbf{Why prune $\mathbf{T}$ and not $\mathbf{C}$?}
The conflict on $\mathbf{T} \cap \mathbf{C}$ has to be resolved in one
direction or the other, and pruning the witnesses is the right choice for
three reasons.
First, $\mathbf{C}$ is the object of the test: it encodes the causal claim
being evaluated, while $\mathbf{T}$ is the auxiliary device that holds
context fixed.
Pruning $\mathbf{C}$ would silently change the claim under test; pruning
$\mathbf{T}$ leaves the claim intact and merely weakens the context held fixed.
Second, the convention extends cleanly to the limit $\mathbf{C} = \mathbf{V}$,
where every variable is a suspect: the pruned $\mathbf{T}$ collapses to
$\emptyset$ and we recover the unconditional but-for test, the right
baseline behaviour.
Pruning $\mathbf{C}$ instead would leave this limit undefined whenever any
witness is sampled.
Third, the convention inherits directly from Halpern's actual-cause
definition (Definition~\ref{def:ac_intuitive}), which makes
$\mathbf{T} \cap \mathbf{C} = \emptyset$ part of the definition;
pruning post-sampling is the operational implementation of that constraint
and lets us specify $\Gamma$ on the unconstrained product
$2^{\mathbf{S}} \times 2^{\mathbf{W}}$ rather than baking disjointness into
the prior.

To evaluate the causal impact of a variable on an outcome, we consider two interventional regimes: the \emph{necessity regime}, 
in which we intervene on the variables $\mathbf{C} \subseteq \mathbf{S}$ to have an alternative value, and the sufficiency regime,
 in which we intervene so that $\mathbf{C}$ takes on factual values. 
Moreover, we hold some elements $\mathbf{T} \subseteq \mathbf{W}$ of the context fixed (i.e., intervened on to retain factual values, despite intervention on $\mathbf{C}$). This is exactly what makes the necessity component \emph{context-sensitive} in Halpern's sense \citep{actualCausalityHalpern}: the alternative cause is evaluated against a held-fixed witness configuration, not in isolation. We will reuse this terminology throughout: ``context-sensitive necessity'' refers to the necessity regime evaluated under an active witness set $\mathbf{T}$.
When looking for an explanation, we perform a search across possible cause sets and possible contexts.
Complete model samples in these regimes will be called \emph{necessity worlds}  and \emph{sufficiency worlds}, respectively.

The intuition is that 
the suspect set $\mathbf{S}$ names every variable that could plausibly have caused the outcome; the witness set $\mathbf{W}$ names context
variables to be selectively held fixed, isolating the causal channels under investigation. Intervening
on suspects tests what would have happened under an alternative cause in the necessity worlds and under the factual values
 in the sufficiency worlds; fixing witnesses
blocks confounding paths that would otherwise absorb the effect.
Responsibility attribution increases when the outcome is far from the factual value in the necessity world and decreases 
when it is far from the factual value in the sufficiency world.

For both Alice (female, bad credit, loan denied) and Bob (male, bad credit, loan denied)
we set $\mathbf{S} = \{\varname{gender}, \varname{credit}\}$ and
$\mathbf{W} = \{\varname{check\text{-}failed}\}$. The witness \varname{check\text{-}failed}
records whether the applicant was checked and had bad credit; holding it at its factual
value isolates the causal path from \varname{gender} or \varname{credit} to \varname{loan}
independently of the credit-check bottleneck. Without this witness, intervening on
\varname{credit} for Alice would leave 80\% of noise realizations unaffected (those
in which she is never checked), making credit's causal role invisible. The witness makes
the bottleneck explicit and manipulable.

In the general framework we search for causes by intervening on subsets of suspects
and holding fixed subsets of witnesses. To make this search tractable, we need a way to prioritize which subsets to consider.

\begin{definition}[Variable Selection Distribution]\label{def:variableselection}
Let $\mathbf{X}$ be a finite set of random variables. A \emph{variable selection distribution} $\Gamma$ is a probability distribution on $2^\mathbf{X}$. We write $\mathbf{Y}\sim\Gamma$ to mean $\mathbf{Y}$ is a random subset of $\mathbf{X}$ drawn from $\Gamma$, with $p^\Gamma(\mathbf{y})$ denoting the probability assigned to $\mathbf{y}\subseteq\mathbf{X}$.
\end{definition}

\medskip\noindent The variable selection distribution $\Gamma$ encodes
prior beliefs about which subsets of suspects are worth examining. A uniform distribution
treats all suspect subsets as equally plausible; a cardinality-constrained choice focuses
the search on interactions of a particular order, analogous to choosing which interaction
terms to include in a regression. More complex distributions could encode domain knowledge
about which variables are more likely to interact, or could be learned from data.

\noindent In particular, we will be using variable selection distributions for the powerset of causal suspects $\mathbf{S}$ and for the powerset of witnesses $\mathbf{W}$. One simple pattern of suspect and witness selection can be derived by setting size limits on the number of activated suspects or witnesses, and sampling uniformly otherwise.

\begin{example}[Cardinality-Constrained Uniform Selection]
Let $\mathbf{X}$ be a candidate set of random variables. Sample a target size
$k \sim \mathrm{Unif}\{J, \dots, K\}$ with $1 \leq J \leq K \leq |\mathbf{X}|$, then
sample $\mathbf{Y}$ uniformly among subsets of that size, i.e.\ from the
\emph{conditional} probability mass function
\[
p^\Gamma(\mathbf{Y} \mid k) \propto \mathbb{I}[|\mathbf{Y}| = k].
\]
Marginalising $k$ out gives the unconditional pmf actually used to define
$\Gamma(2^\mathbf{X})$,
\[
p^\Gamma(\mathbf{Y}) = \frac{1}{K-J+1}\binom{|\mathbf{X}|}{|\mathbf{Y}|}^{-1},
\qquad J \leq |\mathbf{Y}| \leq K,
\]
which is \emph{not} proportional to $\mathbb{I}[J\leq|\mathbf{Y}|\leq K]$: sizes are
weighted uniformly, but within a size, larger coalitions (of which there are more) are
individually rarer.

The parameters $J$ and $K$ determine the cardinality range of the search; specific
choices connect \ourapproach to existing methods and enable efficient Monte Carlo
approximation.\footnote{When $J=1$ and $K=|\mathbf{X}|$, this distribution coincides with the Shapley weighting kernel, which samples coalition size uniformly and then selects a coalition uniformly at that size \citep{shapley1953}. The parameters $J$ and $K$ allow the practitioner to restrict the search to cause sets of a specific cardinality range: setting $J=K=1$ confines attribution to individual variables, while $J=2, K=3$ focuses on small multi-variable interactions. Moreover, whereas exact Shapley computation requires summing over all $2^{|\mathbf{X}|}$ coalitions, \ourapproach treats $\Gamma$ as a generative model: approximation proceeds by sampling subsets from $\Gamma$ and outcomes from $P_{\mathbf{U}}$, making Monte Carlo estimation a native feature of the framework.
We note that the weighting alone does not distinguish PCI from SHAP at $J=1$,
$K=|\mathbf{X}|$: the subset weights are identical. The distinction lies in the
integrand, and it matters for what is being evaluated and what results follow. SHAP
averages marginal contributions using the observational conditional
$\mathbb{E}[f(\mathbf{X}_{\mathcal{S}} = \mathbf{x}_{\mathcal{S}}, \mathbf{X}_{\bar{\mathcal{S}}})]$, marginalising over
the empirical distribution of non-suspect features; this means a variable that is merely
correlated with a cause, but not itself causally active, can receive positive
attribution. \citet{janzing2020feature} frame exactly this as a causal problem,
arguing that dropping a feature should be \emph{interventional} rather than
observational, since the conditional version lets a feature inherit relevance from its
correlates. PCI instead evaluates counterfactual outcomes
$P(Y \in A \mid \mathbf{u},\, \mathrm{do}(\mathbf{C}=\mathbf{c}',
\mathbf{T}=\mathbf{t}^\star))$ computed via the structural model, so attribution flows
only along genuine causal paths. The witness mechanism adds a further distinction: by
holding intermediate variables fixed, PCI can resolve overdetermination and undercutting
scenarios where SHAP assigns equal or incorrect attribution. A detailed comparison of PCI,
SHAP, and Causal SHAP on concrete examples is given in
Section~\ref{sec:shap_examples}.}
\end{example}

In \ourapproach the variable selection distribution $\Gamma$ generalises
Definition~\ref{def:variableselection} (stated there for a single powerset
$2^{\mathbf{X}}$) to a joint object: it is a probability
mass function on $2^{\mathbf{S}} \times 2^{\mathbf{W}}$, governing the joint
choice of an active suspect set $\mathbf{C} \subseteq \mathbf{S}$ and an active
witness set $\mathbf{T} \subseteq \mathbf{W}$. The recommended construction
samples suspects and witnesses separately and then enforces the constraint
that no variable plays both roles in the same draw: pick a marginal $\Gamma_s$
on $2^{\mathbf{S}}$ and a marginal $\Gamma_w$ on $2^{\mathbf{W}}$, draw
$\mathbf{C}\sim\Gamma_s$ and $\mathbf{T}\sim\Gamma_w$ independently, and
\emph{reject} any draw with $\mathbf{T}\cap\mathbf{C}\neq\emptyset$,
resampling until the constraint holds. Equivalently,
\[
p^\Gamma(\mathbf{C}, \mathbf{T}) \;\propto\;
p_s^\Gamma(\mathbf{C})\,p_w^\Gamma(\mathbf{T})\,
\mathbb{I}[\mathbf{T}\cap\mathbf{C}=\emptyset].
\]
 Throughout, $\Gamma_s$ and $\Gamma_w$ are
abbreviations for the sampling marginals of this construction; the formal
primitive in Definition~\ref{def:jointnecsuf} is the joint $\Gamma$.

For the OBCB analysis, with $|\mathbf{S}|=2$, we use a $\Gamma$ built from
$\Gamma_s$ uniform over the three non-empty subsets of $\mathbf{S}$
($\{\varname{gender}\}$, $\{\varname{credit}\}$,
$\{\varname{gender}, \varname{credit}\}$, each with probability $1/3$) and
$\Gamma_w$ uniform over $\{\emptyset, \{\varname{check\text{-}failed}\}\}$,
each with probability $1/2$. Since no element appears in both $\mathbf{S}$
and $\mathbf{W}$, the rejection step is vacuous and $\Gamma$ is exactly the product of marginals $p_s^\Gamma\cdot p_w^\Gamma$. In larger suspect sets the choice of cardinality
range becomes substantive: sampling uniformly from the full powerset yields an
expected intervention size of $|\mathbf{S}|/2$, which spreads attribution
across large coalitions and may make individual feature contributions harder
to isolate; we return to this issue in \cref{sec:ac_benchmark}, where
the cardinality cap is treated as a tunable hyperparameter and a
dynamic-set-sizes ablation measures the empirical cost of bounding it.

The choice of $\Gamma$ is a design decision that shapes the search process and the
resulting attributions.
As a default, a uniform $\Gamma$ over the full powerset is the least-informative
starting point, assigning equal weight to all subset sizes, analogous to including all
interaction orders in a regression. There is no universally principled criterion for
restricting the cardinality range; we therefore recommend treating $J$ and $K$ as
explicit design choices. Two practical considerations bear on this choice, however: with $K=1$, attribution
scores are based on singleton interventions and straightforward to communicate to
stakeholders; higher $K$ averages over joint interventions on $X_k$ with other suspects
simultaneously, which is harder to explain and, in large or structurally complex models,
may amplify estimation noise rather than capturing genuine higher-order causal structure.
For applications where attribution has serious consequences, we recommend a
cardinality sensitivity analysis: if rankings are stable as $K$ increases from $1$ to
$|\mathbf{S}|$, the simpler choice is supported; if they diverge, higher-order
interactions are materially relevant and the full powerset should be used. The advantage of \ourapproach is that this commitment is made transparently rather than
absorbed into an implicit default.\footnote{SHAP faces an analogous decision without
flagging it as one: by averaging over all $2^{|N|}$ feature coalitions, plain SHAP and
Causal SHAP commit to the maximum interaction order $K=|N|$ as an unstated default,
with no diagnostic offered for whether that level of joint conditioning is right for a
given model. Restricting SHAP to $K<|N|$ would break the efficiency axiom, so the
choice is hardcoded into the framework rather than treated as a tunable. \ourapproach{}
makes the same decision an explicit choice.}

Causal attribution hinges on questions such as ``if $X$ had, contrary to fact, been fixed to $x'\ne x$ instead of $x$, what would the outcome have been?'' For non-binary causes, this is ambiguous. To resolve this in a way that does not require committing to a single, contrastive causal event \citep{kawakammi_2024_poc_for_cont_and_vec}, we rely on a \textit{distribution} over alternative values.

\begin{definition}[Alternative Value Distribution]\label{def:alt_value_dist}
Let $\mathbf{X}$ be a set of random variables with domain $\dom(\mathbf{X})$ and
let $\mathbf{X}=\mathbf{x}\in\dom(\mathbf{X})$ be a factual event.
An \emph{alternative value distribution} $\Delta(\mathbf{x})$ is a probability
measure on $\dom(\mathbf{X})$ such that $\mathbf{x} \notin \operatorname{supp}(\Delta(\mathbf{x}))$.
\end{definition}

\medskip\noindent Informally, $\Delta(\mathbf{x})$ specifies what counts as a
genuinely different value for a suspect variable. The support condition ensures no mass
falls on the factual value $\mathbf{x}$ itself, so every sampled alternative is a real
counterfactual departure rather than a noisy copy of what actually happened.

\medskip\noindent Three properties guide the choice of $\Delta$.
\emph{Plausibility under $M$:} the alternative should be a value the structural model
could plausibly produce given the context, so the necessity test interrogates a
counterfactual the model has anything to say about.
\emph{Outcome-independence:} the distribution over $\mathbf{x}'$ should not be informed
by the outcome whose causes we are explaining, on pain of circularity.
\emph{Context-sensitivity:} $\mathbf{x}'$ should reflect the rest of the factual
scenario, so the necessity test is individual-level rather than population-level.
The construction we use throughout is the structural-model posterior of $\mathbf{X}$
conditioned on upstream factual values, optionally excised near the factual value
(Example~\ref{ex:eps_excised} below); this is the natural choice satisfying all three.

\medskip\noindent Upstream-only conditioning operationalises the intuition shared with
the descriptive-normality literature \citep{Halpern2015-HALGCA,
icardNormalityActualCausal2017, hitchcock2009cause}: factual upstream values play the role of the
\emph{default} against which counterfactual alternatives are graded, the same role a
stipulated baseline plays in causal analyses of harm \citep{beckersHarm2022}, where a
default contract fixes what counts as a departure from the expected course of events.
\ourapproach does not aim to resolve the specific counterexamples motivating that
literature; the link is at the level of \emph{why} a non-uniform $\Delta$ matters at
all.
Since $\Delta$ is user-pluggable, other constructions are available when domain
knowledge motivates them: a uniform distribution over $\dom(\mathbf{X})$ trades
plausibility for breadth; the empirical marginal $p(\mathbf{x}')$ trades
context-sensitivity for simplicity; a fully conditional posterior on every observable
(including downstream evidence) trades outcome-independence for sharper conditioning.

\noindent One natural alternative value distribution takes the one already embodied in the model and excises it around the factual setting: ``alternative value'' then means ``a sufficiently different value, following the original distribution otherwise.'' In the discrete case, the analogue is simply the probability mass function conditioned on the factual value not holding.

\begin{example}[$\epsilon$-Excised Distribution]\label{ex:eps_excised}
Assume $\mathbf{X}$ is $\mathbb{R}^d$-valued with density $p$.
Given a factual event $\mathbf{X}=\mathbf{x}$, define the
$\epsilon$-excised alternative value distribution $\Delta_\epsilon(\mathbf{x})$
as the probability measure on $(\mathbb{R}^d,\mathcal{B}(\mathbb{R}^d))$ with density
\[
p_{\Delta_\epsilon}(\mathbf{x}'\mid \mathbf{x})\propto
\begin{cases}
0, & \mathbf{x}'\in B_\epsilon(\mathbf{x}),\\[4pt]
p(\mathbf{x}'), & \text{otherwise},
\end{cases}
\]
where $B_\epsilon(\mathbf{x})$ denotes the $\epsilon$-ball centered at $\mathbf{x}$.
The support is therefore
$\operatorname{supp}(\Delta_\epsilon(\mathbf{x})) = \operatorname{supp}(p) \setminus B_\epsilon(\mathbf{x})$:
the exclusion covers the entire $\epsilon$-neighbourhood of $\mathbf{x}$. This requires $p$ to place positive mass
outside $B_\epsilon(\mathbf{x})$; if $\operatorname{supp}(p)\subseteq B_\epsilon(\mathbf{x})$
(a sharply concentrated posterior, or an overly generous $\epsilon$), the
normalising constant is zero and $\Delta_\epsilon(\mathbf{x})$ is undefined. The
same degeneracy affects the discrete case: a variable whose posterior is a point
mass admits no alternative-value distribution with the support condition of
Definition~\ref{def:alt_value_dist} at all. A practical implementation should
detect this case explicitly and either signal an error, shrink $\epsilon$, or fall
back to a non-degenerate default such as the unconditional marginal.
\end{example}

\medskip\noindent With plausibility, outcome-independence, and context-sensitivity as
design criteria, and $\epsilon$-excision as a concrete construction satisfying them, we
return to the OBCB example to see what $\Delta$ looks like when the suspect variables are
binary.

In the OBCB model both \varname{gender} and \varname{credit} are binary, so $\Delta$ is
trivial. For Alice, the only alternative to $\varname{gender}=0$ (female) is
$\varname{gender}=1$ (male), and the only alternative to $\varname{credit}=0$ (bad) is
$\varname{credit}=1$ (good); for Bob, the only alternative to $\varname{gender}=1$ is
$\varname{gender}=0$. Accordingly, $\Delta(\varname{gender}=0)$ is a point mass on
$\varname{gender}=1$ and $\Delta(\varname{credit}=0)$ a point mass on
$\varname{credit}=1$: no distributional choice is required. The binary case is exactly
where $\Delta$ is unambiguous; it is for continuous suspects such as income or loan amount
that the $\epsilon$-excised distribution becomes necessary, forcing an explicit commitment
to what counts as a meaningfully different value.

\paragraph{Why $\epsilon$-excision?}
Any method that asks ``what would happen if $\mathbf{X}$ had been different'' must
operationalise what \emph{different} means. The observational marginal $p(\mathbf{x}')$ (as in kernel SHAP
\citep{lundberg2017unified}) places mass near the factual value $\mathbf{x}$,
conflating the necessity question (``what if $\mathbf{x}$ had genuinely differed?'')
with evaluating essentially the same setting. The fully conditional posterior
$p(x_k' \mid \mathbf{x}_{-k})$ conditions on all other observed features, potentially
including descendants of $X_k$, which imposes structural constraints that may rule
out otherwise plausible counterfactuals, and still concentrates mass near $x_k$ when
features are correlated. The $\epsilon$-excised distribution occupies the middle
ground: it draws from the marginal informed by upstream variables only, excising the
factual neighbourhood to guarantee genuine departure without downstream conditioning.

The choice parallels the Region of Practical Equivalence (ROPE) in Bayesian hypothesis
testing \citep{kruschke2018}: just as ROPE forces an explicit decision about what
difference is large enough to matter, $\epsilon$-excision forces an explicit decision
about what counterfactual is sufficiently far from the factual to count as genuinely
alternative. The practitioner makes this commitment regardless of method; the only
question is whether it is made transparently or absorbed into a default. When $X$ is
standardised to $\mathcal{N}(0,1)$ the choice acquires a direct probabilistic
interpretation: the excised mass $\Phi(x^\star+\epsilon)-\Phi(x^\star-\epsilon)$ is the
fraction of the distribution ruled out as insufficiently different from $x^\star$. At
$x^\star=0$, $\epsilon=0.2,\,0.5,\,0.8$ excise approximately $16\%, 38\%, 58\%$ of the
distribution, corresponding to small, medium, and large effect thresholds in the sense
of \citet{cohen1988statistical}. Alternatively, $\epsilon$ can reflect measurement
precision or a decision-relevant threshold; sensitivity analysis across $\epsilon$
values is advisable when the appropriate scale is unclear. For concrete examples of how
SHAP's implicit choice leads to problematic attributions, see
Section~\ref{sec:shap_examples}.

The PSCM and intervention apparatus give us a notion of
counterfactual outcome for any fixed noise realisation $\mathbf{u}$; $\Gamma$ tells
us which suspect/witness configurations to ask about; $\Delta$ tells us which
non-factual values to substitute when evaluating necessity. What remains is to plug
those distributions into the counterfactual mechanics and read out a number. The
next two definitions do exactly that: first the joint necessity-sufficiency
measure $P_k^{s,n}$, then its expectation against a user-chosen impact function
$ci$, with the OBCB scores and the desiderata verification immediately after.

The core idea behind \Ourapproach's main definitions is to
measure two complementary things for a variable of interest $X_k$: (a) how much $X_k$'s
factual value contributes to producing $y^\star$, the \emph{sufficiency} question,
and (b) how much replacing $X_k$'s value with alternatives would shift the outcome away
from $y^\star$, the \emph{necessity} question. We capture both by constructing two
counterfactual outcome distributions conditioned on the same noise realization $\mathbf{u}$,
then marginalising their product over $P_{\mathbf{U}}$. Because \ourapproach is not
constrained to particular variable types, we work in general measure-theoretic terms
supporting combinations of continuous and discrete variables.

\begin{definition}[Joint Necessity and Sufficiency Measure] \label{def:jointnecsuf}
Let $M=\langle \mathbf{U},\mathbf{V},\mathbf{F},P_{\mathbf{U}}\rangle$
be a probabilistic structural causal model
and $Y$ an outcome variable with domain $\dom(Y)$.
Let $p^\Gamma$ be the mass function of a variable selection distribution on
$2^{\mathbf{S}} \times 2^{\mathbf{W}}$, and let
$\Delta(\mathbf{s}^\star)$ be an alternative value distribution on $\dom(\mathbf{S})$.
Let $(\mathbf{S},\mathbf{W})=(\mathbf{s}^\star,\mathbf{w}^\star)$ be a factual event.

For a variable of
interest $X_k$, let
$2^{\mathbf{S}}_{k}:=\{\mathbf{C}\subseteq\mathbf{S}:X_k\in\mathbf{C}\}$,
and for each $\mathbf{C}\subseteq\mathbf{S}$ let
$\Delta_{\mathbf{C}}(\mathbf{s}^\star)$ denote the pushforward of
$\Delta(\mathbf{s}^\star)$ under the restriction map
$\mathbf{s}'\mapsto \mathbf{s}'\!\mid_{\mathbf{C}}$.\footnote{For an assignment $\mathbf{x}$ to a collection of variables $\mathbf{X}$ and any subset $\mathbf{Y}\subseteq\mathbf{X}$, the restriction $\mathbf{x}\!\mid_{\mathbf{Y}}$ denotes the sub-assignment of $\mathbf{x}$ to the variables in $\mathbf{Y}$. The restricted marginal $\Delta_{\mathbf{C}}(\mathbf{s}^\star)$ may place positive mass on $\mathbf{s}^\star\!\mid_{\mathbf{C}}$ even though $\Delta(\mathbf{s}^\star)$ excludes the joint factual $\mathbf{s}^\star$; the genuine-departure guarantee is on the joint alternative, not on each restricted coordinate.}
For any measurable $A\subseteq\dom(Y)$ and exogenous realization $\mathbf{u}$, define three
sub-probability measures on $\dom(Y)$ (the first two) and $\dom(Y)\times\dom(Y)$ (the third).

\textbf{(i) Sufficiency measure.} $P_k^s(\cdot\mid\mathbf{u})$ intervenes on each
active suspect set $\mathbf{C}$ to its factual values $\mathbf{s}^\star\!\mid_{\mathbf{C}}$,
holds the active witnesses $\mathbf{T}$ fixed, and takes the $\Gamma$-weighted sum of
$P(Y\in A\mid\mathbf{u},\mathrm{do}(\cdot))$ under factual cause restoration:
\begin{equation}
P_k^{s}(A\mid\mathbf{u})
=
\sum_{\substack{(\mathbf{C},\mathbf{T})\\X_k\in\mathbf{C}}}
P\!\left(Y\in A \mid \mathbf{u},
\mathrm{do}\!\left(\mathbf{C}=\mathbf{s}^\star\!\mid_{\mathbf{C}},
\mathbf{T}=\mathbf{w}^\star\!\mid_{\mathbf{T}}\right)\right)
\,p^\Gamma(\mathbf{C},\mathbf{T}).
\end{equation}

\textbf{(ii) Necessity measure.} $P_k^n(\cdot\mid\mathbf{u})$ instead integrates
over alternative values $\mathbf{c}'$ for $\mathbf{C}$ (with $\mathbf{T}$ still pinned at
factual values), and takes the $\Gamma$-weighted average over alternative cause values:
\begin{equation}
P_k^{n}(A\mid\mathbf{u})
=
\sum_{\substack{(\mathbf{C},\mathbf{T})\\X_k\in\mathbf{C}}}
\int
P\!\left(Y\in A \mid \mathbf{u},
\mathrm{do}\!\left(\mathbf{C}=\mathbf{c}',\mathbf{T}=\mathbf{w}^\star\!\mid_{\mathbf{T}}\right)\right)
\Delta_{\mathbf{C}}(\mathbf{s}^\star)(d\mathbf{c}')
\,p^\Gamma(\mathbf{C},\mathbf{T}).
\end{equation}

\textbf{(iii) Joint necessity-sufficiency measure.} On product rectangles
$A\times B\subseteq\dom(Y)\times\dom(Y)$, set
\begin{align}
P_k^{s,n}(A\times B)
\;:=\;
\int
\sum_{\substack{(\mathbf{C},\mathbf{T})\\X_k\in\mathbf{C}}}
&P\!\left(Y\in A \mid \mathbf{u},
\mathrm{do}\!\left(\mathbf{C}=\mathbf{s}^\star\!\mid_{\mathbf{C}},
\mathbf{T}=\mathbf{w}^\star\!\mid_{\mathbf{T}}\right)\right)
\nonumber\\
&\cdot\left[\int
P\!\left(Y\in B \mid \mathbf{u},
\mathrm{do}\!\left(\mathbf{C}=\mathbf{c}',\mathbf{T}=\mathbf{w}^\star\!\mid_{\mathbf{T}}\right)\right)
\Delta_{\mathbf{C}}(\mathbf{s}^\star)(d\mathbf{c}')\right]
p^\Gamma(\mathbf{C},\mathbf{T})
\;P_{\mathbf{U}}(d\mathbf{u}).
\end{align}
Both terms are evaluated at the \emph{same} drawn $(\mathbf{C},\mathbf{T})$: one
suspect/witness configuration is selected per noise realisation, and both the
sufficiency and necessity worlds are built from it. This rectangle-by-rectangle
specification extends uniquely (by Carath\'eodory's extension theorem) to a
sub-probability measure on the product $\sigma$-algebra of
$\dom(Y)\times\dom(Y)$, with total mass $\Pr_\Gamma[X_k\in\mathbf{C}] \leq 1$:
integrating out $B$ or $A$ collapses the bracketed term to $1$, leaving exactly
the single $\Gamma$-inclusion probability, not its square.
\end{definition}

\medskip\noindent
Informally: $P_k^s(\cdot\mid\mathbf{u})$ is the distribution of $Y$ in the
\emph{sufficiency world}: restore the factual value of a candidate cause set
$\mathbf{C}$ containing $X_k$, hold the witness set $\mathbf{T}$ at its factual values,
and observe where the outcome lands. $P_k^n(\cdot\mid\mathbf{u})$ is the distribution of
$Y$ in the \emph{necessity world}: replace $\mathbf{C}$'s values with alternatives
drawn from $\Delta$, hold $\mathbf{T}$ fixed, and ask how the outcome shifts.
Because only the \emph{joint} alternative is guaranteed to depart from
$\mathbf{s}^\star$, a drawn $\mathbf{c}'$ can coincide with $\mathbf{C}$'s own
factual restriction $\mathbf{s}^\star\!\mid_{\mathbf{C}}$ even though the full
$\mathbf{s}'$ it came from does not; such a draw intervenes on $\mathbf{C}$
without changing it, contributing a sufficiency-like term to the necessity
integral rather than a genuine test of necessity. This can only dilute the
necessity signal, never invalidate it: a $\Delta$ constructed to depart from
$\mathbf{s}^\star$ coordinate-wise, not just jointly, avoids the degeneracy
entirely.
The sum over suspect-containing pairs $\{(\mathbf{C},\mathbf{T}):X_k\in\mathbf{C}\}$
weights each configuration by its $\Gamma$-mass: more focused interventions
receive whatever weight the practitioner assigns them; configurations with
$X_k\notin\mathbf{C}$ would contribute $0$ to the attribution and are absent
from the sum, so $P_k^s$ and $P_k^n$ are sub-probability measures with total
mass $\Pr_\Gamma[X_k\in\mathbf{C}]\le 1$ rather than renormalized expectations. A
consequence worth flagging: because this total mass scales with how often $\Gamma$
includes $X_k$ at all, comparing raw (unnormalised) scores \emph{across suspects}
is $\Gamma$-invariant only when $\Gamma$ assigns every compared suspect the same
inclusion probability $\Pr_\Gamma[X_k\in\mathbf{C}]$, as it does under the uniform and cardinality-constrained $\Gamma$'s used throughout this paper's
examples, since those treat all suspects symmetrically. A non-uniform, asymmetric
$\Gamma$ (e.g.\ one preferring some suspects over others, as when $\Gamma$ is
learned from data) can make a suspect's score larger purely because $\Gamma$ visits
it more often, not because it is a stronger cause; cross-suspect rankings under such
a $\Gamma$ should instead compare each suspect's score after dividing by its own
$\Pr_\Gamma[X_k\in\mathbf{C}]$. In the OBCB example,
for $X_k=\varname{gender}$ under a noise realization in which Alice is credit-checked:
the sufficiency world restores $\varname{gender}=\text{female}$ while holding
$\varname{check\text{-}failed}$ fixed; the necessity world substitutes
$\varname{gender}=\text{male}$ under the same noise.

For a fixed $(\mathbf{C},\mathbf{T})$ and $\mathbf{u}$, the bracketed product inside
$P_k^{s,n}$ is legitimate because the sufficiency and necessity interventions operate
in separate counterfactual regimes with no causal connection:
$\mathrm{do}(\mathbf{C}=\mathbf{c}^\star)$ and $\mathrm{do}(\mathbf{C}=\mathbf{c}')$ act
on the same noise and the same held witnesses but in distinct hypothetical worlds, so
$Y^s$ and $Y^n$ are conditionally independent given $\mathbf{u}$ and the shared
$(\mathbf{C},\mathbf{T})$:
$P(Y^s\in A,\,Y^n\in B\mid\mathbf{u},\mathbf{C},\mathbf{T})=
P(Y^s\in A\mid\mathbf{u},\mathbf{C},\mathbf{T})\,
P(Y^n\in B\mid\mathbf{u},\mathbf{C},\mathbf{T})$.
This conditional independence is exploited configuration-by-configuration, before
summing over $(\mathbf{C},\mathbf{T})$, rather than after each side has already been
marginalised over its own copy of $\Gamma$.

\medskip\noindent\textit{A more general family.} Nothing forces the configuration
$(\mathbf{C},\mathbf{T})$ underlying $Y^s$ to coincide with the one underlying $Y^n$;
Definition~\ref{def:jointnecsuf} is one instance of a broader construction: replace
the single $\Gamma$ with a joint variable-selection distribution $\Gamma^{s,n}$ on
pairs
$((\mathbf{C}_s,\mathbf{T}_s),(\mathbf{C}_n,\mathbf{T}_n))\in(2^{\mathbf{S}}\times2^{\mathbf{W}})^2$
(each marginal recovering $\Gamma$), and weight the product of world-outcomes by
$p^{\Gamma^{s,n}}$ before integrating over $\mathbf{u}$. Definition~\ref{def:jointnecsuf}
is the \emph{diagonal} case, $\Gamma^{s,n}$ supported on
$\{((\mathbf{C},\mathbf{T}),(\mathbf{C},\mathbf{T}))\}$: both worlds share one drawn
configuration. This is the reading illustrated in Figure~\ref{fig:pci-mechanism} and
realised exactly by the per-configuration kernel $\Phi$ of
Section~\ref{sec:relation_with_ac}. The \emph{independent-product} case,
$\Gamma^{s,n}=\Gamma\otimes\Gamma$, instead marginalises each world's configuration
separately before taking the product; the two cases coincide whenever $ci$ is
additively separable in $y^s$ and $y^n$, as in the Absolute Difference score
(Example~\ref{ex:absolute_score}): linearity of expectation then lets each world's
marginal be computed on its own, independent of any coupling. They can differ when
$ci$ is not separable, as in the PNS binary score (Example~\ref{ex:pns_binary}), where
the score depends on the covariance, under $\Gamma^{s,n}$, between
sufficiency-favouring and necessity-favouring configurations. \Ourapproach{} adopts
the diagonal case as its default: it needs only one $\Gamma$-draw per Monte Carlo
sample, and every estimator in this paper (Algorithm~\ref{alg:pci_generic} and its specialisations) implements it. A user who wants
the two worlds explained independently can instantiate the independent-product case
instead.

A second, orthogonal relaxation is available within a single world's own $\Gamma$:
nothing in Definition~\ref{def:jointnecsuf} requires $p^\Gamma(\mathbf{C},\mathbf{T})$
to factor as $p^{\Gamma_s}(\mathbf{C})p^{\Gamma_w}(\mathbf{T})$, as we do from
Section~\ref{sec:relation_with_ac} onward for tractability of those results. A user who
wants witnesses selected structurally instead, e.g.\ pruned to the graph-theoretic
mediators lying on active paths from the sampled $\mathbf{C}$ to $Y$, can substitute a
$\mathbf{T}\mid\mathbf{C}$-dependent $\Gamma$ directly into
Definition~\ref{def:jointnecsuf} and Definition~\ref{def:causalimpact} unchanged, since
both are stated for a general joint mass function $p^\Gamma$ on
$2^{\mathbf{S}}\times2^{\mathbf{W}}$; only the later results that specifically assume
$\Gamma_w$-uniformity (Theorem~\ref{th:local_max_to_ac} and
Corollary~\ref{cor:marginalize}) would need re-derivation for such a $\Gamma$.

With the necessity-sufficiency measure fully defined, we are ready to characterize the causal impact of an event $X_k$. We do so by defining an arbitrary \emph{causal impact function} $ci$, the \emph{impact kernel}, whose expectation over the necessity-sufficiency measure determines the causal impact. Below we first give the formal definition, then several intuitive candidates for the impact kernel.

\begin{definition}[Causal Impact Function and Its Expectation]  \label{def:causalimpact}
Let $y^\star\in\dom(Y)$ denote the factual outcome value, and let $(Y^s,Y^n)$ be
distributed according to the sub-probability law $P_k^{s,n}$ (not a proper
probability measure in general; see below).
For any measurable function $ci: \dom(Y)^3\to\mathbb{R}$, the
\emph{(expected) causal impact} of $X_k$, written $\mathrm{PCI}_{X_k}$, is defined as
\begin{equation}\label{eq:causal-impact}
\mathbb{E}\!\left[\,ci\!\left(Y^s,Y^n,y^\star\right)\right]
\;:=\;
\iint
ci(y^s,y^n,y^\star)\,
P_k^{s,n}(d y^s, d y^n).
\end{equation}
The integral is taken against the sub-probability $P_k^{s,n}$ directly, not against
its renormalisation; the symbol $\mathbb{E}[\cdot]$ is shorthand for this
sub-probability integral throughout. We call $\mathrm{PCI}_{X_k}$ the \ourapproach score
of $X_k$; it is the concrete instance of the placeholder $R$ used in the desiderata
(p.~\pageref{desid}). Suspect configurations with
$X_k\notin\mathbf{C}$ contribute zero mass to $P_k^s$ and $P_k^n$ (and hence to
$P_k^{s,n}$), correctly capturing that those configurations make no claim on $X_k$'s
attribution.
\end{definition}

\medskip\noindent\textit{PCI in pseudocode.} Definitions~\ref{def:jointnecsuf}
and~\ref{def:causalimpact} are stated measure-theoretically because \ourapproach must
cover continuous, discrete, and mixed-type variables uniformly, and because Theorem~\ref{th:ac-exp} needs that generality. Operationally, under the
diagonal default of the remark above, the score they define reduces to a
short Monte Carlo loop:

\begin{algorithm}[htbp]
\caption{Monte Carlo estimator for $\mathbb{E}[ci]$ (diagonal default)}
\label{alg:pci_generic}
\begin{algorithmic}[1]
\Require PSCM $M$; variable of interest $X_k$; factual event $(\mathbf{s}^\star,\mathbf{w}^\star,y^\star)$;
$\Gamma$ on suspect/witness pairs; alternative-value distribution $\Delta$; impact function $ci$; sample budget $N$
\Ensure Monte Carlo estimate of $\mathbb{E}[ci(Y^s,Y^n,y^\star)]$
\For{$j=1,\dots,N$}
  \State draw $\mathbf{u}_j\sim P_{\mathbf{U}}$
  \State draw $(\mathbf{C}_j,\mathbf{T}_j)\sim\Gamma$, conditioned on
         $X_k$ being in the drawn suspect set \Comment{one shared configuration
         for both worlds, as in the remark after Def.~\ref{def:jointnecsuf}}
  \State draw $\mathbf{c}'_j\sim\Delta_{\mathbf{C}_j}(\mathbf{s}^\star)$
  \State $y^s_j \gets M.\mathrm{run}\bigl(\mathbf{u}_j,\ \mathrm{do}(\mathbf{C}_j=\mathbf{s}^\star\!\mid_{\mathbf{C}_j},\ \mathbf{T}_j=\mathbf{w}^\star\!\mid_{\mathbf{T}_j})\bigr)$ \Comment{sufficiency world}
  \State $y^n_j \gets M.\mathrm{run}\bigl(\mathbf{u}_j,\ \mathrm{do}(\mathbf{C}_j=\mathbf{c}'_j,\ \mathbf{T}_j=\mathbf{w}^\star\!\mid_{\mathbf{T}_j})\bigr)$ \Comment{necessity world, same $\mathbf{u}_j$ and $(\mathbf{C}_j,\mathbf{T}_j)$}
  \State $ci_j \gets ci(y^s_j, y^n_j, y^\star)$
\EndFor
\State \Return $\Pr_\Gamma[X_k\in\mathbf{C}]\cdot\frac{1}{N}\sum_j ci_j$
\Comment{rescales the conditional mean back to $\mathbb{E}[ci]$'s sub-probability normalisation}
\end{algorithmic}
\end{algorithm}

\noindent In words: draw noise $\mathbf{u}$; draw one suspect/witness configuration
containing $X_k$, shared by both worlds; run the model twice under
the \emph{same} $\mathbf{u}$ and the \emph{same} configuration, once restoring the
suspects to factual (sufficiency) and once setting them to a fresh alternative
(necessity); score the pair with $ci$; average.
Appendix~\ref{sec:ac_benchmark} (Algorithm~\ref{alg:pci_thin_search}) specialises this
to the necessity-only, per-suspect estimator used in the actual-causality benchmark.

\begin{figure*}[t]
\centering
\includegraphics[width=\textwidth]{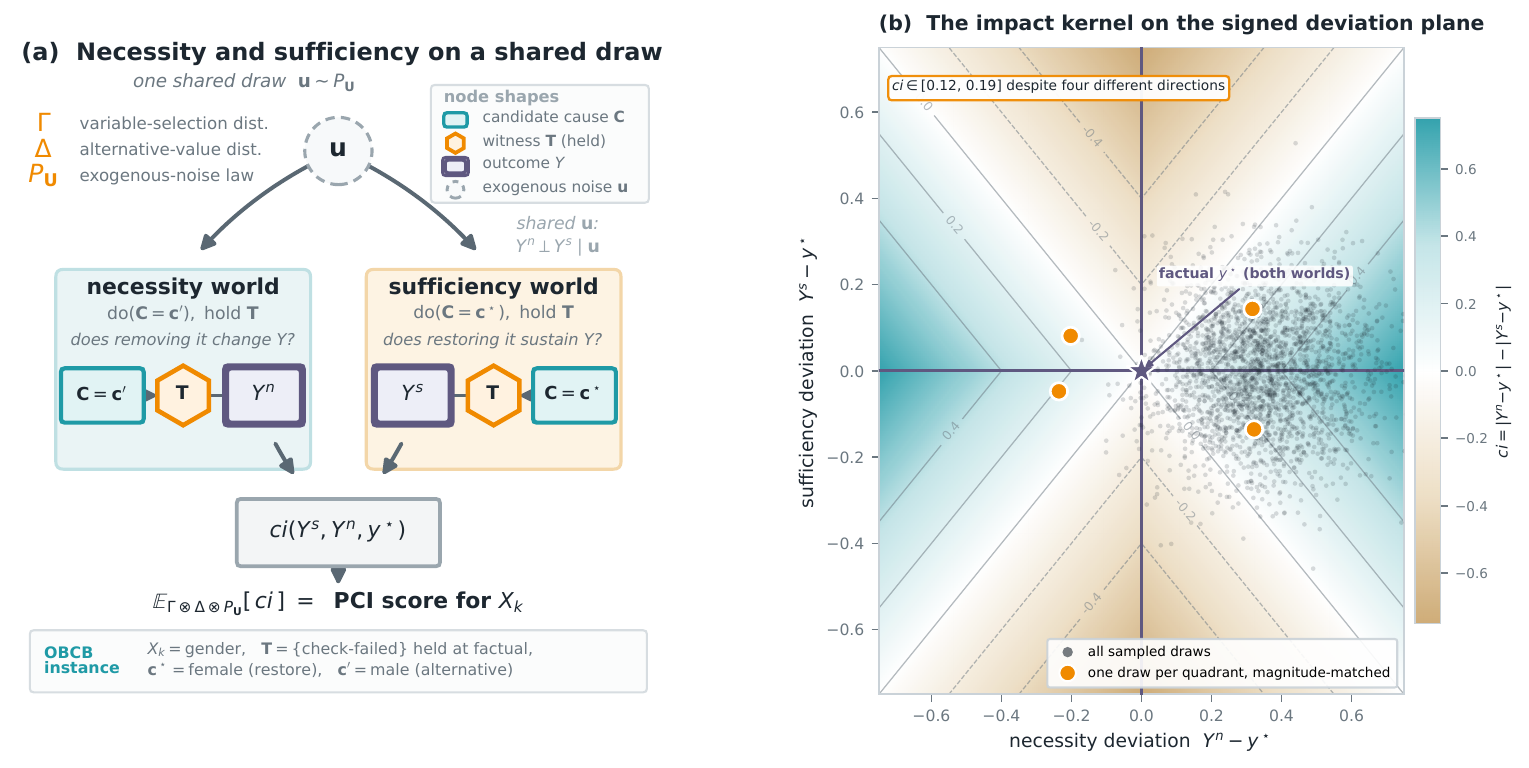}
\caption{The \ourapproach{} mechanism. \textbf{(a)} For a variable of interest
$X_k$, a single exogenous draw $\mathbf{u}\sim P_{\mathbf{U}}$ together with a
candidate cause set $\mathbf{C}\ni X_k$ and witness set $\mathbf{T}$ from
$\Gamma$ and an alternative value $\mathbf{c}'$ from $\Delta$ defines two
counterfactual worlds on the same noise and the same
$(\mathbf{C},\mathbf{T})$: a \emph{necessity world}
($\mathrm{do}(\mathbf{C}=\mathbf{c}')$, witnesses held) giving $Y^n$, and a
\emph{sufficiency world} ($\mathrm{do}(\mathbf{C}=\mathbf{c}^\star)$, witnesses
held) giving $Y^s$. Conditional on $\mathbf{u}$ and $(\mathbf{C},\mathbf{T})$ the
two worlds are independent (Definition~\ref{def:jointnecsuf}); the impact kernel
$ci$ scores the pair, and its expectation over $\Gamma\otimes\Delta\otimes
P_{\mathbf{U}}$ (Definition~\ref{def:causalimpact}) is \ourapproach{}'s realised
score. The strip grounds the schematic in the OBCB running example. \textbf{(b)} The impact kernel on the \emph{signed} plane
$(Y^n-y^\star,\,Y^s-y^\star)$, shown for the Absolute Difference score
$ci=|y^n-y^\star|-|y^s-y^\star|$, one admissible choice of kernel (the
PNS binary score and necessity-only variants are others), before either
deviation is folded by $|\cdot|$. Because $ci$ depends only on the two
absolute deviations, the field is symmetric under an independent sign flip of
either axis: four qualitatively different outcomes, necessity and sufficiency
each landing above or below the factual value, collapse onto the same score
whenever the magnitudes match. The faint cloud is a sample of $(Y^n,Y^s)$
draws; the four gold points are actual draws, one per
quadrant, with closely matched $|Y^n-y^\star|,|Y^s-y^\star|$, so their nearly
equal $ci$ values arise from the sampled data itself.}
\label{fig:pci-mechanism}
\end{figure*}

\medskip\noindent Informally, $ci$ is a user-defined scoring rule that translates
the joint sufficiency--necessity measure into a single scalar
(Figure~\ref{fig:pci-mechanism} summarises the construction: the two
counterfactual worlds in panel~(a), the kernel that scores them in
panel~(b)). The choice is not arbitrary:
the Absolute Difference example below is an instance of L$_1$ distance between outcomes, the
same metric underlying Average Treatment Effect,
 Conditional Average Treatment Effect, and SHAP.
The binary indicator score recovers this as a special
 case when outcomes are binary:
indicators are L$_1$ on $\{0,1\}$. Practitioners
 with reasons to prefer L$_2$ or another
distance measure may substitute it freely; the framework places 
no restriction on $ci$
beyond measurability.

For the OBCB analysis we instantiate $ci$ as the PNS binary score defined in the
following example, with $y^\star = \varname{loan} = 0$. This is the choice that produces
the values in Table~\ref{tab:pci_decomp}.

\begin{example}[PNS for Binary Outcomes]\label{ex:pns_binary}
Assume the outcome $Y$ is binary with $\dom(Y)=\{0,1\}$, and let $y^\star=1$ denote
the factual outcome. Let $Y^s$ and $Y^n$ denote the sufficiency and necessity
outcome variables. Define the causal impact score
\[
ci(y^s,y^n,y^\star)
=
\mathbb{I}\{y^n\neq y^\star\}\,\mathbb{I}\{y^s=y^\star\}.
\]
Then the expected causal impact reduces to
\[
\mathbb{E}\!\left[ci(Y^s,Y^n,1)\right]
=
\mathbb{P}(Y^n=0,\,Y^s=1),
\]
which conceptually coincides with Pearl's probability of necessity and sufficiency (PNS). The alignment is \emph{exact} when $\mathbf{S} = \{X_k\}$, $\mathbf{W}=\emptyset$, and $\dom(X_k)=\{0,1\}$ where, factually, $x_k^\star=1$.
\end{example}

We now illustrate the full computation for Alice's gender attribution using this $ci$.
With $\mathbf{S}=\{\varname{gender},\varname{credit}\}$,
$\mathbf{W}=\{\varname{check\text{-}failed}\}$, $\Gamma_s$ and $\Gamma_w$ uniform,
$\Delta$ a point mass on the alternative binary value, and $y^\star=\varname{loan}=0$,
the exogenous noise $U_\text{check}$ falls into three regions. (Definition~\ref{def:interlaw}
makes the interventional law a point mass given the \emph{full} noise vector
$\mathbf{u}$; below we coarsen $\mathbf{u}$ to these three $U_\text{check}$ regions and
marginalise the remaining noise, in particular $U_{\text{loan-if-checked}}$, inside each region; by iterated expectation this equals the full integral over
$P_{\mathbf{U}}$, and lets a value like ``$P(\varname{loan}=1)=0.05$'' appear
below despite Definition~\ref{def:interlaw}'s pointwise law being deterministic.)
The three regions are:
\begin{itemize}
\item $u_1$ (prob.\ 0.2): Alice is checked regardless of gender;
  $\varname{check\text{-}failed}=1$ factually.
\item $u_2$ (prob.\ 0.7): Alice is not checked as female but would be as male;
  $\varname{check\text{-}failed}=0$ factually.
\item $u_3$ (prob.\ 0.1): Alice is not checked regardless of gender.
\end{itemize}
In all regions, restoring Alice's factual values always yields a denial, so
$P_k^s(\{0\}\mid\mathbf{u})=\tfrac{2}{3}$ everywhere (four suspect-containing
$(\mathbf{C},\mathbf{T})$ combinations each with $p^\Gamma(\mathbf{C},\mathbf{T})=\tfrac{1}{3}\cdot\tfrac{1}{2}=\tfrac{1}{6}$,
since $\mathbf{S}\cap\mathbf{W}=\emptyset$ here, rejection is vacuous and
$\Gamma$ is just the product of marginals, each contributing probability~1).
The necessity measure varies.

In $u_1$, the witness holds $\varname{check\text{-}failed}=1$.
The four $(\mathbf{C},\mathbf{T})$ contributions, each weighted $\tfrac{1}{6}$, are:
\begin{itemize}
\item[$\circ$] $\mathbf{C}=\{\varname{gender}\}$, $\mathbf{T}=\emptyset$:
  male, $\varname{check\text{-}failed}=1$ propagates structurally
  $\Rightarrow P(\varname{loan}=1)=0.05$.
\item[$\circ$] $\mathbf{C}=\{\varname{gender}\}$,
  $\mathbf{T}=\{\varname{check\text{-}failed}\}$:
  male, witness holds failure
  $\Rightarrow P(\varname{loan}=1)=0.05$.
\item[$\circ$] $\mathbf{C}=\{\varname{gender},\varname{credit}\}$, $\mathbf{T}=\emptyset$:
  male and good credit, check passes
  $\Rightarrow P(\varname{loan}=1)=1.0$.
\item[$\circ$] $\mathbf{C}=\{\varname{gender},\varname{credit}\}$,
  $\mathbf{T}=\{\varname{check\text{-}failed}\}$:
  male and good credit, but witness holds failure
  $\Rightarrow P(\varname{loan}=1)=0.05$.
\end{itemize}
\[ P_k^n(\{1\}\mid u_1)
   = \tfrac{1}{6}(0.05+0.05+1.0+0.05) \approx 0.192. \]

In $u_2$, the witness holds $\varname{check\text{-}failed}=0$.
The four contributions:
\begin{itemize}
\item[$\circ$] $\mathbf{C}=\{\varname{gender}\}$, $\mathbf{T}=\emptyset$:
  male would be checked; $\varname{check\text{-}failed}=1$ propagates structurally
  $\Rightarrow P(\varname{loan}=1)=0.05$.
\item[$\circ$] $\mathbf{C}=\{\varname{gender}\}$,
  $\mathbf{T}=\{\varname{check\text{-}failed}\}$:
  male, witness holds $\varname{check\text{-}failed}=0$, check passes
  $\Rightarrow P(\varname{loan}=1)=1.0$.
\item[$\circ$] $\mathbf{C}=\{\varname{gender},\varname{credit}\}$, $\mathbf{T}=\emptyset$:
  male and good credit, check passes
  $\Rightarrow P(\varname{loan}=1)=1.0$.
\item[$\circ$] $\mathbf{C}=\{\varname{gender},\varname{credit}\}$,
  $\mathbf{T}=\{\varname{check\text{-}failed}\}$:
  male and good credit, witness holds $0$
  $\Rightarrow P(\varname{loan}=1)=1.0$.
\end{itemize}
\[ P_k^n(\{1\}\mid u_2)
   = \tfrac{1}{6}(0.05+1.0+1.0+1.0) \approx 0.508. \]

In $u_3$, male Alice is also not checked ($\varname{check}=0$ regardless of gender),
so all necessity worlds yield denial: $P_k^n(\{1\}\mid u_3)=0$.

Combining, with $ci(y^s,y^n,y^\star)=\mathbb{I}\{y^n\neq y^\star\}\,\mathbb{I}\{y^s=y^\star\}$
and $y^\star=\varname{loan}=0$:
\begin{align*}
\mathrm{PCI}_{\varname{gender}}(\text{Alice})
  &= \mathbb{E}\!\left[ci(Y^s,Y^n,y^\star)\right]
   = \mathbb{P}(Y^n=1,\,Y^s=0) \\
  &= \sum_u p(u)\,P_k^s(\{0\}\mid u)\,P_k^n(\{1\}\mid u) \\
  &= 0.2\times\tfrac{2}{3}\times 0.192
   + 0.7\times\tfrac{2}{3}\times 0.508
   + 0.1\times 0 \\
  &\approx 0.026 + 0.237 = 0.263.
\end{align*}

Table~\ref{tab:pci_decomp} gives the full results for Alice and Bob under these
parameter choices: Pearl's PN/PS/PNS, PCI without witnesses
($\mathbf{W}=\emptyset$), and PCI with the $\varname{check\text{-}failed}$ witness
active, with the latter further decomposed into its necessity and sufficiency
marginals.

\begin{table}[ht]
\centering
\small
\setlength{\tabcolsep}{4pt}
\begin{tabular}{l c c c c c c c}
\toprule
& \multicolumn{3}{c}{Pearl} & PCI & \multicolumn{3}{c}{PCI ($\mathbf{W}=\{\varname{check\text{-}failed}\}$)} \\
\cmidrule(lr){2-4} \cmidrule(lr){6-8}
Person, Variable & PN & PS & PNS & ($\mathbf{W}=\emptyset$) & Necessity & Sufficiency & Total \\
\midrule
Alice (\varname{gender}) & $0.045$ & $1.000$ & $0.045$ & $0.210$
                          & $0.394$ & $0.667$ & $\mathbf{0.263}$ \\
Alice (\varname{credit}) & $0.180$ & $1.000$ & $0.180$ & $0.240$
                          & $0.298$ & $0.667$ & $\mathbf{0.199}$ \\
Bob   (\varname{gender}) & $0.000$ & ---     & $0.000$ & $0.038$
                          & $0.030$ & $0.637$ & $\mathbf{0.019}$ \\
Bob   (\varname{credit}) & $0.900$ & $0.955$ & $0.860$ & $0.228$
                          & $0.188$ & $0.637$ & $\mathbf{0.119}$ \\
\bottomrule
\end{tabular}
\caption{Per-individual responsibility scores for \varname{gender} and \varname{credit}
on $\varname{loan}=0$ in the stochastic OBCB model. Pearl's PN/PS/PNS, PCI without
witnesses ($\mathbf{W}=\emptyset$), and PCI with the \varname{check\text{-}failed}
witness shown side by side; the rightmost three columns decompose PCI($\mathbf{W}=\{\varname{check\text{-}failed}\}$)
into its necessity and sufficiency marginals (their joint expectation is the bold
total). Here the total happens to equal the product of the two marginals because the
sufficiency measure $P_k^s(\cdot\mid\mathbf{u})$ is constant in $\mathbf{u}$ for both
individuals; in general the joint $P_k^{s,n}$ does not factor into the product of its
marginals. Bob's PS for gender is undefined because the conditioning event
$(\text{female}, \text{bad}, \varname{loan}=1)$ has probability zero in the OBCB model;
the PCI marginals are well-defined throughout, since they integrate over $P_{\mathbf{U}}$
rather than conditioning on the factual outcome.}
\label{tab:pci_decomp}
\end{table}

For Alice, PNS and PCI without witnesses both fail \textbf{D-A-rank} (p.~\pageref{desid}):
they rank \varname{credit} above \varname{gender} (0.180 vs.\ 0.045; 0.240 vs.\ 0.210).
PCI with witnesses reverses this (0.263 vs.\ 0.199), satisfying \textbf{D-A-rank} as well
as \textbf{D-A1} and \textbf{D-A2} (both values positive). Holding
\varname{check\text{-}failed} at its factual value blocks the credit-check bottleneck in
noise realizations where Alice was checked, isolating gender's structural role. The
\varname{check\text{-}failed} witness is the decisive ingredient.

For Bob, \textbf{D-B2} and \textbf{D-B-rank} hold under all three methods:
\varname{credit} is ranked above \varname{gender} throughout. \textbf{D-B1} requires
strictly positive attribution for \varname{gender}: PNS returns exactly 0.000 and
therefore fails, because $P(\varname{loan}=1\mid\varname{gender}=\text{female},
\varname{credit}=\text{bad})=0$ makes the counterfactual necessity calculation zero,
missing gender's indirect role in opening the credit evaluation. PCI with witnesses
returns 0.019, correctly capturing this indirect contribution. The cross-individual
desideratum \textbf{D-comp} holds under both PCI (0.263 $>$ 0.019) and PNS
(0.045 $>$ 0.000). A full verification of all seven desiderata is given at the end
of this section.

The PNS and PCI (no witnesses) columns differ even though neither uses witnesses,
because PCI still $\Gamma$-averages over candidate suspect sets.  With
$|\mathbf{S}|=2$ PCI averages over $\{\varname{gender}\}$, $\{\varname{credit}\}$,
and $\{\varname{gender},\varname{credit}\}$ (each weighted $\tfrac{1}{3}$),
whereas PNS evaluates a single counterfactual on one variable at a time.
Property~(i) below establishes that the two agree exactly when $|\mathbf{S}|=1$
(single suspect, no witnesses, binary outcome); the OBCB setup here instead uses
$|\mathbf{S}|=2$ deliberately, to demonstrate joint-subset attribution.

\subsection{Continuous Outcomes and the General Causal Impact Function}
\label{sub:continuous_ci}

\noindent The causal impact function $ci$ is a user-defined component of \ourapproach, and different choices encode different causal questions. In the benchmarking experiments of \cref{sec:ac_benchmark}, we instantiate $ci$ using only the necessity component, since Halpern’s actual causality has no sufficiency counterpart: necessity-only attribution allows a direct parallel comparison. The full framework incorporating both $Y^s$ and $Y^n$ is instantiated in the PNS binary
example above and in the Absolute Difference example below; the OBCB analysis of this
section (Tables~\ref{tab:pci_decomp}--\ref{tab:desiderata}) provides a concrete
validation, since the PNS binary $ci$ jointly conditions on $Y^s=y^\star$ and
$Y^n\neq y^\star$ and it is this combination that produces the correct
\textbf{D-A-rank} ordering. Further experiments using the full $ci$ on continuous outcomes are presented in
the signal-with-mediation example (Section~\ref{sec:signal}), the synthetic
evaluation (\cref{sub:synthetic}), and the SIR benchmark
(\cref{sec:sir_benchmark}).

When the outcome variable $Y$ is continuous rather than binary, indicators are replaced by a distance-based score that measures how far the necessity and sufficiency worlds land from the factual value $y^\star$.
The natural choice, paralleling L$_1$-based scores such as ATE and SHAP, is the absolute difference.

\begin{example}[Absolute Difference Impact Score] \label{ex:absolute_score}
We can develop an analogous measure for continuous outcomes where we increase attribution with distance from the factual value $y^\star$ in the necessity world and decrease impact with distance to the factual value in the sufficiency world.
\begin{align*}
ci(y^s, y^n, y^\star) & = \vert y^n - y^\star \vert   - \vert y^s - y^\star \vert.
\end{align*}
\end{example}

\noindent Two one-sided instances of this score recur in the experiments: a
necessity-only $ci_N$ and a sufficiency-only $ci_S$ (\cref{sub:synthetic}), whose sum
returns the absolute-difference $ci$ above. Each is a choice of $ci$ in the sense of
Definition~\ref{def:causalimpact}.

With the full machinery in place (structural model, variable selection, alternative
value distribution, joint necessity-sufficiency measure, and causal impact function),
\ourapproach delivers the following properties.
\begin{enumerate}[label=(\roman*)]
\item \textbf{Generalisation of PN/PS/PNS.} When $\mathbf{S}=\{X_k\}$,
  $\mathbf{W}=\emptyset$, and $\dom(X_k)=\{0,1\}$, the PNS binary $ci$ recovers Pearl's
  probability of necessity and sufficiency exactly; PN and PS are recovered by the
  corresponding one-sided $ci$ choices.
\item \textbf{Connection to actual causality.} For suitable $\Gamma$
  and $\Delta$, \ourapproach recovers Halpern's actual causality judgements
  (Section~\ref{sec:relation_with_ac}).
\item \textbf{Necessity--sufficiency decomposition.} $P_k^s$ and $P_k^n$ are defined and
  interpretable independently; practitioners may choose a $ci$ that uses either alone
  or both, depending on whether the causal question concerns necessity, sufficiency,
  or both.
\item \textbf{Witness-mediated symmetry breaking.} Holding intermediate variables fixed
  via $\mathbf{W}$ disambiguates overdetermination and undercutting scenarios where
  necessity-only methods assign equal or zero attribution, as demonstrated above for
  Alice and Bob, and further developed in Section~\ref{sec:shap_examples} (SHAP and
  Causal SHAP comparison), \cref{sub:synthetic} (synthetic overdetermination
  and undercutting archetypes), and \cref{sec:ac_benchmark} (witness ablation
  on the scaled throwing problem).
\item \textbf{Composability.} $\Gamma$ and $\Delta$ are probability distributions and
  $ci$ is a measurable function, so all three compose directly with the
  probabilistic-programming machinery discussed in Section~\ref{sec:intro};
  \cref{sec:evaluation}'s benchmarks estimate them by Monte Carlo throughout, without
  modification to the core formalism.
\end{enumerate}

\subsection{Back to the Running Example: Verifying the Desiderata}
\label{sub:verify_desiderata}

Table~\ref{tab:desiderata} verifies the seven desiderata introduced in
Section~\ref{sec:motivations} (p.~\pageref{desid}) against the values in Table~\ref{tab:pci_decomp}.
PCI with witnesses satisfies all seven; PCI without witnesses fails \textbf{D-A-rank};
PNS fails \textbf{D-A-rank} and \textbf{D-B1}.
(The table uses the shorthand $|R(V\mid\text{person})|$ for
$|R(V\rightsquigarrow\varname{loan}\mid\text{person})|$.)

\begin{table}[ht]
\centering
\resizebox{\linewidth}{!}{%
\begin{tabular}{lccc}
\toprule
Desideratum & PNS & PCI (no wit.) & PCI (with wit.) \\
\midrule
\textbf{D-A1}: $|R(\varname{gender}\mid\text{Alice})| > 0$
  & $\checkmark$\ 0.045 & $\checkmark$\ 0.210 & $\checkmark$\ 0.263 \\
\textbf{D-A2}: $|R(\varname{credit}\mid\text{Alice})| > 0$
  & $\checkmark$\ 0.180 & $\checkmark$\ 0.240 & $\checkmark$\ 0.199 \\
\textbf{D-A-rank}: $|R(\varname{gender}\mid\text{Alice})| > |R(\varname{credit}\mid\text{Alice})|$
  & $\times$\ $0.045 {<} 0.180$
  & $\times$\ $0.210 {<} 0.240$
  & $\checkmark$\ $0.263 {>} 0.199$ \\
\textbf{D-B1}: $|R(\varname{gender}\mid\text{Bob})| > 0$
  & $\times$\ 0.000 & $\checkmark$\ 0.038 & $\checkmark$\ 0.019 \\
\textbf{D-B2}: $|R(\varname{credit}\mid\text{Bob})| > 0$
  & $\checkmark$\ 0.860 & $\checkmark$\ 0.228 & $\checkmark$\ 0.119 \\
\textbf{D-B-rank}: $|R(\varname{credit}\mid\text{Bob})| > |R(\varname{gender}\mid\text{Bob})|$
  & $\checkmark$\ $0.860 {>} 0.000$
  & $\checkmark$\ $0.228 {>} 0.038$
  & $\checkmark$\ $0.119 {>} 0.019$ \\
\textbf{D-comp}: $|R(\varname{gender}\mid\text{Alice})| > |R(\varname{gender}\mid\text{Bob})|$
  & $\checkmark$\ $0.045 {>} 0.000$
  & $\checkmark$\ $0.210 {>} 0.038$
  & $\checkmark$\ $0.263 {>} 0.019$ \\
\bottomrule
\end{tabular}}
\caption{Verification of the seven desiderata (p.~\pageref{desid}) against the values in
Table~\ref{tab:pci_decomp}. PCI with witnesses satisfies all seven; PCI without witnesses
fails \textbf{D-A-rank}; PNS fails \textbf{D-A-rank} and \textbf{D-B1}.}
\label{tab:desiderata}
\end{table}

\medskip\noindent\textit{PNS fails \textbf{D-A-rank} and \textbf{D-B1}. The D-A-rank
failure reflects that without a \varname{check\text{-}failed} witness,
necessity-based attribution cannot distinguish Alice's gender (which blocked the
evaluation entirely) from her credit (which would only have mattered had she been
evaluated). The D-B1 failure reflects that PNS assigns Bob's gender exactly zero:
because $P(\varname{loan}=1\mid\text{female},\text{bad})=0$,
the counterfactual necessity calculation returns zero, missing gender's indirect role in
opening the credit evaluation. The \varname{check\text{-}failed} witness resolves both.}

To close this section, we compare PCI's decomposition with the one already
present in Pearl's framework. Pearl's framework defines three related but distinct
binary scores: the probability of necessity
$\mathrm{PN} = P(Y_{x'} = y' \mid X = x^\star, Y = y^\star)$, the probability of sufficiency
$\mathrm{PS} = P(Y_{x^\star} = y^\star \mid X = x', Y = y')$, and their joint
$\mathrm{PNS} = P(Y_{x^\star} = y^\star, Y_{x'} = y')$. The first two are conditional
probabilities; $\mathrm{PNS}$ is unconditional, and, \emph{when the cause variable
$X$ is exogenous} (as for the root variables gender and credit in OBCB, so that
$P(Y_{x^\star}=y^\star)=P(X=x^\star,Y=y^\star)$), the three satisfy
$\mathrm{PNS} = \mathrm{PN}\cdot P(Y_{x^\star}=y^\star) = \mathrm{PS}\cdot P(Y_{x'}=y')$ rather
than $\mathrm{PNS} = \mathrm{PN}\cdot\mathrm{PS}$ (which would additionally require
$Y_{x^\star}\perp Y_{x'}$, an assumption typically violated since the two potential
outcomes share exogenous noise). Without exogeneity, $P(Y_{x^\star}=y^\star)$ and
$P(X=x^\star,Y=y^\star)$ can differ and the displayed identity need not hold; the fully
general relation between $\mathrm{PNS}$, $\mathrm{PN}$, and $\mathrm{PS}$ is Pearl's
weighted-sum decomposition (\citealp[Lemma 9.2.6]{causalityPearl}),
$\mathrm{PNS} = P(x^\star,y^\star)\,\mathrm{PN} + P(x',y\neq y^\star)\,\mathrm{PS}$.
PCI admits a related split: with
$ci = \mathbb{I}\{y^n \neq y^\star\}\,\mathbb{I}\{y^s = y^\star\}$,
\[
\mathbb{E}[ci(Y^s, Y^n, y^\star)]
= \int P^s_k(\{y^\star\}\mid\mathbf{u})\,P^n_k(\{1{-}y^\star\}\mid\mathbf{u})\,dP_{\mathbf{U}},
\]
giving a \emph{sufficiency marginal} and a \emph{necessity marginal} that record how
often, on average, restoring the factual values of a suspect set containing $X_k$
reproduces the factual outcome, and how often substituting alternatives overturns it.
The two marginals are independent given $\mathbf{u}$.

The components are not identical to Pearl's: $\mathrm{PN}$ and $\mathrm{PS}$ condition
on the factual outcome and consider a single counterfactual flip; PCI's marginals
integrate over $P_{\mathbf{U}}$, average over suspect sets weighted by $\Gamma$, and
hold the witness configurations fixed when active. Table~\ref{tab:pci_decomp} (presented
earlier) places both decompositions side by side for \varname{gender} and
\varname{credit}.

For Alice's gender, $\mathrm{PN} = 0.045$ is the bare probability that a male with bad
credit would have been approved; PCI's necessity marginal is $0.394$, nearly an order of
magnitude higher, because the \varname{check\text{-}failed} witness keeps the
credit-check bottleneck active when the male counterfactual would otherwise have been
checked and might still have failed, structural information that the unconditional
$\mathrm{PN}$ cannot register. For Bob's gender, $\mathrm{PN}$ is exactly $0$ (a female
with bad credit cannot be approved in this model) and $\mathrm{PS}$ is undefined; PCI's
marginals remain well-defined at $0.030$ and $0.637$, recovering the indirect role of
gender in opening the evaluation.

The next section turns from comparing PCI with Pearl's family of counterfactual scores
to comparing it with the SHAP family, methods grounded in cooperative game theory
rather than in structural counterfactual reasoning.
Section~\ref{sec:shap_examples} re-runs the OBCB analysis under plain and Causal
SHAP, then introduces a continuous mediation example where the two SHAP variants
diverge from each other and from PCI, with explicit desiderata recorded throughout.

\section{Empirical and Comparative Evaluation: Overview}\label{sec:evaluation}

With \ourapproach now defined, we preview how it performs against the dominant
feature-attribution baselines, actual causality and Pearl's probability of
causation, a causal-effect attribution method, and problem sizes ranging from
textbook examples to a deployed system. Each subsection states one comparative
claim and points to the full treatment: the formal comparisons in
Sections~\ref{sec:shap_examples}--\ref{sub:pearl_actual_cause_prob}, and the complete
protocols in the appendices.

\subsection{\ourapproach matches actual-causality verdicts on canonical archetypes}
\label{sec:eval_verdicts}
On the discrete structural models where actual-causality verdicts are available, \ourapproach should
reproduce its verdicts; on the canonical overdetermination and undercutting
patterns it should separate a genuine cause from a preempted one. We test this on a
synthetic structural model that carries three archetypes at once, linear necessary-and-sufficient, overdetermined, and preempted (gated), plus an irrelevance control, all with an analytic ground truth
(Appendix~\ref{sub:synthetic}). Across two contrasting factual cases, all ten
archetype desiderata hold: \ourapproach scores the active cause above the
preempted one, keeps the irrelevant control at the noise floor, and, unlike the
binary actual-cause verdict, does so with graded magnitudes that extend to
continuous variables. The pass threshold is a Monte Carlo noise floor: clearing it means the gap is
distinguishable from sampling variance at the chosen budget; the gap need not
be large in absolute terms (Appendix~\ref{sub:synthetic}). The same agreement holds on Pearl's desert-traveller, where
\ourapproach reproduces his within-scenario ranking
(Section~\ref{sub:pearl_actual_cause_prob}, Appendix~\ref{app:desert}).

\begin{figure}[htbp]
\centering
\includegraphics[width=0.92\linewidth]{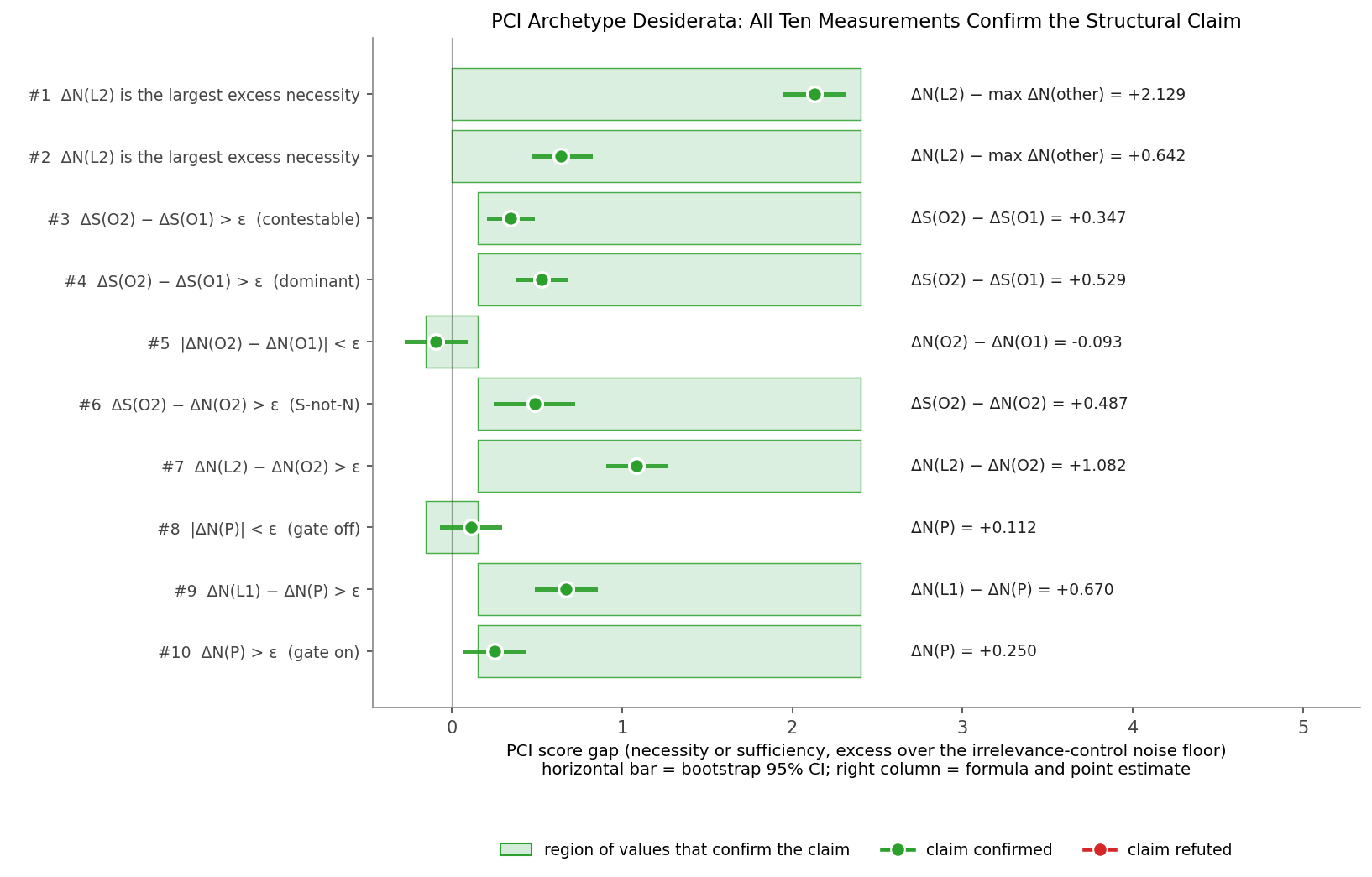}
\caption{\textbf{\ourapproach recovers the actual-causality verdicts on a
synthetic model with three archetypes and an irrelevance control.} Each row is one archetype desideratum (linear
necessity/sufficiency, overdetermination, preemption, irrelevance control); the
dot is the bootstrap point estimate of the relevant $\Delta$ score, the bar its
$95\%$ CI, and the green band the pass region given the data-derived threshold
$\epsilon \approx 0.151$. All ten dots fall inside their bands. Full model,
estimator, and per-row analysis in Appendix~\ref{sub:synthetic}.}
\label{fig:eval_archetypes}
\end{figure}

\subsection{SHAP and Causal SHAP}
\label{sec:eval_shap}
\ourapproach separates direct and indirect causal contributions; SHAP-based
attribution collapses them into one number. Both plain and Causal SHAP decompose a prediction's departure
from a population mean. Because that mean is not the realized outcome, this choice
produces two failures at once. First, they assign attribution backward, against the
arrows of the causal graph; Causal SHAP's interventional correction fixes only part
of this. Second, they tie together the direct and indirect causes of an outcome
that a context-sensitive method should keep separate. \ourapproach{}'s
realized-outcome reference and witness mechanism fix both problems, separating the
direct cause from the indirect one at every instance, as
Table~\ref{tab:eval_shap_desiderata} summarizes on the signal-with-mediation
example ($X \to M \to Y$).

\begin{table}[htbp]
\centering
\small
\begin{tabular}{lccc}
\toprule
Desideratum & Plain SHAP & Causal SHAP & \ourapproach \\
\midrule
D-XY: $X$ genuinely causes $Y$ & \checkmark & \checkmark & \checkmark \\
D-MY: $M$ genuinely causes $Y$ & \checkmark & \checkmark & \checkmark \\
D-XM: $X$ genuinely causes $M$ & \checkmark & \checkmark & \checkmark \\
D-YX: no backward attribution from $Y$ to $X$ & \texttimes & \checkmark & \checkmark \\
D-YM: no backward attribution from $Y$ to $M$ & \texttimes & \texttimes & \checkmark \\
D-MX: no backward attribution from $M$ to $X$ & \texttimes & \texttimes & \checkmark \\
D-MXY: the direct cause $M$ outranks the indirect cause $X$ & \texttimes & \texttimes & \checkmark \\
\bottomrule
\end{tabular}
\caption{Desiderata satisfied (\checkmark) or violated (\texttimes) by plain SHAP,
Causal SHAP, and \ourapproach on the signal-with-mediation example. Causal SHAP's
interventional correction repairs only D-YX. Full numeric values and witness-set
variants in Table~\ref{tab:signal_desid}.}
\label{tab:eval_shap_desiderata}
\end{table}

Section~\ref{sec:shap_examples}
states these failures as formal desiderata and works through both in full on two worked
examples (OBCB and the signal-with-mediation chain), with the formal SHAP and Causal
SHAP definitions in Appendix~\ref{app:signal}.

\subsection{The Differential Causal Effect}
\label{sec:eval_dce}
The next comparison targets a causal-effect attribution method: the
Differential Causal Effect (DCE) attributes an outcome change to features via local
derivatives of the structural map. On a model with a non-monotone internal response,
DCE and \ourapproach can disagree on sign and magnitude where the local gradient is
unrepresentative of the counterfactual contrast \ourapproach integrates over:
tracking the realized-outcome contrast keeps
\ourapproach's attribution aligned with each feature's structural role, without
depending on a single local derivative.

\begin{figure}[htbp]
\centering
\includegraphics[width=\linewidth]{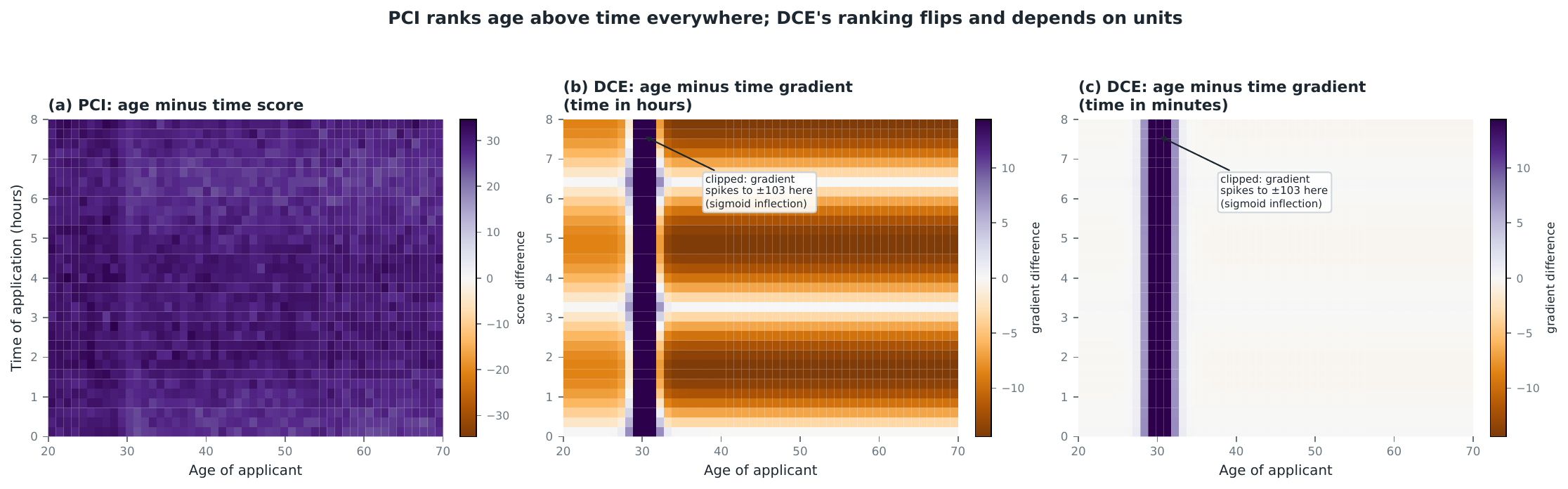}
\caption{\textbf{\ourapproach stays consistent where DCE does not.} \textbf{(a)}
\ourapproach ranks age above time of application everywhere on the credit-limit
grid; the gap is somewhat denser around the age of rapid change (the sigmoid's
steep region) but stays positive throughout. \textbf{(b), (c)} The same contrast for DCE, on the hours and minutes
time scales: DCE favours time of application over most of the grid (the
opposite of \ourapproach's verdict), except in a narrow band where its
gradient spikes at the sigmoid's inflection point, and that pattern itself
depends on which time unit is used. Full discussion in Appendix~\ref{sec:DCE}.}
\label{fig:eval_dce}
\end{figure}

Appendix~\ref{sec:DCE} shows the full contrast over a credit-limit grid.

\subsection{Scaling past exact actual causality}
\label{sec:eval_scaling}
Actual-cause computation requires enumerating counterfactual subsets, which
is exponential; \ourapproach replaces enumeration with a Monte Carlo thin search.
On a controlled family of problems (Appendix~\ref{sec:ac_benchmark}), exact subset
enumeration times out at about $17$ variables, whereas the necessity-only
\ourapproach estimator continues to roughly $73$--$145$ variables under an
undemanding sample budget and a per-size compute cap of $60$ minutes. It sustains a
correct-attribution rate above $0.85$ up to about $60$ variables while visiting
several orders of magnitude fewer cause--witness configurations.

\begin{figure}[htbp]
\centering
\includegraphics[width=0.49\linewidth]{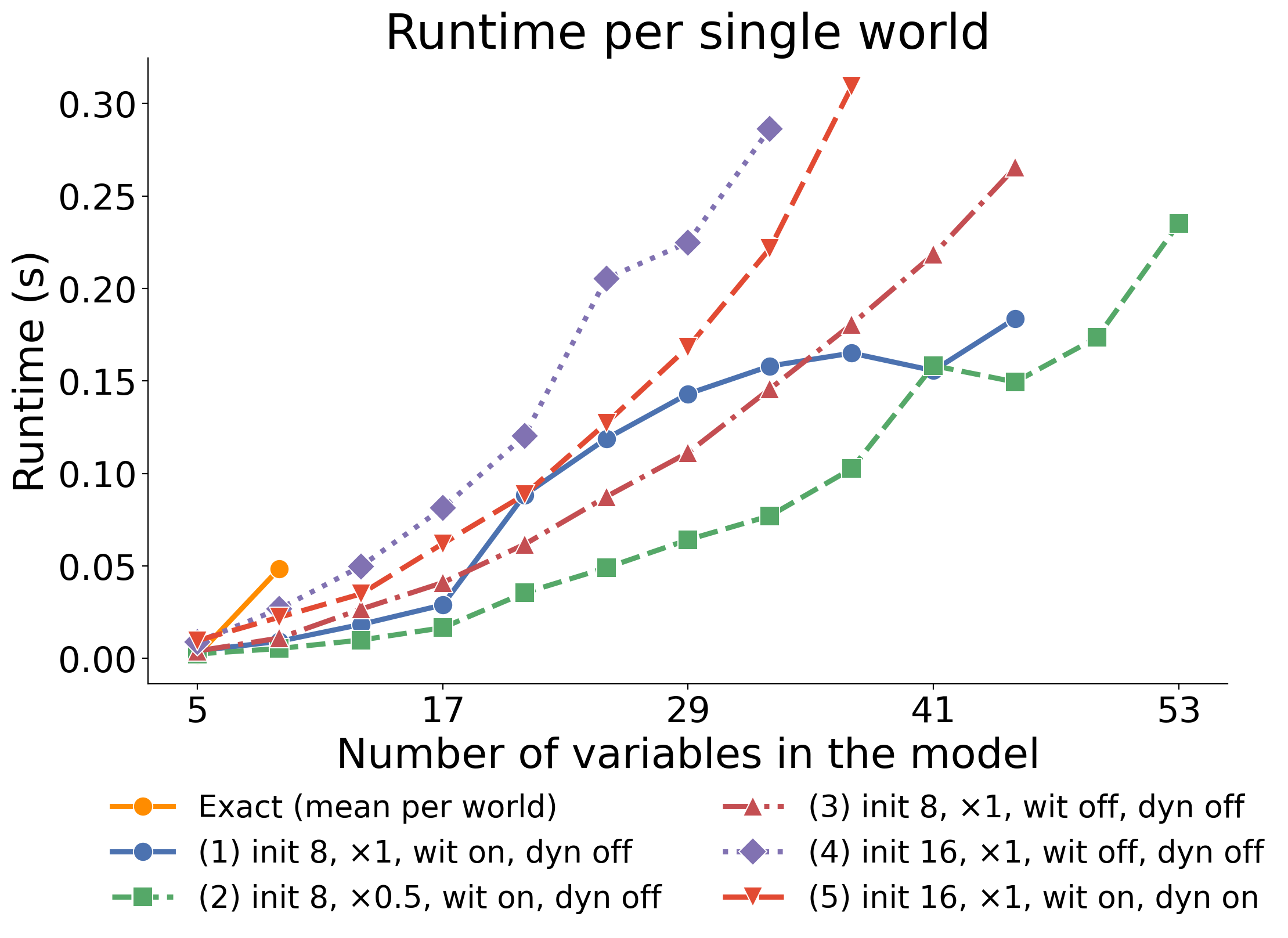}
\hfill
\includegraphics[width=0.49\linewidth]{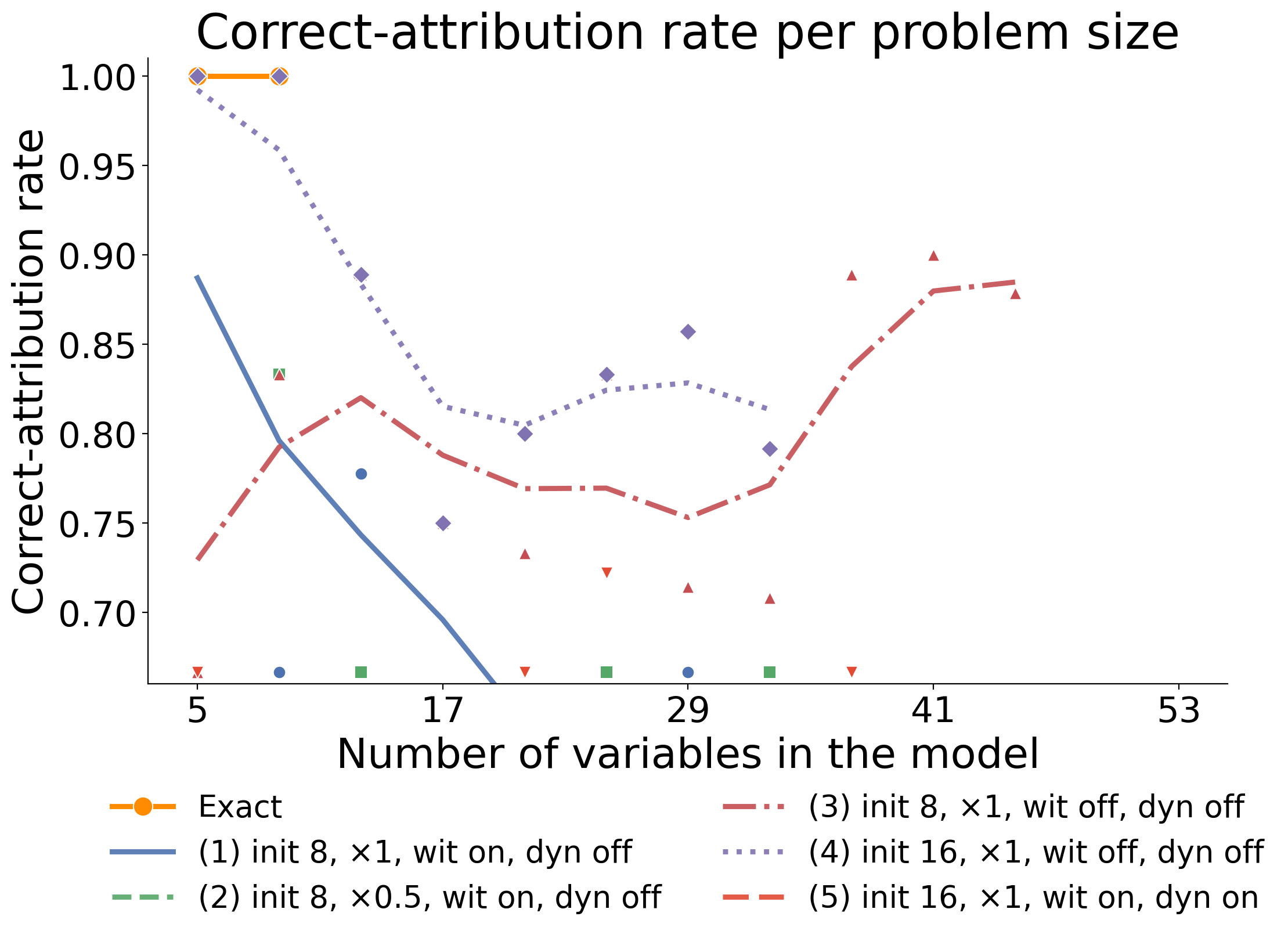}
\caption{\textbf{\ourapproach's estimator continues to work well past the point
where exact actual-cause enumeration gives out.} Left, wall-clock vs.\ problem
size: exact subset enumeration times out near $17$ variables, while the
\ourapproach estimator continues to roughly $73$--$145$. Right,
correct-attribution rate: witness-enabled configurations stay above
$0.85$ well past the exact cut-off, degrading gradually with problem size. Full
benchmark protocol in Appendix~\ref{sec:ac_benchmark}.}
\label{fig:eval_scaling}
\end{figure}

\subsection{Continuous, dynamical outcomes}
\label{sec:eval_sir}
Actual causality has a second limit besides scale: it offers no verdict when the
outcome is continuous and the causes interact through a dynamical system. On a
Bayesian SIR (susceptible--infected--recovered) model with two non-pharmaceutical
policies and a threshold (peak-overshoot) outcome (Appendix~\ref{sec:sir_benchmark}),
but-for analysis cannot separate the two policies; \ourapproach assigns
lockdown roughly twice the responsibility of masking, with the gap driven by the
necessity term under witness pinning. The verdict holds across $20$ factual worlds
drawn from the prior, with lockdown ahead in $18$ of them and a mean gap of
$+0.701$ (standard error $0.111$).

\begin{figure}[htbp]
\centering
\includegraphics[width=0.80\linewidth]{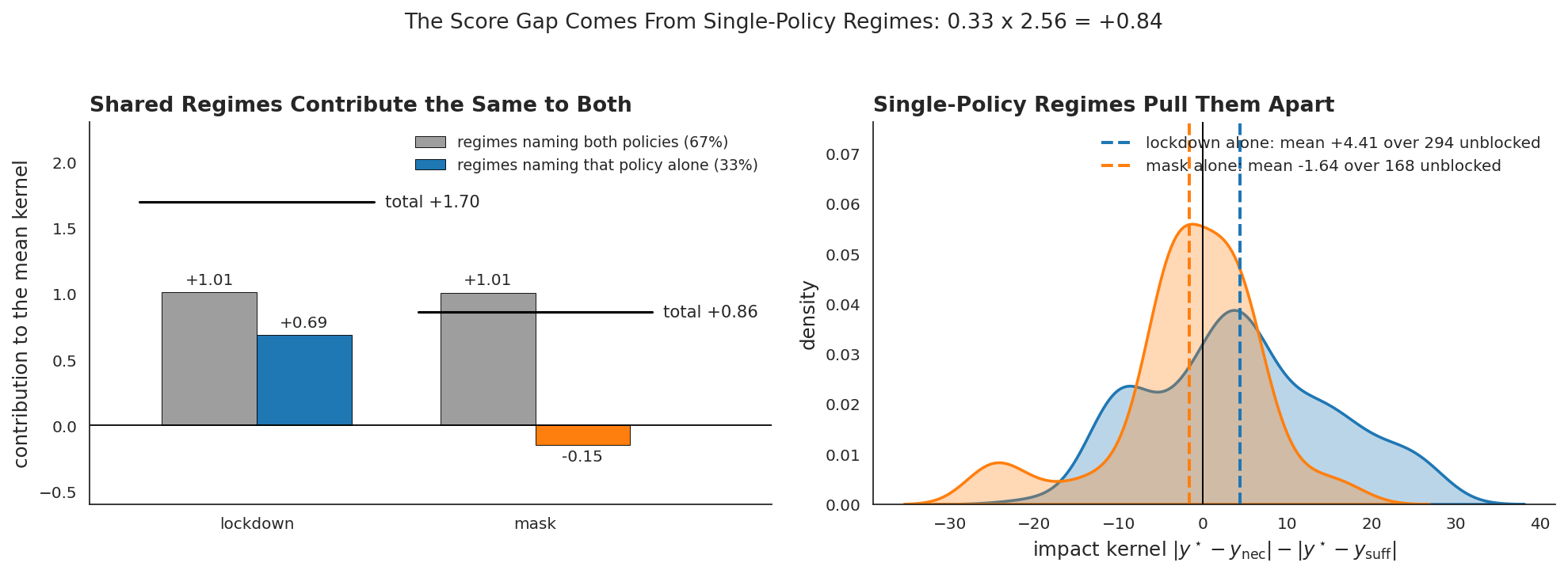}
\caption{\textbf{\ourapproach separates interacting policies on a continuous
dynamical outcome.} Joint necessity-sufficiency over the two policies for the
peak-overshoot event; lockdown carries roughly twice the responsibility of masking
where but-for analysis ties them. Full model and query in
Appendix~\ref{sec:sir_benchmark}.}
\label{fig:eval_sir}
\end{figure}

\subsection{A real-world deployed model}
\label{sec:eval_avm}
Finally, \ourapproach runs on a deployed automated valuation model (AVM) with
machine-learned components trained on millions of points, a regime where actual
causality is intractable and SHAP is the practical default
(Appendix~\ref{sec:avm}). \ourapproach returns a structured, context-sensitive
attribution across the model's internal variables; SHAP, by contrast, concentrates
almost all responsibility on a few downstream variables and places little weight
on the rest. This deployed model has no independently known ground truth.
\ourapproach and SHAP disagree here in the same qualitative direction already
validated by the synthetic benchmark's ground truth, but on its own the AVM
case demonstrates feasibility; it does not independently confirm correctness
(\cref{sec:discussion}).

\begin{figure}[htbp]
\centering
\includegraphics[width=0.49\linewidth]{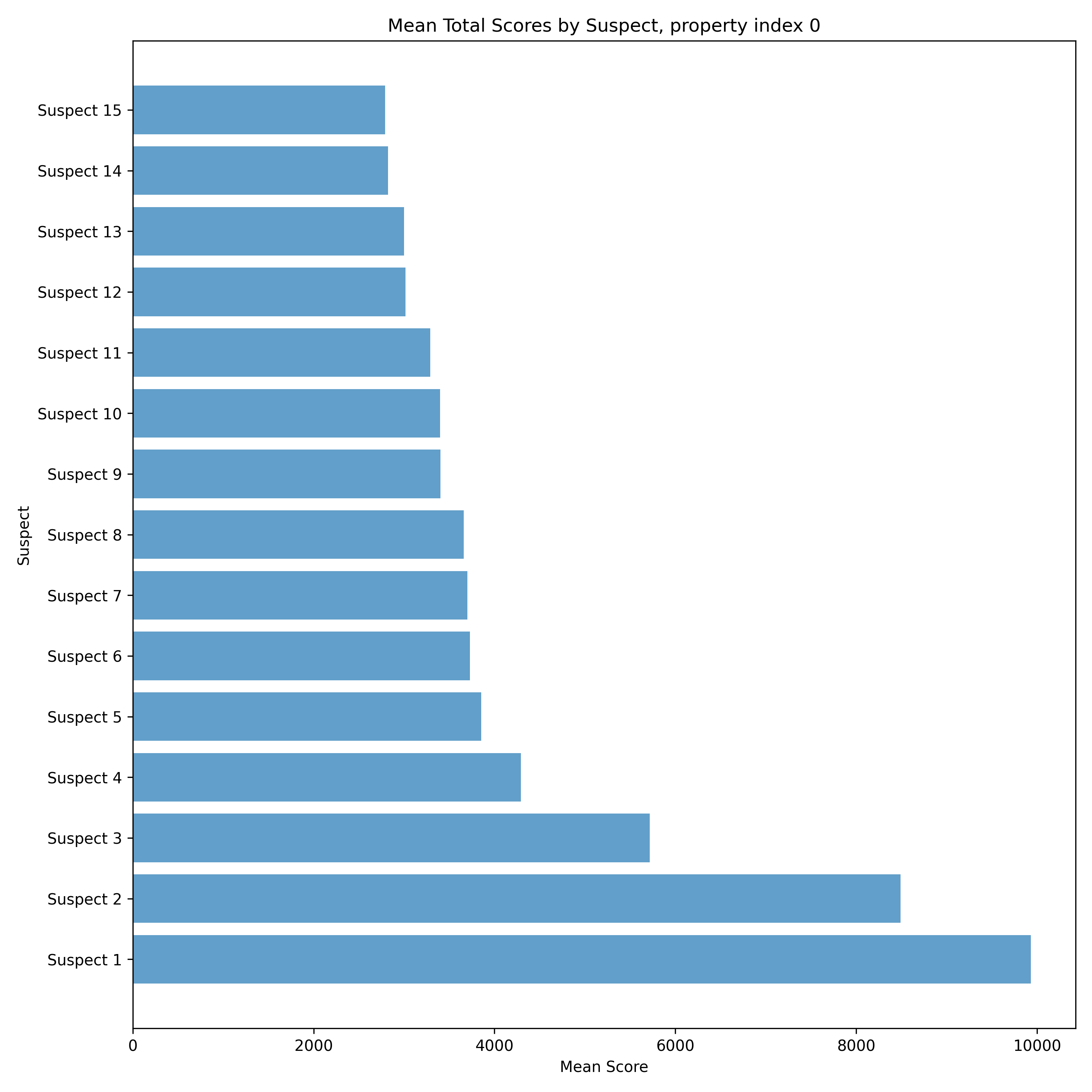}
\hfill
\includegraphics[width=0.49\linewidth]{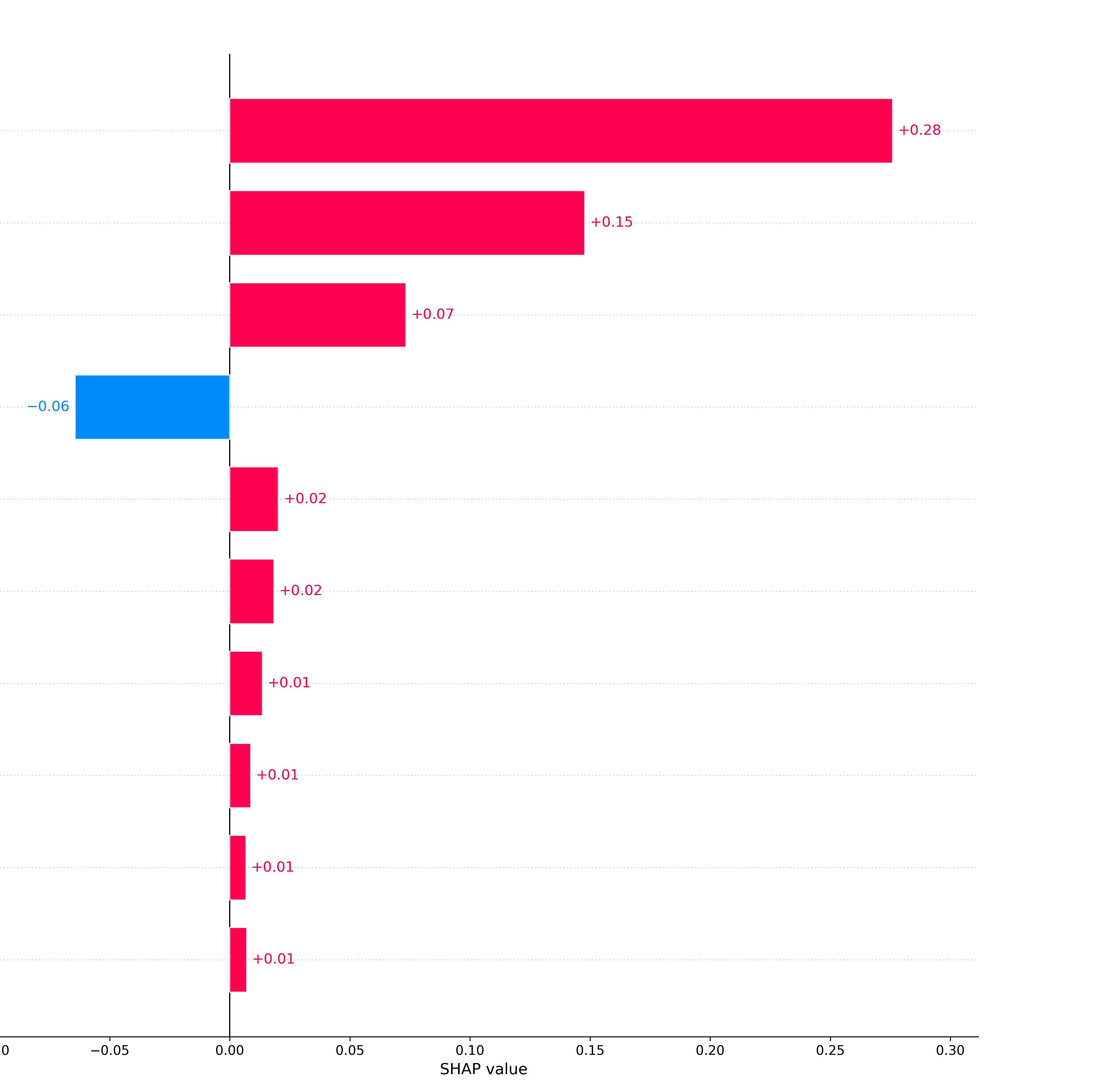}
\caption{\textbf{\ourapproach vs.\ SHAP on a deployed automated valuation model.}
Both panels show a selection of the model's internal variables; real variable
names are withheld under an NDA with the model's operator, and the two panels
rank their own top variables independently, so a given position in one panel
does not correspond to the same variable in the other. Left: \ourapproach's
context-sensitive attribution across the AVM's internal variables. Right: SHAP
concentrates responsibility on a few downstream variables. Full pipeline and
anonymized structure in Appendix~\ref{sec:avm}.}
\label{fig:eval_avm}
\end{figure}

\medskip\noindent Across these settings \ourapproach matches actual-causality verdicts on canonical archetypes and extends past them where they do not apply, at a cost that remains tractable; the full protocols, derivations, and additional figures are collected in Appendices~\ref{app:obcb}--\ref{sec:avm}. Sections~\ref{sec:shap_examples}--\ref{sub:pearl_actual_cause_prob} develop these comparative claims in full before \cref{sec:discussion} draws out the implications.

\section{Examples and Comparison to SHAP and Causal SHAP}\label{sec:shap_examples}

 Section~\ref{sec:causal_impact} defined \ourapproach
and verified, on Alice and Bob, that the resulting attribution satisfies all seven
desiderata of Section~\ref{sec:motivations} (p.~\pageref{desid}), where PN/PS/PNS
fail at least two. The natural next
question is how \ourapproach compares to SHAP and Causal SHAP, the dominant
practical alternatives for feature attribution that practitioners typically
reach for. This section runs that comparison
on two examples, recording where each method passes or fails an explicit
desideratum. Formal definitions of SHAP and Causal SHAP are in \cref{sec:shap_defs}; we re-run OBCB
(Section~\ref{sec:obcb_shap}), where the two SHAP variants coincide because the
feature set is exogenous, so the divergence between PCI and SHAP is the only
contrast to draw. We then introduce a continuous mediation example
(Section~\ref{sec:signal}) where plain and Causal SHAP do diverge (one
desideratum, D-YX, is fixed by the interventional correction; others, D-YM,
D-MX, D-MXY, are not), and tabulate the desiderata in
Section~\ref{sec:shap_signal}. Section~\ref{sec:shap_summary} traces every
failure to one of three structural gaps in the SHAP framework that PCI's witness
mechanism, suspect-set distribution $\Gamma_s$, and realised-outcome reference
jointly close.

Shapley values, grounded in cooperative game theory, decompose a model's
prediction into additive contributions from individual features and have
become one of the most widely used attribution tools in the ML literature.
\citet{janzing2020feature} argued that the conditional expectation in plain SHAP is the
conceptually wrong way to model feature removal: a dropped feature should be
marginalised under a $do$-intervention, not under conditioning.
\citet{heskes2020causal} propose Causal SHAP, which implements this prescription by
replacing the observational marginal used for features outside the active coalition
with an interventional distribution: when a dropped feature is causally downstream of
a coalition member, its distribution under the intervention differs from its marginal,
and Causal SHAP corrects for this using Pearl's do-calculus.\footnote{%
A different counterfactual line is taken by \citet{sharma2022cfshapley}, whose
CF-Shapley values use a structural causal model to attribute the change in a metric
from a single fixed reference $x^{\text{ref}}$ to the observed $x$:
$\phi_j \propto \sum_{\mathcal{S}\subseteq N\setminus\{j\}}
[f(x_\mathcal{S}, x_j^{\text{ref}}) - f(x_\mathcal{S}, x_j)]$.
This is the closest neighbour to \ourapproach{} in spirit (single contrastive
scenario, $do$-style intervention) and differs in three ways.
\emph{(i)} CF-Shapley commits to one reference $x^{\text{ref}}$, while \ourapproach{}
integrates over an alternative-value distribution $\Delta(\mathbf{x})$ and need not pick
a single contrast.
\emph{(ii)} CF-Shapley returns a one-axis contrast
$f(\cdot, x_j^{\text{ref}}) - f(\cdot, x_j)$, while \ourapproach{} reports a joint
$(Y^s, Y^n)$ measure and applies a user-chosen $ci$, decomposing necessity and
sufficiency separately.
\emph{(iii)} CF-Shapley has no witness mechanism, so it cannot break the
overdetermination and undercutting symmetries that \ourapproach{}'s witness set
$\mathbf{W}$ resolves on the OBCB and signal walkthroughs of this section.}

Despite this repair, both methods share deeper limitations. First, both plain and
Causal SHAP can assign non-zero responsibility to variables that play no causal role
with respect to the target: attribution can run against the direction of the causal
graph, a failure the interventional correction does not address. Second, both methods explain a model's prediction as a departure from that model's population-average output, not from the factual outcome an individual actually experienced. When the goal is \emph{causal responsibility
attribution} rather than prediction explanation, this mismatch in target produces
further divergences, which we make explicit below. The first limitation will surface in our signal example as violations of
D-YX, D-YM, and D-MX; the second will surface in OBCB as the D-A-rank overdetermination
effect, and in the signal example as Instance~2's budget collapse when the prediction
equals its baseline.

With these general limitations in view, we trace them through the same two
examples, comparing plain SHAP, Causal SHAP, and PCI against the desiderata of
Section~\ref{sec:causal_impact} together with the desiderata specific to each
example.

\medskip\noindent
At the level of subset weighting alone, \ourapproach{} (with $J=1$, $K=|\mathbf{S}|$) and
SHAP are not distinguishable: $\Gamma$'s subset weights match the Shapley kernel. The
distinction lies in the integrand. Plain SHAP (in the conditional reading of
Eq.~\eqref{eq:plain_shap_value}, which we use throughout except where noted) averages marginal contributions using the observational conditional
$\mathbb{E}[f(\mathbf{X}) \mid \mathbf{X}_{\mathcal{S}} = \mathbf{x}_{\mathcal{S}}]$, drawing the
non-coalition features from their distribution given the coalition; a feature merely
correlated with a true cause can therefore receive positive attribution. \ourapproach{}
evaluates counterfactual
outcomes
$P(Y \in A \mid \mathbf{u},\, \mathrm{do}(\mathbf{C}=\mathbf{c}',\,\mathbf{T}=\mathbf{t}^\star))$
computed via the structural model, so attribution flows only along genuine causal paths.
The witness mechanism adds a further distinction: by holding intermediate variables fixed,
\ourapproach{} can resolve overdetermination and undercutting scenarios where SHAP assigns
equal or incorrect attribution. The remainder of this section makes both points concrete on
two examples.

\subsection{OBCB: SHAP and Causal SHAP}\label{sec:obcb_shap}

To apply the SHAP and Causal SHAP definitions of \cref{sec:shap_defs}
to OBCB we must specify both a feature set $N$ and a function $f$ on the joint
feature space; the value functions and Shapley sums are derived from these.
The natural choice for $f$ is the conditional expectation of the outcome,
$f(x_N)=\mathbb{E}[\varname{loan}\mid X_N=x_N]$. But the OBCB PSCM has five
variables ($\varname{gender}$, $\varname{credit}$, $\varname{check}$,
$\varname{check\text{-}failed}$, $\varname{loan\_if\_checked}$), so the
choice of $N$ is non-trivial. Two options are natural:
\begin{itemize}
\item \textbf{Option A: two upstream roots.}
$N=\{\varname{gender},\varname{credit}\}$. Both features are exogenous; the
three internal variables ($\varname{check}$, $\varname{check\text{-}failed}$,
$\varname{loan\_if\_checked}$) are integrated out inside $f$. This is the
smallest sufficient feature set and matches the way the model is normally
described: predict approval from gender and credit.
\item \textbf{Option B: two roots plus the mediator.}
$N'=\{\varname{gender},\varname{credit},\varname{check\text{-}failed}\}$.
The mediator becomes an observable model input, the same status it has in the
signal-with-mediation example (\S\ref{sec:signal}). This is the comparison
PCI's witness mechanism naturally invites: PCI accesses
$\varname{check\text{-}failed}$ as a witness, so allowing SHAP to access it as
a feature is the fair foil.
\end{itemize}
We work through both. Option~A is simpler and exposes the qualitative failures
of plain SHAP. Option~B then asks whether admitting the mediator rescues Causal
SHAP.

\medskip\noindent\textbf{Option~A: two upstream roots.}
Here $f(g,c)=P(\varname{loan}=1\mid g,c)$ takes values
$f(\text{F},\text{bad})=0.000$, $f(\text{F},\text{good})=0.180$,
$f(\text{M},\text{bad})=0.045$, $f(\text{M},\text{good})=0.900$, with the three
internal variables integrated out. We apply SHAP to the rejection probability
$g(x)=1-f(x)$ for Alice and Bob from Example~\ref{ex:obcb_stochastic},
attributing responsibility for the denial outcome.\footnote{Because $g=1-f$,
the characteristic functions satisfy $v_g(\mathcal{S})=1-v_f(\mathcal{S})$ and
$\phi^g_i=-\phi^f_i$; we compute $\phi^f$ and flip signs. With $|N|=2$ and
uniform marginals $P(\varname{gender})=P(\varname{credit})=\tfrac{1}{2}$,
equation~\eqref{eq:shapley} reduces to
$\phi_i = \tfrac{1}{2}[v(\{i\})-v(\emptyset)] + \tfrac{1}{2}[v(N)-v(N\!\setminus\!\{i\})]$.}
By the SHAP definition (and equivalently by Definition~\ref{def:causalshap}
once we note that no dropped feature is a descendant of any coalition member),
the value function on the 2-feature game and its closed-form specialisation
under independent uniform marginals are derived in
Appendix~\ref{app:obcb} (\S\ref{app:obcb_vS}). Plugging in the four values of
$f$ listed at the start of Option~A yields Table~\ref{tab:vS_optionA}; pushing
those through the closed-form Shapley gives the magnitudes in
Table~\ref{tab:shap_obcb}, which compares against PNS and PCI.

\begin{table}[ht]
\centering
\begin{tabular}{lcc}
\toprule
$\mathcal{S}$ & $v(\mathcal{S})$ for Alice & $v(\mathcal{S})$ for Bob \\
\midrule
$\emptyset$                              & $0.281$ & $0.281$ \\
$\{\varname{gender}\}$                   & $0.090$ & $0.473$ \\
$\{\varname{credit}\}$                   & $0.023$ & $0.023$ \\
$\{\varname{gender},\varname{credit}\}$  & $0.000$ & $0.045$ \\
\bottomrule
\end{tabular}
\caption{Coalition values $v(\mathcal{S})=\mathbb{E}[f(X)\mid X_{\mathcal{S}}=x_{\mathcal{S}}^\star]$ for the
2-feature game ($N=\{\varname{gender},\varname{credit}\}$).
$v(\emptyset)=\mathbb{E}[f(X)]=0.281$ is the population baseline; $v(N)$ is
the factual prediction $f(x^\star)$. Plain SHAP and Causal SHAP coincide on
this game (neither feature is downstream of the other), so a single column
suffices per individual.}
\label{tab:vS_optionA}
\end{table}

\begin{table}[ht]
\centering
\small
\setlength{\tabcolsep}{3pt}
\begin{tabular}{llccccccc}
\toprule
Person & Feature & PNS & SHAP & ATE & CATE & ITE & PCI (no wit.) & PCI (with wit.) \\
\midrule
Alice & \varname{gender} & 0.045 & \phantom{$-$}0.107 & 0.383 & 0.045 & 0.045 & 0.210 & \textbf{0.263} \\
Alice & \varname{credit} & 0.180 & \phantom{$-$}0.174 & 0.518 & 0.180 & 0.180 & 0.240 & \textbf{0.199} \\
Bob   & \varname{gender} & 0.000 & $-$0.107           & 0.383 & 0.045 & 0.000 & 0.038 & \textbf{0.019} \\
Bob   & \varname{credit} & 0.860 & \phantom{$-$}0.343 & 0.518 & 0.855 & 0.895 & 0.228 & \textbf{0.119} \\
\bottomrule
\end{tabular}
\caption{Responsibility attributions for Alice and Bob.
SHAP values are signed ($\phi^g_i$ for $P(\varname{loan}=0\mid x)$); desiderata compare magnitudes $|\cdot|$.
ATE is a population quantity (identical across the two persons by feature); CATE conditions on
the \emph{other} covariate (credit for the gender rows; the person's gender for the credit rows);
ITE abducts the individual's exogenous noise from the full factual record (covariates \emph{and}
the observed denial) and takes the counterfactual contrast on that pinned noise.
PNS and PCI values from Table~\ref{tab:pci_decomp}; ATE/CATE/ITE computed by direct PSCM evaluation in
\nb{obcb\_computations}.}
\label{tab:shap_obcb}
\end{table}

ATE and CATE fail the same two individual-level desiderata that PNS fails, and
for related reasons.
ATE attributes responsibility at the population level only, so its
values are identical across Alice and Bob, immediately violating \textbf{D-comp}
($|R(\varname{gender}\mid\text{Alice})| = |R(\varname{gender}\mid\text{Bob})| = 0.383$).
CATE conditions on observable covariates but inherits PNS's blind spot via the
$\varname{check\text{-}failed}$ propagation discussed in Section~\ref{sec:motivations}:
for Alice it returns $0.045$ for gender and $0.180$ for credit, the same
ranking as $\pns_c$ and the same \textbf{D-A-rank} failure.
CATE also fails \textbf{D-comp}, since
$\mathrm{CATE}(\varname{gender}\mid\varname{credit}{=}\text{bad}) = 0.045$ is the same number
for Alice and Bob.

The individual treatment effect (ITE) is the natural endpoint of this ladder:
where ATE marginalises all exogenous noise and CATE conditions only on observed
covariates, the ITE abducts the unit's noise from its \emph{full} factual record
(covariates \emph{and} the observed denial) and then takes the
counterfactual contrast holding that abducted noise fixed, the
abduction--action--prediction reading already used for $\pn$ in
Section~\ref{sec:motivations}. This buys back exactly the information CATE
discards. The ITE separates the two applicants on \textbf{D-comp} where CATE
cannot: $\mathrm{ITE}(\varname{gender}\mid\text{Alice}) = 0.045$ but
$\mathrm{ITE}(\varname{gender}\mid\text{Bob}) = 0$, because abduction places Bob
in the checked-and-failed branch, where being male is not what blocked him:
flipping him to female (still bad credit) leaves him denied with certainty.
The ITE is also a genuine counterfactual rather than an intervention: because it
conditions on the outcome, $\mathrm{ITE}(\varname{credit}\mid\text{Bob}) = 0.895$
rather than the interventional $P(\varname{loan}{=}1\mid do(\varname{credit}{=}\text{good}),\varname{gender}{=}\text{male}) = 0.90$,
as conditioning on Bob's denial shifts posterior mass toward the unchecked branch.
But the ITE does \emph{not} fix \textbf{D-A-rank}: for Alice it still returns
$0.180$ for credit against $0.045$ for gender, the same inversion as $\pns_c$ and
CATE. Abduction pins the \emph{noise} but not the \emph{mechanism}: the
counterfactual still routes through a hypothetical credit check in the $20\%$ of
Alice's posterior in which she was checked, so credit retains spurious
responsibility. Recovering the correct ranking requires \ourapproach's witness
mechanism, which pins $\varname{check\text{-}failed}$ at its factual value
(dominantly the unchecked branch for Alice, where credit is causally
disconnected from the loan) instead of integrating over its noise. In
this sense \ourapproach is the ITE \emph{plus} the structural machinery needed to
see the mediated path: it inherits the ITE's individual-level, counterfactual
character and adds the context-fixing the ITE alone lacks.

SHAP fails two desiderata. The primary failure is \textbf{D-A-rank}: Alice's
credit outranks her gender ($0.174 > 0.107$). This is an overdetermination
effect: \varname{credit}=bad correlates with the \varname{check\text{-}failed}
mechanism that blocked Alice's credit evaluation, so SHAP's observational
marginalization inflates credit attribution above gender's. PCI includes the factual
\varname{check\text{-}failed} value directly, isolating gender's upstream causal
role and recovering the correct ranking. PNS fails \textbf{D-A-rank} for the same reason;
it also fails \textbf{D-B1} by returning exactly zero for Bob's gender (see
Table~\ref{tab:desiderata}).

The second SHAP failure is \textbf{D-comp}: $|\phi^g_{\varname{gender}}|=0.107$
for both Alice and Bob. But gender played a strictly larger causal role in Alice's
rejection than in Bob's. For Alice, \varname{gender}=female was sufficient on its
own to produce rejection: it blocked the credit evaluation entirely, so her
credit was never a factor. For Bob, \varname{gender}=male was a facilitating
condition: it triggered the credit check, but \varname{credit}=bad was the actual
driver of the denial. SHAP cannot distinguish these because it measures how far
each feature value deviates from the population prediction mean: female and
male are equidistant deviations in opposite directions, so they receive equal
magnitude regardless of their individual causal weight. This is an instance of a
deeper limitation: SHAP explains predicted outputs, not individual factual
outcomes, which we examine further in Section~\ref{sec:signal}.

Causal SHAP coincides exactly with plain SHAP under Option~A: both features
are exogenous roots, so
$P(\varname{credit}\mid do(\varname{gender}=g))=P(\varname{credit})$ and vice
versa, no dropped feature is a descendant of any coalition member, and
$v_{\mathrm{causal}}(\mathcal{S})=v_{\mathrm{plain}}(\mathcal{S})$ for every coalition. Option~A
therefore reduces to a single SHAP analysis, and the failures
above belong to both methods.

\medskip\noindent\textbf{Option~B: two roots plus the mediator.}
Throughout this subsection we abbreviate $\varname{gender}$, $\varname{credit}$, and
$\varname{check\text{-}failed}$ as $g$, $c$, and $cf$ in mathematical expressions and table
cells; running prose retains the full names.
Now we admit \varname{check\text{-}failed} as a third model input and apply
both methods to the 3-feature game $N'=\{\varname{gender},\varname{credit},\varname{check\text{-}failed}\}$.
This requires defining $f$ on $\{0,1\}^3$ rather than on $\{0,1\}^2$. The
natural observational extension is again the conditional expectation
$\tilde f(g,c,cf)=\mathbb{E}[\varname{loan}\mid \varname{gender}=g,\,\varname{credit}=c,\,\varname{check\text{-}failed}=cf]$.
This is well defined for three of the four $(c,cf)$ combinations, but not the
fourth: a failed check requires that a check happened \emph{and} the credit
was bad, so $\varname{check\text{-}failed}=1$ together with
$\varname{credit}=\text{good}$ is impossible under the PSCM and has no
observational data to condition on. We leave $\tilde f$ undefined on this cell.

Both methods need a value function $v(\mathcal{S})$: the expected value of $\tilde f$
when the features in $\mathcal{S}$ are held at their factual values $x_{\mathcal{S}}^\star$ and the
remaining features are averaged out. The methods differ in how that average
is taken. Here plain SHAP switches to the \emph{marginal-independent} reading
flagged after Eq.~\eqref{eq:plain_shap_value} (rather than the conditional
reading used elsewhere in this section), because it is this reading, not the
conditional one, that produces the impossible-cell pathology below: it
averages over each dropped feature using its own
observational marginal $P(X_i)$, with the dropped features treated as
independent of one another. Causal SHAP averages over the joint distribution
the PSCM produces under the intervention $do(X_{\mathcal{S}} = x_{\mathcal{S}}^\star)$, so the dropped
features keep whatever joint dependencies the PSCM imposes on them. Whether
either average ever lands on the impossible cell $(c{=}1, cf{=}1)$ (and so
whether $v(\mathcal{S})$ is finite) depends on both the method and on which features
are held fixed at which factual values. Table~\ref{tab:vS_optionB} works out
all eight coalition values for both individuals.

\begin{table}[ht]
\centering
\small
\begin{tabular}{l cc cc}
\toprule
& \multicolumn{2}{c}{Alice $(g,c,cf){=}(0,0,0)$} & \multicolumn{2}{c}{Bob $(g,c,cf){=}(1,0,1)$} \\
\cmidrule(lr){2-3} \cmidrule(lr){4-5}
$\mathcal{S}$ & $v_{\mathrm{plain}}$ & $v_{\mathrm{causal}}$ & $v_{\mathrm{plain}}$ & $v_{\mathrm{causal}}$ \\
\midrule
$\emptyset$    & \textbf{NaN} & $0.281$ & \textbf{NaN} & $0.281$ \\
$\{g\}$        & \textbf{NaN} & $0.090$ & \textbf{NaN} & $0.473$ \\
$\{c\}$        & $0.007$      & $0.023$ & $0.007$      & $0.023$ \\
$\{cf\}$       & $0.270$      & $0.270$ & \textbf{NaN} & \textbf{NaN} \\
$\{g,c\}$      & $0.000$      & $0.000$ & $0.014$      & $0.045$ \\
$\{g,cf\}$     & $0.090$      & $0.090$ & \textbf{NaN} & \textbf{NaN} \\
$\{c,cf\}$     & $0.000$      & $0.000$ & $0.025$      & $0.025$ \\
$N'$           & $0.000$      & $0.000$ & $0.050$      & $0.050$ \\
\bottomrule
\end{tabular}
\caption{Coalition values for the 3-feature game
$N' = \{\varname{gender},\varname{credit},\varname{check\text{-}failed}\}$.
Bold \textbf{NaN} entries flag coalitions whose value function evaluates
$\tilde f$ at the impossible cell $(c{=}1,cf{=}1)$ with positive weight.
Plain SHAP and Causal SHAP coincide whenever neither $\varname{credit}$ nor
$\varname{check\text{-}failed}$ is dropped, since the only structural coupling
in the model is between these two; they diverge whenever one is in $\mathcal{S}$ and the
other is not.}
\label{tab:vS_optionB}
\end{table}

Plain SHAP treats $\varname{credit}$ and $\varname{check\text{-}failed}$ as
independent, so whenever both are dropped from the coalition the average runs
over all four $(c, cf)$ combinations, with the impossible $(c{=}1, cf{=}1)$
combination receiving weight
$P(c{=}1)\cdot P(cf{=}1) = 0.5\cdot 0.275 \approx 0.140$,
the same nonzero weight as any other corner. For Alice, both $\varname{credit}$ and
$\varname{check\text{-}failed}$ are dropped at $\mathcal{S}\in\{\emptyset,\{\varname{gender}\}\}$,
the first two NaN
entries in Table~\ref{tab:vS_optionB}. For Bob the same two coalitions fail,
and two more break for an additional reason: at
$\mathcal{S}=\{\varname{check\text{-}failed}\}$ and
$\mathcal{S}=\{\varname{gender},\varname{check\text{-}failed}\}$ the value $cf^\star=1$
is held fixed while $\varname{credit}$ is sampled from its marginal, so the $c{=}1$ branch
enters with weight $0.5$ and again hits the impossible cell. Plain SHAP gives
undefined Shapley values for both individuals.\footnote{The underlying issue
is that plain SHAP does not see that $\varname{credit}$ and
$\varname{check\text{-}failed}$ are linked by the structural equation
$cf = \varname{check}\cdot(1 - c)$, and so places weight on a combination
the PSCM rules out.}

Causal SHAP samples $cf$ following its structural equation
$cf = \varname{check}\cdot(1 - c)$, with $\varname{check}\sim\mathrm{Bern}(p_\text{check}[g])$,
given whatever $(g, c)$ values appear in the integrand. When $c=1$ the
equation forces $cf=0$ regardless of $\varname{check}$, so the
$(c{=}1, cf{=}1)$ combination is never produced. Every Alice entry in the
$v_\mathrm{causal}$ column of Table~\ref{tab:vS_optionB} is therefore finite.
Plugging Alice's column into the Shapley formula gives magnitudes
$|\phi_{\varname{gender}}|=0.097$, $|\phi_{\varname{credit}}|=0.176$,
$|\phi_{\varname{check\text{-}failed}}|=0.007$, so
$|\phi_{\varname{credit}}|>|\phi_{\varname{gender}}|$ and \textbf{D-A-rank}
fails: credit outranks gender, exactly the same pathology plain SHAP exhibited
on the 2-feature game.

Two coalitions break for Bob: $\mathcal{S}=\{\varname{check\text{-}failed}\}$ and
$\mathcal{S}=\{\varname{gender},\varname{check\text{-}failed}\}$. Both hold
$\varname{check\text{-}failed}$ at Bob's factual value $cf^\star=1$ and average
$\varname{credit}$ over its
marginal. To see why this is fatal, apply
Definition~\ref{def:causalshap} to the first:
\[
v_\mathrm{causal}(\{cf\})
\;=\; \mathbb{E}\!\left[\tilde f(X)\,\big|\,do(cf = 1)\right]
\;=\; \tfrac{1}{4}\sum_{g\in\{0,1\}}\sum_{c\in\{0,1\}}\tilde f(g,c,1).
\]
The uniform $\tfrac{1}{4}$ weights come from $\varname{gender}$ and
$\varname{credit}$ sitting upstream of $\varname{check\text{-}failed}$ in the causal graph:
intervening on $\varname{check\text{-}failed}$ overrides the structural equation that would
normally determine it
from credit and check, but does not alter the marginal distributions of the
upstream features. Two of the four summands evaluate $\tilde f$ at
$(g, c{=}1, cf{=}1)$ (the impossible cell) with weight $\tfrac{1}{4}$
each. So $v_\mathrm{causal}(\{cf\})$ is undefined, and so is
$v_\mathrm{causal}(\{g, cf\})$ for the same reason. These two coalitions
appear in the Shapley sum for every feature, so all three of Bob's Causal
SHAP values are NaN (Bob's columns in Table~\ref{tab:vS_optionB}).

The same two coalitions arise for Alice; the difference is the value of
$cf^\star$ being held fixed. With Alice's $cf^\star = 0$, the $c{=}1$ summand
in the formula above evaluates $\tilde f(g, 1, 0) =
p_\text{check}[g]\cdot\text{loan\_prob}[g, 0]$, which is defined because the
PSCM does allow $(c{=}1, cf{=}0)$: when credit is good, $cf$ is forced to $0$
deterministically, so the conditional expectation has a value. With Bob's
$cf^\star = 1$, the $c{=}1$ summand evaluates $\tilde f(g, 1, 1)$, the
impossible cell, ``good credit and a failed check'', which the PSCM
rules out, since failing the credit check requires bad credit by the
structural equation $cf = \varname{check}\cdot(1 - c)$.

When the PSCM runs forward, $\varname{credit}$ and $\varname{check\text{-}failed}$ are
linked by $cf = \varname{check}\cdot(1 - c)$, which keeps $c{=}1$ and
$cf{=}1$ from co-occurring. When Causal SHAP intervenes via
$do(cf = cf^\star)$, this link is overridden: $\varname{check\text{-}failed}$ is fixed
externally and
$\varname{credit}$ is sampled freely. The override always produces a
$(c{=}1, cf^\star)$ combination with positive weight; whether $\tilde f$ has
a value for that combination depends on which $cf^\star$ value Causal SHAP is
intervening on.

Promoting the mediator to a feature therefore does not recover the desideratum:
Causal SHAP still misses it for Alice and is left undefined for Bob.

\medskip\noindent
The OBCB analysis exposed two structural failure modes: SHAP's
overdetermination effect on the 2-feature game, and Causal SHAP's breakdown
once the mediator is admitted as a feature under Option~B. The second of
these depends on a structural feature of the OBCB PSCM: the equation
$cf = \varname{check}\cdot(1 - c)$ ties $\varname{credit}$ and
$\varname{check\text{-}failed}$ together, so Option~B's natural feature set
contains one combination of input values (``good credit and a failed
check'') that the model is not defined on. To see whether the same families
of failure appear without this complication, we now turn to the
signal-with-mediation chain $X \to M \to Y$, where the mediator is naturally
an observable input, $\tilde f$ is defined everywhere, and no impossible-cell
issue arises. Causal SHAP's failures there are not artifacts of an undefined
$\tilde f$ but properties of the method itself.

\subsection{Signal with Mediation}\label{sec:signal}

\begin{example}[Signal with mediation]\label{ex:signal_mediation}
We work with the following \textbf{generative model}: $X$ sends a signal,
$M$ mediates it with some noise, and $Y$ is the noisy final outcome:
\begin{align}
X &\sim \mathcal{N}(0.5,\ 0.25) \label{eq:X}\\
M &= X + \varepsilon_M, \quad \varepsilon_M \sim \mathcal{N}(0,\ 0.1) \label{eq:M}\\
Y &= M + \varepsilon_Y, \quad \varepsilon_Y \sim \mathcal{N}(0,\ 0.1) \label{eq:Y}
\end{align}

\noindent The causal graph is $X \to M \to Y$, a simple mediation chain
(Figure~\ref{fig:dag_signal}). All noise terms
are independent of each other and of $X$. Since everything is jointly Gaussian, all
conditional expectations are linear functions of the conditioning variables,
and no approximations are required.

\begin{figure}[htbp]
\centering
\includegraphics[width=0.70\linewidth]{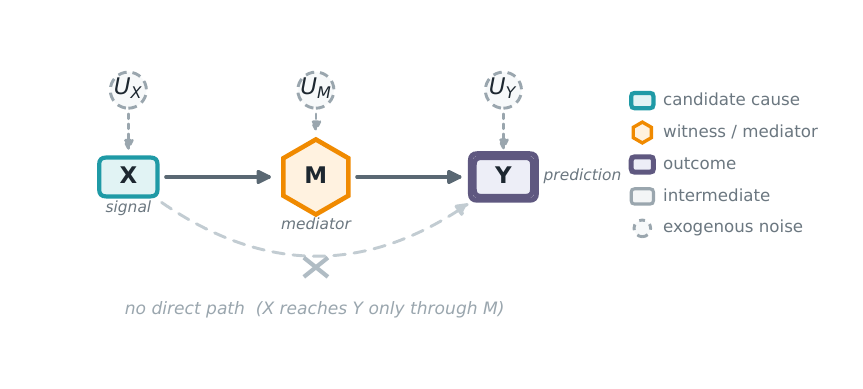}
\caption[Signal-with-mediation graph]{The signal-with-mediation model. $X$ (teal)
reaches the outcome $Y$ (purple double box) \emph{only} through the mediator $M$
(gold hexagon); the faded, crossed-out arc marks the \emph{absent} direct path
$X\to Y$. Dashed circles are the independent Gaussian noises $U_X,U_M,U_Y$ of
\crefrange{eq:X}{eq:Y}. Desideratum \textbf{D-MXY} asks an attribution to register that one missing edge: $M$ acts on $Y$ directly, whereas $X$ acts only
at one remove.}
\label{fig:dag_signal}
\end{figure}

Let $R(X \rightsquigarrow Y)$ denote the causal responsibility of $X$ for $Y$.
Following the convention of Section~\ref{sec:causal_impact}, the desiderata are
stated in terms of absolute magnitudes $|R(\cdot)|$, allowing for directionality.
The \textbf{qualitative causal responsibility desiderata} for this example are:
\begin{align*}
\textbf{D-XY} \quad & |R(X \rightsquigarrow Y)| > 0 \\[4pt]
\textbf{D-MY} \quad & |R(M \rightsquigarrow Y)| > 0 \\[4pt]
\textbf{D-XM} \quad & |R(X \rightsquigarrow M)| > 0 \\[8pt]
\textbf{D-YX} \quad & |R(Y \rightsquigarrow X)| = 0 \\[4pt]
\textbf{D-YM} \quad & |R(Y \rightsquigarrow M)| = 0 \\[4pt]
\textbf{D-MX} \quad & |R(M \rightsquigarrow X)| = 0
\end{align*}

A more subtle desideratum concerns telling apart the \emph{roles} of the two
genuine causes. $X$ reaches $Y$ only through the mediator $M$, whereas $M$ acts
on $Y$ directly; a good attribution method should be able to distinguish these
two roles rather than collapse them into the same number. One systematic way to
register the distinction is to require that the two causes receive different,
rankable scores, and we adopt the convention of assigning the larger score to
the cause closer to the outcome:
\begin{align*}
\textbf{D-MXY} \quad & |R(M \rightsquigarrow Y)| > |R(X \rightsquigarrow Y)|
\end{align*}
\noindent The direction of the inequality is a convention rather than a law of responsibility; what matters is that the method separates the direct from the indirect path at all. A method that assigns the larger score to the more distal cause would be equally principled; we adopt the direction stated here for consistency.

\noindent The \textbf{distributional facts} we will use are:
\begin{align}
\mathbb{E}[X] &= \mathbb{E}[M] = \mathbb{E}[Y] = 0.5, \\
\bigl(\text{Var}(X),\, \text{Var}(M),\, \text{Var}(Y)\bigr) &= (0.25,\, 0.35,\, 0.45), \\
\bigl(\text{Cov}(X,M),\, \text{Cov}(X,Y),\, \text{Cov}(M,Y)\bigr) &= (0.25,\, 0.25,\, 0.35).
\end{align}

\noindent For each target variable, the optimal predictor under squared loss is the
conditional expectation, computable exactly in this Gaussian model:
\begin{align}
f_Y(X, M) &= \mathbb{E}[Y \mid X, M] = M \quad (\{X, M\} \Rightarrow Y) \\
f_M(X, Y) &= \mathbb{E}[M \mid X, Y] = 0.5X + 0.5Y \quad (\{X, Y\} \Rightarrow M) \\
f_X(M, Y) &= \mathbb{E}[X \mid M, Y] = 0.5 + \frac{5}{7}(M - 0.5) \quad (\{M, Y\} \Rightarrow X)
\end{align}

\noindent The distributional facts and the three conditional expectations above
are derived in Appendix~\ref{app:signal}
(\S\ref{app:signal_distfacts}--\S\ref{app:signal_condexp}).

\end{example}

\subsection{Computations and Desiderata Analysis}\label{sec:shap_signal}

We analyze two instances. \textbf{Instance 1} sets $X^\star=M^\star=Y^\star=1$: every
variable lies above its mean, the prediction is non-trivial at every target, and SHAP,
Causal SHAP, and PCI all return non-zero values. \textbf{Instance 2} sets
$X^\star=M^\star=Y^\star=0.5$: every variable sits at its mean, the prediction equals
its baseline, and SHAP's attribution budget $f(x^\star)-\mathbb{E}[f(X)]$ collapses
to zero. We work through Instance~1 first, then return to Instance~2 to isolate what
each method is really tracking.

\paragraph{Computing the three methods.}
At Instance~1 ($X^\star=M^\star=Y^\star=1$) all three methods admit closed-form
evaluation in this Gaussian model. We defer the full per-target coalition-value
and Shapley computations for plain and Causal SHAP, and the
per-$(\mathbf{C},\mathbf{T})$ PCI derivation, to Appendix~\ref{app:signal}
(\S\ref{app:signal_plain}--\S\ref{app:signal_pci}); the resulting attributions
populate Table~\ref{tab:signal_desid}. One computed fact drives the D-YX
discussion below and is worth stating here.

\noindent Plain SHAP assigns $\phi_Y = 0.139 > 0$ for target $X$, even though $Y$ does not
cause $X$. This arises because $v(\{Y\}) = 0.778 > 0.5$: observing $Y{=}1$ is
statistically informative about $X{=}1$ through the chain $X \to M \to Y$. SHAP
cannot distinguish ``$Y$ is informative about $X$'' from ``$Y$ causes $X$''
because the value function is grounded in observational conditionals, which are
symmetric under reversal of causal direction.

\paragraph{Desiderata analysis.}\label{sec:signal_desiderata}
Following the convention of Section~\ref{sec:causal_impact}, we use $R(A \rightsquigarrow B)$
for the method-agnostic responsibility and $\phi_A^B$ exclusively for SHAP/Causal SHAP values;
prose statements about desiderata refer to magnitudes $|R(A \rightsquigarrow B)|$.
Here $A$ denotes the candidate cause and $B$ the target outcome (the generic $Y$ of
Section~\ref{sec:causal_impact}); for the PCI columns $B^n, B^s$ are its
necessity/sufficiency worlds, $B^\star$ its factual value, and
$\bar c = \mathbb{E}\!\left[\,|B^n - B^\star| - |B^s - B^\star|\,\right]$ the expected
absolute-difference impact of Example~\ref{ex:absolute_score} (equivalently
$\mathbb{E}[|B^n-B^\star|] - \mathbb{E}[|B^s-B^\star|]$, as in the caption of
Table~\ref{tab:signal_desid}).
Table~\ref{tab:signal_desid} maps the desiderata from Section~\ref{sec:signal}
to the values computed above.

\begin{table}[ht]
\centering
\footnotesize
\setlength{\tabcolsep}{4pt}
\resizebox{\textwidth}{!}{%
\begin{tabular}{llcccc}
\toprule
Desideratum & Condition & Plain SHAP & Causal SHAP
  & PCI ($\mathbf{W}{=}\emptyset$) & PCI ($\mathbf{W}{=}$3rd) \\
\midrule
\textbf{D-XY}  & $|R(X \rightsquigarrow Y)| > 0$
  & $0.250$\;\checkmark & $0.250$\;\checkmark
  & $0.249$\;\checkmark & $0.186$\;\checkmark$^\dagger$ \\
\textbf{D-MY}  & $|R(M \rightsquigarrow Y)| > 0$
  & $0.250$\;\checkmark & $0.250$\;\checkmark
  & $0.283$\;\checkmark & $0.319$\;\checkmark \\
\textbf{D-XM}  & $|R(X \rightsquigarrow M)| > 0$
  & $0.306$\;\checkmark & $0.375$\;\checkmark
  & $0.253$\;\checkmark & $0.284$\;\checkmark \\
\textbf{D-YX}  & $|R(Y \rightsquigarrow X)| = 0$
  & $0.139$\;\texttimes & $0.000$\;\checkmark
  & $0.000$\;\checkmark & $0.000$\;\checkmark \\
\textbf{D-YM}  & $|R(Y \rightsquigarrow M)| = 0$
  & $0.194$\;\texttimes & $0.125$\;\texttimes
  & $0.000$\;\checkmark & $0.000$\;\checkmark \\
\textbf{D-MX}  & $|R(M \rightsquigarrow X)| = 0$
  & $0.219$\;\texttimes & $0.357$\;\texttimes
  & $0.000$\;\checkmark & $0.000$\;\checkmark \\
\textbf{D-MXY} & $|R(M \rightsquigarrow Y)| > |R(X \rightsquigarrow Y)|$
  & $0.25{=}0.25$\;\texttimes & $0.25{=}0.25$\;\texttimes
  & $0.283{>}0.249$\;\checkmark & $0.319{>}0.186$\;\checkmark$^\ddagger$ \\
\bottomrule
\end{tabular}%
}
\caption{Desiderata satisfied (\checkmark) and violated (\texttimes) by plain SHAP,
Causal SHAP, and PCI on the signal-with-mediation example
($X\!\to\!M\!\to\!Y$), instance $X=M=Y=1$. PCI uses
$\bar c = \mathbb{E}[|B^n-B^\star|] - \mathbb{E}[|B^s-B^\star|]$ with $\Gamma_s,\Gamma_w$
uniform over non-empty subsets and rejection on
$\mathbf{C}\cap\mathbf{T}\neq\emptyset$. \textbf{W=3rd} sets the third variable
(neither $A$ nor $B$) as the witness when not in $\mathbf{C}$.
$^\dagger$~0.186 weights three suspect-containing $(\mathbf{C},\mathbf{T})$ configurations at
$1/4$ each (the rejection-sampling filter keeps $2/3$ of $\Gamma_s\times\Gamma_w$
draws, including a fourth valid pair $(\{M\},\emptyset)$ that does not contain $X$):
one pins $M$ as
witness (contribution~0, necessity and sufficiency coincide) and two let
$M$ propagate freely (contributions~0.320 and 0.425).
$^\ddagger$~The two values use different witness sets ($\mathbf{W}{=}\{M\}$
for D-XY, $\mathbf{W}{=}\{X\}$ for D-MY); this column reports the amplified
verdict under an asymmetric witness rule. The $\mathbf{W}{=}\emptyset$ column
above is the corresponding single-game ranking and already passes D-MXY.}
\label{tab:signal_desid}
\end{table}

\noindent\textbf{Positive results (D-XY, D-MY, D-XM).}
Both methods correctly assign positive attribution to every direct or upstream cause
of the target. Causal prediction models are built from observational conditionals,
which do preserve correlation from cause to effect, so these desiderata pose no
challenge to either method.

\noindent\textbf{Causal SHAP fixes D-YX, not D-YM.}
For target $X$, the interventional value $v_{\mathrm{causal}}(\{Y\})=0.5$ equals
the baseline because $M$ is not downstream of $Y$: $P(M\mid do(Y{=}1))=P(M)$,
so fixing $Y=1$ provides no interventional leverage over $f_X$, and Causal SHAP
correctly yields $\phi_Y^X=0$.
For target $M$, however, $\phi_Y^M=0.125>0$ persists.
The prediction model $f_M(X,Y)=0.5X+0.5Y$ contains $Y$ as an input because $Y$ is
observationally predictive of $M$; setting $Y=1$ raises the model output even under
the interventional regime.
Causal SHAP replaces the conditioning distribution but cannot remove a feature that
the prediction model itself uses.

\noindent\textbf{D-MX: the interventional correction reallocates the spurious attribution.}
Plain SHAP assigns $\phi_M^X=0.219$; Causal SHAP assigns $0.357$.
Fixing D-YX reduces $v_{\mathrm{causal}}(\{Y\})$ from $0.778$ to $0.5$, driving
$\phi_Y^X$ to zero, but reallocates the freed attribution to~$M$.
Since the structural noise that actually drives $X$ is unobserved, all attribution
concentrates on the only visible predictor,~$M$, even though $M$ does not cause~$X$.

\noindent\textbf{D-MXY fails for both.}
Since $f_Y=M$, the singleton values $v(\{X\})=\mathbb{E}[M\mid X{=}1]=1$ and
$v(\{M\})=1$ are identical, giving $\phi_M^Y=\phi_X^Y=0.25$ for both methods.
The Shapley formula cannot distinguish that $X$ reaches $Y$ only through~$M$:
when $X=1$, the expected prediction is already~$1$ regardless of whether $M$ is
observed or not, so both features receive the same marginal contribution.

\noindent\textbf{PCI: structural intervention separates the paths; witnesses
sharpen the gap.}
Table~\ref{tab:signal_desid} reports PCI in two configurations: without
witnesses ($\mathbf{W}=\emptyset$) and with the third variable as witness.

Without witnesses, PCI satisfies all six single-variable desiderata.
D-YX, D-YM, and D-MX are zero exactly: intervening on a variable that does not
structurally affect the target leaves both the necessity and sufficiency
worlds with identical distributions, so $ci=0$ identically.
D-XY, D-MY, and D-XM are positive.
D-MXY also \textbf{passes} at $\mathrm{PCI}(M\to Y)\approx 0.283 > \mathrm{PCI}(X\to Y)\approx 0.249$,
but only by a margin of $0.034$. The asymmetry comes from the sufficiency
world: at $C{=}\{X\}$, $X$'s path to $Y$ accumulates
$\varepsilon_M+\varepsilon_Y$ in $|Y^s-Y^\star|$, while at $C{=}\{M\}$, $M$'s
path accumulates only $\varepsilon_Y$. The necessity expectations match
across paths ($\mathrm{Var}(X)+\mathrm{Var}(\varepsilon_M)+\mathrm{Var}(\varepsilon_Y)
=\mathrm{Var}(M)+\mathrm{Var}(\varepsilon_Y)=0.45$), so the gap is carried
entirely by the sufficiency variance asymmetry $\mathrm{Var}(\varepsilon_M)$.
A small structural margin like this is fragile: under any perturbation of
the noise variances or the contrast function the $\mathbf{W}=\emptyset$
ordering could flip.

Admitting the third variable as a witness amplifies the gap by roughly
fourfold.
$\mathrm{PCI}(X\to Y\mid\mathbf{W}{=}\{M\})\approx 0.186$ weights three
suspect-containing $(\mathbf{C},\mathbf{T})$ configurations at $1/4$ each: in two cases $M$
propagates freely (contributing~$0.320$ and $0.425$) and in one case $M$ is
pinned as witness in both worlds, equalizing them and contributing~$0$.
$\mathrm{PCI}(M\to Y\mid\mathbf{W}{=}\{X\})\approx 0.319$ weights three cases
in which $M$ takes an alternative value; since $X$ does not appear in $Y$'s
structural equation, pinning $X$ has no effect on $Y^n$ or $Y^s$, and all
three cases contribute the same $0.425$.
The gap widens from $0.034$ at $\mathbf{W}=\emptyset$ to
$0.319-0.186=0.133$ at $\mathbf{W}=\{$third$\}$: the ordering is now \emph{robust}.
The comparison relies on different witness sets for the two cells, so it
remains illustrative rather than a single-game ranking.

\paragraph{Instance 2 (baseline).}\label{sec:signal_baseline}
We now return to \textbf{Instance 2}: every variable sits at its mean,
$X^\star = M^\star = Y^\star = 0.5$, and the model predicts the baseline at every
target: $f_Y(0.5,0.5)=0.5$, $f_M(0.5,0.5)=0.5$, $f_X(0.5,0.5)=0.5$. At this
instance, SHAP's attribution budget $f(x^\star)-\mathbb{E}[f(X)]$ vanishes, while
PCI's machinery is unaffected.

\begin{table}[ht]
\centering
\small
\begin{tabular}{lcccccc}
\toprule
& \multicolumn{2}{c}{Plain / Causal SHAP} & \multicolumn{2}{c}{PCI ($\mathbf{W}=\emptyset$)} & \multicolumn{2}{c}{PCI ($\mathbf{W}=\{$third$\}$)} \\
\cmidrule(lr){2-3}\cmidrule(lr){4-5}\cmidrule(lr){6-7}
desideratum & value & status & value & status & value & status \\
\midrule
D-XY & $0.000$ & $\times$ & $0.154$ & $\checkmark$ & $0.115$ & $\checkmark$ \\
D-MY & $0.000$ & $\times$ & $0.189$ & $\checkmark$ & $0.212$ & $\checkmark$ \\
D-XM & $0.000$ & $\times$ & $0.146$ & $\checkmark$ & $0.165$ & $\checkmark$ \\
D-YX & $0.000$ & $\checkmark$ & $0.000$ & $\checkmark$ & $0.000$ & $\checkmark$ \\
D-YM & $0.000$ & $\checkmark$ & $0.000$ & $\checkmark$ & $0.000$ & $\checkmark$ \\
D-MX & $0.000$ & $\checkmark$ & $0.000$ & $\checkmark$ & $0.000$ & $\checkmark$ \\
\midrule
D-MXY & \multicolumn{2}{c}{$0.000=0.000$~$\times$} & \multicolumn{2}{c}{$0.189>0.154$~$\checkmark$} & \multicolumn{2}{c}{$0.212>0.115$~$\checkmark$} \\
\bottomrule
\end{tabular}
\caption{Desiderata at the baseline instance $X^\star=M^\star=Y^\star=0.5$.
SHAP collapses to zero on every cell, assigning zero to the
non-causal triples (though here that zero reflects the vanishing attribution
budget rather than a detected absence of a causal path) while offering no
signal on the causal ones. PCI reports the same structural verdicts as at $X^\star=M^\star=Y^\star=1$.}
\label{tab:signal_desid_baseline}
\end{table}

The pattern in Table~\ref{tab:signal_desid_baseline} reflects the object each
method is decomposing. SHAP attributes the quantity
$f(x^\star)-\mathbb{E}[f(X)]$ across features. At the baseline this gap is
zero on every target, so every cell collapses to zero, uniformly, regardless
of how the realized outcome was produced. The non-causal triples (D-YX, D-YM, D-MX) are satisfied only
by this collapse: SHAP is not isolating the absence of a causal path; it is
reporting that there is nothing to attribute. The same collapse turns D-XY,
D-MY, and D-XM into violations.

PCI never invokes a population expectation. With $\mathbf{S}$ and $\mathbf{W}$
defined relative to a fixed factual instance, the question is whether varying
the suspect would move the realized outcome $y^\star$ away from itself,
not whether $f(x^\star)$ deviates from $\mathbb{E}[f(X)]$. The structural
mechanism is unchanged at the baseline, so the three causal pairs continue
to register positive values; the three non-causal pairs continue to register
zero exactly, because intervening on a variable that does not structurally
affect the target leaves both worlds with identical distributions.
SHAP is not wrong at the baseline: it correctly reports that the
factual prediction does not deviate from its mean. But that is a different
question from causal responsibility for $y^\star$, and when the prediction
equals its baseline SHAP has no budget to distribute even though the PSCM's
causal mechanism remains intact. PCI asks
the responsibility question directly, and answers it whether or not the
prediction happens to coincide with its population mean.

\subsection{Summary Across Both Examples}\label{sec:shap_summary}

The ranking failures above are not artifacts of a particular feature set or
parameter choice: each traces to a structural limitation of the SHAP/Causal
SHAP framework. Causal SHAP strictly improves on plain SHAP for one
desideratum~(D-YX), by replacing observational conditioning with
interventional distributions; the remaining failures share two sources:
prediction models include non-causal features that the conditioning
distribution cannot exclude~(D-YM, D-MX), and Shapley averaging cannot
distinguish direct from indirect causal paths~(D-MXY). The OBCB analysis of
Section~\ref{sec:obcb_shap} exhibits the same family in different surface
form: plain SHAP fails \textbf{D-A-rank} on the 2-feature game, and once the
mediator is admitted as a feature, Causal SHAP inherits that same ranking
failure for Alice. For Bob it instead becomes undefined outright, which
Section~\ref{sec:obcb_shap} traces to this particular PSCM's impossible-cell
structure rather than the ranking failure itself; the signal-mediation
example, which has no such cell, shows the same family of genuine failures
recur regardless. Across both examples,
neither SHAP variant reliably satisfies the desiderata that motivated
\Ourapproach in Section~\ref{sec:causal_impact}. Section~\ref{sec:motivations}
grounds D-A-rank, D-B-rank, and D-comp in distinctions the preemption and
actual-causality literature established well before this paper: active versus
preempted versus irrelevant causes, and direct versus indirect causal paths.
PCI addresses these gaps
through three ingredients the SHAP framework lacks: the
witness mechanism, the suspect-set distribution $\Gamma_s$, and the use of
the realized outcome rather than the prediction baseline. On top of these,
the path-noise asymmetry in the sufficiency world produces a structural gap on D-MXY,
which appears already in the witness-free single game and widens further under
a witness rule.

The two examples here are small enough to inspect
every desideratum individually; the broader empirical comparison (a synthetic
archetype catalogue and a contrast with gradient-based Differential Causal Effect
attribution) was already summarised in \cref{sec:evaluation} and is detailed in
Appendices~\ref{sub:synthetic} and~\ref{sec:DCE}. We turn first, in
\cref{sec:relation_with_ac}, to the formal relationship between \ourapproach and
actual causality.

\section{Relation with Actual Causality} \label{sec:relation_with_ac}

 \cref{sec:shap_examples} positioned \ourapproach against feature-attribution
baselines. We now turn to \ourapproach's relation to the causal-attribution
tradition it inherits the witness mechanism from: Halpern's actual causality.
This section establishes the \emph{formal} connection
(Theorems~\ref{th:ac-exp}--\ref{th:exp-ac} and Proposition~\ref{prop:filter});
\cref{sub:pearl_actual_cause_prob} adds the side-by-side comparison with Pearl's
\emph{probability of actual causation}; \cref{sec:evaluation} already previewed
the empirical comparison (runtime and verdict recovery as model size grows), with
the full protocol in Appendix~\ref{sec:ac_benchmark}. Exact AC-predicate evaluation
enumerates suspect and witness subsets, a cost that grows exponentially with model
size; the theorems below show that the expectation \ourapproach{} estimates via
Monte Carlo sampling recovers AC verdicts under the stated conditions, so the
guarantee carries over to models too large for exact enumeration.

\ourapproach{} is not meant to explicate Halpern's notion of actual causality; the goal
is to combine the intuitions underlying Halpern's notion with those motivating Pearl's
probabilities of necessity and sufficiency to provide a more general tool for causal
attribution. Nevertheless, to show that we preserve the spirit of Halpern's definition,
we now record the connection between actual causality and the joint necessity-sufficiency
machinery of Section~\ref{sec:causal_impact}. Concretely, we exhibit a
natural choice of variable selection distribution $\Gamma$ and alternative-value
distribution $\Delta$ under which \ourapproach{} verdicts agree with actual-causality
verdicts. With one caveat: probability-of-causation maximisation will not identify
subset-minimal causes that are not cardinality-minimal, a misalignment we
characterise and remedy below.

\medskip\noindent Throughout this section $\varphi$ denotes the outcome event whose
causes we attribute, $\{Y=y^\star\}$ in the binary case, or a threshold event such
as $\{Y>y_0\}$ for continuous outcomes (as in \cref{sec:sir_benchmark}), and
$\neg\varphi$ its complement. This is the event written as the measurable set
$A\subseteq\dom(Y)$ in Definition~\ref{def:jointnecsuf}.

\medskip\noindent\textit{Scope: a necessity-only correspondence.} Halpern's definition
involves factivity, context-sensitive necessity, and minimality, but no sufficiency
clause analogous to Pearl's $\mathrm{PS}$ or $\mathrm{PNS}$: a gap
\citet{beckersSufficiency2021} addresses directly, building on the counterfactual NESS
account \citep{beckersNESS2021}: a cause must be a necessary element of a set that is
sufficient for the outcome. The correspondence we establish in this section is
therefore \emph{necessity-only}: the sufficiency component of the kernel $\Phi$
defined below, the $Y^s$ factor of Definition~\ref{def:jointnecsuf}, acts here as a
factivity check (restoring factual values must preserve $\varphi$ under fixed noise).
More generally, $Y^s$ implements the graded, probabilistic analogue of Beckers' NESS
test: it restores a
candidate set $\mathbf{C}$ to factual values, holds the witnesses $\mathbf{T}$ fixed,
and checks that the outcome persists, so that where NESS returns set membership,
the joint measure returns a $\Gamma$-weighted degree, with the witness set supplying
the held-fixed context that NESS sufficient sets leave implicit. That substantive
sufficiency component, together with the $ci_S$ term in \cref{sub:synthetic}, appears
in the synthetic overdetermination and undercutting evaluation (\cref{sub:synthetic}),
the differential causal effect comparison (\cref{sec:DCE}), and the SIR benchmark
(\cref{sec:sir_benchmark}), settings where AC cannot adjudicate. The correspondence below is thus the AC-facing (necessity) half of a
construction whose sufficiency half answers to Beckers' program rather than
Halpern's. The narrowness of this section's correspondence reflects what AC itself
permits.

We first recall Halpern's definition.

\begin{definition}[Actual Cause, after \citealp{actualCausalityHalpern}] \label{def:ac}
    \(\mathbf{C} = \mathbf{c}\) is an actual cause of \(\varphi\) in the causal setting \((M, \mathbf{u})\) if the following three conditions hold:
    \begin{itemize}
        \item \textbf{Factivity}: \((M, \mathbf{u}) \vDash (\mathbf{C} = \mathbf{c}\,)\) and \((M, \mathbf{u}) \vDash \varphi\),
        \item \textbf{Context-sensitive Necessity}: There is a set of variables \(\mathbf{T} \subseteq \mathbf{V}\) and an alternative setting \(\mathbf{c'}\) of variables in \(\mathbf{C}\) such that, writing $\mathbf{t^\star}$ for the factual value of $\mathbf{T}$ in $(M,\mathbf{u})$,
            \[(M, \mathbf{u}) \vDash [\mathbf{C} \leftarrow \mathbf{c'}, \mathbf{T} \leftarrow \mathbf{t^\star}] \neg \varphi,\]
        \item \textbf{Minimality}: \(\mathbf{C}\) is minimal; there is no strict subset \(\mathbf{C_{sub}} \subset \mathbf{C}\) such that \(\mathbf{C_{sub}} = \mathbf{c_{sub}}\) satisfies the above two conditions where $\mathbf{c_{sub}}$ is the restriction of \(\mathbf{c}\) to the variables in \(\mathbf{C_{sub}}\).
        \end{itemize}
\end{definition}

\noindent To connect this to \ourapproach{}, we instantiate the joint
necessity-sufficiency machinery (Definition~\ref{def:jointnecsuf}) at a single
configuration $(\mathbf{C}, \mathbf{T})$ and a fixed noise realisation $\mathbf{u}$,
and read off the configuration-level contribution. The resulting kernel combines the
same three ingredients as Halpern's definition: a sufficiency check at factual values,
a necessity check against alternatives, and a $\Gamma$-weight that records how much
attention $\Gamma$ gives the configuration. Throughout the kernel and its proofs we write
the restored factual values as the restrictions $\mathbf{s}^\star\!\mid_{\mathbf{C}}$
and $\mathbf{w}^\star\!\mid_{\mathbf{T}}$; these are the $\mathbf{c}^\star$ and
$\mathbf{t}^\star$ of Halpern's definition above.

\begin{definition}[AC-aligned PCI kernel] \label{def:ackernel}
Fix a PSCM $M=\langle\mathbf{U},\mathbf{V},\mathbf{F},P_{\mathbf{U}}\rangle$,\footnote{%
The kernel and the theorems below take the full PSCM (structural equations $\mathbf{F}$
plus noise distribution $P_{\mathbf{U}}$) as given, since both the sufficiency and the
necessity factors require sampling counterfactual outcomes under $\mathbf{u}$. As
discussed in Section~\ref{sec:intro}, this is a stronger commitment than approaches
that bound $\mathrm{PN}/\mathrm{PS}/\mathrm{PNS}$ from observational data and a causal
graph alone, but is the price of pointwise (not just bound) verdicts at the
configuration level.} a noise
realisation $\mathbf{u}$, suspect and witness sets $\mathbf{S},\mathbf{W}\subseteq\mathbf{V}$
with factual values $\mathbf{s}^\star,\mathbf{w}^\star$, a variable selection
distribution $\Gamma$ on $2^{\mathbf{S}}\times 2^{\mathbf{W}}$ with mass function
$p^\Gamma$, and an alternative-value distribution $\Delta(\mathbf{s}^\star)$. For an
outcome event $\varphi$, active suspect set $\mathbf{C}\subseteq\mathbf{S}$ and active
witness set $\mathbf{T}\subseteq\mathbf{W}$ with $\mathbf{T}\cap\mathbf{C}=\emptyset$,
define
\begin{align}
\Phi_{\mathbf{u}}(\mathbf{C},\mathbf{T})
&\;:=\;
\underbrace{p^\Gamma(\mathbf{C},\mathbf{T})}_\text{selection weight}
\;\cdot\;
\underbrace{\mathbb{I}\!\left[(M,\mathbf{u})\vDash
[\mathbf{C}\!\leftarrow\!\mathbf{s}^\star\!\mid_{\mathbf{C}},\,
 \mathbf{T}\!\leftarrow\!\mathbf{w}^\star\!\mid_{\mathbf{T}}]\,\varphi\right]}_\text{sufficiency indicator}
\nonumber \\
&\quad\;\cdot\;
\underbrace{\int
\mathbb{I}\!\left[(M,\mathbf{u})\vDash
[\mathbf{C}\!\leftarrow\!\mathbf{c}',\,
 \mathbf{T}\!\leftarrow\!\mathbf{w}^\star\!\mid_{\mathbf{T}}]\,\neg\varphi\right]
\Delta_{\mathbf{C}}(\mathbf{s}^\star)(d\mathbf{c}')}_\text{necessity integral}.
\label{eq:ackernel}
\end{align}
\end{definition}

\medskip\noindent The kernel depends on the fixed noise realisation $\mathbf{u}$
through its two indicators; we write $\Phi_{\mathbf{u}}$ when that dependence is in play
(as in the noise-integrated score $P(\mathbf{C}\!\rightsquigarrow\!\varphi)$ below) and
abbreviate it to $\Phi$ when $\mathbf{u}$ is fixed and clear from context. Concretely,
$\Phi_{\mathbf{u}}(\mathbf{C},\mathbf{T})$ is the
per-configuration term inside the joint necessity-sufficiency measure $P_k^{s,n}$
from Definition~\ref{def:jointnecsuf}: it replaces integration over $P_{\mathbf{U}}$
with point evaluation at $\mathbf{u}$, and the sum over $(\mathbf{C},\mathbf{T})$
with evaluation at a single pair, mirroring AC's structure of
testing factivity and the alternative against a single suspect/witness pair. The
kernel is positive exactly when three conditions hold simultaneously:
$(\mathbf{C},\mathbf{T})$ has positive $\Gamma$-mass; restoring $\mathbf{C}$ to its
factual values while holding $\mathbf{T}$ fixed preserves $\varphi$ under $\mathbf{u}$
(the factivity-as-sufficiency component); and there exists a
$\Delta$-positive-density alternative $\mathbf{c}'$ for which the corresponding
intervention overturns $\varphi$ under $\mathbf{u}$ (the necessity component). The
regularity condition we introduce next ensures that the first of these three, the
$\Gamma$-selection weight, never vanishes on a legal configuration.

\begin{definition}[Regular $\Gamma$]\label{def:regular}
$\Gamma$ is \emph{regular} when $p^\Gamma(\mathbf{C}',\mathbf{T}')>0$ for every
$(\mathbf{C}',\mathbf{T}')$ with $\mathbf{C}'\subseteq\mathbf{S}$,
$\mathbf{T}'\subseteq\mathbf{W}$, $\mathbf{C}'\cap\mathbf{T}'=\emptyset$.
\end{definition}

We can now record the forward direction, under the conditions stated below:
whenever $\mathbf{C}$ is an actual cause in Halpern's sense, the matching kernel
value is strictly positive, so \ourapproach{} never scores an AC-certified cause
as zero.

\begin{theorem}\label{th:ac-exp}
    Suppose that $\mathbf{C}$ is an actual cause of $\varphi$ with witness set $\mathbf{T}$ in $(M,\mathbf{u})$. Take any potential cause set $\mathbf{S}\supseteq\mathbf{C}$ with factual values $\mathbf{S}=\mathbf{s}$, subject to the constraint that no element of $\mathbf{S}$ is a deterministic function (under $M$) of any subset of the remaining elements of $\mathbf{S}$,\footnote{Equivalently, $\mathbf{S}$ is closed under no nontrivial structural equation of $M$. The constraint excludes admitting both $S$ and a deterministic descendant $H=f(S,\ldots)$ as candidates in the same suspect set, but covers all candidate sets standardly considered in the AC literature, where structural descendants of one candidate are not admitted as separate candidates: their contribution is mediated by the equation, and AC's minimality clause already prefers the ancestor alone. The same closure condition discharges a joint-support concern: the next assumption requires $\Delta(\mathbf{s})$ to place positive density on every \emph{joint} alternative configuration $\mathbf{s}'\in\dom(\mathbf{S})\setminus\{\mathbf{s}\}$, not merely positive marginal density on each component. In a chain $S\to H$ with $H=f(S)$, joint assignments such as $(S{=}0, H{=}1)$ are impossible under the structural equations, so any $\Delta$ that placed mass on them would be incompatible with $M$. Closure rules out exactly this case: when no element of $\mathbf{S}$ is a deterministic function of the remaining elements, the joint domain $\dom(\mathbf{S})$ contains no rectangles forced to zero measure by $M$, and a $\Delta$ with positive PMF or Lebesgue density on every $\mathbf{s}'\neq\mathbf{s}$ is consistent with $M$ on the suspect set. In the stone-throwing benchmark of \cref{sec:ac_benchmark}, the suspect set is $\{A_i,B_i\}_i$ and the deterministic mediators $\{W^a_i,W^b_i\}_i$ live in the witness set, not the suspect set, exactly to keep the suspect set closure-respecting. The kernel pins witnesses at their factual values, so they are not subject to a joint-support requirement on alternatives and need no analogous condition.} and any $\mathbf{W}\supseteq\mathbf{T}$ as the potential witness set. Assume $\Gamma$ is regular and that $\Delta(\mathbf{s})$ assigns positive density to every $\mathbf{s}'\in\dom(\mathbf{S})\setminus\{\mathbf{s}\}$ (positive PMF in the discrete case, positive Lebesgue density in the continuous case). In the continuous case, assume additionally:
    \begin{itemize}
    \item[(i)] the structural equations of $M$ downstream of $\mathbf{C}$ are continuous in
    $\mathbf{c}'$ under fixed $\mathbf{u}$;
    \item[(ii)] either $\neg\varphi$ is open in $\dom(Y)$, or, more weakly, the
    boundary of $\neg\varphi$ has measure zero under the pushforward of $Y$
    \emph{and} the witnessing alternative $\mathbf{c}'$ from context-sensitive
    necessity satisfies $Y(\mathbf{c}')\in\operatorname{int}(\neg\varphi)$ (i.e.\
    the witnessed outcome lies strictly inside $\neg\varphi$, not exactly on its
    boundary).
    \end{itemize}
    Then
    \begin{align}
    \Phi(\mathbf{C},\mathbf{T}) > 0,
    \label{eq:pect}
    \end{align}
    Hypotheses (i)--(ii) are a proof device for the necessity integral inside $\Phi$;
    Halpern's Definition~\ref{def:ac} is a qualitative $\vDash$-condition with no
    continuous/discrete distinction of its own, so they play no further role once
    \eqref{eq:pect} is established. In particular, every other term
    $\Phi(\mathbf{C},\mathbf{T}')$ in the sum below is nonnegative by
    Definition~\ref{def:ackernel} regardless of them, and therefore
    \[
    \sum_{\mathbf{T}'\subseteq\mathbf{W}} \Phi(\mathbf{C},\mathbf{T}')\;>\;0.
    \]
\end{theorem}

\begin{proof}
Fix $\mathbf{u}$, $\mathbf{c}'$, and $\mathbf{T}$ as witnesses to the satisfaction of
the conditions in Definition~\ref{def:ac}. The context-sensitive-necessity clause
forces $\mathbf{T}\cap\mathbf{C}=\emptyset$, since otherwise the joint intervention
$[\mathbf{C}\!\leftarrow\!\mathbf{c}',\,\mathbf{T}\!\leftarrow\!\mathbf{w}^\star\!\mid_{\mathbf{T}}]$ would
be ill-defined. It remains to show that each of the three factors in
$\Phi(\mathbf{C},\mathbf{T})$ is strictly positive.
\begin{itemize}
\item \emph{Selection weight.} By regularity of $\Gamma$,
$p^\Gamma(\mathbf{C},\mathbf{T})>0$.
\item \emph{Sufficiency indicator.} With exogenous noise $\mathbf{u}$ held constant,
factivity gives $(M,\mathbf{u})\vDash\varphi$. Restoring $\mathbf{C}$ to
$\mathbf{s}^\star\!\mid_{\mathbf{C}}$ and $\mathbf{T}$ to $\mathbf{w}^\star\!\mid_{\mathbf{T}}$
leaves the model agreeing with $(M,\mathbf{u})$ on every variable, so the indicator
evaluates to $1$.
\item \emph{Necessity integral.} Context-sensitive necessity provides a specific
$\mathbf{c}'$ witnessing
$(M,\mathbf{u})\vDash[\mathbf{C}\!\leftarrow\!\mathbf{c}',\mathbf{T}\!\leftarrow\!\mathbf{w}^\star\!\mid_{\mathbf{T}}]\neg\varphi$.
The closure condition on $\mathbf{S}$ and the positive-density assumption on
$\Delta(\mathbf{s})$ imply that $\mathbf{c}'$ lies in the support of
$\Delta_{\mathbf{C}}(\mathbf{s}^\star)$. We treat the two cases separately.
\emph{Discrete case}: the singleton $\{\mathbf{c}'\}$ has positive PMF
$\Delta_{\mathbf{C}}(\{\mathbf{c}'\})>0$ by hypothesis, so the integral is at least
$\Delta_{\mathbf{C}}(\{\mathbf{c}'\})>0$. \emph{Continuous case}: hypothesis~(i) gives
continuity of the intervention map $\mathbf{c}'\mapsto Y$ at fixed $\mathbf{u}$;
hypothesis~(ii) places $Y(\mathbf{c}')$ in the interior of $\neg\varphi$, directly if $\neg\varphi$ is open (every point of an open set is interior), and by the added
clause otherwise. Pulling back
an open neighbourhood within $\neg\varphi$ yields an open neighbourhood of
$\mathbf{c}'$ on which the indicator equals $1$, and this neighbourhood has positive
Lebesgue measure and hence positive $\Delta_{\mathbf{C}}$-mass by the
positive-density hypothesis. In either case the integral is strictly positive.
\end{itemize}
The product of three strictly positive factors is strictly positive, giving
$\Phi(\mathbf{C},\mathbf{T})>0$. By Definition~\ref{def:ackernel}, every
$\Phi(\mathbf{C},\mathbf{T}')$ is itself a product of a nonnegative selection weight,
a $\{0,1\}$-valued indicator, and an integral of a $\{0,1\}$-valued indicator against
the nonnegative measure $\Delta_{\mathbf{C}}$, hence $\Phi(\mathbf{C},\mathbf{T}')\geq
0$ for every $\mathbf{T}'\subseteq\mathbf{W}$; summing over $\mathbf{T}'$ therefore
adds non-negative terms to the positive term $\Phi(\mathbf{C},\mathbf{T})$, preserving
positivity. The proof uses factivity
and context-sensitive necessity but not minimality: the same conclusion holds whenever
$\mathbf{C}$ satisfies factivity and context-sensitive necessity in $(M,\mathbf{u})$,
a fact we will use in Proposition~\ref{prop:filter}.
\end{proof}

\medskip\noindent In settings with continuous outcomes and a threshold event such as
$\varphi=\{Y>y_0\}$ (for instance the SIR overshoot benchmark of
\cref{sec:sir_benchmark}), both hypotheses are typically satisfied via the weaker
clause of (ii): $\neg\varphi=\{Y\le y_0\}$ is closed rather than open, but its
boundary is the single point $y_0$, which has measure zero under any continuous
pushforward of $Y$; the surviving requirement is that the specific witnessing
alternative satisfies $Y(\mathbf{c}')<y_0$ strictly rather than $Y(\mathbf{c}')=y_0$ exactly, generic under a continuous pushforward, but not automatic, so a witness
that lands exactly on the threshold needs a separate (measure-zero, hence
non-generic) argument.
(Literal openness of $\neg\varphi$ instead requires the strict inequality on the
$\neg\varphi$ side, i.e.\ $\varphi=\{Y\ge y_0\}$, $\neg\varphi=\{Y<y_0\}$.)
Binary-output models, where $Y$ takes values in $\{0,1\}$, satisfy
(ii) automatically (every subset of a discrete codomain is open) but may not satisfy
(i) when the binary output comes from thresholding a continuous score. In such cases
the proof falls back to the discrete branch, which only requires positive PMF on the
witnessing $\mathbf{c}'$.

The converse also holds: a strictly positive summed kernel value forces both
factivity and context-sensitive necessity, the two qualitative AC conditions the
kernel is built to detect (we treat minimality separately below). The proof reads
the argument for Theorem~\ref{th:ac-exp} in reverse on the necessity factor.

\begin{theorem}\label{th:exp-ac}
Under the same hypotheses on $\mathbf{S}$, $\Gamma$, and $\Delta$ as
in Theorem~\ref{th:ac-exp}, if
\[
\sum_{\mathbf{T}'\subseteq\mathbf{W}} \Phi(\mathbf{C},\mathbf{T}')\;>\;0,
\]
then both factivity ($(M,\mathbf{u})\vDash\varphi$) and the
context-sensitive-necessity condition in Definition~\ref{def:ac} hold for
$\mathbf{C}$ in $(M,\mathbf{u})$.
\end{theorem}

\begin{proof}
Since $\Phi$ is non-negative, the assumption forces $\Phi(\mathbf{C},\mathbf{T}')>0$
for at least one $\mathbf{T}'\subseteq\mathbf{W}$. By Definition~\ref{def:ackernel},
each factor of $\Phi(\mathbf{C},\mathbf{T}')$ is strictly positive. The sufficiency
indicator equals $1$, i.e.\
$(M,\mathbf{u})\vDash[\mathbf{C}\!\leftarrow\!\mathbf{s}^\star\!\mid_{\mathbf{C}},\,\mathbf{T}'\!\leftarrow\!\mathbf{w}^\star\!\mid_{\mathbf{T}'}]\varphi$;
since pinning variables to their factual values under fixed $\mathbf{u}$ leaves every
variable unchanged, this is equivalent to $(M,\mathbf{u})\vDash\varphi$, giving
factivity. The necessity integral being positive provides a
$\Delta_{\mathbf{C}}$-positive-density set of values $\mathbf{c}'$ for which
$(M,\mathbf{u})\vDash[\mathbf{C}\!\leftarrow\!\mathbf{c}',\,\mathbf{T}'\!\leftarrow\!\mathbf{w}^\star\!\mid_{\mathbf{T}'}]\neg\varphi$;
any such $\mathbf{c}'$, together with $\mathbf{T}'$, witnesses the existential in the
context-sensitive-necessity clause of Definition~\ref{def:ac}.
\end{proof}

\medskip\noindent\textit{Remark (deterministic intermediates and joint support).}
The hypothesis on $\mathbf{S}$ rules out the following obstruction.
Suppose the suspect set were closed under the structural equations of $M$, e.g., both
Sally's throw $S$ and the bottle's hit $H = f(S,\ldots)$ admitted as candidates. Then the
joint assignment $(S{=}0, H{=}1)$ is structurally impossible, and any alternative-value
distribution consistent with $M$ assigns it density zero, so the joint full-support
condition required to recover the witnessing $\mathbf{c}'$ from $\Phi > 0$ fails.
Restricting suspects to variables that are not deterministic functions of the others
under $M$ removes this obstruction. The constraint is also conservative with respect to
the AC literature: in the canonical stone-throwing setup, the candidate set is
$\{S_{Sally}, S_{Bill}\}$, with $H$ as the outcome, and AC's own
minimality clause already prefers ancestor candidates over their deterministic descendants
when both are admitted. Sections~\ref{sub:synthetic} and~\ref{sec:ac_benchmark} adopt the
same convention experimentally (the continuous synthetic benchmark uses the six roots
as suspects and the discrete scaled-throwing benchmark uses $\{A_i,B_i\}_i$ as suspects
with the deterministic mediators $\{W^a_i,W^b_i\}_i$ as witnesses); in both cases
deterministic intermediates are excluded from the suspect set. Cases where a user genuinely wants to attribute responsibility separately to a
deterministic intermediate (rather than to its structural ancestors) lie outside the
scope of these theorems and require an alternative reading of intervention semantics; we
treat that as a separate problem.

\medskip
With minimality, the situation is somewhat more complicated. Our choice of $\Gamma$
affects which configurations maximise $\Phi$. Moreover, $\Phi$ compares
configurations by $\Gamma$-weight, and Definition~\ref{def:ac} formulates minimality
strictly in terms of subsets, so some misalignment is to be expected. A few structural conditions salvage much of the
correspondence nonetheless, starting with the following monotonicity property.

Enlarging the candidate cause set typically makes the necessity check harder to
satisfy: flipping the outcome now requires a joint alternative across more
coordinates, which is generally a rarer draw under $\Delta$. \emph{Necessity
dilution} names the monotonicity condition under which this intuition always
holds: growing $\mathbf{C}$ never increases the necessity mass, so that
$\Gamma$'s preference for smaller sets is not fighting an opposing pull from $N$.

\begin{definition}[Necessity dilution]\label{def:necessity-dilution}
A pair $(M,\Delta)$ satisfies \emph{necessity dilution} on $\mathbf{S}$ if, writing
$N(\mathbf{C},\mathbf{T})$ for the necessity integral in
Equation~\eqref{eq:ackernel},
\[
N(\mathbf{C}^\sharp,\mathbf{T}^\sharp)\;\geq\;N(\mathbf{C}^\dagger,\mathbf{T}^\dagger)
\]
for every $\mathbf{C}^\sharp\subsetneq\mathbf{C}^\dagger\subseteq\mathbf{S}$ and
witnesses $\mathbf{T}^\sharp,\mathbf{T}^\dagger$ disjoint from the corresponding
suspect set.
\end{definition}

First, there is a simplified setting in which a positive value of $\Phi$,
under factivity, entails minimality with respect to the same parameters. Three
conditions suffice:
\begin{itemize}
\item $\Gamma_w$ is uniform over $2^{\mathbf{W}}$;
\item $\Gamma_s$ is \emph{strictly cardinality-decreasing}, i.e.\
$p^{\Gamma_s}(\mathbf{C}_1) > p^{\Gamma_s}(\mathbf{C}_2)$ whenever
$\mathbf{C}_1\subsetneq\mathbf{C}_2$ (this offsets larger suspect sets' structural
advantage: with more variables to draw on, they are easier to certify as necessary,
so without a compensating weight they would be spuriously preferred over the true
minimal cause);
\item the pair $(M,\Delta)$ satisfies necessity dilution on $\mathbf{S}$
(Definition~\ref{def:necessity-dilution}).
\end{itemize}
Section~\ref{sec:causal_impact}'s default recommendation and its OBCB illustration
both use a uniform $\Gamma_s$. The correspondence below needs the
cardinality-decreasing weighting above; a practitioner obtains this guarantee by
choosing that weighting explicitly.

The same proof strategy, extended below, shows that we can safely marginalise over
witnesses as well.

\begin{theorem}\label{th:local_max_to_ac}
Consider the set of optimisers
$\mathcal{O}^\star = \arg\max_{\mathbf{C}\subseteq\mathbf{S},\,\mathbf{T}\subseteq\mathbf{W}}
\Phi(\mathbf{C},\mathbf{T})$. Assume $(\mathbf{C}^\dagger,\mathbf{T}^\dagger)\in\mathcal{O}^\star$
with $\Phi(\mathbf{C}^\dagger,\mathbf{T}^\dagger)>0$, that factivity is satisfied for
$\mathbf{C}^\dagger$ (as in Definition~\ref{def:ac}), and that no element of
$\mathbf{S}$ is a deterministic function under $M$ of the others. Let
$\Gamma_w$ be uniform on $2^{\mathbf{W}}$ and let $\Gamma_s$ be strictly
cardinality-decreasing, assume that $(M,\Delta)$ satisfies necessity dilution on
$\mathbf{S}$, and assume $\Delta$ satisfies the positive-density hypothesis of
Theorem~\ref{th:ac-exp} (together with hypotheses (i)--(ii) there, in the continuous
case); this lets the proof below invoke Theorem~\ref{th:ac-exp} on
$\mathbf{C}^\sharp$. Then $\mathbf{C}^\dagger$ is an actual cause of $\varphi$
according to Definition~\ref{def:ac}.
\end{theorem}

\begin{proof}
Factivity is assumed; by Theorem~\ref{th:exp-ac} the context-sensitive-necessity
condition holds for $\mathbf{C}^\dagger$ in $(M,\mathbf{u})$. So the only way
$\mathbf{C}^\dagger$ can fail to be an actual cause is by failing minimality. Suppose
for contradiction that minimality fails: there is
$\mathbf{C}^\sharp\subsetneq\mathbf{C}^\dagger$ satisfying factivity and
context-sensitive necessity for $\varphi$ with some witness $\mathbf{T}^\sharp$. By
Theorem~\ref{th:ac-exp} (its proof requires only factivity and necessity, not
minimality), $\Phi(\mathbf{C}^\sharp,\mathbf{T}^\sharp)>0$ as well, with the
sufficiency indicator and the necessity integral both equal to $1$ and a strictly
positive value, respectively. Both pairs therefore satisfy
$\Phi(\mathbf{C},\mathbf{T}) = p^\Gamma(\mathbf{C},\mathbf{T})\cdot N(\mathbf{C},\mathbf{T})$.

By the construction of $\Gamma$ from independent marginals
(Section~\ref{sec:causal_impact}),
\[
p^\Gamma(\mathbf{C},\mathbf{T}) \;\propto\;
p^{\Gamma_s}(\mathbf{C})\,p^{\Gamma_w}(\mathbf{T})\,
\mathbb{I}[\mathbf{T}\cap\mathbf{C}=\emptyset].
\]
The $\mathbb{I}[\cdot]$ factor is $1$ for both pairs (otherwise the contradiction
would have been immediate). Uniformity of $\Gamma_w$ on $2^{\mathbf{W}}$ makes
$p^{\Gamma_w}(\mathbf{T})=2^{-|\mathbf{W}|}$ identically, so the witness factor cancels.
Strict cardinality-decrease of $\Gamma_s$ together with
$\mathbf{C}^\sharp\subsetneq\mathbf{C}^\dagger$ gives
$p^{\Gamma_s}(\mathbf{C}^\sharp)>p^{\Gamma_s}(\mathbf{C}^\dagger)$, and necessity
dilution gives $N(\mathbf{C}^\sharp,\mathbf{T}^\sharp)\geq
N(\mathbf{C}^\dagger,\mathbf{T}^\dagger)$. Multiplying these (and noting
$N(\mathbf{C}^\sharp,\mathbf{T}^\sharp)>0$ to convert weak into strict in the
composite) gives
$\Phi(\mathbf{C}^\sharp,\mathbf{T}^\sharp)>\Phi(\mathbf{C}^\dagger,\mathbf{T}^\dagger)$,
contradicting $(\mathbf{C}^\dagger,\mathbf{T}^\dagger)\in\mathcal{O}^\star$.
\end{proof}

\begin{corollary}\label{cor:marginalize}
Keep the notation and hypotheses of Theorem~\ref{th:local_max_to_ac}. Consider the set
of optimisers
$\mathcal{O}^\Sigma = \arg\max_{\mathbf{C}\subseteq\mathbf{S}}
\sum_{\mathbf{T}\subseteq\mathbf{W}} \Phi(\mathbf{C},\mathbf{T})$. Assume
$\mathbf{C}^\dagger\in\mathcal{O}^\Sigma$ with
$\sum_{\mathbf{T}\subseteq\mathbf{W}} \Phi(\mathbf{C}^\dagger,\mathbf{T})>0$ and that
factivity is satisfied for $\mathbf{C}^\dagger$. Then $\mathbf{C}^\dagger$ is an actual
cause of $\varphi$ according to Definition~\ref{def:ac}.
\end{corollary}

\begin{proof}
The summation hypothesis entails the existence of some
$\mathbf{T}^\dagger\subseteq\mathbf{W}$ with $\Phi(\mathbf{C}^\dagger,\mathbf{T}^\dagger)>0$;
by Theorem~\ref{th:exp-ac}, factivity and context-sensitive necessity hold for
$\mathbf{C}^\dagger$. It remains to rule out a failure of minimality. We cannot simply
invoke Theorem~\ref{th:local_max_to_ac}: $\mathbf{C}^\dagger\in\mathcal{O}^\Sigma$ does
not place any single pair $(\mathbf{C}^\dagger,\mathbf{T})$ in the joint optimiser set
$\mathcal{O}^\star$, since a candidate whose $\Gamma$-mass is spread thinly across many
good witnesses can beat, in sum, a competitor that concentrates all its mass on one
witness pair with strictly higher pointwise $\Phi$. We instead adapt the argument of
Theorem~\ref{th:local_max_to_ac} directly to the summed objective.

Since factivity is global, $(M,\mathbf{u})\models\varphi$, restoring \emph{any}
disjoint pair $(\mathbf{C},\mathbf{T})$ to its own factual values is the identity
intervention, so the sufficiency indicator in Definition~\ref{def:ackernel} equals $1$
for every such pair; hence $\Phi(\mathbf{C},\mathbf{T}) =
p^{\Gamma_s}(\mathbf{C})\,p^{\Gamma_w}(\mathbf{T})\,N(\mathbf{C},\mathbf{T})$ whenever
$\mathbf{T}\cap\mathbf{C}=\emptyset$.

Suppose for contradiction that minimality fails for $\mathbf{C}^\dagger$: there is
$\mathbf{C}^\sharp\subsetneq\mathbf{C}^\dagger$ satisfying factivity and
context-sensitive necessity with some witness $\mathbf{T}^\sharp$ (not necessarily
minimality itself). By Theorem~\ref{th:ac-exp}, whose proof requires only factivity
and necessity, not minimality (as already noted in the proof of
Theorem~\ref{th:local_max_to_ac}), $\Phi(\mathbf{C}^\sharp,\mathbf{T}^\sharp)>0$. For every
$\mathbf{T}$ with $\mathbf{T}\cap\mathbf{C}^\dagger=\emptyset$ (hence also
$\mathbf{T}\cap\mathbf{C}^\sharp=\emptyset$, since $\mathbf{C}^\sharp\subsetneq\mathbf{C}^\dagger$),
necessity dilution gives $N(\mathbf{C}^\sharp,\mathbf{T})\geq N(\mathbf{C}^\dagger,\mathbf{T})$
for that same $\mathbf{T}$, while strict cardinality-decrease gives
$p^{\Gamma_s}(\mathbf{C}^\sharp)>p^{\Gamma_s}(\mathbf{C}^\dagger)$ and uniformity of
$\Gamma_w$ makes $p^{\Gamma_w}(\mathbf{T})$ a common positive constant; together these
give $\Phi(\mathbf{C}^\sharp,\mathbf{T})\geq\Phi(\mathbf{C}^\dagger,\mathbf{T})$
pointwise, with strict inequality at $\mathbf{T}=\mathbf{T}^\dagger$ since
$N(\mathbf{C}^\dagger,\mathbf{T}^\dagger)>0$ there. Summing over
$\mathbf{T}\subseteq\mathbf{W}$ disjoint from $\mathbf{C}^\dagger$, then using
nonnegativity of $\Phi$ to extend the left-hand sum to all $\mathbf{T}$ disjoint from
$\mathbf{C}^\sharp$ (a superset of the domain on the right, since disjointness from the
smaller set $\mathbf{C}^\sharp$ is the weaker condition),
\[
\sum_{\mathbf{T}\subseteq\mathbf{W}} \Phi(\mathbf{C}^\sharp,\mathbf{T})
\;>\;
\sum_{\mathbf{T}\subseteq\mathbf{W}} \Phi(\mathbf{C}^\dagger,\mathbf{T}),
\]
contradicting $\mathbf{C}^\dagger\in\mathcal{O}^\Sigma$. So minimality holds, and
$\mathbf{C}^\dagger$ is an actual cause of $\varphi$.
\end{proof}

\medskip\noindent Necessity dilution is a property of $(M,\Delta)$ alone, not of
$\Gamma$. Definition~\ref{def:necessity-dilution} is deliberately strong: it compares
$N$ across \emph{every} pair of independently-chosen witness sets
$\mathbf{T}^\sharp,\mathbf{T}^\dagger$, not merely at a shared witness set. That
strength is exactly what licenses comparing $\mathbf{T}^\sharp$ against the optimal
$\mathbf{T}^\dagger$ directly in the proof of Theorem~\ref{th:local_max_to_ac} above
(Corollary~\ref{cor:marginalize}'s proof needs only the weaker same-$\mathbf{T}$
special case, $N(\mathbf{C}^\sharp,\mathbf{T})\geq N(\mathbf{C}^\dagger,\mathbf{T})$
for each shared $\mathbf{T}$, which the strong version implies). We write
$N(\mathbf{C})$ for $N(\mathbf{C},\mathbf{T})$ where this does not depend on $\mathbf{T}$ at all; i.e.\ pinning any admissible witness set at its factual
values leaves the necessity integral for $\mathbf{C}$ unchanged, which is a
structural property of $(M,\Delta,\mathbf{C})$ that should be checked case by case,
not assumed from the notation. The condition requires that the $\Delta$-mass of alternatives that flip the outcome does not grow as the candidate cause expands. Three settings make it automatic, and satisfy the witness-independence just described. \emph{(i) Single-pivot models}: one variable $X^\star\in\mathbf{S}$
is pivotal under $(M,\mathbf{u})$ and the others are structurally idle, so $N$ peaks at
$\{X^\star\}$ and falls off as more suspects are added. \emph{(ii) Factored $\Delta$
without Boolean-OR redundancy}: $\Delta_{\mathbf{C}}$ is a product of marginals
$\bigotimes_{X\in\mathbf{C}}\nu_X$ with $\nu_X(\dom(X))\leq 1$, and the flip event
factorises coordinate-wise; $N(\mathbf{C})$ is then a product of factors each at most
$1$, decreasing in $|\mathbf{C}|$. \emph{(iii) Conjunctive outcome events}: if
$\neg\varphi$ requires flipping every variable in a required set
$\mathbf{R}\subseteq\mathbf{S}$, the flip region's $\Delta_{\mathbf{C}}$-mass decreases
as $\mathbf{C}$ grows beyond $\mathbf{R}$. The hypothesis fails in disjunctive
overdetermination (e.g.\ $\varphi:\,A\vee B=1$ with both $A,B$ factual): no proper
subset of $\mathbf{S}$ flips $\varphi$, so $N$ is zero on subsets and positive on the
full set: the opposite of dilution. In such cases no proper subset of $\mathbf{S}$
is an actual cause anyway, so Theorem~\ref{th:local_max_to_ac} is not the appropriate
tool; Proposition~\ref{prop:filter} is.

\medskip\noindent We can now return to the misalignment mentioned earlier. Consider a
reasoning mode that looks for causes of an outcome by inspecting
the integrated probability $P(\mathbf{C}\!\rightsquigarrow\!\varphi):=
\int\sum_{\mathbf{T}}\Phi_{\mathbf{u}}(\mathbf{C},\mathbf{T})\,P_{\mathbf{U}}(d\mathbf{u})$, without knowing
the values of $\mathbf{u}$ but knowing that $\varphi$ occurred. As the next observation
shows,
$\arg\max_{\mathbf{C}\subseteq\mathbf{S}} P(\mathbf{C}\!\rightsquigarrow\!\varphi)$
can fail to include some $\mathbf{C}$ that is an actual cause of $\varphi$ for some
setting $\mathbf{u}$.

\begin{observation}\label{obs:misalignment}
Define $\Xi^{\max}=\arg\max_{\mathbf{C}\subseteq\mathbf{S}} P(\mathbf{C}\!\rightsquigarrow\!\varphi)$.
It is not the case that if $\mathbf{C}^\dagger$ is an actual cause in $(M,\mathbf{u})$
for some $\mathbf{u}$, $\mathbf{C}^\dagger\in\Xi^{\max}$.
\end{observation}

\begin{proof}
Consider a model with two binary parents $A,B$ uniformly distributed and a binary
child $Y=A\vee B$. Take $\varphi:\,Y=1$ and the suspect set $\mathbf{S}=\{A,B\}$ with
$\mathbf{W}=\emptyset$, and take $\Delta$ to be the deterministic bit-flip
distribution on $\{0,1\}$ (the single-point alternative used throughout the
binary examples of this paper, e.g.\ Algorithm~\ref{alg:pci_thin_search}): for
active suspect set $\mathbf{C}$, $\Delta_{\mathbf{C}}(\mathbf{s}^\star)$ places
all its mass on the coordinatewise complement of $\mathbf{s}^\star\!\mid_{\mathbf{C}}$.
Compare $\mathbf{C}^\dagger=\{A,B\}$ at noise $\mathbf{u}_{11}$
($A=B=1$) and $\mathbf{C}^\sharp=\{A\}$ at noise $\mathbf{u}_{10}$ ($A=1,B=0$).
\begin{itemize}
\item In $\mathbf{u}_{11}$, $\mathbf{C}^\dagger$ is an actual cause: setting both
parents to $0$ yields $Y=0$, no proper subset achieves this, and the indicators in
$\Phi$ are $1$. With $\Gamma_w$ uniform on $2^{\emptyset}$ (a single configuration of
weight $1$) and a strictly cardinality-decreasing $\Gamma_s$,
$\Phi_{\mathbf{u}_{11}}(\mathbf{C}^\dagger,\emptyset)\propto p^{\Gamma_s}(\{A,B\})$.
At the other two noise settings with $Y=1$, the bit-flip alternative for
$\mathbf{C}^\dagger$ does not flip $Y$: at $\mathbf{u}_{10}$ it sends $(A,B)$ to
$(0,1)$, and at $\mathbf{u}_{01}$ to $(1,0)$, both of which keep $Y=1$. So only
$\mathbf{u}_{11}$ contributes to $\mathbf{C}^\dagger$.
\item In $\mathbf{u}_{10}$, $\mathbf{C}^\sharp$ is an actual cause:
intervening to $A=0$ (with $B$ left at its factual value $0$, since $B\notin\mathbf{C}$) flips $Y$ to $0$, the singleton is minimal, and again the
indicators are $1$. So $\Phi_{\mathbf{u}_{10}}(\mathbf{C}^\sharp,\emptyset)\propto p^{\Gamma_s}(\{A\})$.
At $\mathbf{u}_{11}$, by contrast, the bit-flip alternative for $\mathbf{C}^\sharp$
sends $A$ to $0$ while $B$ remains at its factual value $1$ (again $B\notin\mathbf{C}$
is not intervened), so $Y=0\vee 1=1$ does not flip: the necessity integral, and
hence $\Phi_{\mathbf{u}_{11}}(\mathbf{C}^\sharp,\emptyset)$, is $0$. At
$\mathbf{u}_{01}$ the same computation with $A$ flipped to $1$ also leaves $Y=1$.
So only $\mathbf{u}_{10}$ contributes to $\mathbf{C}^\sharp$.
\end{itemize}
Integrating uniformly over the four noise realisations $\{\mathbf{u}_{ij}\}$,
$P(\mathbf{C}^\dagger\!\rightsquigarrow\!\varphi)\propto p^{\Gamma_s}(\{A,B\})$ and
$P(\mathbf{C}^\sharp\!\rightsquigarrow\!\varphi)\propto p^{\Gamma_s}(\{A\})$, each
with exactly one contributing noise realisation. Since $\Gamma_s$ is
cardinality-decreasing, $p^{\Gamma_s}(\{A\})>p^{\Gamma_s}(\{A,B\})$ for every such
$\Gamma_s$, so $\mathbf{C}^\sharp$ strictly outscores $\mathbf{C}^\dagger$
regardless of the precise shape of $\Gamma_s$: $\mathbf{C}^\dagger\notin\Xi^{\max}$
even though it is an actual cause at $\mathbf{u}_{11}$, so some actual causes are
dropped. (Under a different $\Delta$ that spreads positive density over every
joint alternative rather than concentrating it on the bit flip, both $A$'s and
$B$'s complements would separately admit nonzero necessity mass at
$\mathbf{u}_{10}$ and $\mathbf{u}_{01}$, changing the counts above; the
conclusion that some actual cause is excluded from $\Xi^{\max}$ is not
sensitive to this choice, but the specific proportionality constants are.)
\end{proof}

The reason is structural: actual cause sets may be subset-minimal without being
cardinality-minimal, while integrated-probability maximisation is governed by
$\Gamma_s$, which (under the cardinality-decreasing parameterisation that aligns with
AC's minimality) penalises larger sets. The misalignment is therefore a mismatch between AC's notion of minimality (subset
relation) and $\Gamma$-weighted scoring.

The fix is to drop the maximisation step and use $\Phi$ directly to enumerate
candidate causes, then filter for subset-minimality. The filter recovers exactly the
actual causes in $\mathbf{S}$.

\begin{proposition}\label{prop:filter}
Under the hypotheses of Theorem~\ref{th:ac-exp}, define
\[
\Xi \;:=\; \Big\{\mathbf{C}^\dagger\subseteq\mathbf{S}\;:\;
\sum_{\mathbf{T}\subseteq\mathbf{W}} \Phi(\mathbf{C}^\dagger,\mathbf{T})>0,\;
\nexists\,\mathbf{C}'\subsetneq\mathbf{C}^\dagger\text{ with }
\sum_{\mathbf{T}\subseteq\mathbf{W}} \Phi(\mathbf{C}',\mathbf{T})>0
\Big\}.
\]
Then $\Xi$ contains exactly the actual causes of $\varphi$ in $(M,\mathbf{u})$ that are
subsets of $\mathbf{S}$.
\end{proposition}

\begin{proof}
\emph{$\Xi$-membership $\Rightarrow$ actual cause.} Let $\mathbf{C}^\dagger\in\Xi$.
Theorem~\ref{th:exp-ac} applied to $\sum_{\mathbf{T}}\Phi(\mathbf{C}^\dagger,\mathbf{T})>0$
gives both factivity and context-sensitive necessity. For minimality, suppose
$\mathbf{C}_{\text{sub}}\subsetneq\mathbf{C}^\dagger$ satisfied factivity and
context-sensitive necessity. By the proof of Theorem~\ref{th:ac-exp} (which uses
only factivity and necessity, not minimality), we would have
$\sum_{\mathbf{T}}\Phi(\mathbf{C}_{\text{sub}},\mathbf{T})>0$, contradicting the
no-positive-proper-subset clause of $\Xi$.

\emph{Actual cause $\Rightarrow$ $\Xi$-membership.} Let $\mathbf{C}^\dagger\subseteq\mathbf{S}$
be an actual cause of $\varphi$ in $(M,\mathbf{u})$. Theorem~\ref{th:ac-exp} gives
$\sum_{\mathbf{T}}\Phi(\mathbf{C}^\dagger,\mathbf{T})>0$. For the no-positive-proper-subset
clause: suppose $\mathbf{C}'\subsetneq\mathbf{C}^\dagger$ had
$\sum_{\mathbf{T}}\Phi(\mathbf{C}',\mathbf{T})>0$. Then by Theorem~\ref{th:exp-ac},
both factivity and context-sensitive necessity would hold for $\mathbf{C}'$. So
$\mathbf{C}'$ would witness a violation of $\mathbf{C}^\dagger$'s minimality,
contradicting the assumption that $\mathbf{C}^\dagger$ is an actual cause.
\end{proof}

\noindent Proposition~\ref{prop:filter} converts the cardinality--subset misalignment
into a procedure: enumerate $\{\mathbf{C}\subseteq\mathbf{S}:\sum_{\mathbf{T}}\Phi(\mathbf{C},\mathbf{T})>0\}$
(which by Theorem~\ref{th:exp-ac} are exactly the cause sets satisfying
context-sensitive necessity), filter for factivity and subset-minimality, and read off
the actual causes. Certain features of the AC notion (in particular, the
subset-minimality clause, which does not align cleanly with $\Gamma$-weighted scoring) are
part of what motivated the more scalable \ourapproach{} formulation of
Section~\ref{sec:causal_impact} in the first place; the filter procedure here is the
right tool when a user wants AC verdicts specifically, but expectation-based
\ourapproach{} attributions are what scale to ML-grade models.

\medskip\noindent\textit{Summary of guarantees.} The section delivers, under explicit
hypotheses on $\mathbf{S}$, $\Gamma$, and $\Delta$:
(i) a forward direction (Theorem~\ref{th:ac-exp})
``actual cause $\Rightarrow\Phi>0$'';
(ii) a backward direction (Theorem~\ref{th:exp-ac})
``$\sum_{\mathbf{T}'}\Phi>0\Rightarrow$ context-sensitive necessity'';
(iii) minimality preservation under cardinality-decreasing $\Gamma_s$
(Theorem~\ref{th:local_max_to_ac}, Corollary~\ref{cor:marginalize});
(iv) a counterexample (Observation~\ref{obs:misalignment}) showing that integration
over noise misaligns with subset-minimality; and
(v) a corrective filter procedure (Proposition~\ref{prop:filter}) that recovers AC
verdicts exactly. The empirical counterpart of the correspondence, a benchmark of
the necessity-only PCI estimator against exact subset enumeration on a scaled
stone-throwing problem, is reported in \cref{sec:ac_benchmark}.

The next
section locates \ourapproach against the construction in the literature closest to
its expectation form: Pearl's \emph{probability of actual causation}. The two share
the same surface idea, ``how often, under the noise, does the actual-cause
verdict fire?'', and they coincide on standard examples, but they disagree on
graded responsibility patterns that PCI registers and Pearl's binary AC verdict
discards. Section~\ref{sub:pearl_actual_cause_prob} walks through Pearl's
canonical desert-traveller illustration and a weak-poison variant to make the
distinction concrete.

\section{Relation with the Probability of Actual Causality}
\label{sub:pearl_actual_cause_prob}

We now discuss the relation between PCI and
Pearl's probability of actual causation. On 
Pearl's Definition~10.3.5 \citep[Ch.~10]{causalityPearl},
\begin{equation}\label{eq:pearl-paac}
P(\text{caused}(x,y\mid e)) \;=\; \frac{P(U_{xy}\cap U_e)}{P(U_e)},
\qquad U_{xy}=\{\mathbf{u}\colon \mathrm{AC}(x,y;M,\mathbf{u})\},
\end{equation}
is the posterior probability, given the evidence, that we are in a noise state
for which the actual-cause verdict holds, with $\mathrm{AC}(x,y;M,\mathbf{u})$
a chosen actual-causation predicate: the causal-beam definition of
\citep[Ch.~10]{causalityPearl} or the Halpern--Pearl AC1--AC3
\citep{actualCausalityHalpern}.\footnote{\label{fn:desert_notebook}All
computations behind the comparison in this section, Pearl's original
desert-traveller and the weak-poison variant, are collected in the companion
notebook \nb{desert\_traveler}, which runs the PCI
Monte Carlo search of Section~\ref{sec:causal_impact} on a scenario-specific
structural model for each forensic posterior alongside a parallel hand-rolled
enumeration; the two agree within Monte Carlo noise on every value reported
below.}
The same construction underwrites Chockler and Halpern's degree of blame
\citep{chocklerResponsibilityBlameStructuralModel2004}, which substitutes a
graded responsibility share inside the expectation; and, more directly,
Halpern and Kleiman-Weiner's \emph{degree of blameworthiness}
\citep{halpernKleimanWeinerBlameworthiness2018}, which takes the expectation
of an actual-causation quantity over an epistemic state (a probability
distribution over causal models). \ourapproach{}'s expectation over noise
states in Definition~\ref{def:jointnecsuf} has exactly this shape. In
multi-agent settings, Friedenberg and Halpern
\citep{friedenbergHalpernBlameworthiness2019} apportion group blameworthiness
to individual agents via the Shapley value (see also
\citealp{alechinaTeamPlans2020}, who compute causality, responsibility, and
blame for multi-agent team plans); this is a \emph{causal} use of the
same game-theoretic averaging that, applied associationally to a predictive
model, underlies SHAP (though we raise our own concerns about SHAP-style
attribution in Section~\ref{sec:shap_examples}). Because
\eqref{eq:pearl-paac} is the closest construction in the literature to
\ourapproach{}'s expectation, we examine in detail how it differs from PCI
and where PCI improves on it.

We replay Pearl's desert-traveller illustration and his binary actual-cause
verdicts, then compute PCI on the same example and show it recovers Pearl's
ranking with graded magnitudes. We then introduce a weak-poison variant that
exhibits three intuitive responsibility patterns: two that Pearl's binary
verdict cannot capture but PCI's graded reading does, and one (within-scenario
ranking) that Pearl preserves. We close with other differences between PCI
and Pearl's probability of causation that make the former more general and
more scalable.

\subsection{Pearl's Worked Example: The Desert Traveller}

We start with the canonical illustration used by Pearl
\citep[\S 10.3.3]{causalityPearl}.

\begin{example}[Desert traveller]\label{ex:desert_traveler}
A traveller has two enemies. The \emph{shooter} (enemy 2) shoots and empties
the canteen ($X{=}1$); the \emph{poisoner} (enemy 1), unaware, puts cyanide in
the water ($P{=}1$). The traveller dies ($Y{=}1$). Who is the actual cause of
death?
\end{example}

The example involves one source of uncertainty packed into a binary noise
$u$:
\begin{itemize}
    \item $u=0$: traveller drank the poisoned water before the canteen was
    emptied; cyanide killed him.
    \item $u=1$: canteen was empty before he could drink; dehydration
    killed him.
\end{itemize}

\begin{figure}[htbp]
\centering
\includegraphics[width=0.62\linewidth]{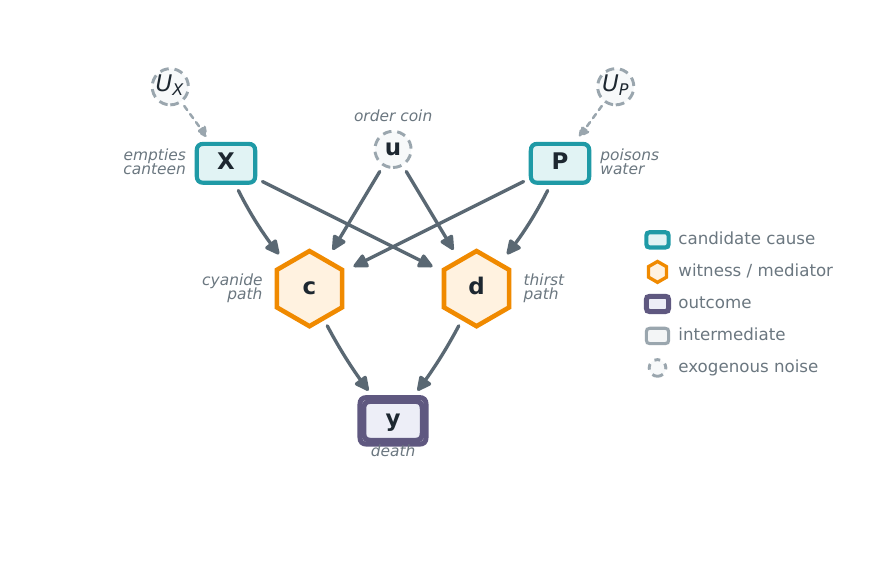}
\caption[Desert-traveller graph]{Pearl's desert traveller as a graph. The two
candidate causes (teal: $X$, who empties the canteen, and $P$, who poisons the
water) feed two competing mechanisms (gold hexagons), the cyanide path $c$ and
the thirst path $d$, whose disjunction is death $y$ (purple double box). The
exogenous noise comprises the order coin $u$ (a named exogenous variable that
enters $c$ and $d$ directly) and the coins $U_X,U_P$ realising the two enemies'
actions; the mechanisms $c,d,y$ are deterministic and carry none. It is $u$,
deciding which path operates first, that creates the preemption making this case
hard for theories of actual causation.}
\label{fig:dag_desert}
\end{figure}

There are at least two intuitions here that we would like to recover. One, in any particular case, we want
to be able to identify the actual cause. The other, slightly weaker, is that we should not call
the dormant cause completely blameless. Actual causality gives us the first half of the 
story, assigning full responsibility to one source here. Pearl's beam analysis \citep[Ch.~10]{causalityPearl} (equivalently
Halpern--Pearl AC1--AC3 of \citealp{actualCausalityHalpern}) gives the
actual-cause verdicts at each deterministic state:
\begin{itemize}
    \item In $u=0$: $P{=}1$ is the actual cause; $X{=}1$ is not.
    \item In $u=1$: $X{=}1$ is the actual cause; $P{=}1$ is not.
\end{itemize}
So $U_{x,y} = \{u=1\}$ and $U_{p,y} = \{u=0\}$. Two natural forensic reports
each push Pearl's posterior to certainty (Table~\ref{tab:pearl-pac-basic}):
toxicology finding cyanide in the body is incompatible with $u=1$, so
$P(\text{caused poisoner}\mid e_A) = 1$; a forensic report of no cyanide is
incompatible with $u=0$, so $P(\text{caused shooter}\mid e_B) = 1$.

\begin{table}[ht]
\centering
\begin{tabular}{lcc}
\toprule
Forensic posterior        & $P(\text{caused shooter})$ & $P(\text{caused poisoner})$ \\
\midrule
$e_A$: cyanide detected ($u=0$)  & $0$ & $\mathbf{1}$ \\
$e_B$: no cyanide ($u=1$)         & $\mathbf{1}$ & $0$ \\
\bottomrule
\end{tabular}
\caption{Pearl's probability of actual causality
$P(\text{caused}\mid e)$ on the basic desert traveller. The verdict saturates
at $1$ for the cause whose path operated in the forensically identified
scenario, and at $0$ for the dormant one.}
\label{tab:pearl-pac-basic}
\end{table}

\subsection{PCI on the Basic Desert Traveller}
\label{sub:pci-basic-dt}

To compare directly with Pearl, we run PCI under the same forensic
conditioning: one structural model per scenario (endogenous variables
$X, P, c, d, Y \in \{0,1\}$, with $c$ the cyanide path operating and $d$ the
dehydration path), with $u$ pinned at the scenario's forensic realisation.
Suspects are $\mathbf{S} = \{X, P\}$ with singleton candidate causes; the
witness pool is the mediators $\{c, d\}$ under a cardinality-uniform $\Gamma$;
$\Delta$ deterministically flips $1 \to 0$; PCI leaves the non-cause suspect
to the SCM rather than intervening on it; and the noise distribution
$P_{\mathbf{U}}$ is the forensic posterior, a point mass at the scenario's
$u$ (we do \emph{not} condition on the outcome $Y$ itself). Appendix~\ref{app:desert}
gives the full structural model and component specification.

For each candidate cause $C$ we report three quantities, all in $[0,1]$,\footnote{We evaluate these
for a fixed candidate cause $C$, hence normalise them to $[0,1]$; this differs
from the sub-probability $\mathbb{E}[ci]$ of Definition~\ref{def:causalimpact}, whose
total mass is $\Pr_\Gamma[X_k\in\mathbf{C}]\le 1$ because it sums only over suspect sets
that contain $X_k$.}
higher meaning more responsibility, jointly averaged over the four sources
of randomness ($\Gamma$, $\Delta$, $P_{\mathbf{U}}$, and the SCM-driven
non-cause suspect):
\begin{equation}\label{eq:nsj}
\begin{aligned}
N_C &= \mathbb{E}\bigl[\,\mathbbm{1}\{Y^n \neq y^\star\}\,\bigr]
       &&\text{(necessity: removal flips $Y$)},\\
S_C &= \mathbb{E}\bigl[\,\mathbbm{1}\{Y^s = y^\star\}\,\bigr]
       &&\text{(sufficiency: holding at factual sustains $Y$)},\\
J_C &= \mathbb{E}\bigl[\,\mathbbm{1}\{Y^n \neq y^\star\}\,\mathbbm{1}\{Y^s = y^\star\}\,\bigr]
       &&\text{(joint PNS-style)}.
\end{aligned}
\end{equation}
Thus $J_C$ is the PCI score $\mathrm{PCI}_C$ of Definition~\ref{def:causalimpact} under
the PNS-binary $ci$, and $N_C$, $S_C$ are its necessity and sufficiency marginals.

Under the spec above (computation in the companion notebook), the results are:

\begin{table}[ht]
\centering
\setlength{\tabcolsep}{4pt}
\begin{tabular}{llcccc}
\toprule
Forensic posterior & Cause & $N$ & $S$ & $J$ & Pearl $P(\text{caused})$ \\
\midrule
$e_A$: $u=0$ (cyanide) & shooter   & $1/4$ & $11/12$ & $5/24$ & $0$ \\
                       & poisoner  & $1/3$ & $1$     & $\mathbf{1/3}$ & $\mathbf{1}$ \\
\addlinespace
$e_B$: $u=1$ (dehydration) & shooter   & $1/3$ & $1$     & $\mathbf{1/3}$ & $\mathbf{1}$ \\
                                                      & poisoner  & $1/4$ & $11/12$ & $5/24$ & $0$ \\
\bottomrule
\end{tabular}
\caption{PCI and Pearl on the basic desert traveller, under the scenario's
forensic posterior. PCI's joint score $J$ ranks the same cause Pearl's
$P(\text{caused})$ does in each scenario; PCI returns a graded reading where
Pearl returns the binary $\{0, 1\}$. The cells use the diagonal coupling of
Definition~\ref{def:jointnecsuf} (Figure~\ref{fig:pci-mechanism}): $Y^s$ and
$Y^n$ share both the noise realisation $u$ and the non-cause suspect's value,
and are otherwise independent.}
\label{tab:pci-basic-dt}
\end{table}

\begin{figure}[htbp]
\centering
\includegraphics[width=0.92\linewidth]{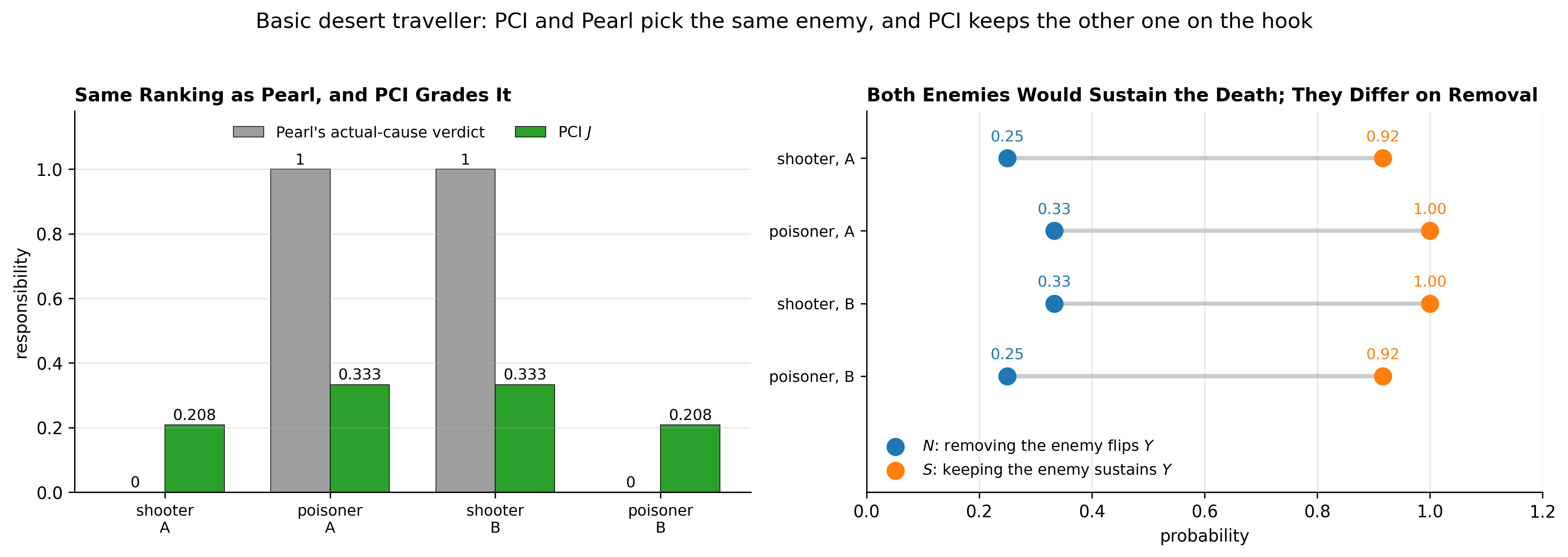}
\caption{Basic desert traveller. Left: Pearl's binary actual-cause verdict
beside \ourapproach's joint score $J$, per scenario and cause. Right: the
necessity and sufficiency components behind each $J$. Both methods name the
same enemy in each scenario; \ourapproach also keeps the dormant enemy on the
hook at $5/24$.}
\label{fig:dt-basic}
\end{figure}

Reading off Table~\ref{tab:pci-basic-dt}:

\textbf{(1) PCI and Pearl agree on the ranking.} In each scenario, $J$ is
strictly larger for the cause Pearl identifies, and strictly smaller for the
other. PCI's ordering matches the actual-cause verdict at each forensic
posterior.

\textbf{(2) PCI is graded; Pearl is binary.} Pearl's $P(\text{caused})$
saturates at $0$ or $1$. PCI's $J$ never reaches $1$ here because the
$\Gamma$-averaging over witness subsets dilutes each cause's score: the
``best'' witness subset for the operative cause does flip the outcome
deterministically and gives $J = 1$ at that single configuration, but the
cardinality-uniform average over the witness subsets (weights
$\tfrac13,\tfrac16,\tfrac16,\tfrac13$ over $\emptyset,\{c\},\{d\},\{c,d\}$)
brings the average down to $1/3$; the per-subset breakdown is in the companion notebook.

\textbf{(3) The dormant cause is not at zero.} PCI assigns $5/24 \approx 0.21$ to the dormant cause. This is not noise: in some witness configurations
the dormant cause's removal does flip $Y$, because averaging marginalises the
non-cause suspect and PCI does not artificially pin the model to ``this
scenario'' beyond the noise realisation itself. Pearl's binary indicator
throws this graded information away.

\subsection{PCI Expectation vs the AC Machinery}

The agreement is reassuring; the substantive change is in how the verdict is
computed. AC's necessity clause is an \emph{existential}:
\emph{there exists} a witness set $\mathbf{T}$ and an alternative value $c'$
such that the intervention flips $Y$. To verify the existential one in
principle enumerates the candidate $(\mathbf{T}, c')$ pairs: this is the
main source of AC's combinatorial hardness. Pearl's definition then averages
the indicator of that already-enumerated predicate over noise states,
inheriting the enumeration unchanged.

\ourapproach{} collapses both moves into one expectation, averaging $\Phi$
over the joint $(\mathbf{C}, \mathbf{T}, \mathbf{u})$ under $\Gamma \otimes
P_{\mathbf{U}}$. The existential ``does \emph{some} witness work?'' becomes a $\Gamma$-weighted measure of how much witnesses collectively support causation: a graded score that a user estimates by sampling, built from the full quantity that AC collapses to a yes/no verdict.

In the desert traveller this difference is minor: a user can trivially
enumerate the four candidate witness sets by hand. The large-$n$ gap is an
analytic fact about AC's existential check, not something this example
demonstrates: for a model with $n$ candidate witness variables, that
existential requires checking up to $2^n$ subsets per noise state by
construction, and Pearl's posterior over the AC indicator inherits that cost
at every evaluation. A user estimates PCI's expectation by Monte Carlo at a
budget they set, with variance independent of that combinatorial size.
\cref{sec:ac_benchmark} reports this gap directly as $n$ grows; the desert
traveller only fixes the small-$n$ case, where the cost difference is
negligible.

So far we have two conceptual differences: PCI's graded reading and its Monte Carlo scaling.
We have already seen that this helps recover the intuition introduced above: the shooter, in the scenario where poison killed the traveller, is less blameworthy than the poisoner but not
completely blameless. Now we turn to another example, exploiting the fact that actual causality
lacks a sufficiency component, where a small change to the model makes Pearl's binary indicator
structurally unable to track the responsibility patterns we think a responsibility attribution
method should capture.

\subsection{Weak-Poison Variant: Path-Reliability and Cross-Scenario Reliability}
\label{sub:weak-poison}

We modify the desert traveller so the cyanide path is unreliable while the
dehydration path stays deterministic. The change is small, but it brings out three
 intuitions that a graded responsibility metric
should capture and a binary actual-cause indicator structurally cannot.
Section~\ref{sec:other-prob-ac} shows this graded sensitivity is structurally
distinct from the probability-raising accounts of Fenton-Glynn and Fang, which
lack \ourapproach's continuous kernel and sufficiency component.\footnote{That
human responsibility judgments are in fact graded, and sensitive to the structure by
which individual contributions combine into a group outcome, has empirical support:
\citet{gerstenbergSpreadingBlame2010} found that across experimental conditions the
structural-model account of \citet{chocklerResponsibilityBlameStructuralModel2004}
predicted participants' responsibility allocations better than competing cognitive
models.}

\begin{example}[Desert traveller, weak poison]\label{ex:desert_traveler_weak_poison}
Same enemies as in Example~\ref{ex:desert_traveler}. The cyanide is a small
dose that turns out fatal only with probability $\alpha = 0.1$, governed by
an independent exogenous coin $\xi \sim \mathrm{Bern}(\alpha)$. The structural
equations are
\[
c = p\,(u' \vee x'), \quad v_C = c\cdot \xi, \quad d = x\,(u \vee p'),
\quad y = v_C \vee d,
\]
where primes here denote Boolean complement ($x'=1-x$, and likewise $u',p'$), not
the alternative-value $\mathbf{c}'$ of the necessity intervention, and $v_C$ is the
new ``cyanide-was-fatal'' mediator: drinking the poison
($c=1$) only kills when the fatality coin also operates ($\xi=1$).
$\alpha = 1$ recovers Pearl's original.
\end{example}

\paragraph{Three responsibility intuitions for an attribution method.}

Before any computation, the three intuitions that we expect an attribution method to recover in this case are:

\begin{enumerate}\setlength\itemsep{0.2em}
    \item\textbf{Adding noise to a path should lower its own cause's
        responsibility.} The cyanide path is now unreliable. In the
        scenario where cyanide killed (Scenario~A, $u=0, \xi=1$), the
        poisoner is still the actual cause, but \emph{less} responsible
        than in the basic deterministic version, because the killing
        depended on a coin landing favourably.
    \item\textbf{Within a scenario, the operative cause should still win.}
        In Scenario~A, even after softening, the poisoner should rank above
        the shooter. In Scenario~B, the shooter should rank above the poisoner.
    \item\textbf{A reliable cause in its own scenario should beat an
        unreliable cause in its own scenario.} Comparing across scenarios:
        the shooter in Scenario~B (dehydration deterministic) should outrank
        the poisoner in Scenario~A (cyanide stochastic). Their respective
        paths are operative in their own scenarios, but one is reliable and
        the other is not.
\end{enumerate}

Pearl's $P(\text{caused})$ captures (2) (the within-scenario ranking),
but misses (1) and (3). The AC predicate
fires deterministically once we condition on the forensic posterior: in
Scenario~A, $u=0$ and $\xi=1$ jointly imply the cyanide path was the operative
chain, so $P(\text{caused poisoner}\mid e_A) = 1$; in Scenario~B, $u=1$
implies the dehydration chain operated, so $P(\text{caused shooter}\mid e_B)
= 1$. Both scenarios saturate at $1$ for their respective operative cause,
identical to the basic version (Table~\ref{tab:pearl-pac-weak}). Pearl's
binary indicator is structurally indifferent to whether the cyanide path is
deterministic or stochastic, and blind to any cross-scenario comparison.

\begin{table}[ht]
\centering
\begin{tabular}{lcc}
\toprule
Forensic posterior & $P(\text{caused shooter})$ & $P(\text{caused poisoner})$ \\
\midrule
$e_A$: cyanide killed ($u=0, \xi=1$) & $0$ & $\mathbf{1}$ \\
$e_B$: dehydration killed ($u=1$)     & $\mathbf{1}$ & $0$ \\
\bottomrule
\end{tabular}
\caption{Pearl's $P(\text{caused})$ on the weak-poison variant. Identical to
the basic version (Table~\ref{tab:pearl-pac-basic}); the binary AC indicator
is insensitive to $\alpha$.}
\label{tab:pearl-pac-weak}
\end{table}

\paragraph{PCI on the weak-poison model.}

We run PCI exactly as in Section~\ref{sub:pci-basic-dt}, with one mechanical
change: the witness pool becomes $\mathbf{W} = \{c, d, v_C\}$ (the new mediator
$v_C$ is included), and the forensic posterior in Scenario~A pins both
$u = 0$ and $\xi = 1$ (cyanide-was-fatal is directly forensically
observable). Scenario~B's posterior pins $u = 1$; $\xi$ is structurally
irrelevant since the cyanide path is dormant.

Table~\ref{tab:pci-weak-dt} reports PCI's three quantities alongside Pearl's
verdict. The companion notebook computes the values, and a hand-rolled
enumeration matching the framework verifies them.

\begin{table}[ht]
\centering
\setlength{\tabcolsep}{4pt}
\begin{tabular}{llcccc}
\toprule
Forensic posterior & Cause & $N$ & $S$ & $J$ & Pearl $P(\text{caused})$ \\
\midrule
$e_A$: $u=0, \xi=1$ (cyanide killed) & shooter   & $1/6$ & $23/24$ & $7/48$ & $0$ \\
                                      & poisoner  & $5/24$ & $1$             & $\mathbf{5/24}$ & $\mathbf{1}$ \\
\addlinespace
$e_B$: $u=1$ (dehydration killed)    & shooter   & $1/2$           & $1$             & $\mathbf{1/2}$  & $\mathbf{1}$ \\
                                                                            & poisoner  & $1/4$           & $3/4$ & $1/8$ & $0$ \\
\bottomrule
\end{tabular}
\caption{PCI and Pearl on the weak-poison desert traveller, under the
scenario's forensic posterior. PCI's $J$ ranks the same cause Pearl's
binary verdict does in each scenario; the magnitudes also reflect the
asymmetry between the deterministic dehydration path and the stochastic
cyanide path. As in Table~\ref{tab:pci-basic-dt}, cells use the diagonal
coupling of Definition~\ref{def:jointnecsuf}: $Y^s$ and $Y^n$ share both the
noise realisation and the non-cause suspect's value.}
\label{tab:pci-weak-dt}
\end{table}

\begin{figure}[htbp]
\centering
\includegraphics[width=0.92\linewidth]{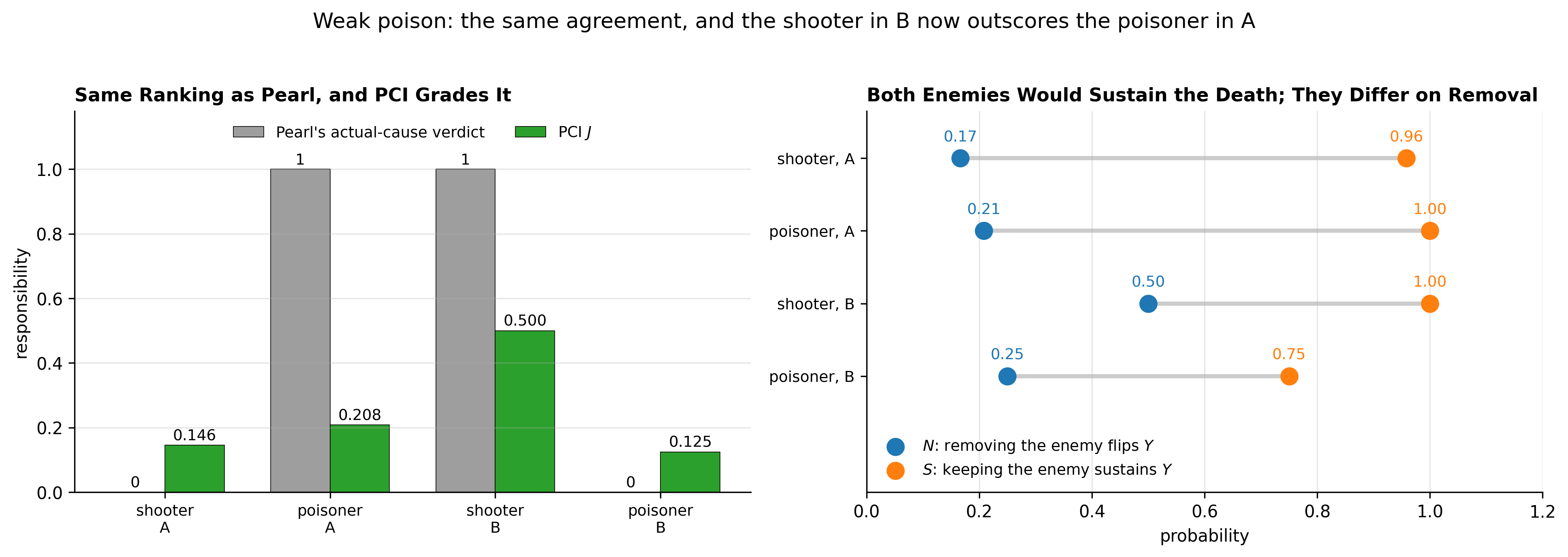}
\caption{Weak-poison variant, laid out as Figure~\ref{fig:dt-basic}. Pearl
scores both operative enemies at $1$ and cannot separate the scenarios;
\ourapproach separates them, ranking the shooter in Scenario~B ($1/2$) above
the poisoner in Scenario~A ($5/24$).}
\label{fig:dt-weak}
\end{figure}

The three intuitions are all captured by $J$.

\textbf{(1) Adding noise lowers the noisy cause's responsibility.} The
poisoner in Scenario~A drops from $J = 1/3$ in the basic version
(Table~\ref{tab:pci-basic-dt}) to $J = 5/24$ in the weak version
(Table~\ref{tab:pci-weak-dt}). The drop is a $\Gamma$ re-weighting,
not a loss of sufficiency: the poisoner's sufficiency factor is exactly $1$ in
both versions. Adding the fatality coin adds $v_C$ to the witness pool, so the
eight subsets of $\{c, v_C, d\}$ replace the four subsets of $\{c, d\}$, and the
single configuration that scores $1$ for the poisoner now carries a twelfth of
the averaging weight where it carried a sixth. The empty witness set still
contributes $1/2$ at weight $1/4$, and $\tfrac14\cdot\tfrac12 + \tfrac1{12}\cdot 1
= 5/24$.

\textbf{(2) Within a scenario, the operative cause still wins.} In
Scenario~A, $J_{\text{poisoner}} = 5/24 > J_{\text{shooter}} = 7/48$. Cyanide
was the operative path, and PCI tracks that. In Scenario~B the margin is
fourfold, $J_{\text{shooter}} = 1/2$ against $J_{\text{poisoner}} = 1/8$. At
$u=1$ with $\xi=0$ the cyanide path is dead, so nothing the poisoner did could
have killed this traveller, and his $1/8$ records no responsibility for the
death. It is the residue of marginalising over the shooter's action, which
credits the poisoner in those drawn worlds where the shooter did not act
either. PCI grades him as dormant: above a variable with no path to the
outcome, and well below the cause that did the work.

\textbf{(3) Reliable beats unreliable, across scenarios.} The shooter in
Scenario~B ($J = 1/2$) outranks the poisoner in Scenario~A ($J = 5/24$). Neither ceiling
differs: both operative causes reach $J = 1$ at their best witness
configuration, so the coin does not lower what the poisoner can attain. What
differs is how many configurations attain it. Conditional on $u=1$ the cyanide
path is dead, so removing the shooter leaves nothing that can kill and four of
the eight witness subsets score a full $1$. Conditional on $u=0,\xi=1$ the
shooter is still standing with a working dehydration path, so removing the
poisoner lets the shooter cover the death half the time, and only one subset
reaches $1$.

Pearl's $P(\text{caused})$ misses intuitions (1) and (3) for related reasons:
(1) requires sensitivity to mechanism noise, which the AC indicator does not
have once the forensic posterior is applied; (3) requires cross-scenario
comparison of \emph{strength of responsibility}, which Pearl's
$P(\text{caused})$ does not produce because each posterior independently
saturates at $0$ or $1$.

\subsection{Further Differences Between Pearl's Probability of Causation and PCI}

\medskip\noindent\textbf{Breadth of aggregation.}
Def.~10.3.5 aggregates over one dimension: the noise $\mathbf{u}$, conditional on
evidence $e$, for a claim $(X=x, Y=y)$ already fixed. PCI aggregates over four: the
suspect set $\mathbf{C}$ and witness set $\mathbf{T}$ via $\Gamma$, the alternative
value $\mathbf{c}'$ via $\Delta$, and the noise via $P_{\mathbf{U}}$. This supports per-feature comparison: the same PCI expectation, evaluated over varying
$\mathbf{C}$ inside a fixed $\mathbf{S}$, returns the relative responsibility of each
candidate cause. Def.~10.3.5 does not itself provide that comparison as a single expectation:
obtaining it requires a separate evaluation per candidate cause, each inheriting
the AC-predicate cost discussed above, tabulated by hand as in
Table~\ref{tab:pearl-pac-basic} above.

\medskip\noindent\textbf{Continuous variables.}
The AC predicates inside Def.~10.3.5 are event-shaped: they read off
$\mathrm{AC}(x,y;M,\mathbf{u})$ for propositional $x, y$, and extensions to
continuous variables require discretising the relevant event before the predicate
applies. PCI's $ci$ function is continuous-native (the absolute-difference impact
score of Section~\ref{sec:causal_impact}), and the synthetic benchmark of
\cref{sub:synthetic} exercises that continuity directly.

\medskip\noindent\textbf{Lineage in Pearl's PN/PS.}
Def.~10.3.5 takes a binary actual-causation predicate (itself a
structural-equation construction that does not appeal to PN/PS at all) and lifts
it to a probability through $P(\mathbf{u}\mid e)$; $\mathrm{PN}$, $\mathrm{PS}$, and
$\mathrm{PNS}$ appear nowhere in the construction. PCI generalises
$\mathrm{PN}$ and $\mathrm{PS}$ directly: the necessity factor $Y^n$ is the
counterfactual operation underlying $\mathrm{PN}$, the sufficiency factor $Y^s$ is
the counterfactual operation underlying $\mathrm{PS}$, both indexed by the
cause--witness--alternative--noise tuple, with $\Gamma$ and $\Delta$ exposing the
design choices $\mathrm{PN}/\mathrm{PS}$ makes implicitly. So Pearl's Def.~10.3.5
and \ourapproach{} are not competing implementations of the same idea: they inherit
from different parts of Pearl's own toolkit, with \ourapproach{} playing the role of
a context-sensitive, multi-variable extension of the $\mathrm{PN}/\mathrm{PS}$
machinery rather than a probabilistic wrapper around AC.

\medskip
Pearl's Def.~10.3.5 and \ourapproach{} share the same explanatory target,
probabilistic responsibility built on top of structural information, but
instantiate it through different mechanisms. Pearl's definition averages a binary AC
indicator against a posterior over noise; PCI integrates a continuous
necessity--sufficiency kernel against a joint distribution over suspect sets,
witness sets, alternative values, and noise. The basic desert traveller shows the
two agree on ranking with PCI providing graded magnitudes; the weak-poison variant
shows three intuitions that PCI's graded reading captures and Pearl's binary
indicator structurally cannot. That actual causation admits of degrees is not itself
new: \citet{Halpern2015-HALGCA} already make it graded and context-sensitive through
defaults and normality. That account delivers an ordinal, comparative verdict,
whereas \ourapproach supplies a cardinal probabilistic magnitude on the same structural
information. The cases where sufficiency does even more graded
work (the continuous signal-mediation walkthrough of
Section~\ref{sec:signal}, the synthetic archetypes of
\cref{sub:synthetic}, and the SIR benchmark of
\cref{sec:sir_benchmark}) are where these constructions answer
recognisably different questions; for the high-dimensional, continuous,
machine-learned models this paper targets, PCI's is the question that remains
tractable and well-defined.

\subsection{Other probabilistic extensions of actual causation}
\label{sec:other-prob-ac}

Pearl's probability of causation is not the only probabilistic reading of the
Halpern--Pearl definition. \citet{Fenton-Glynn2017-FENAPP-2} proposes a probabilistic
extension that keeps HP's witness/contingency machinery intact but replaces the
deterministic counterfactual-dependence clause with a \emph{probability-raising}
condition: with the off-path variables $\mathbf{W}$ pinned at their actual values
$\mathbf{w}^\star$, the cause $\mathbf{C}=\mathbf{c}$ qualifies when intervening to set
it raises the probability of the outcome relative to an alternative $\mathbf{c}'$,
$P(\varphi \mid \mathrm{do}(\mathbf{C}=\mathbf{c}),\mathbf{W}=\mathbf{w}^\star) >
P(\varphi \mid \mathrm{do}(\mathbf{C}=\mathbf{c}'),\mathbf{W}=\mathbf{w}^\star)$,
required to hold for every subset of the on-path mediators held at their factual values.
\citet{fang2022probabilistic} modifies this in turn, replacing the interventional
probability-raising with a \emph{counterfactual} one (conditioning on the actual
situation rather than intervening afresh) which repairs the verdicts
Fenton-Glynn's account returns on voting, trumping, and overdetermination cases.

Both are close neighbours of \ourapproach, and the contrast is worth drawing precisely.
First, each
probabilifies only the \emph{necessity} side of HP: it raises or conditions a
probability of the witnessed-necessity test, with no analogue of \ourapproach's
\emph{sufficiency} component, the $Y^s$ factor of Definition~\ref{def:jointnecsuf},
which asks whether \emph{restoring} the factual cause sustains the outcome. Second, the
contingency enters differently. Fenton-Glynn fixes the witness $\mathbf{W}$
deterministically at its actual values and quantifies \emph{universally} over the
mediator subsets (picking the alternative $\mathbf{c}'$ existentially); \ourapproach
replaces those quantifiers with \emph{distributions} (a $\Gamma$-weighted expectation
over which suspects and witnesses to activate, and an alternative-value distribution
$\Delta$), both presented as user-facing knobs and approximated by Monte Carlo rather
than enumerated. Third, \ourapproach integrates a \emph{continuous} necessity--sufficiency kernel where those accounts threshold a binary actual-cause indicator.
\ourapproach is therefore not a third probabilistic-HP variant but a PN/PS-style
necessity-and-sufficiency measure that \emph{contains} the witnessed-necessity test they
probabilify as one regime: the actual-causality analogue of the contrast drawn just
above with Pearl.

These accounts have begun to see applied use. \citet{maldonado2025robot} instantiate
Fenton-Glynn's definition in a robot-pouring task, learning the causal graph by
discovery and the conditional distributions with neural networks, and use it to
identify the cause of spillage and to select a corrective action, searching by hand
over contrastive values and a probability threshold for an alternative that would avert
the outcome. That search is precisely the alternative-value reasoning \ourapproach
packages into $\Delta$ and a causal-impact threshold, and it inherits the differences
just drawn: it remains necessity-only, with no sufficiency check, and it fixes the
contingency rather than averaging over it. The application is nonetheless evidence that
tractable probabilistic actual causation is in demand in exactly the ML-driven settings
\ourapproach targets.

\paragraph{Necessity and sufficiency as an XAI method.}
Closer to the ML audience's likely prior exposure, \citet{watson2021necessity} adapt
Pearl's PN/PS directly into a local explanation method for black-box classifiers:
given a specified causal structure over input features, they define probabilistic
necessity and sufficiency scores for a feature (or feature subset) with respect to a
model's prediction, and use them to generate minimal sufficient explanations and
necessary-feature rankings, benchmarked against LIME, Anchors, and SHAP-style
attributions. The relation to \ourapproach is close in spirit to the relation with
Fenton-Glynn and Fang drawn above: both build directly on Pearl's PN/PS rather than
Halpern--Pearl's witnessed-necessity machinery, and neither packages the witness set
as a user-facing, $\Gamma$-sampled search dimension the way \ourapproach does; unlike
\ourapproach, \citeauthor{watson2021necessity}'s method targets model explanation
specifically (attributing a classifier's prediction) rather than causal responsibility
attribution for a realized outcome in a structural model, the distinction drawn in
Section~\ref{sec:causal_impact}'s discussion of SHAP and Causal SHAP.

\paragraph{Continuous and vector-valued PNS.}
\citet{kawakammi_2024_poc_for_cont_and_vec} extend Tian and Pearl's binary PNS, PN, and
PS to continuous and vector-valued treatments and outcomes, defining necessity and
sufficiency via an order relation on the outcome space ($Y_{x_0} < y \le Y_{x_1}$ for a
threshold $y$) rather than a fixed binary event. Under exogeneity together with a monotonicity
assumption, they identify PNS, PN, and PS nonparametrically from observational
distributions alone, with no need for a fitted causal model. They offer several
monotonicity variants and identify under one stated on the potential outcomes
themselves. That is a genuine extension to continuous
domains, but it inherits Tian and Pearl's population-level framing: their PNS asks how
often a random individual's outcome crosses an externally chosen threshold under two
treatment levels, not whether \emph{this} instance's already-realized outcome depended
on \emph{this} instance's cause, which is the single-case reading of PN/PS
\citep{pearl2022probabilities} that \ourapproach builds on. Like ATE and CATE, it also
has no analogue of the witness mechanism, so it cannot hold mediators fixed at their
factual values and register preemption. \Ourapproach reaches the same domains by a
different route, and the difference starts in the definition. Definition~3.1 fixes a
contrast pair $(x_0,x_1)$ and a threshold $y$, so a single inequality settles necessity
and sufficiency. \Ourapproach puts an alternative-value distribution $\Delta$ where the
fixed contrast was, and a user-chosen impact kernel $ci$ where the threshold was. One
construction then yields a continuous necessity--sufficiency pair rather than a
binarised event. Monotonicity drops out too: \ourapproach evaluates both counterfactual
worlds inside a fitted model, which supplies the cross-world quantities by simulation
and asks nothing of the observational distribution. The burden moves onto that model.
\citeauthor{kawakammi_2024_poc_for_cont_and_vec} buy identification with restrictions on
the true system; \ourapproach buys it by assuming the fitted model is that system, which
for flexible structural equations such as neural networks is the harder warrant
(\cref{sec:discussion}). The witness mechanism then supplies the context-sensitivity
their population-level framing forfeits.

\paragraph{FANS and LEWIS.}
Two recent attribution methods also build on Pearl's probabilities of necessity and sufficiency, bringing them closer still to \ourapproach. FANS \citep{chenFANS2024} explains a
prediction by searching a small region around the input for the largest PNS attainable there, estimating counterfactual probabilities by resampling. LEWIS \citep{galhotraLewis2021} works from a
causal graph and input--output data, computing necessity, sufficiency, and combined
scores for a black-box decision at the population or subgroup level, and turns them into recourse. Both are well suited to their stated aims and share our starting point
that attribution is a probability of causation. They differ from \ourapproach{} on four
design choices, each reflecting our narrower focus on token-level, context-sensitive
attribution when a full model is available.

First, we keep necessity and sufficiency apart. FANS combines them into a single PNS
number and LEWIS reports them as separate population averages, whereas \ourapproach{}
carries both as a joint pair $(Y^s, Y^n)$ and applies the user's impact function $ci$
only afterwards. As the synthetic models of \cref{sub:synthetic} illustrate, a cause can
be necessary without being sufficient, or the reverse, and keeping the two components
together lets a single sweep tell these archetypes apart.

Second, we average rather than optimise for a best case. FANS reports the maximum PNS over the local regions it searches
and uses gradient-based optimisation to extract the single feature subset with the highest
PNS; LEWIS scores individual attributes or a user-specified set. \ourapproach{} averages
instead: the selection distribution $\Gamma$ samples suspect sets within a chosen
cardinality band, so a feature's responsibility reflects its participation across many
sets rather than its most favourable one, with the Shapley kernel arising as a special
case.

Third, we use witnesses where FANS uses a local region and LEWIS a population average.
When two causes overdetermine an outcome, or one acts only through a mediator, recovering
the intended verdict means holding the inactive path fixed at its factual value, which a
ball around the input cannot express and a population average smooths over. \ourapproach{}
holds witness variables $\mathbf{W}$ fixed and averages over which variables serve as
witnesses, as the OBCB and signal walkthroughs of \cref{sec:shap_examples} illustrate.

Fourth, we compare against a distribution of alternatives.
FANS measures change from one baseline input and LEWIS from one contrast value;
\ourapproach{} averages over an alternative-value distribution $\Delta$, removing the
dependence on a chosen reference at the cost of assuming more of the model. That cost is a
deliberate trade: LEWIS recovers its scores from a graph and observational data at the
population level; \ourapproach{} assumes a fully specified model and returns a
value for the individual case.

\paragraph{Probabilistic versus nondeterministic models.}
\ourapproach lives in a probabilistic world: we push all the indeterminism into the
exogenous noise $\mathbf{U}$, so that once a noise value $\mathbf{u}$ is fixed the
mechanisms are ordinary functions, each necessity or sufficiency question becomes a
plain interventional expectation, and we average those over $P_{\mathbf{U}}$.
\citet{beckers2025nondeterministic} takes a different route out of determinism. His
nondeterministic causal models let a mechanism return a \emph{set} of admissible values
rather than a single one, with no probabilities attached, and he reads actual causation
off two modalities: a cause \emph{must} make a difference if it does so in every
admissible resolution, and \emph{might} make one if it does so in some. He notes in
passing that this brings the logical and probabilistic pictures of causation closer
together than they are in the deterministic case, since the order-of-operations
distinction that matters for Pearl's probabilities of counterfactuals resurfaces here
as a distinction between these two modalities.

It is tempting to read the two pictures as lining up, though we do not try to make this
precise. If one were to turn a nondeterministic mechanism into a probabilistic one, by
placing some full-support distribution over its admissible values, then \emph{might}
would seem to answer to ``the outcome changes with positive probability'' and
\emph{must} to ``the outcome changes with probability one.'' On that reading the modal
verdict would depend only on which alternatives are possible and not on how they are
weighted, so any full-support choice (the uniform one being the most agnostic)
would give the same yes-or-no answer, with a graded score like \ourapproach's filling
in the range between. Where genuine probabilities are not on hand, the modal account is
the natural tool and \ourapproach has nothing to add; where they are, \ourapproach
trades the bare list of possibilities for a plausibility-weighted magnitude. We mention
this only to locate the two frameworks relative to one another, not to claim a
correspondence.

\subsection{Actual-causality-adjacent quantities}
\label{sec:recovering-ac}

The actual-causality tradition has accumulated a family of named quantities beyond
causation itself. \citet{chocklerResponsibilityBlameStructuralModel2004} introduce
\emph{degree of responsibility} and \emph{degree of blame}; \citet{actualCausalityHalpern}
develops \emph{causal explanation} and \emph{explanatory power} at book length, and
\citet{beckersExplanationsXAI2022} recasts both for explainable AI;
\citet{beckersSufficiency2021} and \citet{beckersNESS2021} reconstruct causation through
\emph{causal sufficiency} and a counterfactual NESS condition; and
\citet{beckersHarm2022} define \emph{harm}, with \citet{beckersQuantifyingHarm2023} adding
its graded refinement and \citet{beckersMoralResponsibility2023} an allied account of
moral responsibility. Each comes with its own definitional apparatus (a
counterfactual test, a minimisation over contingencies, an averaging over an agent's
epistemic state, or a stipulated baseline), and the relations among them remain largely
implicit.

We conjecture that most of these are not independent primitives but
\emph{reparametrisations of the single counterfactual contrast} that \ourapproach already
makes explicit; or, where a control is missing, notions that fold into the same framework by an
appropriate generalisation, \emph{mutatis mutandis}. The construction of Section~\ref{sec:causal_impact} carries a small number of
user-facing knobs: the variable selection distribution $\Gamma$ (which and how many
suspects enter a candidate cause set $\mathbf{C}$), the witness set $\mathbf{T}$ (which
context a user holds fixed, in the contingency sense of \citealp{actualCausalityHalpern}), the
alternative value distribution $\Delta$ (what counts as a counterfactual departure,
equivalently the \emph{default} against which PCI grades alternatives), the causal impact
function $ci$ (how PCI scores a sufficiency--necessity pair,
Definition~\ref{def:causalimpact}), and the exogenous law $P_{\mathbf{U}}$ (what PCI
averages the attribution over). Read against these knobs, the named quantities look like
settings rather than rivals. \emph{Actual causation} in the sense of
\citet{actualCausalityHalpern} is the necessity component at a single-suspect $\Gamma$ with
witnesses admitted; the \emph{causal explanation} of \citet{actualCausalityHalpern} (a
context that, given what the agent already takes for granted, makes the cause necessary)
is recovered in intent by the held-fixed witness configuration $\mathbf{T}$ together with the
marginalisation over $P_{\mathbf{U}}$ that stands in for the agent's epistemic state; the
sufficient and counterfactual \emph{explanations} of \citet{beckersExplanationsXAI2022} are
organised around a held-fixed witness that coincides with \ourapproach's witness set
$\mathbf{T}$; both descend from Halpern's modified definition. His additional
\emph{safeguarded network} (variables protected from intervention, as opposed to held at their
factual values) is not among the present knobs, and would enter only as a natural
generalisation, a set safeguarded from intervention alongside the witnesses held at factual
values. Where he rates an
explanation only as \emph{good} or not (non-domination, i.e.\ minimality), a $ci$ on the
necessity--sufficiency differential of Definition~\ref{def:jointnecsuf} would supply a graded
refinement rather than recover an existing measure; the \emph{degree of blame} of
\citet{chocklerResponsibilityBlameStructuralModel2004} is the same
attribution marginalised over a prior on models alongside $P_{\mathbf{U}}$, with their
\emph{degree of responsibility} its single-model special case; the \emph{harm} of \citet{beckersHarm2022}
and \citet{beckersQuantifyingHarm2023} is \ourapproach under a $\Delta$ pinned to a default
contract together with a utility-valued $ci$; and the \emph{causal sufficiency} of
\citet{beckersSufficiency2021} is the sufficiency regime itself, the necessity of $X_k$
\emph{within} a sufficient set $\mathbf{C}$.

The limits of the reframing. Each notion's specific convention, the $1/(N{+}1)$
responsibility weight of
\citet{chocklerResponsibilityBlameStructuralModel2004},\footnote{That this
particular weight reflects a modelling choice rather than a necessity has precedent:
\citet{salimiQuantifyingCausalEffects2016}, applying Chockler--Halpern
responsibility to database query answering, found the $1/(N{+}1)$ measure
gave unintuitive rankings in their setting and replaced it with a graded \emph{causal-effect}
measure that, they argue, recovers the intended results.} the all-or-nothing causation
threshold, the non-domination criterion that rates an explanation only \emph{good} or not
\citep{actualCausalityHalpern,beckersExplanationsXAI2022}, or the stipulated default in the harm contrast of
\citet{beckersHarm2022}, encodes a deliberate modelling commitment suited to that
work's aims, summarising a graded, model-relative phenomenon under a single label.
\ourapproach retains the underlying graded probability and relocates
each of these to an \emph{explicit} downstream choice: a cardinality band in $\Gamma$, a
threshold on $ci$, a prior over models, a particular $\Delta$. On this view the named
quantities are complementary readings of one underlying measure, a vantage point on
the original constructions rather than a shortcoming in them.

We state this only as a conjecture and do not develop it here. We raise it because the
knobs we introduced to separate necessity from sufficiency appear, prima facie, to be the
degrees of freedom along which the quantities of
\citet{chocklerResponsibilityBlameStructuralModel2004}, \citet{actualCausalityHalpern}, and
\citet{beckersSufficiency2021,beckersHarm2022} themselves vary. We leave a full accounting
to future work.

Sections~\ref{sec:relation_with_ac}--\ref{sub:pearl_actual_cause_prob}
have laid out the formal relation between \ourapproach and actual causality and
its Pearlian probabilistic wrapper, completing the comparative story previewed in
\cref{sec:evaluation}: a claim-by-claim comparison against feature-attribution
baselines, the Differential Causal Effect, exact actual-cause enumeration, a
continuous dynamical model, and a deployed valuation system, with the full
protocols collected in the appendices. \cref{sec:discussion} closes the paper.

\section{Discussion and Conclusions}
\label{sec:discussion}

We can now state the main differences between \ourapproach and the
two approaches that inspired it, against the background of what the rest of the paper has
shown. The main lesson from the literature on actual causality is that intuitive causal
attribution often requires holding some endogenous variables, members of a ``witness
set'', fixed at their factual values. This ensures context sensitivity by controlling
complicating causal paths. We treat witnesses differently: instead of looking for
a single witness set that reveals a causal relationship, we use a distribution over
witness sets and integrate the causal impact score across them, which lets
\ourapproach account for the number of witness sets and the strength of
within-witness-set effects. The construction also admits off-the-shelf probabilistic
inference, avoiding worst-case subset enumeration. The relationship between the two
views is now formal where it was once gestural: under matching choices of $\Gamma$ and
$\Delta$, the support of $\sum_{\mathbf{T}}\Phi$ coincides with factivity and
context-sensitive necessity (Theorems~\ref{th:ac-exp}--\ref{th:exp-ac}), and a
subset-minimality filter over that support recovers AC verdicts exactly
(Proposition~\ref{prop:filter}); the empirical companion in \cref{sec:ac_benchmark}
confirms this at scales where AC enumeration is intractable.

The main concept borrowed from Pearl's probability of necessity and sufficiency is the
joint distribution over outcomes from two interventional regimes, pushed forward through
a user-chosen causal impact function. The difference is that we use witness sets to
allow for context-sensitivity and to break symmetries present in otherwise pathological
cases like overdetermination and undercutting. Section~\ref{sec:shap_examples} works
through the binary instantiation in which the necessity and sufficiency components
coincide with PN and PS, and \cref{sub:synthetic} reads the two components apart on
a continuous synthetic model where each is needed to bring out a different causal
archetype.

The further generalisations are mostly bookkeeping in the same spirit: a single
alternative value for a potential cause becomes a factual-excised alternative-value
distribution we integrate over, which generalises immediately to continuous variables;
asking whether a particular set is an actual cause becomes sampling suspect sets and
gauging a feature's responsibility by its participation in sets with causal impact; and
asking whether the outcome differs becomes measuring how much, via a customisable causal
impact function $ci$ that reads differences across counterfactual worlds, including
between continuous variables. The variable-selection distribution $\Gamma$, the
alternative-value distribution $\Delta$, and the impact function $ci$ become explicit,
user-facing components of the framework, no longer implicit defaults.

Against SHAP-family attributions, the framework's central commitment is to
\emph{intervention with witnesses} over \emph{observational marginalisation}. The
controlled comparisons of Section~\ref{sec:shap_examples} show that this commitment does
real work on canonical overdetermination and continuous-mediation examples, where SHAP
and Causal SHAP fail in characteristic ways that \ourapproach avoids. The synthetic
benchmark of \cref{sub:synthetic} confirms this on a model with known ground
truth, and the case study of \cref{sec:avm} reports the same qualitative pattern
on a production-grade trained model with millions of data points. A separate comparison
to gradient-based attribution, where the difference is one of differentiability and locality,
as distinct from the causal/non-causal axis, is reported in \cref{sec:DCE}. The
thread connecting these results is structural: the witness mechanism that satisfies
D-A-rank for Alice in the binary OBCB walkthrough of
Section~\ref{sec:shap_examples} is the same machinery driving the upstream-versus-downstream
attribution split on the AVM in \cref{sec:avm}.

A broader test is how \ourapproach fares against the desiderata that the XAI
literature typically asks counterfactual-explanation methods to meet, of which
\cite{verma2022recourses} list five. \emph{Sparsity}, the requirement that
explanations change as few features as possible, follows directly from the
cardinality cap on suspect sets in $\Gamma$, which biases the search toward small
$\mathbf{C}$. \emph{Multiple counterfactuals}, the ability to return several
alternative explanations, falls out of the Monte Carlo construction: a single sweep
over $\Gamma$ and $\Delta$ produces a population of
$(\mathbf{C}, \mathbf{c}', \mathbf{T})$ triples.
\emph{Causal-constraint satisfaction}, the demand that explanations respect
downstream propagation rather than treat features as independently editable, is
automatic once suspect interventions are routed through the causal model: changing
$\mathbf{C}$ propagates to descendants by construction. \emph{Model-agnosticism}
(applicability to arbitrary predictors), \ourapproach satisfies only partially: it
can wrap a black-box predictor, but at the cost of the causal sensitivity that
motivates the framework in the first place.

We part company with the literature on \emph{actionability}, the requirement that explanations not recommend changes to features users cannot in practice change. The standard recommendation is to fold ease-of-change into the
cause-identification procedure itself, weighting features by mutability so that
immutable ones are never returned as causes (\cite{verma2022recourses}; cf.\
\cite{Karimi2021}, who shift the focus from explanations altogether to interventions).
We think this conflates two distinct questions. A procedure that suppresses gender as
a possible cause whenever a loan decision is at stake will help applicants satisfy the
existing system without surfacing the kind of structural
problem that responsibility attribution can in principle reveal. Cost-of-action
layers can sit downstream of \ourapproach for users who need actionable
recommendations, but they should compose with responsibility attribution instead of
replacing it.

Set against that favourable comparison, several limitations temper the picture.
\Ourapproach presumes access to a trained
probabilistic causal model from which it can sample counterfactual outcomes, a
stronger requirement than methods that work from observational data and a causal graph
alone. It also inherits whatever error is in that model: a misspecified structural
equation or noise distribution yields systematically wrong attributions, with
nothing in the construction to flag this from the scores alone, since \ourapproach
evaluates counterfactuals under the model the user supplies, without checking that
model against data. The Monte Carlo variance of the estimator grows with model size and with the
cardinalities of suspect and witness sets; \cref{sec:ac_benchmark} reports
empirical scaling and an ablation that isolates the contribution of witness pinning.
In the worst case, a single revealing witness set among $2^{|\mathbf{W}|}$ candidates
under uniform $\Gamma_w$ costs as many samples in expectation as exhaustive
enumeration would. The escape is that pinning a variable that is causally inert under
the tested intervention is a no-op, so revealing witness sets come in neighbourhoods
of size $2^m$ for $m$ such inert candidates, giving revealing mass at least
$2^{m-|\mathbf{W}|}$ rather than $2^{-|\mathbf{W}|}$. Whether this rescues sampling
generically, or only in graphs sparse relative to the tested intervention, is a
sharper open question than we have characterised here. A $\Gamma$ built from the
graph, weighting witness candidates by their relevance to the tested intervention,
would exploit exactly this sparsity, but building it requires
the graph's relevance structure to be legible in the first place, which trades away
desideratum~\hyperref[des:broad]{(iii)}'s reach into grey-box models where structural
facts are not cheaply available.
Exposing $\Gamma$, $\Delta$, and $ci$ as user-facing choices is also a burden:
a practitioner with no strong prior view on which suspect sets or alternative values are plausible has nothing
principled to fall back on, and a bad choice can silently shift which variables
the score favours, with no signal in the output that anything went wrong.

\paragraph{Choosing $\mathbf{S}$ and $\mathbf{W}$ in practice.} Every worked
example in this paper hand-picks the suspect and witness pools with full
knowledge of the model's structure, which is not available to a practitioner
starting from a large learned PSCM. Our default recommendation: take
$\mathbf{S}$ to be whichever variables are under live consideration as
candidate causes for the question at hand (the same choice a practitioner
would have to make for but-for testing or SHAP alike), and take $\mathbf{W}$
to be all remaining endogenous non-outcome variables not already in
$\mathbf{S}$, i.e.\ every variable Definition~\ref{def:ac_intuitive}'s witness
pool could legally contain. This costs the most sampling (\cref{sec:ac_benchmark}
quantifies the resulting variance), but it never silently suppresses a genuine
cause by omitting a witness that would have revealed it, since the joint
$\Gamma$ marginalises over which subset of $\mathbf{W}$ is actually pinned on
each draw rather than committing to one witness set up front. The failure
mode to watch for is the opposite direction: if $\mathbf{S}$ itself omits a
variable that mediates between two included suspects, PCI cannot detect that
mediation (it can only witness against it), so a practitioner who suspects
an omitted mediator should move it into $\mathbf{S}$, not $\mathbf{W}$, and
compare the two runs directly, exactly as \cref{sec:obcb_shap}'s Option~B does
for \varname{check-failed}. When sampling cost makes the full-$\mathbf{W}$
default impractical, restricting $\mathbf{W}$ to variables structurally
downstream of $\mathbf{S}$ and upstream of $Y$ (candidate mediators and
confounding paths) is the next-cheapest defensible choice, since variables
off every $\mathbf{S}\to Y$ path are causally inert witnesses whose pinning
or freeing cannot change any counterfactual outcome (\cref{sec:discussion}'s
sampling-cost discussion above formalises why inert witnesses are cheap to
average over rather than enumerate).

Finally, we deliberately limited what the AVM case study discloses: the apples-to-apples
SHAP-vs-PCI comparison rests on the synthetic benchmark of \cref{sub:synthetic},
with the production model serving as a feasibility witness, not as quantitative
evidence.

\Ourapproach provides a context-sensitive causal explanation method for causal
probabilistic machine-learning models that stays tractable, at the sampling cost
discussed above, past the point where exact actual-causality enumeration does not
(\cref{sec:ac_benchmark}). It arrives at this as a natural probabilistic
generalisation of ideas already in the literature, yielding an estimator that
pairs with standard inference algorithms. We benchmarked it against exact
actual-causality computations, characterised its qualitative behaviour on synthetic
models with known ground truth and on canonical SHAP failure cases, and demonstrated
computational feasibility on a production-grade model.

This estimator also sits within a longer line of efforts to mechanise
actual-causality reasoning in software: the Actual Causality Canvas of
\citet{ibrahimCanvas2020} and the SAT- and optimization-based computations of
\citet{ibrahimChecking2020} operationalise the Halpern--Pearl definition for
forensic investigation, fault diagnosis, and explainable AI; those tools are built
for \emph{discrete} (typically binary) structural models, where the actual cause is
recovered by combinatorial search over candidate sets. The complementary regime of large,
continuous-valued models has received less attention, and it is there that the exact
notions become harder to apply directly. \Ourapproach targets exactly that regime,
turning approximate causal attribution into something workable for users of models of
non-trivial size, with $\Gamma$, $\Delta$, and $ci$ exposed as knobs for the user to set
rather than defaults buried in the implementation.

The same mechanisation goal motivates a parallel literature in databases, where
responsibility and causal-effect measures explain query answers over relational
data rather than outcomes of a structural causal model (surveyed by
\citealp{bertossiAttributionScores2023}; \citealp{salimiQuantifyingCausalEffects2016},
already noted in \cref{sub:pearl_actual_cause_prob}, is one instance). In that
vocabulary, \ourapproach's witness-integrated score is a graded, model-relative
responsibility measure of the same family, specialised to probabilistic structural
causal models.

\bibliographystyle{plainnat}
\bibliography{references}

\begin{thebibliography}{68}
\providecommand{\natexlab}[1]{#1}
\providecommand{\url}[1]{\texttt{#1}}
\expandafter\ifx\csname urlstyle\endcsname\relax
  \providecommand{\doi}[1]{doi: #1}\else
  \providecommand{\doi}{doi: \begingroup \urlstyle{rm}\Url}\fi

\bibitem[Alechina et~al.(2017)Alechina, Halpern, and
  Logan]{alechinaTeamPlans2020}
Natasha Alechina, Joseph~Y. Halpern, and Brian Logan.
\newblock Causality, responsibility and blame in team plans.
\newblock In \emph{Proceedings of the 16th International Conference on
  Autonomous Agents and Multiagent Systems (AAMAS)}, pages 1091--1099, 2017.

\bibitem[Balke and Pearl(1994)]{balke1994probabilistic}
Alexander Balke and Judea Pearl.
\newblock Probabilistic evaluation of counterfactual queries.
\newblock In \emph{Proceedings of the Twelfth AAAI National Conference on
  Artificial Intelligence (AAAI)}, pages 230--237, 1994.

\bibitem[Beckers(2021{\natexlab{a}})]{beckersNESS2021}
Sander Beckers.
\newblock The counterfactual {NESS} definition of causation.
\newblock In \emph{Proceedings of the AAAI Conference on Artificial
  Intelligence}, volume~35, pages 6210--6217, 2021{\natexlab{a}}.
\newblock \doi{10.1609/aaai.v35i7.16772}.

\bibitem[Beckers(2021{\natexlab{b}})]{beckersSufficiency2021}
Sander Beckers.
\newblock Causal sufficiency and actual causation.
\newblock \emph{Journal of Philosophical Logic}, 50:\penalty0 1341--1374,
  2021{\natexlab{b}}.
\newblock \doi{10.1007/s10992-021-09601-z}.

\bibitem[Beckers(2022)]{beckersExplanationsXAI2022}
Sander Beckers.
\newblock Causal explanations and {XAI}.
\newblock In \emph{Proceedings of the First Conference on Causal Learning and
  Reasoning (CLeaR)}, volume 177 of \emph{PMLR}, pages 90--109, 2022.

\bibitem[Beckers(2023)]{beckersMoralResponsibility2023}
Sander Beckers.
\newblock Moral responsibility for {AI} systems.
\newblock In \emph{Advances in Neural Information Processing Systems
  (NeurIPS)}, volume~36, 2023.

\bibitem[Beckers(2025)]{beckers2025nondeterministic}
Sander Beckers.
\newblock Actual causation and nondeterministic causal models.
\newblock In \emph{Proceedings of the Fourth Conference on Causal Learning and
  Reasoning (CLeaR)}, volume 275 of \emph{PMLR}, pages 514--532, 2025.

\bibitem[Beckers and Vennekens(2015)]{Beckers2015CombiningPC}
Sander Beckers and Joost Vennekens.
\newblock Combining probabilistic, causal, and normative reasoning in
  {{CP-logic}}.
\newblock \emph{ArXiv}, abs/1503.01051, 2015.

\bibitem[Beckers and Vennekens(2017)]{beckersTransitivity2017}
Sander Beckers and Joost Vennekens.
\newblock The transitivity and asymmetry of actual causation.
\newblock \emph{Ergo}, 4\penalty0 (1):\penalty0 1--27, 2017.
\newblock \doi{10.3998/ergo.12405314.0004.001}.

\bibitem[Beckers and Vennekens(2018)]{beckersPrincipled2018}
Sander Beckers and Joost Vennekens.
\newblock A principled approach to defining actual causation.
\newblock \emph{Synthese}, 195\penalty0 (2):\penalty0 835--862, 2018.
\newblock \doi{10.1007/s11229-016-1247-1}.

\bibitem[Beckers et~al.(2022)Beckers, Chockler, and Halpern]{beckersHarm2022}
Sander Beckers, Hana Chockler, and Joseph~Y. Halpern.
\newblock A causal analysis of harm.
\newblock In \emph{Advances in Neural Information Processing Systems
  (NeurIPS)}, volume~35, 2022.

\bibitem[Beckers et~al.(2023)Beckers, Chockler, and
  Halpern]{beckersQuantifyingHarm2023}
Sander Beckers, Hana Chockler, and Joseph~Y. Halpern.
\newblock Quantifying harm.
\newblock In \emph{Proceedings of the Thirty-Second International Joint
  Conference on Artificial Intelligence (IJCAI)}, pages 363--371, 2023.
\newblock \doi{10.24963/ijcai.2023/41}.

\bibitem[Bertossi(2023)]{bertossiAttributionScores2023}
Leopoldo Bertossi.
\newblock Attribution-scores in data management and explainable machine
  learning, 2023.
\newblock Paper associated with a tutorial at ADBIS 2023.

\bibitem[Butler et~al.(2022)Butler, Feng, and Djurić]{butler2022differential}
K.~Butler, G.~Feng, and P.~M. Djurić.
\newblock A differential measure of the strength of causation.
\newblock \emph{IEEE Signal Processing Letters}, 29:\penalty0 2208--2212, 2022.

\bibitem[Chen et~al.(2024)Chen, Cai, Huang, Zhu, Horwood, Hao, Li, and
  Hern\'andez-Lobato]{chenFANS2024}
Xuexin Chen, Ruichu Cai, Zhengting Huang, Yuxuan Zhu, Julien Horwood, Zhifeng
  Hao, Zijian Li, and Jos\'e~Miguel Hern\'andez-Lobato.
\newblock Feature attribution with necessity and sufficiency via dual-stage
  perturbation test for causal explanation.
\newblock In \emph{Proceedings of the 41st International Conference on Machine
  Learning (ICML)}, volume 235 of \emph{Proceedings of Machine Learning
  Research}. PMLR, 2024.

\bibitem[Cheshire and Magrini(2009)]{CheshireMagrini2009}
Paul Cheshire and Stefano Magrini.
\newblock Growing urban gdp or attracting people? different causes, different
  consequences.
\newblock In Charlie Karlsson, A{\r{a}}ke~E. Andersson, Paul~C. Cheshire, and
  Roger~R. Stough, editors, \emph{New Directions in Regional Economic
  Development}, Advances in Spatial Science, pages 291--315. Springer, Berlin,
  Heidelberg, 2009.
\newblock ISBN 978-3-642-01016-3, 978-3-642-01017-0.
\newblock \doi{10.1007/978-3-642-01017-0_16}.

\bibitem[Chockler and
  Halpern(2004)]{chocklerResponsibilityBlameStructuralModel2004}
H.~Chockler and J.~Y. Halpern.
\newblock Responsibility and {{Blame}}: {{A Structural-Model Approach}}.
\newblock \emph{Journal of Artificial Intelligence Research}, 22:\penalty0
  93--115, October 2004.
\newblock ISSN 1076-9757.
\newblock \doi{10.1613/jair.1391}.

\bibitem[Cohen(1988)]{cohen1988statistical}
Jacob Cohen.
\newblock \emph{Statistical Power Analysis for the Behavioral Sciences}.
\newblock Lawrence Erlbaum Associates, Hillsdale, NJ, 2nd edition, 1988.

\bibitem[Eiter and Lukasiewicz(2002)]{eiterComplexityResultsStructurebased2002}
Thomas Eiter and Thomas Lukasiewicz.
\newblock Complexity results for structure-based causality.
\newblock \emph{Artificial Intelligence}, 142\penalty0 (1):\penalty0 53--89,
  November 2002.
\newblock ISSN 00043702.
\newblock \doi{10.1016/S0004-3702(02)00271-0}.

\bibitem[Fang(2022)]{fang2022probabilistic}
Jingzhi Fang.
\newblock \emph{On probabilistic reasoning of actual causation}.
\newblock Doctoral thesis, Lingnan University, Hong Kong, 2022.
\newblock Retrieved from \url{https://commons.ln.edu.hk/otd/155/}.

\bibitem[Fenton{-}Glynn(2017)]{Fenton-Glynn2017-FENAPP-2}
Luke Fenton{-}Glynn.
\newblock A proposed probabilistic extension of the halpern and pearl
  definition of ?actual cause?
\newblock \emph{British Journal for the Philosophy of Science}, 68\penalty0
  (4):\penalty0 1061--1124, 2017.
\newblock \doi{10.1093/bjps/axv056}.

\bibitem[Friedenberg and Halpern(2019)]{friedenbergHalpernBlameworthiness2019}
Meir Friedenberg and Joseph~Y. Halpern.
\newblock Blameworthiness in multi-agent settings.
\newblock In \emph{Proceedings of the Thirty-Third AAAI Conference on
  Artificial Intelligence (AAAI)}, 2019.

\bibitem[Galhotra et~al.(2021)Galhotra, Pradhan, and Salimi]{galhotraLewis2021}
Sainyam Galhotra, Romila Pradhan, and Babak Salimi.
\newblock Explaining black-box algorithms using probabilistic contrastive
  counterfactuals.
\newblock In \emph{Proceedings of the 2021 ACM SIGMOD International Conference
  on Management of Data}, 2021.

\bibitem[Gerstenberg and Lagnado(2010)]{gerstenbergSpreadingBlame2010}
Tobias Gerstenberg and David~A. Lagnado.
\newblock Spreading the blame: The allocation of responsibility amongst
  multiple agents.
\newblock \emph{Cognition}, 115\penalty0 (1):\penalty0 166--171, 2010.
\newblock \doi{10.1016/j.cognition.2009.12.011}.

\bibitem[Gustafson(2015)]{gustafson2015partial}
Paul Gustafson.
\newblock \emph{Bayesian Inference for Partially Identified Models: Exploring
  the Limits of Limited Data}.
\newblock Monographs on Statistics and Applied Probability. Chapman and
  Hall/CRC, 2015.

\bibitem[Hall(2004)]{hall2004two}
Ned Hall.
\newblock Two concepts of causation.
\newblock In John Collins, Ned Hall, and L.~A. Paul, editors, \emph{Causation
  and Counterfactuals}, pages 225--276. MIT Press, 2004.

\bibitem[Halpern(2015)]{halpern2015modification}
Joseph~Y. Halpern.
\newblock A modification of the {Halpern--Pearl} definition of causality.
\newblock In \emph{Proceedings of the 24th International Joint Conference on
  Artificial Intelligence (IJCAI)}, pages 3022--3033, 2015.

\bibitem[Halpern(2016)]{actualCausalityHalpern}
Joseph~Y. Halpern.
\newblock \emph{Actual Causality}.
\newblock The MIT Press, 2016.
\newblock ISBN 9780262035026.
\newblock URL \url{http://www.jstor.org/stable/j.ctt1f5g5p9}.

\bibitem[Halpern and Hitchcock(2015)]{Halpern2015-HALGCA}
Joseph~Y. Halpern and Christopher Hitchcock.
\newblock Graded causation and defaults.
\newblock \emph{British Journal for the Philosophy of Science}, 66\penalty0
  (2):\penalty0 413--457, 2015.
\newblock \doi{10.1093/bjps/axt050}.

\bibitem[Halpern and
  Kleiman-Weiner(2018)]{halpernKleimanWeinerBlameworthiness2018}
Joseph~Y. Halpern and Max Kleiman-Weiner.
\newblock Towards formal definitions of blameworthiness, intention, and moral
  responsibility.
\newblock In \emph{Proceedings of the Thirty-Second AAAI Conference on
  Artificial Intelligence (AAAI)}, 2018.

\bibitem[Heckman and Pinto(2022)]{HeckmanPinto2022CausalPolicyAnalysis}
James~J. Heckman and Rodrigo Pinto.
\newblock The econometric model for causal policy analysis.
\newblock \emph{Annual Review of Economics}, 14\penalty0 (1):\penalty0
  893--923, 2022.
\newblock \doi{10.1146/annurev-economics-051520-015456}.
\newblock URL \url{https://doi.org/10.1146/annurev-economics-051520-015456}.

\bibitem[Hensman et~al.(2013)Hensman, Fusi, and Lawrence]{hensman2013gaussian}
James Hensman, Nicol{\`o} Fusi, and Neil~D Lawrence.
\newblock Gaussian processes for big data.
\newblock In \emph{Proceedings of the Twenty-Ninth Conference on Uncertainty in
  Artificial Intelligence}, pages 282--290, 2013.

\bibitem[Heskes et~al.(2020)Heskes, Sijben, Bucur, and
  Claassen]{heskes2020causal}
Tom Heskes, Evi Sijben, Ioan~Gabriel Bucur, and Tom Claassen.
\newblock Causal {Shapley} values: Exploiting causal knowledge to explain
  individual predictions of complex models.
\newblock In \emph{Advances in Neural Information Processing Systems},
  volume~33, pages 4778--4789, 2020.

\bibitem[Hitchcock and Knobe(2009)]{hitchcock2009cause}
Christopher Hitchcock and Joshua Knobe.
\newblock Cause and norm.
\newblock \emph{The Journal of Philosophy}, 106\penalty0 (11):\penalty0
  587--612, 2009.
\newblock \doi{10.5840/jphil20091061128}.

\bibitem[Ibrahim and Pretschner(2020)]{ibrahimChecking2020}
Amjad Ibrahim and Alexander Pretschner.
\newblock From checking to inference: Actual causality computations as
  optimization problems.
\newblock In \emph{Automated Technology for Verification and Analysis (ATVA)},
  volume 12302 of \emph{Lecture Notes in Computer Science}. Springer, 2020.

\bibitem[Ibrahim et~al.(2020)Ibrahim, Klesel, Zibaei, Kacianka, and
  Pretschner]{ibrahimCanvas2020}
Amjad Ibrahim, Tobias Klesel, Ehsan Zibaei, Severin Kacianka, and Alexander
  Pretschner.
\newblock Actual causality canvas: A general framework for explanation-based
  socio-technical constructs.
\newblock In \emph{Proceedings of the 24th European Conference on Artificial
  Intelligence (ECAI)}, volume 325 of \emph{Frontiers in Artificial
  Intelligence and Applications}, pages 2978--2985. IOS Press, 2020.
\newblock \doi{10.3233/FAIA200472}.

\bibitem[Icard et~al.(2017)Icard, Kominsky, and
  Knobe]{icardNormalityActualCausal2017}
Thomas~F. Icard, Jonathan~F. Kominsky, and Joshua Knobe.
\newblock Normality and actual causal strength.
\newblock \emph{Cognition}, 161:\penalty0 80--93, April 2017.
\newblock ISSN 00100277.
\newblock \doi{10.1016/j.cognition.2017.01.010}.

\bibitem[Janzing et~al.(2020)Janzing, Minorics, and
  Bl{\"o}baum]{janzing2020feature}
Dominik Janzing, Lenon Minorics, and Patrick Bl{\"o}baum.
\newblock Feature relevance quantification in explainable {AI}: A causal
  problem.
\newblock In \emph{International Conference on Artificial Intelligence and
  Statistics ({AISTATS})}, volume 108 of \emph{PMLR}, pages 2907--2916, 2020.

\bibitem[Karimi et~al.(2021)Karimi, Sch\"{o}lkopf, and Valera]{Karimi2021}
Amir-Hossein Karimi, Bernhard Sch\"{o}lkopf, and Isabel Valera.
\newblock Algorithmic recourse: from counterfactual explanations to
  interventions.
\newblock In \emph{Proceedings of the 2021 ACM Conference on Fairness,
  Accountability, and Transparency}, FAccT ’21, page 353–362. ACM, March
  2021.
\newblock \doi{10.1145/3442188.3445899}.
\newblock URL \url{http://dx.doi.org/10.1145/3442188.3445899}.

\bibitem[Kawakami et~al.(2024)Kawakami, Kuroki, and
  Tian]{kawakammi_2024_poc_for_cont_and_vec}
Yuta Kawakami, Manabu Kuroki, and Jin Tian.
\newblock Probabilities of {{Causation}} for {{Continuous}} and {{Vector
  Variables}}.
\newblock In Negar Kiyavash and Joris~M. Mooij, editors, \emph{Proceedings of
  the Fortieth Conference on Uncertainty in Artificial Intelligence}, volume
  244 of \emph{Proceedings of Machine Learning Research}, pages 1901--1921.
  PMLR, July 2024.

\bibitem[Kruschke(2018)]{kruschke2018}
John~K. Kruschke.
\newblock Rejecting or accepting parameter values in {Bayesian} estimation.
\newblock \emph{Advances in Methods and Practices in Psychological Science},
  1\penalty0 (2):\penalty0 270--280, 2018.
\newblock \doi{10.1177/2515245918771304}.

\bibitem[Lewis(1973)]{lewis1973causation}
David Lewis.
\newblock Causation.
\newblock \emph{The Journal of Philosophy}, 70\penalty0 (17):\penalty0
  556--567, 1973.

\bibitem[Lundberg and Lee(2017)]{lundberg2017unified}
Scott~M Lundberg and Su-In Lee.
\newblock A unified approach to interpreting model predictions.
\newblock In \emph{Advances in Neural Information Processing Systems},
  volume~30, pages 4765--4774, 2017.

\bibitem[Maldonado et~al.(2025)Maldonado, Krumme, Zetzsche, Didelez, and
  Schill]{maldonado2025robot}
Jaime Maldonado, Jonas Krumme, Christoph Zetzsche, Vanessa Didelez, and Kerstin
  Schill.
\newblock Robot pouring: Identifying causes of spillage and selecting
  alternative action parameters using probabilistic actual causation.
\newblock \emph{Frontiers in Cognition}, 4:\penalty0 1565059, 2025.
\newblock \doi{10.3389/fcogn.2025.1565059}.

\bibitem[Manski(2003)]{manski2003partial}
Charles~F. Manski.
\newblock \emph{Partial Identification of Probability Distributions}.
\newblock Springer Series in Statistics. Springer, 2003.

\bibitem[Moon and Schorfheide(2012)]{moon2012bayesian}
Hyungsik~Roger Moon and Frank Schorfheide.
\newblock Bayesian and frequentist inference in partially identified models.
\newblock \emph{Econometrica}, 80\penalty0 (2):\penalty0 755--782, 2012.
\newblock \doi{10.3982/ECTA8360}.

\bibitem[Mueller et~al.(2022)Mueller, Li, and Pearl]{muellerpearl2022causes}
Scott Mueller, Ang Li, and Judea Pearl.
\newblock Causes of effects: Learning individual responses from population
  data.
\newblock In \emph{Proceedings of the Thirty-First International Joint
  Conference on Artificial Intelligence (IJCAI)}, pages 2712--2718, 2022.
\newblock \doi{10.24963/ijcai.2022/376}.

\bibitem[Ness(2025)]{ness2025causal}
Robert~Osazuwa Ness.
\newblock \emph{Causal AI}.
\newblock Manning Publications, 2025.

\bibitem[Papageorgiou et~al.(2022)Papageorgiou, Theodosiou, Rapti,
  Papageorgiou, Dimitriou, Tzovaras, and Margetis]{Papageorgiou2022}
Konstantinos Papageorgiou, Theodosios Theodosiou, Aikaterini Rapti, Elpiniki~I.
  Papageorgiou, Nikolaos Dimitriou, Dimitrios Tzovaras, and George Margetis.
\newblock A systematic review on machine learning methods for root cause
  analysis towards zero-defect manufacturing.
\newblock \emph{Frontiers in Manufacturing Technology}, Volume 2 - 2022, 2022.
\newblock ISSN 2813-0359.
\newblock \doi{10.3389/fmtec.2022.972712}.
\newblock URL
  \url{https://www.frontiersin.org/journals/manufacturing-technology/articles/10.3389/fmtec.2022.972712}.

\bibitem[Paul and Hall(2013)]{paulhall2013}
L.~A. Paul and Ned Hall.
\newblock \emph{Causation: A User's Guide}.
\newblock Oxford University Press, 2013.

\bibitem[Pearl(1999)]{pearl1999probabilities}
Judea Pearl.
\newblock Probabilities of causation: Three counterfactual interpretations and
  their identification.
\newblock \emph{Synthese}, 121\penalty0 (1--2):\penalty0 93--149, 1999.
\newblock \doi{10.1023/A:1005233831499}.

\bibitem[Pearl(2009)]{causalityPearl}
Judea Pearl.
\newblock \emph{Causality: Models, Reasoning and Inference}.
\newblock Cambridge University Press, USA, 2nd edition, 2009.
\newblock ISBN 052189560X.

\bibitem[Pearl(2022)]{pearl2022probabilities}
Judea Pearl.
\newblock Probabilities of causation: three counterfactual interpretations and
  their identification.
\newblock In \emph{Probabilistic and Causal Inference: The Works of Judea
  Pearl}, pages 317--372. 2022.

\bibitem[Poirier(1998)]{poirier1998revising}
Dale~J. Poirier.
\newblock Revising beliefs in nonidentified models.
\newblock \emph{Econometric Theory}, 14\penalty0 (4):\penalty0 483--509, 1998.
\newblock \doi{10.1017/S0266466698144043}.

\bibitem[Richardson et~al.(2011)Richardson, Evans, and
  Robins]{richardson2011transparent}
Thomas~S. Richardson, Robin~J. Evans, and James~M. Robins.
\newblock Transparent parametrizations of models for potential outcomes.
\newblock In \emph{Bayesian Statistics 9}, pages 569--610. Oxford University
  Press, 2011.

\bibitem[Rubin(1978)]{rubin1978bayesian}
Donald~B. Rubin.
\newblock Bayesian inference for causal effects: The role of randomization.
\newblock \emph{The Annals of Statistics}, 6\penalty0 (1):\penalty0 34--58,
  1978.
\newblock \doi{10.1214/aos/1176344064}.

\bibitem[Salimi et~al.(2016)Salimi, Bertossi, Suciu, and Van~den
  Broeck]{salimiQuantifyingCausalEffects2016}
Babak Salimi, Leopoldo Bertossi, Dan Suciu, and Guy Van~den Broeck.
\newblock Quantifying causal effects on query answering in databases.
\newblock In \emph{Proceedings of the 8th USENIX Workshop on the Theory and
  Practice of Provenance (TaPP)}, 2016.

\bibitem[Shapley(1953)]{shapley1953}
Lloyd~S. Shapley.
\newblock A value for $n$-person games.
\newblock In Harold~W. Kuhn and Albert~W. Tucker, editors, \emph{Contributions
  to the Theory of Games, Volume {II}}, volume~28 of \emph{Annals of
  Mathematics Studies}, pages 307--317. Princeton University Press, Princeton,
  1953.

\bibitem[Sharma et~al.(2022)Sharma, Li, and Jiao]{sharma2022cfshapley}
Amit Sharma, Hua Li, and Jian Jiao.
\newblock The counterfactual-{S}hapley value: Attributing change in system
  metrics.
\newblock \emph{arXiv preprint arXiv:2208.08399}, 2022.

\bibitem[Shepherd et~al.(2024)Shepherd, Amor, and
  Moreno‐Betancur]{Shepherd2024}
Daisy~A. Shepherd, David~J. Amor, and Margarita Moreno‐Betancur.
\newblock Statistical analysis of observational studies in disability research.
\newblock \emph{Developmental Medicine \& Child Neurology}, 66\penalty0
  (11):\penalty0 1408--1418, 2024.
\newblock \doi{10.1111/dmcn.15948}.

\bibitem[Sprenger(2018)]{sprenger2018foundations}
Jan Sprenger.
\newblock Foundations of a probabilistic theory of causal strength.
\newblock \emph{The Philosophical Review}, 127\penalty0 (3):\penalty0 371--398,
  2018.
\newblock \doi{10.1215/00318108-6718797}.

\bibitem[Tian and Pearl(2000)]{tian2000probabilities}
Jin Tian and Judea Pearl.
\newblock Probabilities of causation: Bounds and identification.
\newblock \emph{Annals of Mathematics and Artificial Intelligence}, 28\penalty0
  (1--4):\penalty0 287--313, 2000.
\newblock \doi{10.1023/A:1018912507879}.

\bibitem[Titsias(2009)]{titsias2009variational}
Michalis Titsias.
\newblock Variational learning of inducing variables in sparse {G}aussian
  processes.
\newblock In \emph{Artificial Intelligence and Statistics}, pages 567--574.
  PMLR, 2009.

\bibitem[Verma et~al.(2024)Verma, Boonsanong, Hoang, Hines, Dickerson, and
  Shah]{verma2022recourses}
Sahil Verma, Varich Boonsanong, Minh Hoang, Keegan Hines, John Dickerson, and
  Chirag Shah.
\newblock Counterfactual explanations and algorithmic recourses for machine
  learning: A review.
\newblock \emph{ACM Computing Surveys}, 56\penalty0 (12):\penalty0 1--42, 2024.

\bibitem[Wang and Jordan(2021)]{wangJordanDesiderata2021}
Yixin Wang and Michael~I. Jordan.
\newblock Desiderata for representation learning: A causal perspective, 2021.

\bibitem[Watson et~al.(2021)Watson, Gultchin, Taly, and
  Floridi]{watson2021necessity}
David~S. Watson, Limor Gultchin, Ankur Taly, and Luciano Floridi.
\newblock Local explanations via necessity and sufficiency: unifying theory and
  practice.
\newblock In \emph{Proceedings of the Thirty-Seventh Conference on Uncertainty
  in Artificial Intelligence}, volume 161 of \emph{Proceedings of Machine
  Learning Research}, pages 1382--1392. PMLR, 2021.

\bibitem[Xia et~al.(2023)Xia, Pan, and Bareinboim]{xia2022neural}
Kevin Xia, Yushu Pan, and Elias Bareinboim.
\newblock Neural causal models for counterfactual identification and
  estimation.
\newblock In \emph{International Conference on Learning Representations
  (ICLR)}, 2023.
\newblock arXiv:2210.00035.

\bibitem[Zhang et~al.(2022)Zhang, Tian, and Bareinboim]{zhang2022partial}
Junzhe Zhang, Jin Tian, and Elias Bareinboim.
\newblock Partial counterfactual identification from observational and
  experimental data.
\newblock In \emph{Proceedings of the 39th International Conference on Machine
  Learning}, volume 162 of \emph{Proceedings of Machine Learning Research},
  pages 26548--26558. PMLR, 2022.

\end{thebibliography}

\appendix
\section{Computations for the OBCB Running Example}\label{app:obcb}

This appendix collects the worked computations for the stochastic OBCB model of
Example~\ref{ex:obcb_stochastic} that the main text summarises: the
per-branch traces behind Alice's individual $\pn$ and $\ps$ values in
Table~\ref{tab:pn_ps_pns} (Section~\ref{sec:motivations}), and the closed-form
SHAP value function behind Tables~\ref{tab:vS_optionA}--\ref{tab:shap_obcb}
(Section~\ref{sec:obcb_shap}). We obtain the remaining individual and population
entries of Table~\ref{tab:pn_ps_pns} analogously; the notebook
\nb{obcb\_computations} reproduces all values.

\subsection{Alice's probability of necessity}\label{app:obcb_pn}

Alice is factually $\varname{gender}=\text{female}$, $\varname{credit}=\text{bad}$,
$\varname{loan}=F$; the counterfactual intervenes to set
$\varname{gender}=\text{male}$ while leaving $\varname{credit}=\text{bad}$.
Tracing through the model:
\begin{itemize}
\item Under the intervention, $\varname{check}\sim\mathrm{Bern}(0.9)$: men are
checked $90\%$ of the time.
\item If checked, $\varname{check\text{-}failed}=1$ (Alice's credit is bad), and from
the loan-probability table
$P(\varname{loan\text{-}if\text{-}checked}=1\mid\varname{gender}=\text{male},\varname{check\text{-}failed}=1)=0.05$.
\item If not checked, $\varname{loan}=0$ deterministically since
$\varname{loan}=\varname{loan\text{-}if\text{-}checked}\cdot\varname{check}$.
\end{itemize}
Combining the two branches:
\begin{multline*}
\pn(\varname{gender}=\text{female}, \varname{loan}=F \mid \text{Alice}) \\
= P(\varname{loan}_{\varname{gender}=\text{male}}=T
\mid \varname{gender}=\text{female}, \varname{credit}=\text{bad}) \\
= 0.9 \cdot 0.05 = 0.045.
\end{multline*}

\subsection{Alice's probability of sufficiency}\label{app:obcb_ps}

Following the same convention as for $\pn$, the individual-level computation
additionally conditions on Alice's credit ($\varname{credit}=\text{bad}$).
The PS conditional then has Alice's credit fixed but the cause and outcome flipped
to their non-factual values ($\varname{gender}=\text{male}$, $\varname{loan}=T$); we
intervene to restore $\varname{gender}=\text{female}$ and ask whether the loan is
again denied:
\begin{itemize}
\item With $\varname{gender}=\text{female}$ and $\varname{credit}=\text{bad}$,
$\varname{check}\sim\mathrm{Bern}(0.2)$ \emph{unconditionally}; but PS conditions on
$\varname{gender}=\text{male}, \varname{loan}=T$, which (since $\varname{loan}=T$
forces $\varname{check}=1$) restricts $U_{\varname{check}}$ to the region
compatible with $\varname{check}=1$ under $\varname{gender}=\text{male}$, measure
$0.9$. Under the twin-network coupling of Footnote~\ref{fn:twin-network} (both
worlds share $U_{\varname{check}}$, so $\varname{check}=1$ iff
$U_{\varname{check}}<\gamma_M=0.9$ for men and $U_{\varname{check}}<\gamma_F=0.2$
for women), restoring $\varname{gender}=\text{female}$ on this restricted
$U_{\varname{check}}$ gives $\varname{check}=1$ with probability
$0.2/0.9=2/9$, not the unconditional $0.2$; this does not change the answer below,
since both branches deny the loan regardless of the exact branch probability.
\item If checked, $\varname{check\text{-}failed}=1$, and from the loan-probability table
$P(\varname{loan\text{-}if\text{-}checked}=1\mid\varname{gender}=\text{female},\varname{check\text{-}failed}=1)=0$.
\item If not checked, $\varname{loan}=0$ deterministically.
\end{itemize}
Either branch gives $\varname{loan}=F$, so
\begin{multline*}
\ps(\varname{gender}=\text{female}, \varname{loan}=F \mid \text{Alice}) \\
= P(\varname{loan}_{\varname{gender}=\text{female}}=F \mid \varname{gender}=\text{male},
\varname{loan}=T, \varname{credit}=\text{bad}) = 1.
\end{multline*}

\subsection{SHAP value function for Option~A}\label{app:obcb_vS}

For the 2-feature game $N=\{\varname{gender},\varname{credit}\}$ with
$f(g,c)=P(\varname{loan}=1\mid g,c)$, by the SHAP definition (and equivalently by
Definition~\ref{def:causalshap} once we note that no dropped feature is a
descendant of any coalition member), the value function is
\[
v(\mathcal{S}) \;=\; \mathbb{E}\!\left[f(X)\,\big|\,X_{\mathcal{S}} = x_{\mathcal{S}}^\star\right]
\;=\; \sum_{x_{N\setminus\mathcal{S}}\in\{0,1\}^{|N\setminus\mathcal{S}|}} P(X_{N\setminus\mathcal{S}}=x_{N\setminus\mathcal{S}})\,f(x_{\mathcal{S}}^\star,\,x_{N\setminus\mathcal{S}}).
\]
Under independent uniform marginals
$P(\varname{gender})=P(\varname{credit})=\tfrac{1}{2}$ this specialises to
\begin{align*}
v(\emptyset)                              &= \tfrac{1}{4}\sum_{g,c\in\{0,1\}} f(g,c),
&&[\text{both dropped: weight } \tfrac{1}{2}\cdot\tfrac{1}{2}] \\
v(\{\varname{gender}\})                   &= \tfrac{1}{2}\!\left[f(g^\star,0)+f(g^\star,1)\right],
&&[\text{only credit dropped: weight } P(c) = \tfrac{1}{2}] \\
v(\{\varname{credit}\})                   &= \tfrac{1}{2}\!\left[f(0,c^\star)+f(1,c^\star)\right],
&&[\text{only gender dropped: weight } P(g) = \tfrac{1}{2}] \\
v(\{\varname{gender},\varname{credit}\})  &= f(g^\star,c^\star).
&&[\text{nothing dropped: no average}]
\end{align*}
Plugging in the four values
$f(\text{F},\text{bad})=0.000$, $f(\text{F},\text{good})=0.180$,
$f(\text{M},\text{bad})=0.045$, $f(\text{M},\text{good})=0.900$ yields the
coalition values of Table~\ref{tab:vS_optionA}; pushing those through the
closed-form Shapley
$\phi_i = \tfrac{1}{2}[v(\{i\})-v(\emptyset)] + \tfrac{1}{2}[v(N)-v(N\!\setminus\!\{i\})]$
(and flipping signs, since we attribute rejection $g=1-f$) gives the magnitudes
in Table~\ref{tab:shap_obcb}.

\section{SHAP and Causal SHAP: Definitions and Signal-with-Mediation Computations}\label{app:signal}

This appendix collects the closed-form derivations for the
signal-with-mediation example of Section~\ref{sec:signal}: the Gaussian chain
$X\to M\to Y$ of Eqs.~\eqref{eq:X}--\eqref{eq:Y}, with
$X\sim\mathcal{N}(0.5,0.25)$, $M=X+\varepsilon_M$, $Y=M+\varepsilon_Y$, and
independent noise $\varepsilon_M,\varepsilon_Y\sim\mathcal{N}(0,0.1)$. The
distributional facts and conditional expectations stated in
Example~\ref{ex:signal_mediation}, and the per-target SHAP and per-pair PCI
computations summarised in Table~\ref{tab:signal_desid}, are derived here.
Unless noted otherwise the factual instance is Instance~1,
$X^\star=M^\star=Y^\star=1$.

\subsection{SHAP and Causal SHAP}\label{sec:shap_defs}

In this section $\mathcal{S} \subseteq N$ denotes a SHAP coalition (a subset of features held
to their factual values during evaluation), playing the same role here that
$\mathbf{C} \subseteq \mathbf{S}$ played in Section~\ref{sec:causal_impact}. We use the
calligraphic $\mathcal{S}$ to keep the SHAP coalition typographically distinct from PCI's
suspect set $\mathbf{S}$ on the printed page; the SHAP literature's plain $S$ is the same
object.

\begin{definition}[SHAP]
A cooperative game is a pair $(N, v)$,
where $N = \{1, \ldots, n\}$ is the set of players and $v: 2^N \to \mathbb{R}$
is the \emph{characteristic function} mapping each subset
(coalition) $\mathcal{S} \subseteq N$ to the value that coalition can produce. The Shapley value
\citep{shapley1953} assigns each player $i$ a payoff $\phi_i$, its share of the total
value $v(N) - v(\emptyset)$, and is the unique attribution satisfying four axioms:
\begin{enumerate}
\item \textbf{Efficiency}: $\sum_i \phi_i = v(N) - v(\emptyset)$
\item \textbf{Symmetry}: If $v(\mathcal{S} \cup \{i\}) = v(\mathcal{S} \cup \{j\})$ for all
      $\mathcal{S} \subseteq N \setminus \{i,j\}$, then $\phi_i = \phi_j$
\item \textbf{Dummy}: If $v(\mathcal{S} \cup \{i\}) = v(\mathcal{S})$ for all $\mathcal{S}$, then $\phi_i = 0$
\item \textbf{Linearity}: For two games $v_1, v_2$: $\phi_i(\alpha_1 v_1 + \alpha_2 v_2)
      = \alpha_1 \phi_i(v_1) + \alpha_2 \phi_i(v_2)$
\end{enumerate}
The unique solution is:
\begin{equation}
\phi_i = \sum_{\mathcal{S} \subseteq N \setminus \{i\}}
\frac{|\mathcal{S}|!\,(|N|-|\mathcal{S}|-1)!}{|N|!}
\bigl[v(\mathcal{S} \cup \{i\}) - v(\mathcal{S})\bigr]
\label{eq:shapley}
\end{equation}
The weight $\frac{|\mathcal{S}|!(|N|-|\mathcal{S}|-1)!}{|N|!}$ equals the probability that, in a uniformly
random ordering of all players, exactly the members of $\mathcal{S}$ precede $i$. Equivalently,
$\phi_i$ is the average marginal contribution of $i$ across all orderings.
\citet{lundberg2017unified} apply Shapley values to explain model predictions by taking
players to be input features, and the characteristic function to be the expected model
output given a subset of features held at their factual values:
\begin{equation}
v(\mathcal{S}) = \mathbb{E}\bigl[f(X) \,\big\vert\, X_{\mathcal{S}} = x_{\mathcal{S}}\bigr]
\label{eq:plain_shap_value}
\end{equation}
\noindent where $\bar{\mathcal{S}} = N \setminus \mathcal{S}$ and the missing features
$X_{\bar{\mathcal{S}}}$ are drawn from their distribution \emph{conditional} on
$X_{\mathcal{S}} = x_{\mathcal{S}}$, as in \citeauthor{lundberg2017unified}'s original
formulation. This is not the only reading in use: \citet{janzing2020feature} single
out exactly this choice as contested, and a common alternative (also called ``plain'' or ``interventional'' SHAP in parts of the literature) instead averages
each dropped feature over its own \emph{marginal} distribution $P(X_i)$, treating the
dropped features as mutually independent and independent of $X_{\mathcal{S}}$. The two
coincide whenever $X_{\mathcal{S}}$ and $X_{\bar{\mathcal{S}}}$ are independent, but
differ in general. Throughout this paper, ``plain SHAP'' denotes the conditional form
of Eq.~\eqref{eq:plain_shap_value} unless stated otherwise; the one exception is
\S\ref{sec:obcb_shap}'s Option~B analysis, which switches explicitly to the
marginal-independent form to expose a separate failure mode (undefined values on
coalitions whose marginal average places positive weight on a structurally impossible
combination of feature values), and says so at the point of use.
\end{definition}

\begin{definition}[Causal SHAP]\label{def:causalshap}
Causal SHAP replaces the characteristic function with an interventional
expectation \citep{heskes2020causal}:
\begin{equation}
v_{\text{causal}}(\mathcal{S}) = \mathbb{E}\bigl[f(X) \mid do(X_{\mathcal{S}} = x_{\mathcal{S}})\bigr]
= \int P(X_{\bar{\mathcal{S}}} \mid do(X_{\mathcal{S}} = x_{\mathcal{S}}))\, f(X_{\bar{\mathcal{S}}}, x_{\mathcal{S}})\, dX_{\bar{\mathcal{S}}}
\label{eq:causal_v}
\end{equation}
By the rules of do-calculus, intervening on $X_{\mathcal{S}}$ only affects the distribution of
descendants of $X_{\mathcal{S}}$. For non-descendants, $P(X_j \mid do(X_{\mathcal{S}} = x_{\mathcal{S}})) = P(X_j)$.
Writing $X_{\bar{\mathcal{S}},\,\mathrm{nd}}$ for the features in $\bar{\mathcal{S}}$ that are not
descendants of any member of $\mathcal{S}$, and $X_{\bar{\mathcal{S}},\,\mathrm{d}}$ for those that are,
equation~\eqref{eq:causal_v} is equivalent to:
\begin{equation}
v_{\text{causal}}(\mathcal{S}) = \mathbb{E}_{X_{\bar{\mathcal{S}},\,\mathrm{nd}} \sim P(\cdot),\;
X_{\bar{\mathcal{S}},\,\mathrm{d}} \sim P(\cdot \mid do(X_{\mathcal{S}}=x_{\mathcal{S}}))}\bigl[f(X_{\bar{\mathcal{S}}}, x_{\mathcal{S}})\bigr]
\end{equation}
The Shapley formula~\eqref{eq:shapley} is otherwise unchanged. The method retains all
four Shapley axioms.
\end{definition}

\subsection{Distributional facts}\label{app:signal_distfacts}

From the structural equations,
$\mathrm{Var}(M)=\mathrm{Var}(X)+\mathrm{Var}(\varepsilon_M)=0.25+0.1=0.35$ and
$\mathrm{Var}(Y)=\mathrm{Var}(M)+\mathrm{Var}(\varepsilon_Y)=0.35+0.1=0.45$.
For the covariances,
$\mathrm{Cov}(X,M)=\mathrm{Cov}(X,X+\varepsilon_M)=\mathrm{Var}(X)=0.25$;
$\mathrm{Cov}(X,Y)=\mathrm{Cov}(X,M+\varepsilon_Y)=\mathrm{Cov}(X,M)=0.25$;
and
$\mathrm{Cov}(M,Y)=\mathrm{Cov}(M,M+\varepsilon_Y)=\mathrm{Var}(M)=0.35$.

\subsection{Conditional expectations}\label{app:signal_condexp}

$\mathbb{E}[Y\mid X,M]=M$ since $Y=M+\varepsilon_Y$ with $\varepsilon_Y\perp(X,M)$.
For $\mathbb{E}[M\mid X,Y]$, the covariance matrix of $(X,Y)$ is
$\Sigma_{XY}=\bigl(\begin{smallmatrix}0.25&0.25\\0.25&0.45\end{smallmatrix}\bigr)$
with $\det(\Sigma_{XY})=0.05$, so
$\Sigma_{XY}^{-1}=20\bigl(\begin{smallmatrix}0.45&-0.25\\-0.25&0.25\end{smallmatrix}\bigr)$;
the regression coefficients are
$[\mathrm{Cov}(M,X),\,\mathrm{Cov}(M,Y)]\,\Sigma_{XY}^{-1}
=[0.25,\,0.35]\cdot 20\bigl(\begin{smallmatrix}0.45&-0.25\\-0.25&0.25\end{smallmatrix}\bigr)
=[0.5,\,0.5]$,
giving $\mathbb{E}[M\mid X,Y]=0.5+0.5(X-0.5)+0.5(Y-0.5)=0.5X+0.5Y$.
For $\mathbb{E}[X\mid M,Y]$, by $d$-separation, $X\perp Y\mid M$ in the chain
$X\!\to\!M\!\to\!Y$, so $\mathbb{E}[X\mid M,Y]=\mathbb{E}[X\mid M]
=0.5+\tfrac{\mathrm{Cov}(X,M)}{\mathrm{Var}(M)}(M-0.5)
=0.5+\tfrac{0.25}{0.35}(M-0.5)=0.5+\tfrac{5}{7}(M-0.5)$.

\subsection{Plain SHAP}\label{app:signal_plain}

\textit{Target $Y$, features $\{X, M\}$.} First, compute the coalition values:
\begin{align}
v(\emptyset) &= \mathbb{E}[M] = 0.5 \\
v(\{X\}) &= \mathbb{E}[M \mid X=1] = 1 \\
v(\{M\}) &= \mathbb{E}[M \mid M=1] = 1 \\
v(\{X,M\}) &= f_Y(1,1) = 1
\end{align}

Next, apply equation~\eqref{eq:shapley} with $|N|=2$:
\begin{equation}
\phi_X = \tfrac{1}{2}(1 - 0.5) + \tfrac{1}{2}(1 - 1) = 0.25
\qquad
\phi_M = \tfrac{1}{2}(1 - 0.5) + \tfrac{1}{2}(1 - 1) = 0.25
\end{equation}

\textit{Target $M$, features $\{X,Y\}$.}
$v(\emptyset)=0.5$;\enspace $v(\{X\})=1$;\enspace
$v(\{Y\})=0.5\,\mathbb{E}[X\mid Y{=}1]+0.5=0.889$
(using $\mathbb{E}[X\mid Y{=}1]=0.5+\tfrac{0.25}{0.45}(0.5)=0.778$);\enspace
$v(\{X,Y\})=1$.
\begin{align}
\phi_X &= \tfrac{1}{2}(1-0.5)+\tfrac{1}{2}(1-0.889)=0.306, \\
\phi_Y &= \tfrac{1}{2}(0.889-0.5)+\tfrac{1}{2}(1-1)=0.194.
\end{align}

\textit{Target $X$, features $\{M,Y\}$.}
$v(\emptyset)=0.5$;\enspace $v(\{M\})=0.5+\tfrac{5}{7}(0.5)=0.857$;\enspace
$v(\{Y\})=\mathbb{E}[X\mid Y{=}1]=0.778$ (as above);\enspace
$v(\{M,Y\})=\mathbb{E}[X\mid M{=}1]=0.857$ (since $X\perp Y\mid M$).
\begin{align}
\phi_M &= \tfrac{1}{2}(0.857-0.5)+\tfrac{1}{2}(0.857-0.778)=0.219, \\
\phi_Y &= \tfrac{1}{2}(0.778-0.5)+\tfrac{1}{2}(0.857-0.857)=0.139.
\end{align}

\subsection{Causal SHAP}\label{app:signal_causal}

We apply Causal SHAP (Definition~\ref{def:causalshap}) and compare against
plain SHAP.

\textit{Target $Y$, features $\{X,M\}$.}
$v_{\mathrm{causal}}(\emptyset)=0.5$;\enspace
$v_{\mathrm{causal}}(\{X\})=\mathbb{E}[M\mid do(X{=}1)]=1$
(M is downstream of X: $P(M\mid do(X{=}1))=\mathcal{N}(1,0.1)$);\enspace
$v_{\mathrm{causal}}(\{M\})=1$ ($do(M{=}1)$ sets $f_Y=M=1$ directly,
$P(X\mid do(M{=}1))=P(X)$ since X has no causal parents);\enspace
$v_{\mathrm{causal}}(\{X,M\})=1$.
All values match plain SHAP, so $\phi_X=\phi_M=0.25$.

\textit{Target $M$, features $\{X,Y\}$.}
$v_{\mathrm{causal}}(\emptyset)=0.5$;\enspace
$v_{\mathrm{causal}}(\{X\})=1$ ($Y$ is downstream of $X$, same as plain SHAP);\enspace
$v_{\mathrm{causal}}(\{Y\})=\mathbb{E}_{X\sim P(X)}[0.5X+0.5{\cdot}1]=0.75$
($Y$ does not cause $X$, so $P(X\mid do(Y{=}1))=P(X)$;
differs from plain SHAP's $0.889$);\enspace
$v_{\mathrm{causal}}(\{X,Y\})=1$.
\begin{align}
\phi_X &= \tfrac{1}{2}(1-0.5)+\tfrac{1}{2}(1-0.75)=0.375, \\
\phi_Y &= \tfrac{1}{2}(0.75-0.5)+\tfrac{1}{2}(1-1)=0.125.
\end{align}

\textit{Target $X$, features $\{M,Y\}$.}
$v_{\mathrm{causal}}(\emptyset)=0.5$;\enspace
$v_{\mathrm{causal}}(\{M\})=0.857$ ($Y$ is downstream of $M$ but $f_X$
does not depend on $Y$ since $X\perp Y\mid M$, so
$\mathbb{E}_{Y\sim P(Y\mid do(M{=}1))}[f_X(1,Y)]=f_X(1,\cdot)=0.857$;
same as plain SHAP);\enspace
$v_{\mathrm{causal}}(\{Y\})=\mathbb{E}_{M\sim P(M)}[f_X(M,1)]=0.5$
($M$ is not a descendant of $Y$, so $P(M\mid do(Y{=}1))=P(M)$
and $\mathbb{E}[f_X(M,1)]=0.5+\tfrac{5}{7}(\mathbb{E}[M]-0.5)=0.5$;
differs from plain SHAP's $0.778$);\enspace
$v_{\mathrm{causal}}(\{M,Y\})=0.857$.
\begin{align}
\phi_M &= \tfrac{1}{2}(0.857-0.5)+\tfrac{1}{2}(0.857-0.5)=0.357, \\
\phi_Y &= \tfrac{1}{2}(0.5-0.5)+\tfrac{1}{2}(0.857-0.857)=0.
\end{align}

\subsection{PCI}\label{app:signal_pci}

We work at the factual instance $X^\star = M^\star = Y^\star = 1$ with the
absolute-difference score (Example~\ref{ex:absolute_score}):
$ci(y^s, y^n, y^\star) = |y^n - y^\star| - |y^s - y^\star|$. Both worlds run
the PSCM with fresh exogenous noise drawn from $P_{\mathbf{U}}$: the necessity
world replaces the suspect $A$ with $A' \sim P(A)$ and propagates; the
sufficiency world restores it to $a^\star$ and propagates. For target $B$,
$\mathbf{S}$ is the two input features of $B$ and $\mathbf{W}$ the remaining
input; $\Gamma_s$ is uniform over the three non-empty subsets of $\mathbf{S}$
and $\Gamma_w$ uniform over $2^{\mathbf{W}}$.

\textit{$R(X \rightsquigarrow Y)$.}
$\mathbf{S} = \{X, M\}$, $\mathbf{W} = \{M\}$, $w^\star = m^\star = 1$, and
$X' \sim \mathcal{N}(0.5,\, 0.25)$. Three of the four accepted
$(\mathbf{C}, \mathbf{T})$ pairs intervene on $X$, each with weight
$\tfrac{1}{4}$; the fourth, $(\mathbf{C}{=}\{M\}, \mathbf{T}{=}\emptyset)$,
does not touch $X$ and contributes nothing. The
witness pair $(\mathbf{C}{=}\{X\}, \mathbf{T}{=}\{M\})$ pins $M$ at $m^\star$
in both worlds, which equalizes their distributions and zeros out $\bar c$;
the other two pairs let $M$ propagate freely. Computing
$\mathbb{E}[|Y^n - 1|]$ and $\mathbb{E}[|Y^s - 1|]$ in closed form via the
folded-normal formula gives:
\begin{center}
\begin{tabular}{lccc}
\toprule
$(\mathbf{C}, \mathbf{T})$ & $\mathbb{E}[|Y^n - 1|]$ & $\mathbb{E}[|Y^s - 1|]$ & $\bar c$ \\
\midrule
$(\{X\}, \emptyset)$              & $0.677$ & $0.357$ & $0.320$ \\
$(\{X\}, \{M\})$ \emph{[witness]} & $0.252$ & $0.252$ & $0.000$ \\
$(\{X, M\}, \emptyset)$           & $0.677$ & $0.252$ & $0.425$ \\
\bottomrule
\end{tabular}
\end{center}
Combining at weight $\tfrac{1}{4}$ each,
\[
R(X \rightsquigarrow Y \mid \mathbf{W}{=}\{M\})
= \tfrac{1}{4}(0.320 + 0 + 0.425) \approx 0.186.
\]

\textit{$R(Y \rightsquigarrow X)$.}
$X$ is exogenous: no $do$-intervention on $\{M, Y\}$ alters its structural
equation, so necessity and sufficiency distributions coincide and
$R(Y \rightsquigarrow X \mid \mathbf{W}{=}\{M\}) = 0$.

\textit{Other cells.}
The remaining entries of Table~\ref{tab:signal_desid} (D-MY, D-XM, D-YM, D-MX,
D-MXY, and the $\mathbf{W} = \emptyset$ column) follow the same closed-form
machinery; full per-pair computations are in the companion notebook
\nb{signal\_mediation\_computations}.

\section{Setup for the Desert-Traveller PCI Computation}\label{app:desert}

This appendix records the structural model and the full PCI specification behind
the desert-traveller comparison of
Section~\ref{sub:pearl_actual_cause_prob}. The basic-variant results
(Table~\ref{tab:pci-basic-dt}) and the weak-poison results
(Table~\ref{tab:pci-weak-dt}) are produced by running the PCI Monte Carlo
search of Section~\ref{sec:causal_impact} on the scenario-specific model below,
with the component choices listed here; a parallel hand-rolled enumeration matches every
value within Monte Carlo noise (companion notebook
\nb{desert\_traveler}).

\subsection{Structural model}\label{app:desert_model}

To compare directly with Pearl, we run PCI under the same forensic
conditioning: one structural model per scenario, with $u$ pinned at the
scenario's forensic realisation. The endogenous variables are $X, P, c, d,
Y \in \{0,1\}$ with
\[
c = P\,(u' \vee X'), \qquad d = X\,(u \vee P'), \qquad y = c \vee d,
\]
where primes denote Boolean complement ($u'=1-u$, etc.), $c$ indicates the cyanide
path operating (poisoned water drunk) and $d$
the dehydration path operating (drank from empty canteen).

\subsection{PCI components}\label{app:desert_components}

The PCI components, all specified in advance:
\begin{itemize}\setlength\itemsep{0.1em}
    \item Suspects $\mathbf{S} = \{X, P\}$; candidate causes are the
        singletons $\{X\}$ and $\{P\}$.
    \item Witness pool $\mathbf{W} = \{c, d\}$ (the mediators); four subsets
        $\mathbf{T} \subseteq \mathbf{W}$, weighted by $\Gamma$
        cardinality-uniform (uniform on $|\mathbf{T}|$, then uniform within
        each cardinality class).
    \item We do \emph{not} intervene on or hard-code the non-cause suspect
        (the one not in $\mathbf{C}$); it follows the SCM under the noise
        prior on its own coin, consistent with the formal definition's
        requirement that we intervene only on $\mathbf{C}$ and $\mathbf{T}$
        (Definition~\ref{def:jointnecsuf}).
    \item Alternative-value distribution $\Delta$ deterministically flips
        $1 \to 0$.
    \item Noise distribution $P_{\mathbf{U}}$ is the forensic posterior, a
        point mass at the scenario's $u$. We do \emph{not} condition on the
        outcome $Y$ itself; the forensic evidence sits above $Y$ in the SCM,
        so conditioning on it is not the ``fully conditional'' anti-pattern of
        Section~\ref{sec:causal_impact}.
\end{itemize}

For the weak-poison variant (Section~\ref{sub:weak-poison}) the only mechanical
changes are that the witness pool becomes $\mathbf{W} = \{c, d, v_C\}$ (the new
``cyanide-was-fatal'' mediator $v_C$ is included) and that the Scenario~A
forensic posterior pins both $u = 0$ and $\xi = 1$; Scenario~B's posterior pins
$u = 1$, with $\xi$ structurally irrelevant since the cyanide path is dormant.

\section{Synthetic Evaluation Across Linear, Overdetermined, and Undercutting Archetypes}
\label{sub:synthetic}

This appendix gives the full model, estimator, and per-row analysis behind the synthetic-archetype result summarised in \cref{sec:evaluation}.

Section~\ref{sec:shap_examples} has compared \ourapproach
with SHAP and Causal SHAP on two small examples and shown where the SHAP
framework's desiderata fail. Those examples cover only two causal patterns:
overdetermination through a gating mediator (OBCB) and correlated mediation
(signal). This appendix stress-tests \ourapproach across a wider catalogue of
patterns, with an analytic ground truth available throughout: the question is
whether \ourapproach correctly separates linear, overdetermined, and preempted
contributions on a single structural model, and stays at the noise floor on a
variable that plays no role.

To evaluate \ourapproach in a setting that exhibits multiple distinct causal patterns
simultaneously, we construct a synthetic structural causal model carrying three archetypes,
linear necessary-and-sufficient, overdetermined-not-necessary, and preempted
(the undercutting archetype named in the section title, implemented here via a
gate),
together with an irrelevance control.\footnote{The classical motivating example is
Sally and Billy throwing stones at a bottle: Sally's stone arrives first and shatters it;
Billy's would have done the same had Sally's not preempted it. The case is canonical in
the actual-causality literature \citep{actualCausalityHalpern} and motivates the preempted
archetype below; we use generic letters rather than personal names to keep variable
labelling tight.}

The model has six independent root variables and three deterministic branches feeding the
outcome $E$ (the generic outcome $Y$ of Section~\ref{sec:causal_impact}, written $E$
here to match the model diagram; its factual value is $y^\star$ and its
necessity/sufficiency worlds $Y^n, Y^s$):
\begin{align*}
L_1, L_2 &\sim \mathcal{N}(0, 1), \quad O_1, O_2 \sim \mathcal{N}(1, 1), \quad
  P, D \sim \mathcal{N}(0, 1), \\
\mathrm{lin}      &= 5 L_1 + 10 L_2, \\
\mathrm{od}       &= \max(5 O_1,\, 5 O_2), \\
\mathrm{gate}     &= \mathbbm{1}\{\lvert L_2 \rvert \le \tau\}, \\
\mathrm{p\_branch}&= 5 P \cdot \mathrm{gate}, \\
E &= \mathrm{lin} + \mathrm{od} + \mathrm{p\_branch}.
\end{align*}
We pick $\tau = 0.674$ so the gate is on / off about half the time
($\Pr(\lvert L_2 \rvert \le 0.674) \approx 0.5$ for $L_2 \sim \mathcal{N}(0,1)$).
The DAG is shown in Figure~\ref{fig:archetypes_dag}. The variable $D$ is sampled but never
enters $E$, serving as an irrelevance control, so any non-zero score for $D$ is the
noise floor of the procedure.

\begin{figure}[htbp]
\centering
\includegraphics[width=0.78\linewidth]{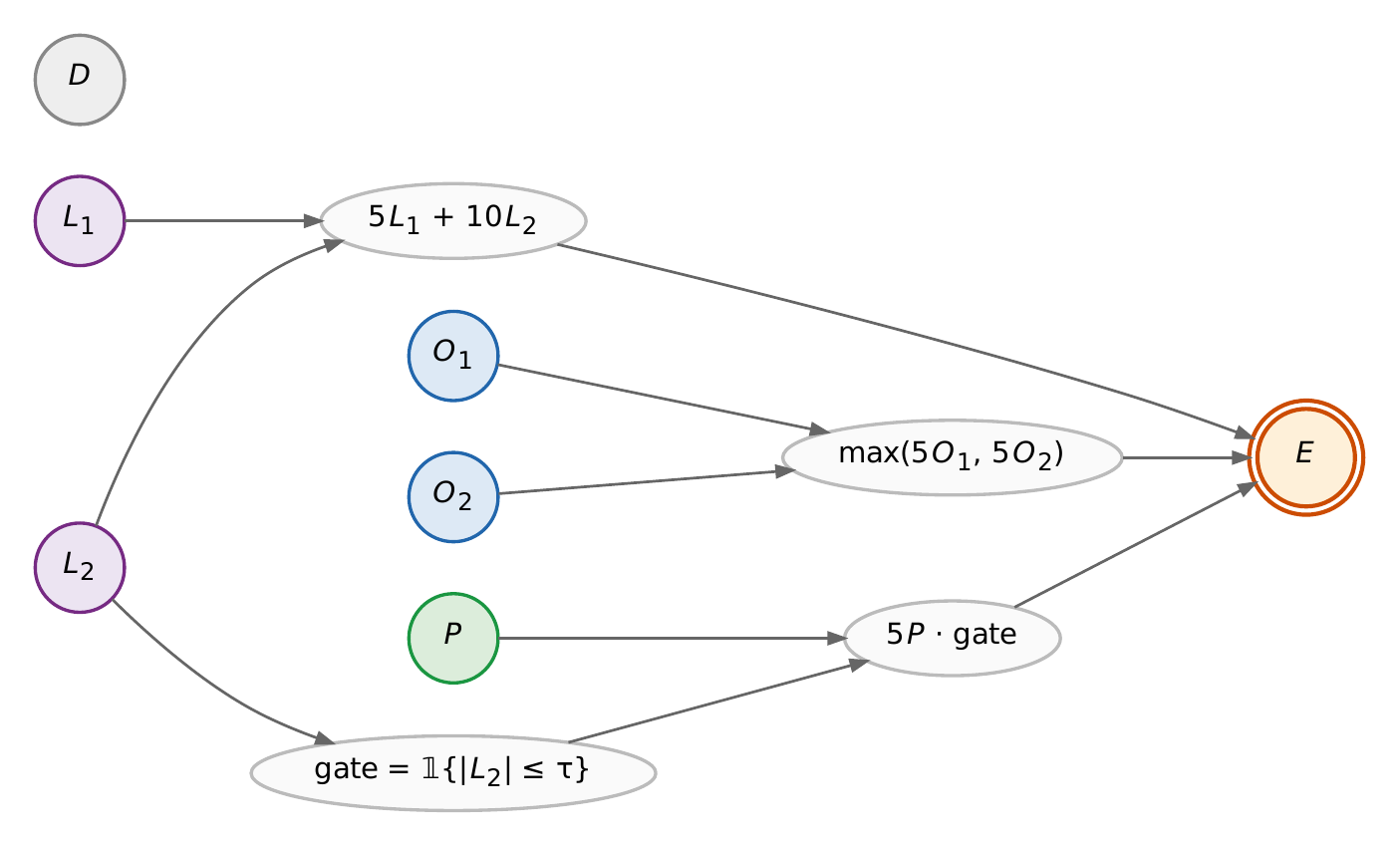}
\caption{Synthetic archetype model. Roots (circles) are independent draws:
$L_1, L_2 \sim \mathcal{N}(0,1)$ (linear, purple); $O_1, O_2 \sim \mathcal{N}(1,1)$
(overdetermined via $\max$, blue); $P \sim \mathcal{N}(0,1)$ (gated by
$\lvert L_2 \rvert \le \tau$, green); $D \sim \mathcal{N}(0,1)$ (irrelevance control,
grey). Mediators (white ellipses) carry the deterministic functions of their parents.
Outcome $E$ (orange double-circle) is the sum
$\mathrm{lin} + \mathrm{od} + \mathrm{p\_branch}$. The variable $D$ has no outgoing edges
by construction.}
\label{fig:archetypes_dag}
\end{figure}

The \ourapproach evaluation treats the six roots as the suspect set; deterministic
mediators are not suspects. Restricting suspects to roots loses no causal information here
since downstream values are determined by the roots (formal statement deferred to
Section~\ref{sec:causal_impact}).

\paragraph{Causal impact, decomposed.}
We follow the impact function $ci(y^s, y^n, y^\star) = \lvert y^n - y^\star \rvert -
\lvert y^s - y^\star \rvert$ defined in Section~\ref{sec:causal_impact}, but report its
two components separately as $ci_N$ and $ci_S$:
\begin{align*}
ci_N(V) &= \mathbb{E}\!\left[\,\lvert Y^n - y^\star \rvert \,\big|\, V \text{ intervened}\right]
  \quad \text{(necessity)}, \\
ci_S(V) &= -\,\mathbb{E}\!\left[\,\lvert Y^s - y^\star \rvert \,\big|\, V \text{ pinned at factual}\right]
  \quad \text{(sufficiency)},
\end{align*}
with $ci(V) = ci_N(V) + ci_S(V)$. Higher $ci_N$ means more necessary; higher $ci_S$
(closer to zero) means more sufficient. The raw scores carry a co-intervention baseline
that does not vanish for $D$, so we report \emph{excess} over $D$ within each case:
\[
\Delta_N(V) = ci_N(V) - ci_N(D),
\qquad
\Delta_S(V) = ci_S(V) - ci_S(D).
\]
$D$ is the noise floor; any other variable's $\Delta$ is its causal signal above it.

\paragraph{Two factual cases.}
We instantiate the model with one batch of $500$ events and select two factual
observations to read all archetypes off.

\textit{Case 1 (preempted, contestable $O$ pair).} Filter:
$\lvert L_2 \rvert > \tau$ (gate off, so $\mathrm{p\_branch} = 0$ regardless of $P$) and
$\lvert O_1 - O_2 \rvert < \kappa$ with $\kappa = 1$ (the two overdetermined contributors
are close: the winner is barely the winner; we call such a pair \emph{contestable},
since a small perturbation to either contributor could flip which one wins
$\max(5O_1, 5O_2)$). Linear N+S on the $L$s, between-variable
sufficiency asymmetry on a contestable $O$ pair, and noise floor for both $P$ and $D$
should all be readable from this single observation.

\textit{Case 2 (unpreempted, dominant $O$ winner).} Filter:
$\lvert L_2 \rvert \le \tau$ (gate on, $P$ branch active) and
$\lvert O_1 - O_2 \rvert > \kappa$. The dominant winner makes the within-variable $S$-not-$N$ gap visible: the winner is large enough that alternatives
rarely beat it (high $\Delta_S$), and the loser provides a partial floor when the winner
is intervened (reduced $\Delta_N$). A contestable pair cannot show this within-variable
gap, since pinning either one still lets the partner flip the max.

The two cases together test the structural-role-not-identity claim: the same variable $P$
should sit at the noise floor in Case~1 and mirror $L_1$ in Case~2.

\paragraph{Desiderata and measurements.}
Table~\ref{tab:archetypes_desiderata} pairs each archetype expectation with the formal
inequality on the excess scores $\Delta_N, \Delta_S$, the measured value at the
production sample budget ($N = 20\,000$ Monte Carlo draws per intervention), and a
bootstrap 95\% confidence interval (4\,000 resamples, joint across $V$ and $D$ so
per-sample correlations are preserved). The noise-floor variable $D$ itself drifts
across the two cases: co-intervention variance is structurally larger in the gate-on
regime than in the gate-off regime, so $\mathrm{ci}_N(D)$ shifts between regimes by a
non-negligible amount. This is exactly why every archetype claim is stated on the
excess scores $\Delta_N, \Delta_S$ rather than raw ci values: subtracting $D$ within
each case normalises out the regime-dependent floor, so the archetype patterns can be
read off stably regardless of $D$'s absolute drift.

\paragraph{The threshold comes from the data.}
The directional desiderata of the form $\Delta > \epsilon$ or $\lvert \Delta \rvert <
\epsilon$ all depend on a single tolerance $\epsilon$. We do not hand-pick it; we
derive it from the Monte Carlo budget. For every $\Delta$ score appearing in any
desideratum we bootstrap a standard error $\sigma_V$ and take
\[
\epsilon \;=\; 1.645 \cdot \sigma_{\max},
\qquad
\sigma_{\max} \;=\; \max_{V,\text{case}}\, \sigma_V,
\]
the one-sided $95\%$ Z-quantile times the largest standard error across all measured
$\Delta$s. Under the null hypothesis $\Delta = 0$, fewer than $5\%$ of measurements
would exceed $\epsilon$; a $\Delta$ that clears $\epsilon$ is therefore not Monte Carlo
chatter at the chosen budget. At $N = 20\,000$ samples we measure $\sigma_{\max} \approx
0.092$, giving $\epsilon \approx 0.151$. Smaller budgets give a noisier $\epsilon$ that
defeats the point (e.g.\ at $N = 5\,000$, $\sigma_{\max} \approx 0.18$ would push
$\epsilon$ into the signal range); this fixes the production sample size for the
section. Ten of the twelve desiderata we considered survive this
threshold; the two that do not (cross-regime $P$ flip, and quantitative match of
$\Delta_N(P)$ to $\Delta_N(L_1)$ in Case~2) are discussed below as diagnostics.

\begin{table}[ht]
\centering
\footnotesize
\setlength{\tabcolsep}{4pt}
\begin{tabular}{r p{7.5cm} r r r}
\toprule
\# & Formal desideratum & Measured & 95\% CI & Result \\
\midrule
1  & $\Delta_N(L_2)$ is the largest excess necessity, Case~1
   & $+2.129$ & $[+1.953,\,+2.300]$ & \checkmark \\
2  & $\Delta_N(L_2)$ is the largest excess necessity, Case~2
   & $+0.642$ & $[+0.476,\,+0.813]$ & \checkmark \\
3  & $\Delta_S(O_w) - \Delta_S(O_l) > \epsilon$, Case~1 (contestable)
   & $+0.347$ & $[+0.216,\,+0.473]$ & \checkmark \\
4  & $\Delta_S(O_w) - \Delta_S(O_l) > \epsilon$, Case~2 (dominant)
   & $+0.529$ & $[+0.387,\,+0.667]$ & \checkmark \\
5  & $\lvert \Delta_N(O_w) - \Delta_N(O_l) \rvert < \epsilon$, Case~1
   & $-0.093$ & $[-0.266,\,+0.081]$ & \checkmark \\
6  & $\Delta_S(O_w) - \Delta_N(O_w) > \epsilon$, Case~2
       \emph{(within-variable $S$-not-$N$ signature)}
   & $+0.487$ & $[+0.256,\,+0.712]$ & \checkmark \\
7  & $\Delta_N(L_2) - \Delta_N(O_w) > \epsilon$, Case~2
   & $+1.082$ & $[+0.915,\,+1.251]$ & \checkmark \\
8  & $\lvert \Delta_N(P) \rvert < \epsilon$, Case~1 (gate off)
   & $+0.112$ & $[-0.058,\,+0.280]$ & \checkmark \\
9  & $\Delta_N(L_1) - \Delta_N(P) > \epsilon$, Case~1
   & $+0.670$ & $[+0.499,\,+0.843]$ & \checkmark \\
10 & $\Delta_N(P) > \epsilon$, Case~2 (gate on)
   & $+0.250$ & $[+0.079,\,+0.422]$ & \checkmark \\
\bottomrule
\end{tabular}
\caption{Desiderata satisfaction on the synthetic archetype model.
Threshold $\epsilon = 1.645 \cdot \sigma_{\max} \approx 0.151$ is the smallest $\Delta$
gap that, at $N = 20\,000$ samples, exceeds Monte Carlo noise at one-sided $95\%$
confidence. Subscripts $w/l$ denote the winner / loser of the
$\max(5 O_1, 5 O_2)$ branch in the given case. Per-row check identifiers and the
underlying numerical computations are listed in the companion notebook
\nb{responsibility\_archetypes}.}
\label{tab:archetypes_desiderata}
\end{table}

Figure~\ref{fig:eval_archetypes} visualises the same ten claims as a forest
plot: each row carries a plain-language statement of the structural expectation; the
dot is the point estimate of the relevant $\Delta$ score (or difference of $\Delta$
scores); the horizontal bar is its bootstrap $95\%$ CI; and the green band is the row's
\emph{pass region}, i.e.\ the set of $\Delta$ values that would confirm that row's
claim given $\epsilon$. The pass region depends on the claim type: $\Delta > \epsilon$
claims pass to the right of $\epsilon$; $\lvert \Delta \rvert < \epsilon$ closeness
claims pass inside the strip $(-\epsilon, +\epsilon)$; rank claims (\#1 and \#2) pass
to the right of $0$, since they compare $\Delta_N(L_2)$ against the largest competing
suspect. We count a row as confirmed when the point estimate lies in its green band; we show the
CI separately so the reader can see when a measurement is statistically tight
and when its CI crosses the pass boundary. All ten dots fall inside their pass bands.
The tightest row is \#10 ($\Delta_N(P) > \epsilon$ in Case~2, gate on): the point
estimate $+0.250$ clears $\epsilon \approx 0.151$ by roughly one $\sigma$, with its CI
lower edge at $+0.079$ still inside the band on the right side of zero.

\paragraph{Diagnostics that do not survive the bootstrap-derived threshold.}
Two further claims about the preempted variable $P$
fail at this threshold; we report them as diagnostics rather than as numbered
desiderata. (D1) The cross-regime flip
$\Delta_N(P,\text{Case~2}) - \Delta_N(P,\text{Case~1})$ measures $+0.139$ with CI
$[-0.081, +0.365]$: the CI straddles zero, so we cannot statistically resolve $P$'s
rise across regimes on its own at the production budget. The qualitative content of
this claim is already carried by desiderata \#8 (P at the noise floor in Case~1) and
\#10 (P above $\epsilon$ in Case~2): if both pass, the regime flip is a corollary, not
an independent claim. (D2) The quantitative match
$\lvert \Delta_N(P) - \Delta_N(L_1) \rvert < \epsilon$ in Case~2: the signed gap
$\Delta_N(P) - \Delta_N(L_1)$ measures $-0.267$ with
CI $[-0.441, -0.106]$; the point estimate's magnitude $0.267$ lies above $\epsilon$, though
the interval's near edge ($0.106$) dips below it. At this sample budget,
gated $P$ does not quantitatively mirror ungated $L_1$; what we can claim is the
weaker, qualitative statement that both clear the noise floor in Case~2, which is
exactly desideratum \#10 for $P$ and built into Case~2 by construction for $L_1$.

The archetype evaluation establishes that \ourapproach
recovers the right qualitative pattern on every standard causal motif against an
analytic ground truth. \cref{sec:DCE} now turns to a complementary
continuous-input comparison: how does \ourapproach behave next to gradient-based
attribution (Differential Causal Effect of \citealp{butler2022differential}) when
the same continuous cause is rescaled or moved out of the steep region of the
response surface, where the gradient itself collapses but causal responsibility
arguably should not?

\section{Comparison to Differential Causal Effect}\label{sec:DCE}

This appendix details the Differential Causal Effect comparison summarised in \cref{sec:evaluation}.

Having tested \ourapproach against ground truth in
\cref{sub:synthetic}, we now compare it with gradient-based
attribution on a continuous-input credit-limit example.

A natural baseline for continuous-input attribution is to read causal influence off
the gradient of an interventional mean function. The motivation is direct:
$\mathbb{E}\!\bigl[Y\mid do(X{=}1)\bigr]-\mathbb{E}\!\bigl[Y\mid do(X{=}0)\bigr]$
is the binary ATE, and for a continuous cause one ``arrives at derivative
calculus'' \citep[Section~7.4.3]{ness2025causal} by letting the contrast shrink:
$\partial_x\,\mathbb{E}\!\bigl[Y\mid do(X{=}x)\bigr]$. Generalised to a vector
$\mathbf{x}=(x_1,\ldots,x_d)$ with interventional mean
$m(\mathbf{x})=\mathbb{E}\!\bigl[Y_{\mathbf{X}=\mathbf{x}}\bigr]$, the
\emph{causal gradient} is
\[
\nabla_{\mathbf{x}}\,m(\mathbf{x})
=\Bigl(\partial_{x_1}\mathbb{E}\!\bigl[Y_{\mathbf{X}=\mathbf{x}}\bigr],\;\ldots,\;
        \partial_{x_d}\mathbb{E}\!\bigl[Y_{\mathbf{X}=\mathbf{x}}\bigr]\Bigr),
\]
and each partial is the \emph{differential causal effect} (DCE) of a single
variable at $\mathbf{x}$. \citet{butler2022differential} introduced this
quantity in the Gaussian Process regression setting, where, with the posterior
expected mean expressible as $\sum_n k(x,x^{(n)})\alpha_n$, the DCE has the
closed form $\partial_{x_i}\hat{F}=\sum_n \partial_{x_i}k(x,x^{(n)})\alpha_n$.

\ourapproach and DCE are not direct competitors (DCE estimates a rate of
change, \ourapproach estimates individual-level responsibility), but
\ourapproach inherits the territory of continuous attribution, so a contrast
on a single example is informative. We highlight the differences:
\begin{enumerate}
\item DCE requires $y=F(\mathbf{X},\epsilon)$ with $F$ differentiable;
  \ourapproach requires only that the causal model be sampleable, and the
  inference algorithm for the expanded model is the user's choice (we use Monte
  Carlo throughout).
\item DCE inherits the units of both $\mathbf{X}$ and $Y$, so cross-feature
  comparison is unit-sensitive. The \ourapproach score integrates over the
  necessity- and sufficiency-world densities, so changes to input scale leave
  the score essentially unchanged.\footnote{This unit-invariance is contingent
  on how the alternative-value distribution $\Delta$ is specified. The
  $\epsilon$-excised $\Delta$ of Section~\ref{sec:causal_impact} is invariant to
  a reparameterisation of the input scale precisely when $\epsilon$ is fixed by
  the \emph{probability mass} excised around the factual value (equivalently, in
  standardised units), as in the ROPE/effect-size reading given there, rather
  than as an absolute distance in the raw units of $\mathbf{X}$: an
  absolute-distance $\epsilon$ held fixed across a rescaling would itself become
  unit-dependent and reintroduce the very sensitivity PCI otherwise avoids.}
\item DCE varies abruptly with the local geometry of $F$. The \ourapproach
  score depends on where the factual value sits relative to the
  necessity/sufficiency densities and is therefore non-zero in flat regions
  while remaining locally informative.
\end{enumerate}

\begin{example}[credit-limit assignment]\label{ex:age}
We model an applicant's approved credit limit as a function of \emph{age} and
the \emph{time of application}. The deterministic component is a sigmoidal
rise in age with a small sinusoidal modulation in time-of-day; intuitively,
age should dominate. We consider four settings on two axes: time measured in
hours vs.\ minutes (a units rescaling that should not affect attribution), and
an age prior centred at $30$ vs.\ $40$ (a genuine change in what counts as
unusual). The deterministic component is shown in
Figure~\ref{fig:internal}. Computational details and code are in
\texttt{docs/source/gradient\_based\_attribution.ipynb}.
\end{example}

\begin{figure}[t]
    \centering
    \includegraphics[width=\linewidth]{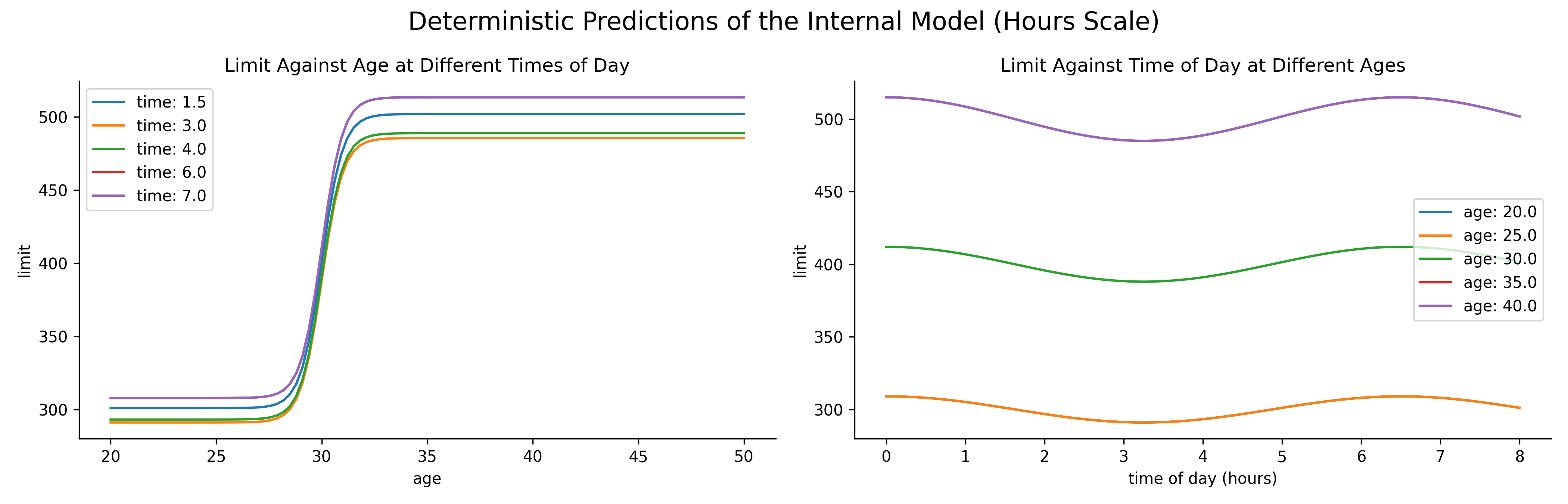}
    \caption{Deterministic component of the credit-limit model
    (Example~\ref{ex:age}), on the hours scale: limit vs.\ age at several times
    of day (left) and limit vs.\ time of day at several ages (right). Age
    dominates the limit, with a small sinusoidal modulation from time of day;
    the minute scale is identical modulo a linear reparameterisation (verified
    in the notebook).}\label{fig:internal}
\end{figure}

\paragraph{\ourapproach scores.} We instantiate $ci$ as the absolute-difference
score from Example~\ref{ex:absolute_score}, $ci(y^s,y^n,y^\star)=|y^n-y^\star|-|y^s-y^\star|$,
and run the search of Section~\ref{sec:causal_impact} on each setting,
conditioning on age and on time of application as singleton suspects.
Figure~\ref{fig:mechanism} shows the mechanism for the canonical setting
(age-$30$ prior, hours): when age is the active suspect the necessity-world
outcome tracks the deterministic curve while the sufficiency-world outcome
stays near factual, and the same picture, with a smaller displacement, holds
for time of application. Table~\ref{tab:dce_means} reports the mean total
scores across all four settings. Two patterns stand out. First, age outscores
time of application in every setting, by roughly a factor of $3/2$, matching the
intuition that the age sigmoid moves the outcome by a much larger absolute
amount than the daily modulation does. Second, the scores are essentially
unchanged between the hours and minutes columns, confirming the unit-invariance
discussed above, whereas the shift from the age-$30$ to the age-$40$ prior
raises every score, since a more unusual factual age means a larger search-space
displacement.

\begin{figure}[t]
    \centering
    \includegraphics[width=\linewidth]{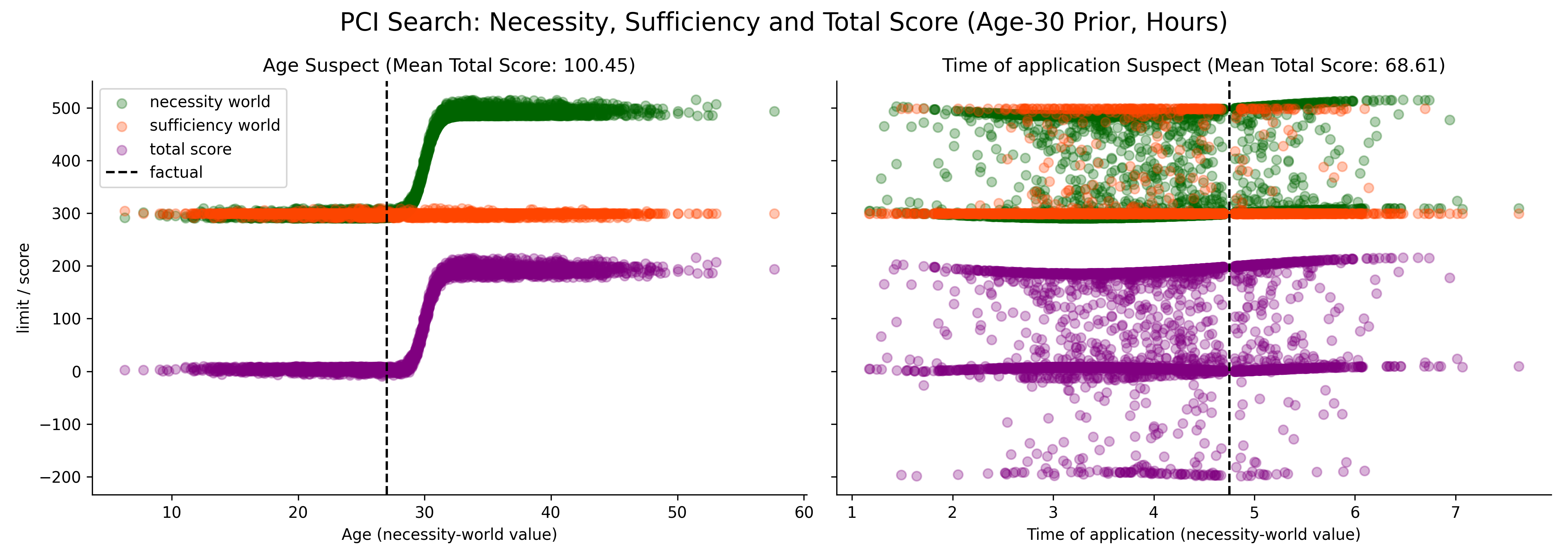}
    \caption{PCI search for the canonical setting (age-$30$ prior, hours), with
    age (left) and time of application (right) as the active suspect.
    Necessity-world (green) and sufficiency-world (orange) limit samples and the
    total score (purple) are plotted against the necessity-world value of the
    suspect; the dashed line marks the factual. The necessity world departs from
    factual along the deterministic response while the sufficiency world stays
    near it; each panel title gives the mean total score.}\label{fig:mechanism}
\end{figure}

\begin{table}[t]
\centering
\caption{Mean \ourapproach total scores for age and time of application across
the four settings (time unit $\times$ age prior), at the factual instance. Age
outscores time in every setting; the hours and minutes columns agree
(unit-invariance), while the age-$40$ prior raises all scores relative to
age-$30$ (a more unusual factual age).}
\label{tab:dce_means}
\vspace{4pt}
\begin{tabular}{l cc cc}
\toprule
& \multicolumn{2}{c}{age-$30$ prior} & \multicolumn{2}{c}{age-$40$ prior} \\
\cmidrule(lr){2-3}\cmidrule(lr){4-5}
Suspect & hours & minutes & hours & minutes \\
\midrule
Age                  & $100.45$ & $97.95$  & $177.87$ & $176.56$ \\
Time of application  & $68.61$  & $66.51$  & $117.97$ & $118.42$ \\
\bottomrule
\end{tabular}
\end{table}

\paragraph{PCI vs.\ DCE across the input grid.}
Figure~\ref{fig:eval_dce} (\cref{sec:evaluation}) contrasts the two methods on
the same grid of factual values (age-$30$ prior); the other three settings
give qualitatively the same picture. Panel (a) plots the difference between
the \ourapproach scores for age and for time of application: it is positive
everywhere, so age is consistently the more responsible variable. Panels (b)
and (c) plot the corresponding difference in DCE magnitudes, on the hours and
minutes scales respectively, and two differences with \ourapproach stand out.
First, on the hour scale the DCE for time of application exceeds the DCE for
age over most of the grid, the opposite of \ourapproach's verdict; only a
narrow band near the sigmoid's inflection point favours age, and only because
the gradient spikes sharply at that single point. Second, switching the time
unit from hours to minutes flattens the time DCE by roughly a factor of $60$,
collapsing almost all of panel (b)'s time-favouring region to near zero in
panel (c) and making the unit sensitivity explicit. Both are direct
consequences of DCE's role as a rate-of-change estimator rather than a
responsibility score, and both make it harder to compare across features and
across reparameterisations.

\FloatBarrier  

\section{Scaling Actual-Cause Computation: An Empirical Comparison}\label{sec:ac_benchmark}

This appendix details the actual-cause scaling benchmark summarised in \cref{sec:evaluation}.

The correspondence in Theorems~\ref{th:ac-exp}--\ref{th:exp-ac} and
Proposition~\ref{prop:filter} is a statement about verdicts: under the stated
hypotheses, AC actual causes coincide with the support of $\sum_{\mathbf{T}}\Phi$.
This appendix reports an empirical companion to the correspondence. We ask whether
the necessity-only PCI kernel of Definition~\ref{def:ackernel}, instantiated with a
finite Monte Carlo estimator, recovers AC verdicts at problem sizes where exact
subset enumeration is no longer tractable. The full implementation, raw outputs and
plotting code are in the companion notebook
\nb{actual\_causality\_benchmark}.

\medskip\noindent\textit{Generalised throwing problem.}
For a size parameter $n$ we replicate the canonical Sally/Billy structure $n$ times
in parallel. Binary suspects $A_1,\dots,A_n,B_1,\dots,B_n\in\{0,1\}$ feed
deterministic mediators
\(W^a_i = A_i,\quad W^b_i = B_i \land \lnot A_i\)
that capture overdetermination via preemption: $B_i$ contributes only if $A_i$ does
not. The outcome is the conjunction $Y = \bigwedge_{i=1}^{n} (W^a_i \lor W^b_i)$, and
the model size grows as $4n+1$ variables. Ground-truth responsibility at site $i$ is
read off during the forward pass: $A_i$ is responsible if $W^a_i=1$, $B_i$ is
responsible if $W^b_i=1$. The structure scales the canonical example without softening it: overdetermination,
undercutting, and witness-mediated symmetry-breaking are present at every site,
while keeping a closed-form ground truth available for evaluation.

The variable partition takes the roots $\{A_i,B_i\}_i$ as suspects and the
deterministic mediators $\{W^a_i,W^b_i\}_i$ as witnesses, with no suspect a
deterministic function of the others. This partition satisfies the closure
condition required by Theorem~\ref{th:ac-exp} (see the Remark on deterministic
intermediates and joint support, which addresses precisely the scenario in which
both a root and its structural descendant are admitted as suspects). The
benchmark is therefore within the scope of the formal correspondence; in
particular, joint full support of the alternative-value distribution holds on the
admissible suspect set $\{A_i,B_i\}_i$ even though the model itself is
deterministic, because witnesses are not themselves part of the alternative-value
distribution.

\medskip\noindent\textit{Exact subset enumeration (AC reference).}
The exact procedure mechanises Definition~\ref{def:ac}: for each non-empty cause
subset $\mathbf{C}\subseteq\{A_i,B_i\}_i$ and each non-empty witness subset
$\mathbf{T}\subseteq\{W^a_i,W^b_i\}_i$, we intervene
$[\mathbf{C}\!\leftarrow\!\mathbf{1}-\mathbf{c}^\star,\mathbf{T}\!\leftarrow\!\mathbf{w}^\star]$,
record whether the outcome flips, and accept $\mathbf{C}$ as actual cause if some
witness $\mathbf{T}$ achieves a flip and no proper subset of $\mathbf{C}$ does.
The search space has $\bigl(2^{2n}-1\bigr)^2$ cause-witness pairs ($9, 225, 3969,
65025$ for $n=1,2,3,4$), which already determines the ceiling reached by the exact
method.

\medskip\noindent\textit{Approximate PCI estimator.}
The PCI estimator instantiates the kernel of Definition~\ref{def:ackernel}
via the Monte Carlo search procedure of Section~\ref{sec:causal_impact}. Algorithm~\ref{alg:pci_thin_search} is the
explicit procedure we use to produce the curves in
Figure~\ref{fig:eval_scaling} and Figure~\ref{fig:ac_search_space}.
In words: for each of $N$ draws it picks a random small set of suspects to flip and a
random set of witnesses to pin at their factual values, then runs the model twice under the
\emph{same} noise, once with the suspects restored to factual (the sufficiency world) and
once with them flipped to their alternatives (the necessity world), and records whether
the flip changed the outcome. Each suspect's score is the fraction of its draws in which
flipping it changed the outcome, and at each site the procedure returns the higher-scoring
of the two rivals $A_i, B_i$ as the cause. The procedure transcribes the \texttt{ThinSearchSampler}
used in the companion notebook \nb{actual\_causality\_benchmark}.

\begin{algorithm}[ht]
\caption{Approximate PCI estimator on the scaled throwing problem.}
\label{alg:pci_thin_search}
\begin{algorithmic}[1]
\Require PSCM $M$; factual world
$(\mathbf{s}^\star,\mathbf{w}^\star,y^\star)$; cardinality bounds
$n_s, n_w$; sample budget $N$; witness flag
$\mathrm{wit}\in\{\text{on},\text{off}\}$
\Ensure per-site verdicts $v_i\in\{A_i,B_i\}$ for $i=1,\dots,n$
\If{$\mathrm{wit} = \text{off}$}
  \State $\mathbf{W} \gets \emptyset$ \Comment{no witness candidates}
\EndIf
\For{each $X\in\mathbf{S}$}
  \State $L[X] \gets [\,]$ \Comment{per-suspect score list}
\EndFor
\For{$j = 1,\dots,N$}
  \State draw $\mathbf{u}_j \sim P_{\mathbf{U}}$
  \State sample $k_s \sim \mathrm{Unif}\{1,\dots,n_s\}$,\quad
         $k_w \sim \mathrm{Unif}\{0,\dots,n_w\}$
  \State sample $\mathbf{C}_j \subseteq \mathbf{S}$ uniformly with
         $|\mathbf{C}_j| = k_s$
  \State sample $D_j \subseteq \mathbf{W}$ uniformly with $|D_j| = k_w$
         \Comment{witnesses dropped}
  \State $\mathbf{T}_j \gets \mathbf{W} \setminus D_j$
         \Comment{active witnesses}
  \State $\mathbf{c}'_j \gets \mathbf{1} - \mathbf{s}^\star\!\mid_{\mathbf{C}_j}$
         \Comment{bit-flip on $\{0,1\}$}
  \State $y^s_j \gets M.\mathrm{run}\bigl(\mathbf{u}_j,\,
         \mathrm{do}(\mathbf{C}_j = \mathbf{s}^\star\!\mid_{\mathbf{C}_j},\,
         \mathbf{T}_j = \mathbf{w}^\star\!\mid_{\mathbf{T}_j})\bigr)$
         \Comment{sufficiency}
  \State $y^n_j \gets M.\mathrm{run}\bigl(\mathbf{u}_j,\,
         \mathrm{do}(\mathbf{C}_j = \mathbf{c}'_j,\,
         \mathbf{T}_j = \mathbf{w}^\star\!\mid_{\mathbf{T}_j})\bigr)$
         \Comment{necessity, shared $\mathbf{u}_j$}
  \For{each $X \in \mathbf{C}_j$}
    \State append $|y^n_j - y^\star|$ to $L[X]$
  \EndFor
\EndFor
\For{each $X\in\mathbf{S}$}
  \State $\mathrm{score}[X] \gets \mathrm{nanmean}(L[X])$
\EndFor
\For{$i = 1,\dots,n$}
  \State $v_i \gets \arg\max\bigl(\mathrm{score}[A_i],\,\mathrm{score}[B_i]\bigr)$
\EndFor
\State \Return $(v_1, \dots, v_n)$
\end{algorithmic}
\end{algorithm}

The algorithm draws $N$ intervention configurations under shared noise.
Line~10 draws the active suspect set $\mathbf{C}_j$ at uniformly chosen
cardinality $k_s \in \{1,\dots,n_s\}$; line~11 draws the dropped witness set $D_j$
the same way at $k_w \in \{0,\dots,n_w\}$, and the active
witnesses on line~12 are the complement $\mathbf{W}\setminus D_j$, so
$n_w$ bounds the number of witnesses \emph{dropped}, not the number kept,
and the active witness set has size $|\mathbf{W}|-n_w$ or larger. Line~2
handles the witness ablation flag $\mathrm{wit}=\text{off}$ once,
by emptying $\mathbf{W}$; this collapses the witness machinery on
lines~11--12 and 14--15 to no-ops, producing the dotted accuracy curves of
Figure~\ref{fig:eval_scaling}.

Alternative values on line~13 are bit flips: on $\{0,1\}$ the
$\epsilon$-excised alternative-value distribution has a single support
point and $\epsilon$ plays no role; we preserve the explicit sampling form
so that the same routine runs on continuous benchmarks. The two
interventional worlds on lines~14--15 share the noise realisation
$\mathbf{u}_j$ drawn on line~8, the standard counterfactual coupling,
which lets us read the per-suspect score on lines~16--18 off as the
nan-mean of $|y^n_j - y^\star|$ over the draws in which
$X\in\mathbf{C}_j$. The algorithm computes the sufficiency outcome $y^s_j$ but does not
consume it: the AC verdict against which we benchmark has no sufficiency
clause, so we read only the necessity world. We exercise the joint estimator with
sufficiency active in Appendices~\ref{sub:synthetic}, \ref{sec:DCE}, and~\ref{sec:sir_benchmark}, where AC has no comparable
verdict to benchmark against. At each site
$i = 1,\dots,n$, line~24 returns the suspect, $A_i$ or $B_i$, with
the higher mean.

\medskip\noindent\textit{Benchmark protocol.}
For each problem size $n$ we sample up to ten distinct factive worlds (with
\(Y=1\)), instantiate the model, run both methods to verdicts, and compare to the
analytic ground truth. We grant each method a 60-minute global budget per run; if
the budget runs out mid-loop, the largest completed $n$ caps the curve. We exercise the
approximate estimator in five configurations that vary along three axes
identified in Section~\ref{sec:causal_impact}:
(i) initial sample budget $N_0\in\{250,500\}$, scaled linearly with $n$ by a factor
$\alpha\in\{0.5,1\}$ so that the total sample at size $n$ is $\alpha\cdot N_0\cdot n$;
(ii) presence of witness interventions ($\mathrm{wit}\in\{\text{on},\text{off}\}$),
ablating $\mathbf{T}\!\leftarrow\!\mathbf{w}^\star$ to test whether witness pinning is
decisive; (iii) cardinality bounds $n_s,n_w$, fixed at $4$ in the default runs
(answering ``does some witness set of size $\geq|\mathbf{W}|-4$ work?'') and grown
linearly as $4n$ in the dynamic-set-size configuration (recovering the unrestricted
subset search at the cost of higher variance). The cardinality bound is a tunable
approximation knob, not a hidden assumption: the dynamic-set-size run measures the
empirical cost of bounding $n_s,n_w$ at $4$ on this problem, and the ablation measures the gap between the two configurations. Section~\ref{sec:causal_impact} gives the conceptual case for
constraining cardinality at all: uniform sampling over the full powerset
biases attributions toward large-coalition interventions, analogous to
declining to enumerate all $n$-way regression interaction terms.

\medskip\noindent\textit{Results.}
The results cover runtime, search-space size, and accuracy. First, runtime
(Figure~\ref{fig:eval_scaling}): the exact procedure times out at $4n+1=17$ variables
($n=4$) once the cause-witness search space crosses sixty-five thousand pairs, while
the approximate methods continue to roughly $73$--$145$ variables depending on
sample budget, with the frugal $\alpha{=}0.5$ run reaching the largest sizes.
Second, search-space size (Figure~\ref{fig:ac_search_space}, log scale): the exact
method visits the full powerset, while the approximate methods visit at most a few
hundred unique configurations per problem size: the gap is several orders of
magnitude and widens with $n$. Third, accuracy (Figure~\ref{fig:eval_scaling}): the
exact method is correct by construction up to its early demise; the witness-using
approximate methods stay above $0.85$ correct-attribution rate up to $73$--$89$
variables and degrade gracefully beyond that, with the dynamic-set-size run
sustaining among the highest accuracies at moderate problem sizes. Witness-free
runs (dotted) underperform at nearly every problem size, starting near $0.67$
already at $n=1$; at matched budget ($250$ samples, $\alpha{=}1$) the witness runs
average $0.941$ against $0.861$ and lead at $22$ of the $26$ shared sizes. The
ablation confirms a structural claim of Section~\ref{sec:causal_impact}: pinning
factual witness values is doing real work.

\begin{figure}[htbp]
\centering
\includegraphics[width=0.92\linewidth]{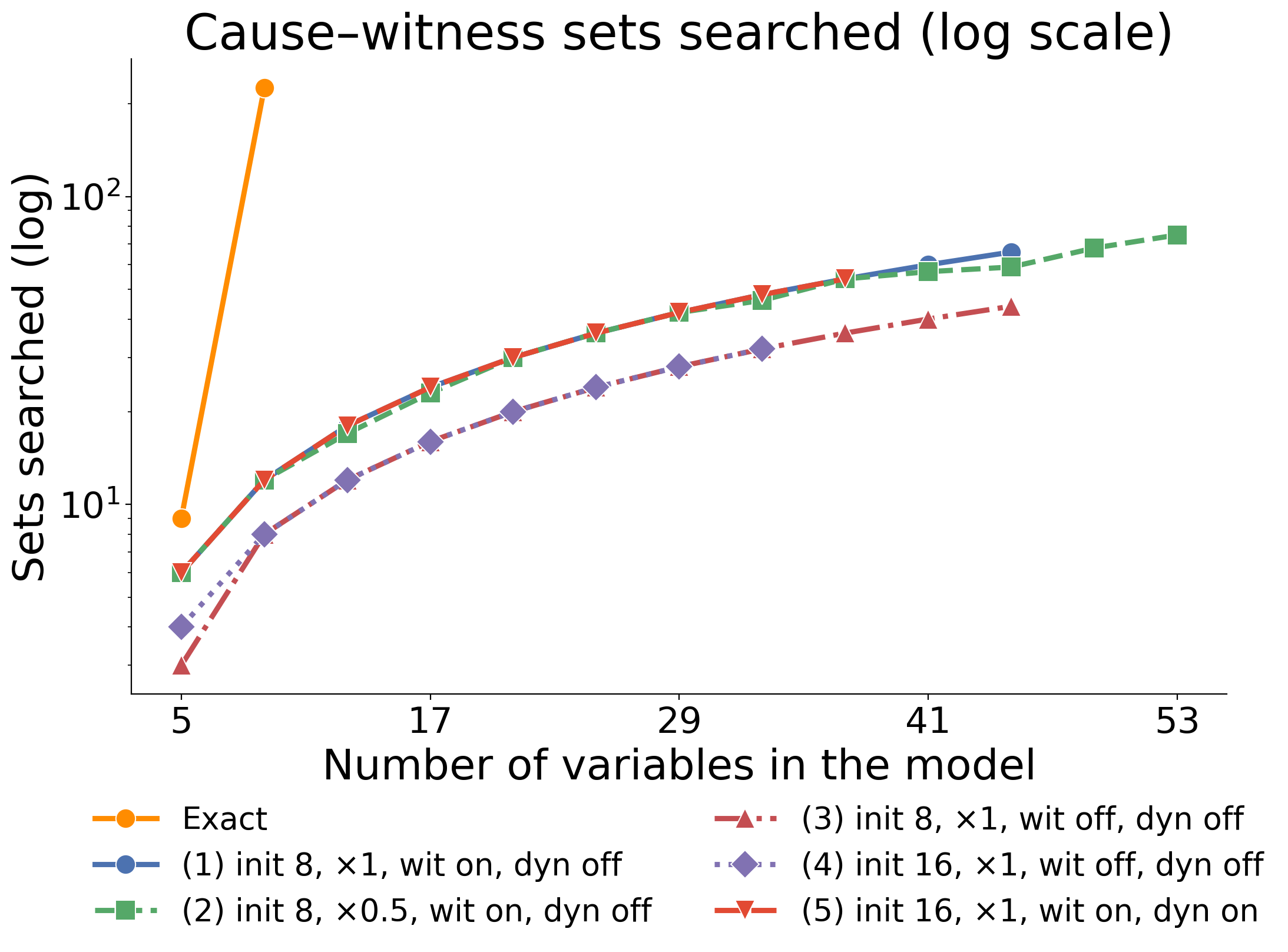}
\caption{Number of cause-witness configurations explored per problem size, log
scale. The exact method visits the full powerset; the approximate methods visit a
few hundred unique configurations per problem size regardless of $n$, a gap of
several orders of magnitude that grows with $n$.}
\label{fig:ac_search_space}
\end{figure}

\medskip\noindent\textit{Scope of the conclusion.}
The benchmark establishes that the necessity-only PCI estimator tracks AC verdicts
on a faithful scaling of the canonical example with several orders of magnitude
fewer evaluated configurations. The problem itself, however, is structured:
responsibility attribution is independent across sites, exactly one of $A_i,B_i$ is responsible at
each site given $C{=}1$, supersets of any witness set witness the same attribution
(so a small witness-cardinality bound costs little), and all variables are binary.
These regularities account for some of the favourable scaling: in particular,
the cardinality-bounded sampling losing little to dynamic-set-sizes sampling on
this problem. Two extensions are natural follow-ups. First, we defer a comparison
against SHAP and causal SHAP on the same scaled-throwing problem to a
companion benchmark (\cref{sec:shap_examples} gives the SHAP-vs-\ourapproach{}
comparison on the smaller overdetermination and continuous-mediation examples); the sufficiency component of PCI distinguishes it from those baselines, and a clean head-to-head
requires the joint kernel rather than the necessity-only restriction used here.
Second, benchmarks that break each of the structural regularities above,
particularly continuous-variable extensions where AC's mechanical enumeration is
no longer even well-defined, would test the estimator outside its current
favourable regime. As a sanity check that the correspondence in
Theorems~\ref{th:ac-exp}--\ref{th:exp-ac} is empirically operational at scale, the
present comparison is sufficient.

The AC benchmark exercises the necessity-only
specialisation of \ourapproach on a discrete-outcome problem with a closed-form
ground truth. \cref{sec:sir_benchmark} shifts to the regime where AC has no
verdict to compare against: a Bayesian dynamical SIR model with a continuous
outcome and asymmetrically interacting policies. The point is to exercise the joint
necessity-sufficiency machinery on threshold events that the actual-causality
machinery cannot adjudicate, and to show that PCI separates lockdown and masking
where but-for analysis collapses them.

\section{Dynamical SIR Benchmark: PCI on Continuous Outcomes}\label{sec:sir_benchmark}

This appendix details the continuous-outcome SIR benchmark summarised in \cref{sec:evaluation}.

The actual-causality benchmark of \cref{sec:ac_benchmark} exercises the
necessity-only specialisation of \ourapproach on a discrete outcome with a
closed-form ground truth. This appendix shifts to the regime where AC has no
verdict to compare against: a Bayesian dynamical model with a continuous
outcome and asymmetrically interacting causes. The benchmark mirrors a public
ChiRho tutorial%
\footnote{\url{https://basisresearch.github.io/chirho/explainable_sir.html}}
in setup and query. We keep the same model and
the same query but replace the explanatory machinery with the \ourapproach
thin-search sampler. The full implementation, raw outputs and plotting code
are in the companion notebook
\nb{sir\_benchmark}.

\medskip\noindent\textit{Model.}
The dynamics are the standard SIR system
\(\dot S = -\beta S I,\, \dot I = \beta S I - \gamma I,\, \dot R = \gamma I\),
with Bayesian priors $\beta\sim\mathrm{Beta}(18,600)$ and
$\gamma\sim\mathrm{Beta}(1600,1600)$. Two non-pharmaceutical policies, each
enacted with prior probability $1/2$, modulate the transmission rate via an
intervention strength $l\in[0,1]$ that scales $\beta_0$ to $(1-l)\beta_0$.
Lockdown alone has efficiency $0.6$; masking efficiency depends on lockdown
($0.1$ under lockdown, $0.45$ on its own); the joint efficiency is the
clamped sum, capped at $0.95$. Lockdown is enacted at $t=1$, masking at
$t=1.5$. The outcome is the \emph{overshoot}, the peak-to-final $S$ gap
of the simulated trajectory, and the undesirable event is
$\mathrm{os\_too\_high}=\mathbbm{1}[\mathrm{overshoot}>24]$ on a $100$-person
population. A policy can \emph{raise} this probability relative to
no intervention at all: suppressing or flattening the infectious peak leaves
more susceptibles unspent at the peak, so they are drawn down more slowly
afterward, which can widen the peak-to-final gap that overshoot measures even
though the intervention is clearly the safer outcome in absolute infection
terms. A high $\Pr(\mathrm{overshoot}>24)$ under intervention is therefore not
a sign error; it is a property of this particular summary statistic.

\medskip\noindent\textit{But-for analysis gives a symmetric verdict.}
Conditioning on each pair of policy decisions and drawing $100$ predictives
yields $\Pr(\mathrm{overshoot}>24) \approx 0.07$ with no intervention, $0.76$
with both policies, $0.84$ with lockdown only, and $0.81$ with mask only; the two single-policy regimes are statistically indistinguishable by
but-for probability alone at this sample budget (a $0.03$ gap on $100$
predictive draws is well within Monte Carlo noise), which is exactly the
symmetry \ourapproach is designed to break.
The mechanism is the asymmetric joint efficiency: lockdown and mask together
reach a transmission reduction of only $0.7$, well short of the model's $0.95$
clamp, and mask's efficiency itself depends on whether lockdown is also
active (falling from $0.45$ alone to $0.1$ under lockdown). But-for cannot
disentangle this context dependence: it treats each policy as a binary cause
and reports an aggregate effect, with no way to express that lockdown's
causal contribution depends on whether masking is in force. Figure~\ref{fig:sir_trajectories}
visualises the four interventional regimes: slowing transmission delays the
infectious peak and leaves more susceptibles unspent when it arrives, so all
three intervention regimes show a similarly elevated overshoot relative to no
intervention, mirroring the near-equal but-for overshoot probabilities.

\begin{figure}[htbp]
\centering
\includegraphics[width=0.92\linewidth]{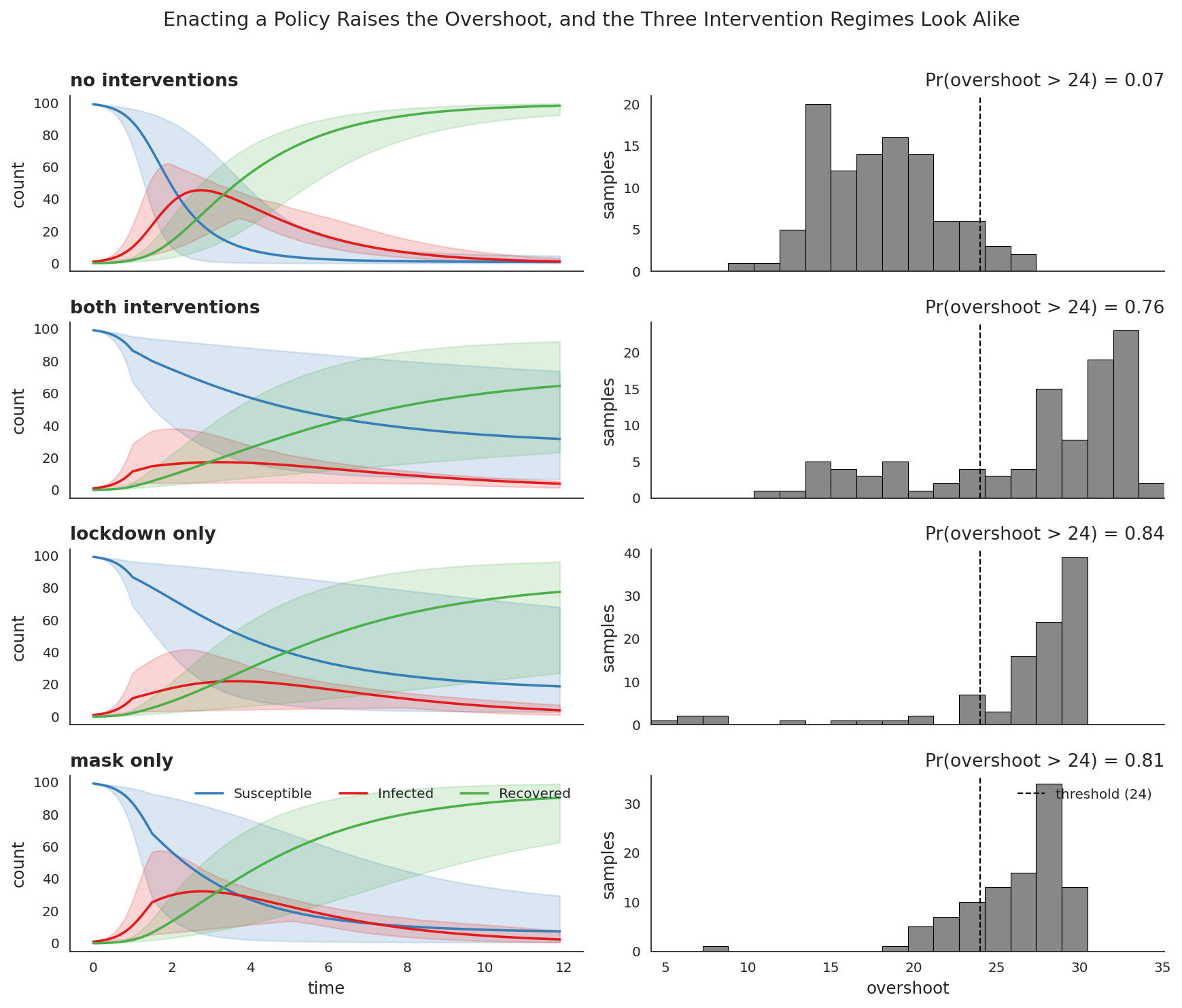}
\caption{SIR trajectories and overshoot distributions under the four
interventional regimes (no interventions, both policies, lockdown only, mask
only). Left: posterior bands for susceptible (blue), infectious (red) and
recovered (green) over time, with the lockdown enacted at $t=1$ and the mask
mandate at $t=1.5$. Right: per-regime histograms of the overshoot, with the
$24$-overshoot threshold marked. Delaying the infectious peak raises the
overshoot in every intervention regime, and the lockdown-only and mask-only
overshoot probabilities are near-indistinguishable: the asymmetric joint
efficiency that \ourapproach disentangles.}
\label{fig:sir_trajectories}
\end{figure}

\medskip\noindent\textit{The qualitative reading the model supports.}
The but-for numbers are symmetric, but the model is not, and the asymmetry
points to a clear qualitative reading that any responsibility method ought to
recover. First, \emph{lockdown is the proximate cause of excessive overshoot}:
it is the policy that does the real work of suppressing the infectious peak,
and the joint regime clears the $24$-overshoot threshold only because lockdown
is in force. Second, \emph{masking is at most a context-sensitive cause}: its
contribution depends on whether lockdown is also active. Masking is
comparatively effective only when lockdown is absent (efficiency $0.45$ versus
$0.1$ under lockdown), so once lockdown is present masking adds little. A good
attribution method should therefore rank lockdown above mask and, more finely,
flag mask's causal status as contingent on context rather than robust. These
two intuitions are the target of the analysis below; they follow from the
model's asymmetric joint efficiency, not from any particular tool, and they
coincide with the verdict the ChiRho tutorial reaches via classical
actual-cause probabilities. We state them here so the comparison that follows
does not depend on consulting that tutorial.

\medskip\noindent\textit{\ourapproach thin-search procedure.}
The suspects are $\{\mathrm{lockdown},\mathrm{mask}\}$. The witnesses are
drawn from the policy efficiencies
$\{\mathrm{lockdown\_efficiency}, \mathrm{mask\_efficiency},
\mathrm{joint\_efficiency}\}$, with the suspect itself excluded from the
witness set on each draw. The factual world fixes both policies on
($\mathrm{lockdown}=1, \mathrm{mask}=1$), so the factual outcome
$y^\star=\mathrm{overshoot}$ is read from the same trajectory. For each
suspect the sampler draws $2{,}500$ regimes, scoring each with the
$|y^\star-y_{\mathrm{nec}}| - |y^\star-y_{\mathrm{suff}}|$ statistic and
averaging across regimes. The $\mathrm{nec}$ term rewards regimes where
removing the suspect moves the outcome \emph{away} from $y^\star$; the
$\mathrm{suff}$ term rewards regimes where fixing the suspect at factual
keeps the outcome \emph{close} to $y^\star$.

\medskip\noindent\textit{The joint necessity--sufficiency plane.}
Before turning to numbers, it helps to see what the search produces. For each
suspect the sampler returns, per regime, a pair of outcomes: the
\emph{necessity-world} overshoot $y_{\mathrm{nec}}$ (the focal suspect set to
its alternative, any other active suspects in $\mathbf{C}\setminus\{X\}$
intervened, witnesses pinned at factual) and the \emph{sufficiency-world}
overshoot $y_{\mathrm{suff}}$ (the focal suspect held at factual, other active
suspects intervened, witnesses at factual). These are the necessity- and
sufficiency-world outcomes $y^n, y^s$ of Section~\ref{sec:causal_impact}; we
write $y_{\mathrm{nec}}, y_{\mathrm{suff}}$ here to match the figure axes. Plotting one point per regime in the
$(y_{\mathrm{nec}}, y_{\mathrm{suff}})$ plane gives a per-suspect cloud
(Figure~\ref{fig:eval_sir}) that exposes the score geometry before any
averaging. Three reference lines orient the plane: the dotted red cross-hair
marks the factual outcome $y^\star$ at the both-policies-on world
($\mathrm{lockdown}=\mathrm{mask}=1$); because that world is the high-overshoot
scenario, $y^\star$ sits near the upper-right of the cloud rather than at its
centre. The dash-dot blue line marks the $24$-overshoot threshold, and the
dashed black diagonal is $y_{\mathrm{nec}}=y_{\mathrm{suff}}$. The per-regime
score $|y^\star-y_{\mathrm{nec}}| - |y^\star-y_{\mathrm{suff}}|$ is positive
exactly when a point lies \emph{further from the cross-hair along the necessity
axis than along the sufficiency axis}: necessity pulls the outcome away from
factual (evidence the suspect was needed), while sufficiency keeps it close
(evidence the suspect alone sustains the outcome). A cause therefore shows up as
a cloud displaced toward smaller $y_{\mathrm{nec}}$ at fixed $y_{\mathrm{suff}}$ (the \emph{cause signature} on this geometry), whereas a non-cause hugs the
diagonal. Reading the two panels this way, the lockdown cloud is clearly
displaced above the diagonal toward smaller $y_{\mathrm{nec}}$, while the mask
cloud sits closer to it: the visual form of the verdict the next paragraph
quantifies.

\medskip\noindent\textit{Single-world results.}
\ourapproach ranks lockdown well above mask, a gap well outside Monte Carlo
noise:
\begin{center}
\small
\begin{tabular}{lccccc}
\toprule
suspect & $\overline{|y^\star\!-\!y_{\mathrm{nec}}|}$ &
          $-\overline{|y^\star\!-\!y_{\mathrm{suff}}|}$ &
          $\overline{ci}$ & $q$ & \ourapproach $=q^2\,\overline{ci}$ \\
\midrule
lockdown & $8.276$ & $-6.578$ & $1.698$ & $0.754$ & $\mathbf{0.965 \pm 0.091}$ \\
mask     & $7.536$ & $-6.673$ & $0.863$ & $0.754$ & $\mathbf{0.491 \pm 0.079}$ \\
\bottomrule
\end{tabular}
\end{center}
Here $q = \Pr_\Gamma[X_k\in\mathbf{C}] = 0.754$ is the sub-probability mass of
the suspect-containing regimes, so the reported score is the $q^2$-scaled
expectation of Definition~\ref{def:causalimpact} rather than a renormalised
conditional mean; the quarter of draws that leave the suspect out contribute
zero and stay in the average. The gap is driven mostly by the necessity
term: removing lockdown moves the overshoot further from $y^\star$ than
removing mask does ($8.276$ vs $7.536$ in mean absolute distance, $0.740$ of
the $0.835$ total gap), while the sufficiency terms differ by only $0.095$. This is consistent with the but-for finding that either
policy \emph{alone} is roughly equally bad. What distinguishes the suspects
is that lockdown's removal \emph{under context} (with witnesses pinned to
factual) pulls the world further from $y^\star$ than mask's removal does: the lockdown cloud's displacement above the diagonal toward smaller
$y_{\mathrm{nec}}$ in Figure~\ref{fig:eval_sir}, against the mask cloud's
near-diagonal position.

\begin{figure}[htbp]
\centering
\includegraphics[width=0.92\linewidth]{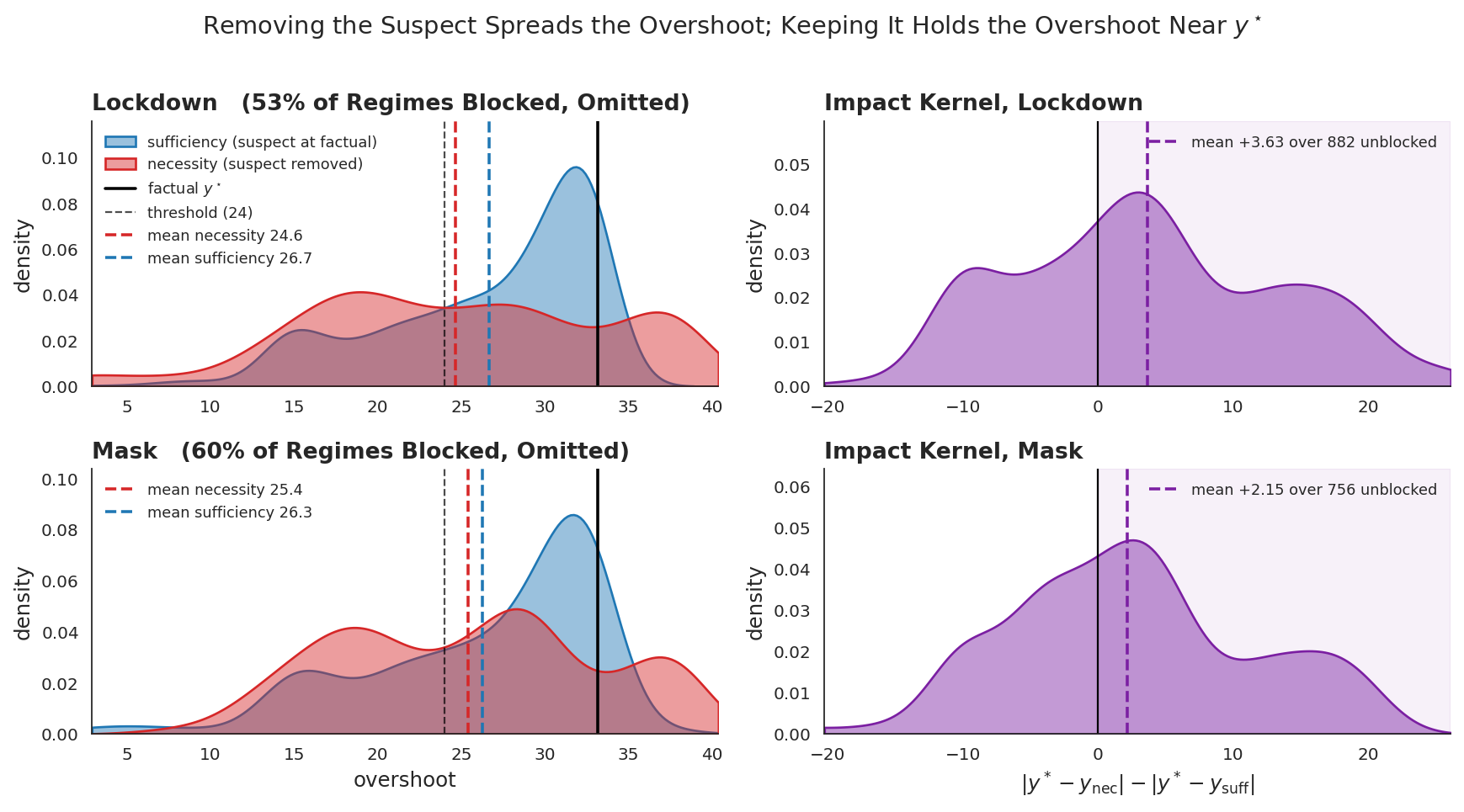}
\caption{Impact-kernel densities per suspect. Left: the necessity world
(suspect removed) spreads the overshoot away from $y^\star$, while the
sufficiency world (suspect held at factual) keeps it near $y^\star$; the
printed means are the two columns of the table above. Right: the resulting
per-draw impact kernel $|y^\star-y_{\mathrm{nec}}|-|y^\star-y_{\mathrm{suff}}|$,
whose $q^2$-scaled mean is the \ourapproach score. The quarter of draws whose
candidate set omits the suspect contribute zero and are excluded from the
density.}
\label{fig:sir_kernels}
\end{figure}

\medskip\noindent\textit{Robustness across factual worlds.}
We replicate the analysis across $20$ factual worlds drawn from the prior
(rather than the single hand-set world above), running thin search on each
with $200$ regimes per suspect. The draws span $\beta \in [0.0160, 0.0398]$
and $\gamma \in [0.481, 0.525]$, giving factual overshoots from $9.86$ to
$33.80$ people, $12$ of the $20$ above the threshold. Lockdown averages
$+0.819$ against mask's $+0.118$, a mean gap of $+0.701$ with a standard
error of $0.111$, and lockdown leads in $18$ of the $20$ worlds. The verdict
is not an artifact of the chosen factual.

\begin{figure}[htbp]
\centering
\includegraphics[width=0.92\linewidth]{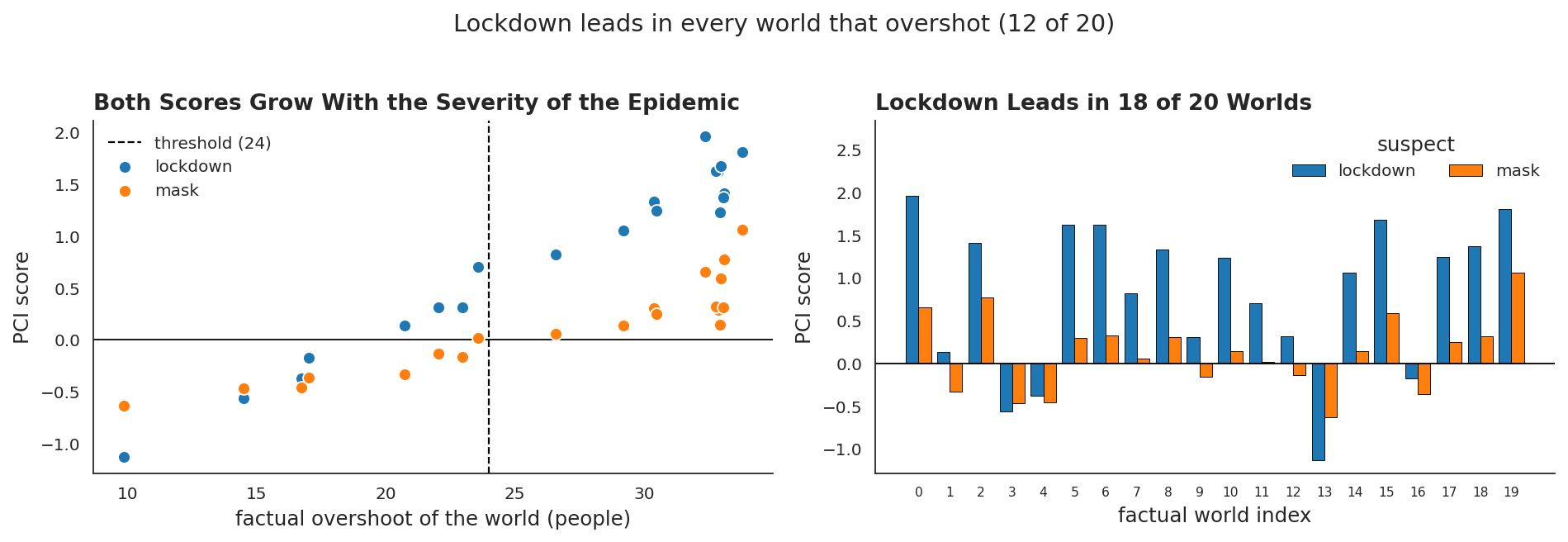}
\caption{The single-world verdict across $20$ factual worlds drawn from the
prior. Left: \ourapproach score against the world's factual overshoot, for
each suspect. Right: the per-world scores in draw order. Lockdown leads in
$18$ of the $20$ worlds, with a mean gap of $+0.701$ (standard error
$0.111$).}
\label{fig:sir_many_worlds}
\end{figure}

\medskip\noindent\textit{Effect of context.}
The ChiRho tutorial isolates one downstream variable at a time, asking
whether the verdict survives when that variable is held fixed at factual
versus left free. We replicate the experiment by partitioning the
\ourapproach regimes by whether the partner-policy efficiency variable was
sampled into the witness set
(Figure~\ref{fig:sir_contexts}). For lockdown, with
$\mathrm{mask\_efficiency}$ \emph{fixed} as a witness the score is $0.494$;
\emph{free}, it rises to $1.684$. For mask, with $\mathrm{lockdown\_efficiency}$
\emph{fixed} the score is $-0.161$ (mask fails to register as a cause at
all under that context regime) and \emph{free} it rises to $1.497$.
This experiment partitions the regimes differently from the single hand-set
world reported earlier, so the fixed/free scores are not expected to average
to that world's scores (lockdown $0.965$, mask $0.491$).
Lockdown therefore registers as a cause in both context regimes, while
mask's status as a cause is contingent on its partner being allowed to
adapt. \ourapproach thus expresses both \emph{which} policy is the dominant cause and \emph{how robustly} that causal claim holds across
counterfactual contexts: a finer-grained verdict than the binary
``actual cause / not'' that classical AC produces.

\begin{figure}[htbp]
\centering
\includegraphics[width=0.92\linewidth]{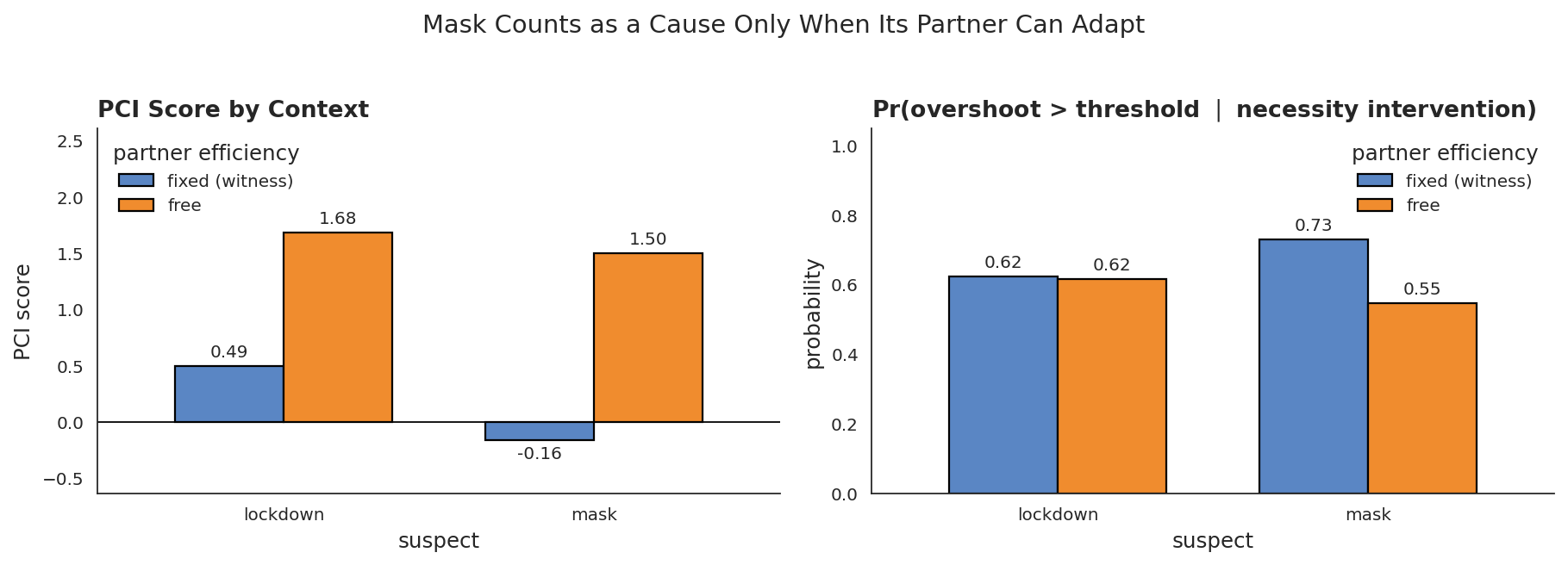}
\caption{Different-contexts experiment. Bars compare PCI total score (left)
and $\Pr(\mathrm{overshoot}>24\mid\mathrm{necessity})$ (right) when the
partner-policy efficiency variable is fixed as a witness versus left free.
Lockdown's score is positive in both regimes; mask's score flips sign
across regimes, indicating that mask is a context-sensitive rather than a
robust cause.}
\label{fig:sir_contexts}
\end{figure}

\medskip\noindent\textit{Takeaway.}
On the dynamical SIR setting, \ourapproach recovers both intuitions set out
at the start of this appendix (lockdown is the proximate cause of excessive overshoot and masking only a context-sensitive one, despite both policies appearing similar under but-for analysis), and expresses them on a
continuous, decomposable scale. The
necessity-vs-sufficiency decomposition distinguishes suspects whose
isolation effects are similar but whose removal-under-context effects
diverge; the contexts experiment further distinguishes robust from
context-sensitive causes. The same machinery applies to any ChiRho
probabilistic program with a continuous outcome and a tractable enumeration
over interventional regimes.

The SIR benchmark scales \ourapproach to a
research-grade Bayesian dynamical model and demonstrates that the joint kernel
captures patterns that AC cannot adjudicate. \cref{sec:avm} pushes the
demonstration one step further: a deployed automated valuation model trained on
millions of points, with deep neural-network components and a sparse variational
GP outcome head. The question there is no longer whether \ourapproach agrees with
a known ground truth (on a proprietary production model none is available),
but whether the computation is feasible at all, and whether the qualitative
SHAP-vs-PCI divergences observed on the synthetic side persist on a real trained
system.

\section{Scaling \ourapproach to a Real-World Automated Valuation Model}\label{sec:avm}

 This appendix details the deployed-AVM study summarised in \cref{sec:evaluation}.

The synthetic archetypes of \cref{sub:synthetic}, the AC benchmark of
\cref{sec:ac_benchmark}, and the SIR benchmark of \cref{sec:sir_benchmark} all
controlled the model and the ground truth, so that \ourapproach could be judged
against an analytic verdict or a competing AC enumeration. Here we drop both,
applying \ourapproach to a deployed valuation model for residential housing on
which neither ground truth nor exact AC enumeration is available. Disclosure
constraints on the trained model and its training data prevent a quantitative
accuracy comparison on this system, so the quantitative validation of
\ourapproach{} lives entirely in the controlled synthetic benchmark of
\cref{sub:synthetic}. The role of this appendix is correspondingly narrow: to
report (i) that \ourapproach{} is computationally feasible at production scale on a
model with millions of training points, and (ii) that the qualitative SHAP-vs-PCI
divergence observed in \cref{sec:shap_examples} persists when the model is a real
trained one: the shape of the two attribution distributions and where
attribution mass lands, not the substantive content of the features they assign to.

\paragraph{Setting.}
The unit of analysis is a residential property transaction. \ourapproach is run
against the full country-level production AVM; what we show here is an
illustrative prototype based on the state of New York rather than
the full national model structure, which we cannot disclose in detail. The NY
prototype shares the architecture and the causal-kernel construction described
below with the full system but is smaller in scope and DAG complexity. The
computational-feasibility claim refers to the full national system; the
qualitative SHAP-vs-PCI behaviour we illustrate is a property of the shared
architecture and carries over from the prototype. The
outcome variable $y$ is the (log-epsilon-standardized) sale price. Transactions are
described by a set of categorical variables $x_{cat} = x_{1}, \dots, x_k$ and continuous
(transformed) variables\footnote{Usually log-epsilon-standardized,
standardized-truncated, or standardized.} $x_{con} = x_{k+1}, \dots, x_n$.

\paragraph{Causal model.}
A causal DAG $\mathcal{G}$ expresses our causal assumptions about the relationships between
variables. We wrote it manually in collaboration with domain experts, then refined it by
adding edges to break implied conditional independencies not upheld by the
data. The simplified NY-prototype DAG has 25 nodes and 99 edges; topology with anonymised node
labels is shown in Figure~\ref{fig:avm_dag}; the full national DAG is larger and is
not shown. As demonstrated in \cref{sec:ac_benchmark}, even the prototype is
already large enough that exact actual-causality enumeration is computationally
infeasible, and the same conclusion applies a fortiori to the national system.

Variables $x_i$ are drawn from distributions $\mathrm{Dist}_i$ parameterized by functions
$f_i(\mathrm{pa}_i, \varphi_i)$ of the parent variables
$\mathrm{pa}_i$\footnote{For root nodes, $\mathrm{pa}_i$ is empty.}:
$x_i \sim \mathrm{Dist}_i(f_i(\mathrm{pa}_i, \varphi_i))$. We call these distributions
\emph{causal kernels}. The distributions are either (truncated) normal or categorical; we choose the
family based on the domain of the variable it models. The functions $f_i$ are deep
neural networks with parameters $\varphi_i$, which we optimize using cross-entropy
loss for categorical distributions and average negative log likelihood for
continuous distributions. The outcome is modeled by a sparse variational
Gaussian Process (SVGP; \citealp{titsias2009variational,hensman2013gaussian}),
$x_{gp} = (x_{con}, \eta(x_{cat}, \varphi_{cat}))$,
$y \sim \mathcal{SVGP}\!\left(\mu(x_{gp}, \varphi_\mu), \kappa(x_{gp}, \varphi_\kappa)\right)$,
where $\eta$ is a linear embedding of the categorical variables into the continuous
space, jointly optimized with the GP.

\begin{figure}[htbp]
\centering
\includegraphics[width=0.85\linewidth]{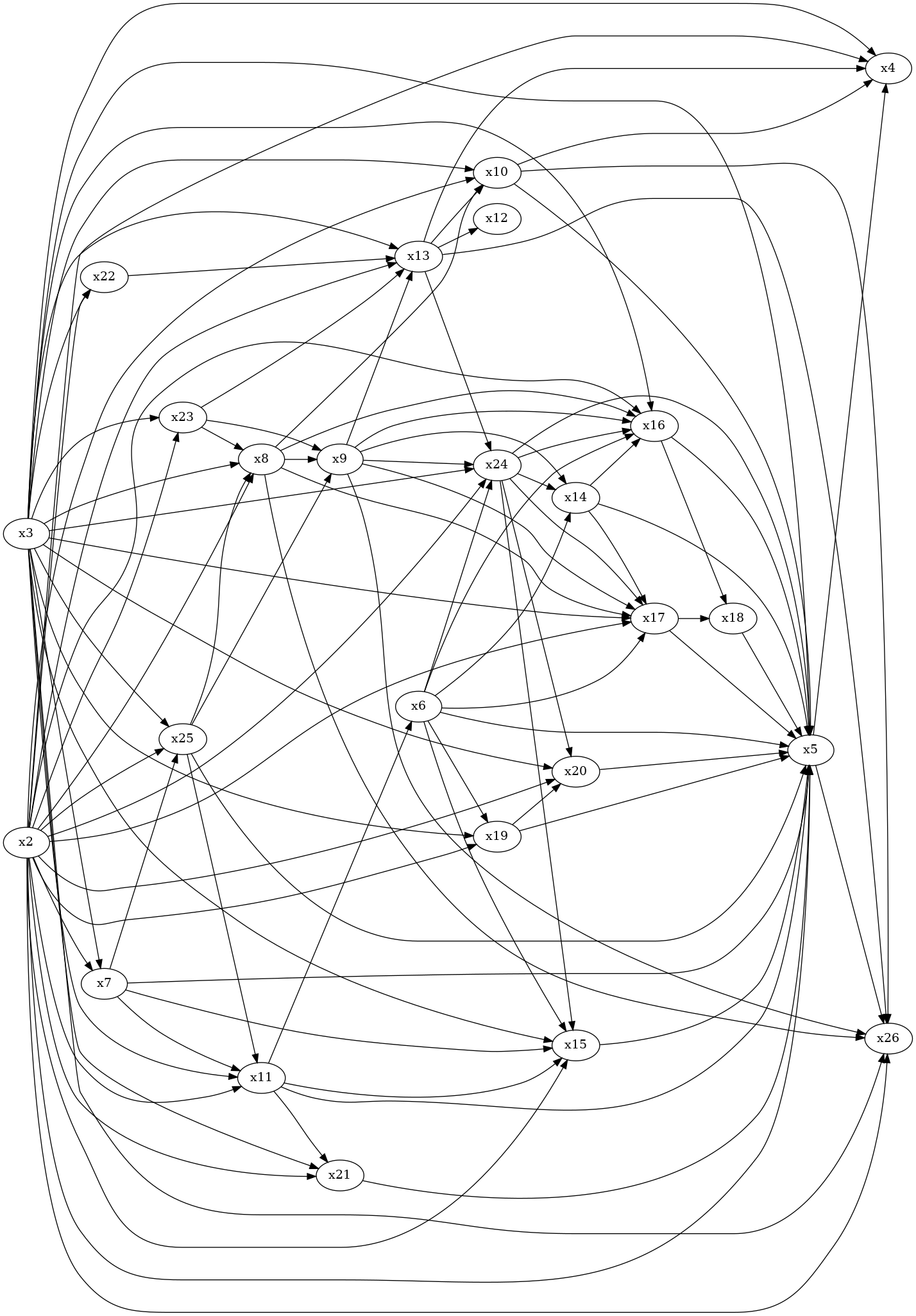}
\caption{Topology of the AVM causal DAG (NY-prototype version), with node labels
anonymised. The graph has 25 nodes and 99 edges; the full national DAG is larger
and is not shown. Several variables sit deep in the DAG with high in-degree; these
are the nodes towards which SHAP allocates disproportionate attribution mass (see
Figure~\ref{fig:eval_avm}).}
\label{fig:avm_dag}
\end{figure}

\paragraph{Feasibility at production scale.}
To apply \ourapproach we transform the model as described in
Section~\ref{sec:causal_impact} and evaluate the impact of features from a list of
potential causal suspects, using the absolute difference impact score of
Example~\ref{ex:absolute_score}. We draw Monte Carlo samples from the expanded model.
Conditioning on a given suspect $X_i$ being active gives us a sample of the scores
corresponding to that feature, which we use to estimate the expected value thereof. The
estimation converges at around \num{25000} samples; for a batch of 50 data points this
takes around 60 minutes on commodity hardware. We need to sample only once; we can then reuse the samples to compute the
attribution score of any suspect. This
establishes that the method is computationally tractable at production scale, not a
foregone conclusion, since the expanded-model construction in
Section~\ref{sec:causal_impact} introduces additional latent dimensions per suspect and
per witness.

\paragraph{Qualitative comparison to SHAP.}
We compare \ourapproach attributions to SHAP scores \citep{lundberg2017unified} on the
same trained model. SHAP heavily weights variables that are downstream of many other
variables in the causal graph, effectively assigning attribution to what amount to
summary statistics of the outcome. \ourapproach distributes attribution more evenly
across the upstream causal structure. The pattern is consistent with what
Section~\ref{sec:shap_examples} establishes on controlled examples: SHAP's observational
marginalisation allocates weight differently from intervention-and-witness-based PCI
scores, and on a deep DAG this difference manifests as concentration on
summary-statistic-like nodes. Figure~\ref{fig:eval_avm} shows the two attribution
distributions with real feature identities withheld under NDA; the two panels
rank their own top variables independently, so only the shape of each
distribution is comparable, not which specific features receive attribution in
one panel versus the other.

\end{document}